%% file: TC_inference.tex
\documentclass[english]{article}
\usepackage[T1]{fontenc}
\usepackage[latin9]{inputenc}
\usepackage{geometry}
\usepackage{babel}
\usepackage{verbatim}
\usepackage{float}
\usepackage{mathtools}
\usepackage{bm}
\usepackage{amsmath}
\usepackage{amssymb}
\usepackage{graphicx}
\usepackage[unicode=true,
 bookmarks=false,
 breaklinks=false,pdfborder={0 0 1},colorlinks=false]
 {hyperref}
\hypersetup{
 colorlinks,linkcolor=red,anchorcolor=blue,citecolor=blue}

\makeatletter

\providecommand{\tabularnewline}{\\}
\floatstyle{ruled}
\newfloat{algorithm}{tbp}{loa}
\providecommand{\algorithmname}{Algorithm}
\floatname{algorithm}{\protect\algorithmname}

\usepackage{babel}

\usepackage{cite}\usepackage{amsthm}\usepackage{dsfont}\usepackage{array}\usepackage{mathrsfs}\usepackage{comment}\onecolumn

\usepackage{color}\usepackage{babel}

\allowdisplaybreaks

\usepackage{xfrac}
\usepackage{enumitem}
\setlist[itemize]{leftmargin=1.5em}
\setlist[enumerate]{leftmargin=1.5em}

\usepackage{babel}
\usepackage{algorithm}% http://ctan.org/pkg/algorithms
\usepackage{algorithmicx}% http://ctan.org/pkg/algorithmicx
\usepackage{algpseudocode}
\usepackage{arydshln}

\newcommand{\ind}{\mathds{1}}  % Indicator

\definecolor{yxc}{RGB}{255,0,0}

\definecolor{cxc}{RGB}{0,140,0}

\usepackage[outdir=./]{epstopdf}

\makeatother

\begin{document}
\theoremstyle{plain} \newtheorem{lemma}{\textbf{Lemma}} \newtheorem{prop}{\textbf{Proposition}}\newtheorem{theorem}{\textbf{Theorem}}\setcounter{theorem}{0}
\newtheorem{corollary}{\textbf{Corollary}} \newtheorem{example}{\textbf{Example}}
\newtheorem{definition}{\textbf{Definition}} \newtheorem{fact}{\textbf{Fact}}
\newtheorem{claim}{\textbf{Claim}}\newtheorem{assumption}{\textbf{Assumption}}

\theoremstyle{definition}

\theoremstyle{remark}\newtheorem{remark}{\textbf{Remark}}\newtheorem{conjecture}{Conjecture}\newtheorem{condition}{\textbf{Condition}}
\title{Uncertainty quantification for nonconvex tensor completion: Confidence
intervals, heteroscedasticity and optimality\footnotetext{Corresponding
author: Yuxin Chen.}}
\author{Changxiao Cai\thanks{Department of Electrical Engineering, Princeton University, Princeton,
NJ 08544, USA; email: \texttt{\{ccai,poor,yuxin.chen\}@princeton.edu}.} \and H.~Vincent Poor\footnotemark[1] \and Yuxin Chen\footnotemark[1]}

\maketitle
\input{abstract.tex}


\tableofcontents{}

\input{intro.tex}

\input{notation.tex}

\input{problem.tex}

\input{results.tex}

\input{experiment.tex}

\input{related-work.tex}

\input{analysis.tex}

\input{discussion.tex}

\section*{Acknowledgements}

Y.~Chen is supported in part by the grants AFOSR YIP award FA9550-19-1-0030,
ONR N00014-19-1-2120, ARO YIP award W911NF-20-1-0097, ARO W911NF-18-1-0303,
NSF CCF-1907661, IIS-1900140 and DMS-2014279, and the Princeton SEAS
Innovation Award. H.~V.~Poor is supported in part by the National
Science Foundation under Grant CCF-1908308, and in part by a Princeton
Schmidt Data-X Research Award. C.~Cai is supported in part by the
Gordon Y.~S.~Wu Fellowship in Engineering.

\appendix
\input{init-algorithms.tex}

\input{preliminaries.tex}

\input{proof-U-dist.tex}

\input{proof-T-dist.tex}

\input{proof-CI.tex}

\input{proof-estimation.tex}

\input{proof_lower_bounds.tex}

\input{proof-preliminaries.tex}

\input{aux-lemma.tex}

\bibliographystyle{alphaabbr}
\bibliography{bibfileNonconvex}

\end{document}

%% file: abstract.tex
\begin{abstract}
We study the distribution and uncertainty of nonconvex optimization
for noisy tensor completion --- the problem of estimating a low-rank
tensor given incomplete and corrupted observations of its entries.
Focusing on a two-stage estimation algorithm proposed by Cai et al.~\cite{cai2019nonconvex},
we characterize the distribution of this nonconvex estimator down
to fine scales. This distributional theory in turn allows one to construct
valid and short confidence intervals for both the unseen tensor entries
and the unknown tensor factors. The proposed inferential procedure
enjoys several important features: (1) it is fully adaptive to noise
heteroscedasticity, and (2) it is data-driven and automatically adapts
to unknown noise distributions. Furthermore, our findings unveil the
statistical optimality of nonconvex tensor completion: it attains
un-improvable $\ell_{2}$ accuracy --- including both the rates and
the pre-constants --- when estimating both the unknown tensor and
the underlying tensor factors.
\end{abstract}

%% file: intro.tex
\section{Introduction}

\label{sec:Introduction}

\subsection{Noisy low-rank tensor completion}

Tensor data are routinely employed in data and information sciences
to model (structured) multi-dimensional objects \cite{kolda2009tensor,sidiropoulos2017tensor,anandkumar2014tensor,zhang2019cross}.
In many practical scenarios of interest, however, we do not have full
access to a large-dimensional tensor of interest, as only a sampling
of its entries are revealed to us; yet we would still wish to reliably
infer all missing data. This task, commonly referred to as \emph{tensor
completion}, finds applications in numerous domains including medical
imaging \cite{semerci2014tensor}, visual data analysis \cite{liu2013tensor},
seismic data reconstruction \cite{kreimer2013tensor}, to name just
a few. In order to make meaningful inference about the unseen entries,
additional information about the unknown tensor plays a pivotal role
(otherwise one is in the position with fewer equations than unknowns).
A common type of such prior information is low-rank structure, which
hypothesizes that the unknown tensor is decomposable into the superposition
of a few rank-one tensors. Substantial attempts have been made in
the past few years to understand and tackle such low-rank tensor completion
problems.

To set the stage for a formal discussion, we formulate the problem
as follows. Imagine that we are interested in reconstructing a third-order
tensor $\bm{T}^{\star}=[T_{i,j,k}]_{1\leq i,j,k\leq d}\in\mathbb{R}^{d\times d\times d}$,
which is \emph{a priori} known to have low canonical polyadic (CP)
rank \cite{kolda2009tensor}. This means that $\bm{T}^{\star}$ admits
the following CP decomposition\footnote{For any vectors $\bm{u},\bm{v},\bm{w}\in\mathbb{R}^{d}$, we denote
by $\bm{u}\otimes\bm{v}\otimes\bm{w}\in\mathbb{R}^{d\times d\times d}$
a three-way array whose $(i,j,k)$-th element is given by the product
of the corresponding vector entries $u_{i}v_{j}w_{k}$.} 
\begin{equation}
\bm{T}^{\star}=\sum_{l=1}^{r}\bm{u}_{l}^{\star}\otimes\bm{u}_{l}^{\star}\otimes\bm{u}_{l}^{\star}=:\sum_{l=1}^{r}(\bm{u}_{l}^{\star})^{\otimes3},\label{eq:unknown-tensor}
\end{equation}
where $\bm{u}_{l}\in\mathbb{R}^{d}$ ($1\leq l\leq r$) represents
the unknown tensor factor, and the rank $r$ is considerably smaller
than the ambient dimension $d$. What we have obtained is a highly
incomplete collection of noisy observations about the entries of $\bm{T}^{\star}\in\mathbb{R}^{d\times d\times d}$;
more precisely, we observe
\begin{equation}
T_{i,j,k}^{\mathsf{obs}}=T_{i,j,k}^{\star}+E_{i,j,k},\qquad(i,j,k)\in\Omega,\label{eq:observed-T}
\end{equation}
where $\Omega\subseteq[d]\times[d]\times[d]$ with $[d]:=\{1,\cdots,d\}$
is a subset of entries, $T_{i,j,k}^{\mathsf{obs}}$ denotes the observed
entry in the $(i,j,k)$-th position, and we use $E_{i,j,k}$ to represent
the associated measurement noise, in an attempt to model more realistic
scenarios. The presence of missing data and noise, as well as the
``notorious'' tensor structure (which is often not as computationally
friendly as its matrix analog), poses severe computational and statistical
challenges for reliable tensor reconstruction.

\subsection{Review: a nonconvex optimization approach}

A natural reconstruction strategy based on the partial data in hand
is to resort to the following least-squares problem: 
\begin{equation}
\underset{\bm{U}\in\mathbb{R}^{d\times r}}{\text{minimize}}\quad f(\bm{U}):=\sum_{(i,j,k)\in\Omega}\Bigg[\Big(\sum_{l=1}^{r}\bm{u}_{l}^{\otimes3}\Big)_{i,j,k}-T_{i,j,k}^{\mathsf{obs}}\Bigg]^{2}.\label{eq:nonconvex-formulation}
\end{equation}
Here and in the sequel, we use $\bm{U}:=[\bm{u}_{1},\cdots,\bm{u}_{r}]$
to concisely represent the set $\{\bm{u}_{l}\}_{1\leq l\leq r}$.
Unfortunately, owing to its highly nonconvex nature, the optimization
problem (\ref{eq:nonconvex-formulation}) is in general daunting to
solve.

To alleviate computational intractability, a number of polynomial-time
algorithms have been proposed; partial examples include convex relaxation
\cite{gandy2011tensor,huang2015provable,romera2013new}, spectral
methods \cite{montanari2018spectral,cai2019subspace}, sum of squares
hierarchy \cite{barak2016noisy,potechin2017exact}, alternating minimization
\cite{jain2014provable,liu2020tensor}, and so on. Nevertheless, most of these algorithms
either are still computationally prohibitive for large-scale problems,
or do not come with optimal statistical guarantees; see Section~\ref{sec:Related-Work}
for detailed discussions. To address the computational and statistical
challenges at once, the recent work \cite{cai2019nonconvex} proposed
a two-stage nonconvex paradigm that guarantees efficient yet reliable
solutions. In a nutshell, this algorithm starts by computing a rough
(but reasonable) initial guess $\bm{U}^{0}=[\bm{u}_{1}^{0},\cdots,\bm{u}_{r}^{0}]$
for all tensor factors, and iteratively refines the estimate by means
of the gradient descent (GD) update rule: 
\begin{equation}
\bm{U}^{t+1}=\bm{U}^{t}-\eta_{t}\nabla f(\bm{U}^{t}),\qquad t=0,1,\cdots\label{eq:GD-update}
\end{equation}
See Algorithm~\ref{alg:gd} (note that the initialization scheme
is more complex to describe, and is hence postponed to Appendix~\ref{sec:Initialization-scheme}).
Encouragingly, despite the nonconvex optimization landscape, theoretical
guarantees have been developed for Algorithm~\ref{alg:gd} under
a suitable random sampling and random noise model. Take the noiseless
case for instance: this approach converges linearly to the ground
truth under near-minimal sample complexity. Furthermore, the algorithm
achieves intriguing $\ell_{2}$ and $\ell_{\infty}$ statistical accuracy
under a broad family of noise models.

\begin{algorithm}[h]
\caption{A nonconvex algorithm for tensor completion.}

\label{alg:gd}\begin{algorithmic}

\State \textbf{{Initialize}} $\bm{U}^{0}=[\bm{u}_{1}^{0},\cdots,\bm{u}_{r}^{0}]$
via Algorithm \ref{alg:init}.

\State \textbf{{Gradient updates}}: \textbf{for }$t=0,1,\ldots,t_{0}-1$
\textbf{do}

\State \vspace{-2.5em}
 
\begin{align}
\bm{U}^{t+1} & =\bm{U}^{t}-\eta_{t}\nabla f(\bm{U}^{t}).\label{eq:gradient_update_ncvx-TC}
\end{align}

\State \textbf{{Output} }$\bm{U}=[\bm{u}_{1},\cdots,\bm{u}_{r}]:=\bm{U}^{t_{0}}$.

\end{algorithmic}
\end{algorithm}

\subsection{Uncertainty quantification for nonconvex tensor completion}

In various practical scenarios (e.g.~medical imaging), in order to
enable informative decision making and trustworthy prediction, it
is crucial not only to provide the users with the reconstruction outcome,
but also to inform them of the uncertainty or risk underlying the
reconstruction. The latter task, often termed\emph{ uncertainty quantification,}
can be accomplished by characterizing the (approximate) distributions
of our reconstruction, which can be further employed to construct
valid confidence intervals (namely, giving lower and upper bounds)
for the unknowns. In particular, two questions are of fundamental
importance: given an estimate returned by the above nonconvex algorithm,
how to identify a confidence interval when predicting an unseen entry,
and how to produce a confidence region that is likely to contain the
tensor factors of interest?

Unfortunately, classical distributional theory available in the statistics
literature, which typically operates in the large-sample regime (with
a fixed number of unknowns and a sample size tending to infinity),
is not applicable to assess the uncertainty of the above nonconvex
algorithm in high dimension. In fact, due to the nonconvex nature
of the algorithm, it becomes remarkably challenging to track the distribution
of the solution returned by Algorithm~\ref{alg:gd} or other nonconvex
alternatives. The absence of distributional characterization prevents
us from communicating a trustworthy uncertainty estimate to the users.
While the statistical performance of Algorithm~\ref{alg:gd} has
been investigated in \cite{cai2019nonconvex}, existing statistical
guarantees$\,$---$\,$which hide the (potentially huge) pre-constants$\,$---$\,$can
only yield confidence intervals that are overly wide and, as a result,
practically uninformative. In contrast, one should aim for valid confidence
intervals that are as short as possible.

Furthermore, an ideal procedure for uncertainty quantification should
automatically adapt to unknown noise distributions. Accomplishing
this goal, however, becomes particularly challenging when the noise
levels are not only unknown but also location-varying$\,$---$\,$a
scenario commonly referred to as \emph{heteroscedasticity}. In fact,
there is no shortage of realistic scenarios in which the data heteroscedasticity
makes it infeasible to pre-estimate local variability in a uniformly
reliable manner. Addressing this challenge calls for the design of
model-agnostic data-driven procedures that are fully adaptive to noise
heteroscedasticity.

\subsection{Main contributions and insights}

We now give an informal overview of the main contributions and insights
of this paper. To the best of our knowledge, results of this kind
were previously unavailable in the tensor completion/estimation literature.
\begin{enumerate}
\item \emph{A distributional theory for nonconvex tensor completion.} Despite
its nonconvex nature, the distributional representation of the estimate
returned by Algorithm~\ref{alg:gd} can be established down to quite
fine scales. Under mild conditions, (1) the resulting estimates for
both the tensor factors and the tensor entries are nearly unbiased,
and (2) the associated uncertainty follows a zero-mean Gaussian distribution
whose (co)-variance can be accurately determined in a data-driven
manner.
\item \emph{Construction of entrywise confidence intervals}. The above distributional
characterization leads to construction of \emph{entrywise} confidence
intervals for both the unknown tensor and the associated tensor factors.
The proposed inferential procedure is fully data-driven: it does not
require prior knowledge about the noise distributions, and it automatically
adapts to local variability of noise levels.
\item \emph{Optimality w.r.t.~both inference and estimation}. We develop
fundamental lower bounds under i.i.d.~Gaussian noise, confirming
that the proposed entrywise confidence intervals are in some sense
the shortest possible. As a byproduct, our results also reveal that
nonconvex optimization achieves un-improvable $\ell_{2}$ statistical
accuracy$\,$---$\,$including both the rates and the pre-constants$\,$---$\,$for
estimating both the unknown tensor and its underlying tensor factors.
\end{enumerate}
All in all, our results shed light on the \emph{unreasonable effectiveness
of nonconvex optimization} in noisy tensor completion, which enables
optimal estimation and uncertainty quantification all at once.

The rest of the paper is organized as follows. Section~\ref{subsec:Models-and-assumptions}
formulates the problem settings. Section~\ref{sec:main-results}
presents our distributional theory, discusses construction of confidence
intervals, and develops fundamental lower bounds. Section~\ref{sec:Related-Work}
provides an overview of related prior work. The proof outline of our
main theory is supplied in Section~\ref{sec:Analysis}, with the
proofs of auxiliary lemmas provided in the appendix. We conclude the
paper with a discussion of future directions in Section~\ref{sec:Discussion}.

%% file: notation.tex
\subsection{Notation}

\label{subsec:Notations}

For any matrix $\bm{M}$, we use $\|\bm{M}\|$ and $\|\bm{M}\|_{\mathrm{F}}$
to denote the spectral norm (operator norm) and the Frobenius norm
of $\bm{M}$, respectively, and let $\bm{M}_{i,:}$ and $\bm{M}_{:,i}$
stand for the $i$-th row and $i$-th colomn, respectively. We denote
by $\|\bm{M}\|_{2,\infty}:=\max_{l}\|\bm{M}_{l,:}\|_{2}$ (resp.~$\|\bm{M}\|_{\infty}:=\max_{i,j}|M_{i,j}|$)
the $\ell_{2,\infty}$ norm (resp.~entrywise $\ell_{\infty}$ norm)
of $\bm{M}$. In addition, let $\lambda_{1}(\bm{M})\geq\lambda_{2}(\bm{M})\geq\cdots$
denote the eigenvalues of $\bm{M}$ and $\sigma_{1}(\bm{M})\geq\sigma_{2}(\bm{M})\geq\cdots$
denote the singular values of $\bm{M}$. For any matrices $\bm{M},\bm{N}$
of compatible dimensions, we let $\bm{M}\odot\bm{N}$ stand for the
Hadamard (entrywise) product.

For any tensor $\bm{T}\in\mathbb{R}^{d\times d\times d}$, denote
by $\mathcal{P}_{\Omega}(\bm{T})$ the Euclidean projection of $\bm{T}$
onto the subset of tensors that vanish outside the index set $\Omega$.
With this notation in place, the observed data (\ref{eq:observed-T})
can be succinctly described as
\begin{equation}
\mathcal{P}_{\Omega}\big(\bm{T}^{\mathsf{obs}}\big)=\mathcal{P}_{\Omega}\big(\bm{T}^{\star}+\bm{E}\big),\label{eq:POmega-T}
\end{equation}
where $\bm{T}^{\mathsf{obs}}:=[T_{i,j,k}^{\mathsf{obs}}]_{1\leq i,j,k\leq d}$
and $\bm{E}:=[E_{i,j,k}]_{1\leq i,j,k\leq d}$. Here and throughout,
we let $T_{i,j,k}^{\mathsf{obs}}=0$ for any $(i,j,k)\notin\Omega$.
In addition, we use $u_{l,i}$ (resp.~$u_{l,i}^{\star}$) to denote
the $i$-th entry of $\bm{u}_{l}\in\mathbb{R}^{d}$ (resp.~$\bm{u}_{l}^{\star}\in\mathbb{R}^{d}$).

For any tensor $\bm{T}\in\mathbb{R}^{d\times d\times d}$, let $\bm{T}_{i,:,:}\in\mathbb{R}^{d\times d}$
denote the mode-$1$ $i$-th slice with entries $(\bm{T}_{i,:,:})_{j,k}=T_{i,j,k},$
and $\bm{T}_{:,i,:}$ and $\bm{T}_{:,:,i}$ are defined analogously.
Let $\mathsf{unfold}(\bm{T})$ represent the mode-3 matricization
of $\bm{T}$, namely, $\mathsf{unfold}\left(\bm{T}\right)$ is a matrix
in $\mathbb{R}^{d\times d^{2}}$ whose entries are given by
\begin{equation}
\big(\mathsf{unfold}(\bm{T})\big)_{k,d(i-1)+j}=T_{i,j,k},\qquad1\leq i,j,k\leq d.
\end{equation}
For any tensors $\bm{T}\in\mathbb{R}^{d\times d\times d}$, the Frobenius
norm of $\bm{T}$ is defined accordingly as $\|\bm{T}\|_{\mathrm{F}}:=\sqrt{\sum_{i,j,k}T_{i,j,k}^{2}}$.
We use $\|\bm{T}\|_{\infty}:=\max_{i,j,k}|T_{i,j,k}|$ to denote the
entrywise $\ell_{\infty}$ norm. For any vectors $\bm{u},\bm{v}\in\mathbb{R}^{d}$,
we define the vector products $\bm{T}\times_{3}\bm{u}\in\mathbb{R}^{d\times d}$
and $\bm{T}\times_{1}\bm{u}\times_{2}\bm{v}\in\mathbb{R}^{d}$ such
that\begin{subequations}
\begin{align}
\left[\bm{T}\times_{3}\bm{u}\right]_{i,j} & :=\sum\nolimits _{k}T_{i,j,k}u_{k},\qquad1\leq i,j\leq d;\\
\left[\bm{T}\times_{1}\bm{u}\times_{2}\bm{v}\right]_{k} & :=\sum\nolimits _{i,j}T_{i,j,k}u_{i}v_{j},\qquad1\leq k\leq d.
\end{align}
\end{subequations}The products $\bm{T}\times_{2}\bm{u}\in\mathbb{R}^{d\times d}$,
$\bm{T}\times_{3}\bm{u}\in\mathbb{R}^{d\times d}$, $\bm{T}\times_{1}\bm{u}\times_{3}\bm{v}\in\mathbb{R}^{d}$,
$\bm{T}\times_{2}\bm{u}\times_{3}\bm{v}\in\mathbb{R}^{d}$ are defined
analogously. In addition, the spectral norm of $\bm{T}$ is defined
as $\left\Vert \bm{T}\right\Vert :=\sup_{\bm{u},\bm{v},\bm{w}\in\mathbb{S}^{d-1}}\left\langle \bm{T},\bm{u}\otimes\bm{v}\otimes\bm{w}\right\rangle $,
where we denote by $\mathbb{S}^{d-1}:=\{\bm{u}\in\mathbb{R}^{d}|\|\bm{u}\|_{2}=1\}$
the unit sphere in $\mathbb{R}^{d}$.

We use $[a\pm b]$ to denote the interval $[a-b,a+b]$, and we shall
often let $(i,j)$ denote $(i-1)d+j$ whenever it is clear from the
context. We denote by $[d]:=\{1,2,\cdots,d\}$. The notation $f(d)\lesssim g(d)$
or $f(d)=O(g(d))$ (resp.~$f(d)\gtrsim g(d)$) means that there exists
a constant $C_{0}>0$ such that $|f(d)|\leq C_{0}|g(d)|$ (resp.~$|f(d)|\geq C_{0}|g(d)|$).
The notation $f(d)\asymp g(d)$ means that $C_{0}|f(d)|\leq|g(d)|\leq C_{1}|f(d)|$
holds for some universal constants $C_{0},C_{1}>0$. In addition,
$f(d)=o(g(d))$ means that $\lim_{d\rightarrow\infty}f(d)/g(d)=0$,
$f(d)\ll g(d)$ means that $f(d)\leq c_{1}g(d)$ for some small constant
$c_{0}>0$ and $f(d)\gg g(d)$ means that $f(d)\geq c_{1}g(d)$ for
some large constant $c_{1}>0$.

%% file: problem.tex
\section{Models and assumptions \label{subsec:Models-and-assumptions}}

In this paper, we shall consider a setting with random sampling and
independent random noise as follows.

\begin{assumption}[Random sampling]\label{assumption:random-sampling}Suppose
that $\Omega$ is a symmetric index set.\footnote{This means that if $(i,j,k)\in\Omega$, then $(j,i,k)$, $(i,k,j)$,
$(j,k,i)$, $(k,i,j)$, $(k,j,i)$ are all in $\Omega$.} Assume that each $(i,j,k)$ with $i\leq j\leq k$ is included in
$\Omega$ independently with probability $p$. Throughout this paper,
we shall define
\begin{equation}
\chi_{i,j,k}:=\ind\{(i,j,k)\in\Omega\},\qquad1\leq i,j,k\leq d.\label{eq:def:chi}
\end{equation}
 \end{assumption}

\begin{assumption}[Random noise]\label{assumption:random-noise}Suppose
that $\bm{E}=[E_{i,j,k}]_{1\leq i,j,k\leq d}$ is a symmetric tensor.\footnote{This means that $E_{i,j,k}=E_{j,i,k}=E_{i,k,j}=E_{j,k,i}=E_{k,i,j}=E_{k,j,i}$
for any $1\leq i,j,k\leq d$.} Assume that the noise components $\{E_{i,j,k}\}_{1\leq i\leq j\leq k\leq d}$
are independent sub-Gaussian random variables satisfying $\mathbb{E}[E_{i,j,k}]=0$
and $\mathsf{Var}(E_{i,j,k})=\sigma_{i,j,k}^{2}$. Denoting $\sigma_{\min}:=\min_{i,j,k}\sigma_{i,j,k}$
and $\sigma_{\max}:=\max_{i,j,k}\sigma_{i,j,k}$, we assume throughout
that $\sigma_{\max}/\sigma_{\min}=O(1)$.\end{assumption}

Next, we introduce additional parameters about the unknown tensor
of interest. Recall that 
\[
\bm{T}^{\star}=\sum_{l=1}^{r}\bm{u}_{l}^{\star}\otimes\bm{u}_{l}^{\star}\otimes\bm{u}_{l}^{\star}=\sum_{l=1}^{r}\bm{u}_{l}^{\star\otimes3}\in\mathbb{R}^{d\times d\times d}.
\]
To begin with, we define the strength of each rank-one tensor component
as follows
\begin{equation}
\lambda_{\min}^{\star}:=\min_{1\leq l\leq r}\left\Vert \bm{u}_{l}^{\star}\right\Vert _{2}^{3}\qquad\text{and}\qquad\lambda_{\max}^{\star}:=\max_{1\leq l\leq r}\left\Vert \bm{u}_{l}^{\star}\right\Vert _{2}^{3},\label{eq:defn-lambda-max-min}
\end{equation}
allowing us to define the condition number by
\begin{equation}
\kappa:=\lambda_{\max}^{\star}/\lambda_{\min}^{\star}.\label{eq:defn-cond-number}
\end{equation}
To enable reliable tensor completion, we impose further assumptions
on the tensor factors $\{\bm{u}_{l}^{\star}\}$ as follows.

\begin{assumption}[Incoherence and well-conditionedness]\label{assumption:incoherence}Suppose
that $\bm{T}^{\star}$ satisfies\begin{subequations}\label{eq:incoh}
\begin{align}
\left\Vert \bm{T}^{\star}\right\Vert _{\infty} & \leq\sqrt{\frac{\mu_{0}}{d^{3}}}\,\left\Vert \bm{T}^{\star}\right\Vert _{\mathrm{F}};\label{assumption:T-inf-norm}\\
\left\Vert \bm{u}_{l}^{\star}\right\Vert _{\infty} & \leq\sqrt{\frac{\mu_{1}}{d}}\,\left\Vert \bm{u}_{l}^{\star}\right\Vert _{\mathrm{2}},\qquad1\leq l\leq r;\label{assumption:u-inf-norm}\\
\left|\left\langle \bm{u}_{l}^{\star},\,\bm{u}_{j}^{\star}\right\rangle \right| & \leq\sqrt{\frac{\mu_{2}}{d}}\,\left\Vert \bm{u}_{l}^{\star}\right\Vert _{\mathrm{2}}\left\Vert \bm{u}_{j}^{\star}\right\Vert _{\mathrm{2}},\qquad1\leq l\neq j\leq r.\label{assumption:u-inner-prod}
\end{align}
\end{subequations}Further, assume that $\bm{T}^{\star}$ is well-conditioned
in the sense that $\kappa$ (cf.~(\ref{eq:defn-cond-number})) satisfies
$\kappa=O\left(1\right)$.\end{assumption}

Informally, when both $\mu_{0}$ and $\mu_{1}$ are small, the $\ell_{2}$
energy of both $\bm{T}^{\star}$ and $\bm{u}_{l}^{\star}$ ($1\leq l\leq r$)
is dispersed more or less evenly across their entries. In addition,
a small $\mu_{2}$ necessarily implies that every pair of the tensor
factors of interest is nearly orthogonal to (and hence incoherent
with) each other. Finally, the well-conditionedness assumption guarantees
that no single tensor component has significantly higher energy compared
to the rest of them. For the sake of notational simplicity, we shall
combine them into a single incoherence parameter
\begin{equation}
\mu:=\max\left\{ \mu_{0},\,\mu_{1},\,\mu_{2}\right\} .\label{eq:defn-mu}
\end{equation}

The focal point of this paper lies in obtaining distributional characterization
of, and uncertainty assessment for, the nonconvex estimate (i.e.~the
solution $\bm{U}$ returned by Algorithm \ref{alg:gd}) in a strong
entrywise sense. In particular, we set out the goal to

\begin{itemize}

\item[1)]establish distributional representation of the estimate
$\bm{U}$;

\item[2)]construct short yet valid confidence intervals for each
entry of the tensor factor $\{\bm{u}_{l}^{\star}\}_{1\leq l\leq r}$
as well as each entry of the unknown tensor $\bm{T}^{\star}$.

\end{itemize} 

To cast the latter task in more precise terms: given any target coverage
level $0<1-\alpha<1$, any $1\leq l\leq r$ and any $1\leq i,j,k\leq d$,
the aim is to compute intervals $[c_{1,u},c_{2,u}]$ and $[c_{1,T},c_{2,T}]$
such that\begin{subequations}
\begin{align}
\mathbb{P}\left\{ u_{l,i}^{\star}\in[c_{1,u},c_{2,u}]\right\}  & =1-\alpha+o\left(1\right)\qquad(\text{up to global permutation});\label{eq:CI-tensor-factor}\\
\mathbb{P}\left\{ T_{i,j,k}^{\star}\in[c_{1,T},c_{2,T}]\right\}  & =1-\alpha+o\left(1\right).\label{eq:CI-tensor-entry}
\end{align}
\end{subequations}Here, (\ref{eq:CI-tensor-factor}) is phrased accounting
for global permutation, since one cannot possibly distinguish $\{\bm{u}_{l}^{\star}\}_{1\leq l\leq r}$
and any permutation of them given only the observations (\ref{eq:observed-T}).
Ideally, the above tasks should be accomplished in a data-driven manner
without requiring prior knowledge about the noise distributions.

%% file: results.tex
\section{Main results\label{sec:main-results}}

This section presents our distributional theory for nonconvex tensor
completion, and demonstrates how to conduct data-driven and optimal
uncertainty quantification. For notational convenience, in the sequel
we denote by $\bm{U}=[\bm{u}_{1},\cdots,\bm{u}_{r}]\in\mathbb{R}^{d\times r}$
the estimate returned by Algorithm~\ref{alg:gd}, and let $\bm{T}\in\mathbb{R}^{d\times d\times d}$
indicate the resulting tensor estimate as follows
\begin{equation}
\bm{T}:=\sum_{l=1}^{r}\bm{u}_{l}\otimes\bm{u}_{l}\otimes\bm{u}_{l}.\label{eq:defn-estimate-T}
\end{equation}
In addition, recognizing that one can only hope to recover $\bm{U}^{\star}$
up to global permutation, we introduce a permutation matrix as follows
\begin{equation}
\bm{\Pi}:=\min_{\bm{Q}\in\mathsf{perm}_{r}}\left\Vert \bm{U}\bm{Q}-\bm{U}^{\star}\right\Vert _{\mathrm{F}},\label{eq:defn-permutation}
\end{equation}
where $\mathsf{perm}_{r}$ represents the set of permutation matrices
in $\mathbb{R}^{r\times r}$. Additionally, in order to guarantee
reliable convergence of Algorithm~\ref{alg:gd}, there are several
algorithmic parameters (e.g.~the learning rates) that need to be
properly chosen. We shall adopt the choices suggested by \cite{cai2019nonconvex}
throughout this paper. Given that our theory can be presented regardless
of whether one understands these algorithmic choices, we defer the
specification of these parameters to Appendix~\ref{subsec:Choices-of-algorithmic-pars}
to avoid distraction.

\subsection{Distributional guarantees for nonconvex estimates\label{sec:confidence-intervals}}

We now establish distributional guarantees for the nonconvex estimate.
For notational convenience, we introduce an auxiliary matrix $\widetilde{\bm{U}}^{\star}\in\mathbb{R}^{d^{2}\times r}$
as well as a collection of diagonal matrices $\bm{D}_{k}^{\star}\in\mathbb{R}^{d^{2}\times d^{2}}$
($1\leq k\leq d$) such that 
\begin{align}
\widetilde{\bm{U}}^{\star} & :=\big[\bm{u}_{1}^{\star}\otimes\bm{u}_{1}^{\star},\cdots,\bm{u}_{r}^{\star}\otimes\bm{u}_{r}^{\star}\big]\in\mathbb{R}^{d^{2}\times r};\label{eq:defn-Ustar-tilde}\\
\big(\bm{D}_{k}^{\star}\big)_{(i,j),(i,j)} & :=\sigma_{i,j,k}^{2},\qquad1\leq i,j\leq d;\label{eq:cov-matrix-m-diag}
\end{align}
here, we abuse the notation $(i,j)$ to denote $(i-1)d+j$ whenever
it is clear from the context. In words, $\widetilde{\bm{U}}^{\star}$
lifts the tensor factors to a higher order, and $\bm{D}_{k}^{\star}$
collects the noise variance in the $k$-th slice of $\bm{E}$. To
simplify presentation, we begin with the case with independent Gaussian
noise.

\begin{theorem}[Distributional guarantees for tensor factor estimates
(Gaussian noise)]\label{thm:U-loss-dist-simple-Gaussian}Suppose
that the $E_{i,j,k}$'s are Gaussian, and that Assumptions \ref{assumption:random-sampling}-\ref{assumption:incoherence}
hold. Assume that $\mu,\kappa,r=O(1)$ and that $t_{0}=c_{0}\log d$,
\begin{align}
p & \geq c_{1}\frac{\log^{5}d}{d^{3/2}}\qquad\text{and}\qquad\frac{c_{2}}{d^{100}}\leq\frac{\sigma_{\min}}{\|\bm{T}^{\star}\|_{\infty}}\leq\frac{\sigma_{\max}}{\|\bm{T}^{\star}\|_{\infty}}\leq c_{3}\sqrt{\frac{pd^{3/2}}{\log^{4}d}}\label{eq:requirement-p-sigma-simple}
\end{align}
for some sufficiently large (resp.~small) constant $c_{0},c_{2},c_{3}>0$
(resp.~$c_{1}>0$). Then with probability at least $1-o(1)$, one
has the following decomposition:
\[
\bm{U}\bm{\Pi}-\bm{U}^{\star}=\bm{Z}+\bm{W},
\]
where $\bm{\Pi}$ is defined in (\ref{eq:defn-permutation}), $\left\Vert \bm{W}\right\Vert _{2,\infty}=o\big(\frac{\sigma_{\min}}{\lambda_{\max}^{\star2/3}\sqrt{p}}\big)$,
and for any $1\leq k\leq d$ one has $\bm{Z}_{k,:}\sim\mathcal{N}(\bm{0},\bm{\Sigma}_{k}^{\star})$
with
\begin{equation}
\bm{\Sigma}_{k}^{\star}:=\frac{2}{p}\big(\widetilde{\bm{U}}^{\star\top}\widetilde{\bm{U}}^{\star}\big)^{-1}\widetilde{\bm{U}}^{\star\top}\bm{D}_{k}^{\star}\widetilde{\bm{U}}^{\star}\big(\widetilde{\bm{U}}^{\star\top}\widetilde{\bm{U}}^{\star}\big)^{-1}.\label{eq:cov-matrix-m}
\end{equation}
\end{theorem}

\begin{remark}As an interpretation of Condition~(\ref{eq:requirement-p-sigma-simple}):
(i) the sample complexity is $pd^{3}\gtrsim d^{3/2}\mathrm{poly}\log(d)$,
which is widely conjectured to be computationally optimal (up to some
log factor) \cite{barak2016noisy}; (ii) the typical size of each
noise component (as captured by $\{\sigma_{i,j,k}\}$) is allowed
to be substantially larger than the maximum magnitude of the entries
of $\bm{T}^{\star}$ under the sample size assumption stated here.
\end{remark}

In words, Theorem~\ref{thm:U-loss-dist-simple-Gaussian} reveals
that the estimation error of $\bm{U}$ can be decomposed into a Gaussian
component $\bm{Z}$ and a residual term $\bm{W}$. Encouragingly,
the residual term $\bm{W}$ is, in some sense, dominated by the Gaussian
term and can be safely neglected. To see this, recall that $\sigma_{i,j,k}\geq\sigma_{\min}$,
leading to a lower bound\footnote{To see why the penultimate inequality holds, note that under our assumptions,
\[
\widetilde{\bm{U}}^{\star\top}\widetilde{\bm{U}}^{\star}=\big[(\bm{u}_{i}^{\star\top}\bm{u}_{j}^{\star})^{2}\big]_{1\leq i,j\leq r}\preceq\mathsf{diag}\big(\big[\|\bm{u}_{i}^{\star}\|_{2}^{4}\big]_{1\leq i\leq r}\big)+\big(r\max\nolimits _{i\neq j}(\bm{u}_{i}^{\star\top}\bm{u}_{j}^{\star})^{2}\big)\bm{I}_{r}=(1+o(1))\mathsf{diag}\big(\big[\|\bm{u}_{i}^{\star}\|_{2}^{4}\big]_{1\leq i\leq r}\big).
\]
}
\begin{align*}
\bm{\Sigma}_{k}^{\star} & \succeq\frac{2\sigma_{\min}^{2}}{p}\big(\widetilde{\bm{U}}^{\star\top}\widetilde{\bm{U}}^{\star}\big)^{-1}\succeq\frac{(2-o(1))\sigma_{\min}^{2}}{p}\mathrm{diag}\left(\big[\|\bm{u}_{i}^{\star}\|_{2}^{-4}\big]_{1\leq i\leq r}\right)\succeq(1-o(1))\frac{2\sigma_{\min}^{2}}{p\lambda_{\max}^{\star4/3}}\bm{I}.
\end{align*}
This tells us that the typical $\ell_{2}$ norm of each row $\bm{Z}_{k,:}$
exceeds the order of $\frac{\sigma_{\min}\sqrt{r}}{\sqrt{p}\lambda_{\max}^{\star2/3}}$,
which is hence much larger than $\|\bm{W}\|_{2,\infty}$ (by virtue
of Theorem~\ref{thm:U-loss-dist-simple-Gaussian}). To conclude,
the nonconvex estimate $\bm{U}$ is$\,$---$\,$up to global permutation$\,$---$\,$a
nearly un-biased estimate of the true tensor factors $\bm{U}^{\star}$,
with estimation errors being approximately Gaussian.

As it turns out, this distributional characterization can be extended
to accommodate a much broader class of noise beyond Gaussian noise,
as stated below.

\begin{theorem}[Distributional guarantees for tensor factor estimates
(general noise)]\label{thm:U-loss-dist-simple-nonGaussian}Suppose
that $\{E_{i,j,k}\}$ are not necessarily Gaussian but satisfy Assumption~\ref{assumption:random-noise}.
Then the decomposition in Theorem~\ref{thm:U-loss-dist-simple-Gaussian}
continues to hold, except that $\bm{Z}$ is not necessarily Gaussian
but instead obeys
\[
\left|\mathbb{P}\big\{\bm{Z}_{k,:}\in\mathcal{A}\big\}-\mathbb{P}\left\{ \bm{g}_{k}\in\mathcal{A}\right\} \right|=o\left(1\right),\qquad1\leq k\leq d
\]
for any convex set $\mathcal{A}\subset\mathbb{R}^{r}$. Here, $\bm{g}_{k}\sim\mathcal{N}\left(\bm{0},\bm{\Sigma}_{k}^{\star}\right)$
with covariance matrix $\bm{\Sigma}_{k}^{\star}$ defined in (\ref{eq:cov-matrix-m}).\end{theorem}

Before continuing, there is another important observation that is
worth pointing out (which is not included in Theorems~\ref{thm:U-loss-dist-simple-Gaussian}-\ref{thm:U-loss-dist-simple-nonGaussian}
but will be made precise in the analysis): for any three different
rows $i,j,k$, the corresponding errors $\bm{Z}_{i,:}$, $\bm{Z}_{j,:}$
and $\bm{Z}_{k,:}$ are nearly statistically independent. This is
a crucial observation that immediately allows for entrywise distributional
characterizations for the resulting tensor estimate $\bm{T}$, as
summarized below.

\begin{theorem}[Distributional guarantees for tensor entry estimates]\label{thm:T-loss-dist}Instate
the assumptions of Theorem~\ref{thm:U-loss-dist-simple-nonGaussian}.
Consider any $1\leq i\leq j\leq k\leq d$ obeying
\begin{equation}
\frac{\big\|\widetilde{\bm{U}}_{(j,k),:}^{\star}\big\|_{2}+\big\|\widetilde{\bm{U}}_{(i,j),:}^{\star}\big\|_{2}+\big\|\widetilde{\bm{U}}_{(i,k),:}^{\star}\big\|_{2}}{\|\widetilde{\bm{U}}^{\star}\|_{2,\infty}}\geq c_{5}\frac{\sigma_{\max}}{\|\bm{T}^{\star}\|_{\infty}}\sqrt{\frac{\log^{3}d}{d^{2}p}}\label{eq:U-tilde-2norm-simple-LB}
\end{equation}
for some large constant $c_{5}>0$, with $\widetilde{\bm{U}}^{\star}$
defined in (\ref{eq:defn-Ustar-tilde}). Then the estimate $\bm{T}$
defined in (\ref{eq:defn-estimate-T}) obeys
\begin{equation}
\sup_{\tau\in\mathbb{R}}\,\left|\mathbb{P}\Big\{ T_{i,j,k}\leq T_{i,j,k}^{\star}+\tau\sqrt{v_{i,j,k}^{\star}}\Big\}-\Phi(\tau)\right|=o\left(1\right),\label{eq:T-approximate-Gaussian}
\end{equation}
where $\Phi(\cdot)$ is the CDF of a standard Gaussian random variable.
Here, the variance parameters $\{v_{i,j,k}^{\star}\}$ are defined
such that for any three distinct numbers $i,j,k$, \begin{subequations}\label{def:T-entry-var}
\begin{align}
v_{i,j,k}^{\star} & :=\widetilde{\bm{U}}_{(j,k),:}^{\star}\bm{\Sigma}_{i}^{\star}\big(\widetilde{\bm{U}}_{(j,k),:}^{\star}\big)^{\top}+\widetilde{\bm{U}}_{(i,k),:}^{\star}\bm{\Sigma}_{j}^{\star}\big(\widetilde{\bm{U}}_{(i,k),:}^{\star}\big)^{\top}+\widetilde{\bm{U}}_{(i,j),:}^{\star}\bm{\Sigma}_{k}^{\star}\big(\widetilde{\bm{U}}_{(i,j),:}^{\star}\big)^{\top},\label{def:T-entry-var-ijk}\\
v_{i,i,k}^{\star} & :=4\,\widetilde{\bm{U}}_{(i,k),:}^{\star}\bm{\Sigma}_{i}^{\star}\big(\widetilde{\bm{U}}_{(i,k),:}^{\star}\big)^{\top}+\widetilde{\bm{U}}_{(i,i),:}^{\star}\bm{\Sigma}_{k}^{\star}\big(\widetilde{\bm{U}}_{(i,i),:}^{\star}\big)^{\top},\label{def:T-entry-var-ik}\\
v_{i,i,i}^{\star} & :=9\,\widetilde{\bm{U}}_{(i,i),:}^{\star}\bm{\Sigma}_{i}^{\star}\big(\widetilde{\bm{U}}_{(i,i),:}^{\star}\big)^{\top},\label{def:T-entry-var-i}
\end{align}
\end{subequations}where $\bm{\Sigma}_{k}^{\star}$ is defined in
(\ref{eq:cov-matrix-m}).\end{theorem}

\begin{comment}
\begin{remark}The right-hand side of Condition (\ref{eq:U-tilde-2norm-simple-LB})
depends linearly on the noise level $\sigma_{\max}$, which vanishes
when $\sigma_{\max}\rightarrow0$. In addition,

\end{remark}
\end{comment}

In short, the above theorem indicates that: if the ``strength''
of a tensor entry $T_{i,j,k}^{\star}$ is not exceedingly small, then
our nonconvex estimate of this entry is nearly unbiased, whose estimation
error is approximately zero-mean Gaussian with variance $v_{i,j,k}^{\star}$.
To see this, note that when (\ref{eq:requirement-p-sigma-simple})
holds, the right-hand side of Condition~(\ref{eq:U-tilde-2norm-simple-LB})
is at most $O\left(d^{-1/4}/\sqrt{\log d}\right)$, which is vanishingly
small. In other words, the Gaussian approximation is nearly tight
unless the energy $\big\|\widetilde{\bm{U}}_{(j,k),:}^{\star}\big\|_{2}+\big\|\widetilde{\bm{U}}_{(i,k),:}^{\star}\big\|_{2}+\big\|\widetilde{\bm{U}}_{(i,j),:}^{\star}\big\|_{2}$
is vanishingly small compared to the average size. This entrywise
distributional theory allows one to accommodate a broad family of
noise models.

\subsection{Confidence intervals\label{subsec:Confidence-intervals}}

The preceding distributional guarantees pave the way for uncertainty
quantification. However, an outstanding challenge remains in computing$\,$/$\,$estimating
the covariance matrices $\{\bm{\Sigma}_{k}^{\star}\}$ and the variance
parameters $\{v_{i,j,k}^{\star}\}$. In particular, these crucial
parameters are functions of both the ground truth $\{\bm{u}_{l}^{\star}\}$
and the noise variance $\{\sigma_{i,j,k}^{2}\}$, which we do not
have access to \emph{a priori}. To further complicate matters, in
the heteroscedastic case where $\{\sigma_{i,j,k}^{2}\}$ are location
varying, it is in general infeasible to estimate all variance parameters
reliably.

\paragraph{Variance and covariance estimation.} Fortunately, despite
the absence of prior knowledge about the truth and the noise parameters,
we are still able to faithfully estimate these important parameters
from the data in hand, using a simple plug-in procedure. Specifically: 
\begin{enumerate}
\item Rather than estimating all $\{\sigma_{i,j,k}\}$ directly, we turn
attention to estimating the noise components $\{E_{i,j.k}\}$ instead,
with the assistance of our tensor estimate $\bm{T}$ as follows
\begin{equation}
\widehat{E}_{i,j,k}:=T_{i,j,k}^{\mathsf{obs}}-T_{i,j,k},\qquad(i,j,k)\in\Omega.\label{def:noise-est}
\end{equation}
We then construct a diagonal matrix $\bm{D}_{k}\in\mathbb{R}^{d^{2}\times d^{2}}$
($1\leq k\leq d$) obeying
\begin{equation}
\big(\bm{D}_{k}\big)_{(i,j),(i,j)}=p^{-1}\widehat{E}_{i,j,k}^{2}\ind_{\{(i,j,k)\in\Omega\}}.\label{eq:defn-Dk}
\end{equation}
Note that $\bm{D}_{k}$ is not really a faithful estimate of the $\bm{D}_{k}^{\star}$
defined in (\ref{eq:cov-matrix-m-diag}), but it suffices for our
purpose.
\item Estimate $\widetilde{\bm{U}}^{\star}$ (cf.~(\ref{eq:defn-Ustar-tilde}))
via the plug-in estimator $\widetilde{\bm{U}}:=[\bm{u}_{s}\otimes\bm{u}_{s}]_{1\leq s\leq r}\in\mathbb{R}^{d^{2}\times r}$.
\item Substitute the above estimators into the expressions of the variance
parameters to yield our estimate. Specifically, for any $1\leq k\leq d$,
we compute
\begin{equation}
\bm{\Sigma}_{k}=\frac{2}{p}(\widetilde{\bm{U}}^{\top}\widetilde{\bm{U}})^{-1}\widetilde{\bm{U}}^{\top}\bm{D}_{k}\widetilde{\bm{U}}(\widetilde{\bm{U}}^{\top}\widetilde{\bm{U}})^{-1}\label{eq:defn-Sigmak}
\end{equation}
as an estimate of $\bm{\Sigma}_{k}^{\star}$ (cf.~(\ref{eq:cov-matrix-m})).
We also compute the estimates for $\{v_{i,j,k}^{\star}\}$ such that:
for any three distinct numbers $1\leq i,j,k\leq d$, \begin{subequations}\label{def:T-entry-var-estimate}
\begin{align}
v_{i,j,k} & :=\widetilde{\bm{U}}_{(j,k),:}\bm{\Sigma}_{i}\big(\widetilde{\bm{U}}_{(j,k),:}\big)^{\top}+\widetilde{\bm{U}}_{(i,k),:}\bm{\Sigma}_{j}\big(\widetilde{\bm{U}}_{(i,k),:}\big)^{\top}+\widetilde{\bm{U}}_{(i,j),:}\bm{\Sigma}_{k}\big(\widetilde{\bm{U}}_{(i,j),:}\big)^{\top};\label{def:T-entry-var-estimate-ijk}\\
v_{i,i,k} & :=4\,\widetilde{\bm{U}}_{(i,k),:}\bm{\Sigma}_{i}\big(\widetilde{\bm{U}}_{(i,k),:}\big)^{\top}+\widetilde{\bm{U}}_{(i,i),:}\bm{\Sigma}_{k}\big(\widetilde{\bm{U}}_{(i,i),:}\big)^{\top};\label{def:T-entry-var-estimate-ik}\\
v_{i,i,i} & :=9\,\widetilde{\bm{U}}_{(i,i),:}\bm{\Sigma}_{i}\big(\widetilde{\bm{U}}_{(i,i),:}\big)^{\top}.\label{def:T-entry-var-estimate-i}
\end{align}
\end{subequations}
\end{enumerate}
\paragraph{Confidence intervals.} With the above variance$\,$/$\,$covariance
estimates in place, we are positioned to introduce our uncertainty
quantification procedure, which consists in constructing\emph{ entrywise}
confidence intervals for both the tensor factors and the unknown tensor
as follows.
\begin{itemize}
\item For each $1\leq k\leq d$ and $1\leq l\leq r$, we construct a $(1-\alpha)$-confidence
interval for the $k$-th entry of the $l$-th tensor factor (up to
global permutation) as follows
\begin{equation}
\mathsf{CI}_{u_{l,k}}^{1-\alpha}:=\big[u_{l,k}\pm\sqrt{(\bm{\Sigma}_{k})_{l,l}}\cdot\Phi^{-1}(1-\alpha/2)\big],\label{eq:defn-CI-u}
\end{equation}
where $\Phi^{-1}(\cdot)$ is the inverse CDF of a standard Gaussian,
$[a\pm b]:=[a-b,a+b]$, and $\bm{\Sigma}_{k}$ is constructed in (\ref{eq:defn-Sigmak}).
\item For each $1\leq i,j,k\leq d$, we construct a $(1-\alpha)$-confidence
interval for the $(i,j,k)$-th entry of $\bm{T}^{\star}$ as follows
\begin{equation}
\mathsf{CI}_{T_{i,j,k}}^{1-\alpha}:=\big[T_{i,j,k}\pm\sqrt{v_{i,j,k}}\cdot\Phi^{-1}(1-\alpha/2)\big],\label{eq:defn-CI-T}
\end{equation}
where $v_{i,j,k}$ is constructed in (\ref{def:T-entry-var-estimate}).
\end{itemize}
As it turns out, the proposed (entrywise) confidence intervals are
nearly accurate, as revealed by the following theorem. The proof is
postponed to Appendix~\ref{sec:Analysis-of-CI}.

\begin{theorem}[Validity of constructed confidence intervals]\label{thm:entry-CI}Instate
the assumptions of Theorem~\ref{thm:U-loss-dist-simple-nonGaussian}.
There is a permutation $\pi(\cdot):[d]\mapsto[d]$ such that for any
$0<\alpha<1$, the confidence interval constructed in (\ref{eq:defn-CI-u})
obeys
\[
\mathbb{P}\left\{ u_{\pi(l),k}^{\star}\in\mathsf{CI}_{u_{l,k}}^{1-\alpha}\right\} =1-\alpha+o(1),\qquad\forall1\leq l\leq r,1\leq k\leq d.
\]
In addition, for any $1\leq i,j,k\leq d$ obeying (\ref{eq:U-tilde-2norm-simple-LB})
and any $0<\alpha<1$, the confidence interval constructed in (\ref{eq:defn-CI-T})
obeys
\[
\mathbb{P}\left\{ T_{i,j,k}^{\star}\in\mathsf{CI}_{T_{i,j,k}}^{1-\alpha}\right\} =1-\alpha+o\left(1\right).
\]
\end{theorem}

This theorem justifies the validity of the uncertainty quantification
procedure we propose. Several important features are worth emphasizing:
\begin{itemize}
\item \emph{``Fine-grained'' entrywise uncertainty quantification}. Our
results enable trustworthy uncertainty quantification down to quite
fine scale, namely, we are capable of assessing the uncertainty reliably
at the entrywise level for both the tensor factors and the tensor
of interest. To the best of our knowledge, accurate entrywise uncertainty
characterization for tensor completion is previously unavailable.
\item \emph{Adaptivity to heterogeneous and unknown noise distributions}.
The proposed confidence intervals do not require prior knowledge about
the noise distributions, and automatically adapt to noise heteroscedasticity
(i.e.~the case when the noise variance varies across entries). Such
model-free and adaptive features are of important practical value.
\item \emph{No need of sample splitting}. The whole procedure studied here$\,$---$\,$including
both estimation and uncertainty quantification$\,$---$\,$does not
rely on any sort of data splitting, thus effectively preventing unnecessary
information loss due to sample splitting.
\end{itemize}
\paragraph{Lower bounds.} One might naturally wonder whether the
proposed confidence intervals can be further improved; concretely,
is it possible to identify a shorter confidence interval that remains
valid? As it turns out, our procedures are, in some sense, statistically
optimal under Gaussian noise, as confirmed by the following fundamental
lower bound.

\begin{theorem}[Entrywise lower bounds]\label{thm:lower-bound-entrywise}Consider
any unbiased estimator $\widehat{\bm{u}}_{l}$ for $\bm{u}_{l}^{\star}$
$(1\leq l\leq r)$ and any unbiased estimator $\widehat{\bm{T}}$
for $\bm{T}^{\star}$. Suppose that $\{E_{i,j,k}\}$ are i.i.d.~Gaussians
and that Assumptions \ref{assumption:random-sampling}-\ref{assumption:incoherence}
hold. If $\mu,\kappa,r=O(1)$ and 
\[
p\geq c_{6}\frac{\log^{2}d}{d^{2}}
\]
for some sufficiently large constant $c_{6}>0$, then the following
holds with probability at least $1-O(d^{-10})$:
\begin{align*}
\mathsf{Var}\big[\widehat{u}_{l,k}\big] & \geq\left(1-o\left(1\right)\right)\big(\bm{\Sigma}_{k}^{\star}\big)_{l,l},\qquad1\leq k\leq d;\\
\mathsf{Var}\big[\widehat{T}_{i,j,k}\big] & \geq\left(1-o\left(1\right)\right)v_{i,j,k}^{\star},\qquad\quad1\leq i,j,k\leq d.
\end{align*}
\end{theorem}Taken collectively with Theorems~\ref{thm:U-loss-dist-simple-nonGaussian}
and \ref{thm:T-loss-dist}, the above result reveals that our nonconvex
estimators $\{\bm{u}_{l}\}$ and $\bm{T}$ achieve minimal mean square
estimation errors in a very sharp manner at the entrywise level. Recognizing
that the proposed confidence intervals allow for accurate assessment
of the uncertainty (by virtue of Theorem~\ref{thm:entry-CI}), we
conclude that the proposed inferential procedures are, in some sense,
un-improvable under i.i.d.~Gaussian noise (including both the rates
and the pre-constants).

\subsection{Back to estimation: $\ell_{2}$ optimality of nonconvex estimates\label{subsec:Back-to-estimation}}

Thus far, we have established optimality of the estimators $u_{l,k}$
($1\leq l\leq r,1\leq k\leq d$) and $T_{i,j,k}$ (for those $i,j,k$
obeying (\ref{eq:U-tilde-2norm-simple-LB})) in an entrywise sense.
These results taken together allow one to uncover the $\ell_{2}$
optimality of the nonconvex optimization approach as well. Our result
is this:

\begin{theorem}[Optimality w.r.t.~$\ell_{2}$ estimation accuracy]\label{thm:optimality-L2}Instate
the assumptions of Theorem~\ref{thm:U-loss-dist-simple-nonGaussian}.
With probability exceeding $1-o(1)$, the estimates returned by Algorithm~\ref{alg:gd}
obey\begin{subequations}
\begin{align}
\big\|\bm{u}_{\pi(l)}-\bm{u}_{l}^{\star}\big\|_{2}^{2} & =\frac{\left(2+o\left(1\right)\right)\sigma_{\max}^{2}d}{p\,\big\|\bm{u}_{l}^{\star}\big\|_{2}^{4}},\qquad\forall1\leq l\leq r\label{eq:L2-u}\\
\left\Vert \bm{T}-\bm{T}^{\star}\right\Vert _{\mathrm{F}}^{2} & =\frac{\left(6+o\left(1\right)\right)\sigma_{\max}^{2}dr}{p}\label{eq:L2-T}
\end{align}
\end{subequations}for some permutation $\pi(\cdot):[d]\mapsto[d]$.
In addition, if $\{E_{i,j,k}\}$ are Gaussian, \end{theorem}

\begin{theorem}[Lower bound w.r.t~$\ell_{2}$ estimation accuracy]\label{thm:l2-error-lower-bound}Instate
the assumptions of Theorem~\ref{thm:lower-bound-entrywise}. The
following holds with probability at least $1-O(d^{-10})$: any unbiased
estimator $\widehat{\bm{u}}_{l}$ (resp.~$\widehat{\bm{T}}$) for
$\bm{u}_{l}^{\star}$ (resp.~$\bm{T}^{\star}$) necessarily obeys
\begin{align}
\mathbb{E}\left[\big\|\widehat{\bm{u}}_{l}-\bm{u}_{l}^{\star}\big\|_{2}^{2}\right] & \geq\frac{\left(2-o\left(1\right)\right)\sigma_{\min}^{2}d}{p\,\big\|\bm{u}_{l}^{\star}\big\|_{2}^{4}};\qquad\mathbb{E}\left[\big\|\widehat{\bm{T}}-\bm{T}^{\star}\big\|_{\mathrm{F}}^{2}\right]\geq\frac{\left(6-o\left(1\right)\right)\sigma_{\min}^{2}dr}{p}.\label{eq:L2-lower-bound}
\end{align}
\end{theorem}

Here, the characterization of the $\ell_{2}$ risk (\ref{eq:L2-u})
for $\bm{u}_{l}$ is a straightforward consequence of Theorems~\ref{thm:U-loss-dist-simple-Gaussian}-\ref{thm:U-loss-dist-simple-nonGaussian},
and the lower bounds (\ref{eq:L2-lower-bound}) follow immediately
from Theorem~\ref{thm:lower-bound-entrywise}. Establishing the $\ell_{2}$
risk (\ref{eq:L2-T}) for $\bm{T}$ requires more work, as Theorem
\ref{thm:T-loss-dist} is valid only for a set of entries obeying
(\ref{eq:U-tilde-2norm-simple-LB}). Fortunately, a majority of the
entries of $\bm{T}^{\star}$ satisfy (\ref{eq:U-tilde-2norm-simple-LB}),
thus allowing for a nearly accurate approximation of the Euclidean
risk of $\bm{T}$. All in all, Theorems~\ref{thm:optimality-L2}
and \ref{thm:l2-error-lower-bound} deliver an encouraging news: when
the noise components are i.i.d.~Gaussian, nonconvex optimization
is information-theoretically optimal when estimating both the unknown
tensor and its underlying tensor factors.

%% file: experiment.tex
\subsection{Numerical experiments}

\label{subsec:Numerical-experiments}

To validate our theory and demonstrate the practical applicability
of our inferential procedures, we perform a series of numerical experiments
for a variety of settings. Specifically, we set $d=100$, $p=0.2$,
and generate the ground-truth tensor $\bm{T}^{\star}=\sum_{l=1}^{r}\bm{u}_{l}^{\star}$
in a random fashion such that $\bm{u}_{l}^{\star}\overset{\mathsf{i.i.d.}}{\sim}\mathcal{N}(\bm{0},\bm{I}_{d})$.
Regarding the algorithmic parameters for nonconvex optimization (i.e.~Algorithm~\ref{alg:gd}),
we choose $L=r^{2}$, $\epsilon_{\mathsf{th}}=0.4$, $\eta_{t}\equiv3\times10^{-5}/p$,
and $t_{0}=100$. The noise components are independently drawn from
Gaussian distributions, obeying $E_{i,j,k}\sim\mathcal{N}(0,\sigma_{i,j,k}^{2}),1\leq i\leq j\leq k\leq d$
with variance $\sigma_{i,j,k}^{2}$ constructed as follows. We generate
$w_{i,j,k}\overset{\mathsf{i.i.d.}}{\sim}\mathsf{Unif}[0,1],1\leq i,j,k\leq d$
and let
\[
\sigma_{i,j,k}^{2}=\frac{\sigma^{2}w_{i,j,k}^{\beta}}{\sum_{1\leq i\leq j\leq k\leq d}w_{i,j,k}^{\beta}}\frac{d^{3}}{6},
\]
where $\beta$ represents the degree of heteroscedasticity. The noise
becomes more heteroscedastic as $\beta$ increases, and setting $\beta=0$
reduces to the homoscedastic case where the noise variances are identical
across all entries. In what follows, we set $\beta=5$.

\begin{comment}
Simulation results are reported over 100 independent trials in all
experiments.
\end{comment}

\paragraph{Tensor factor entries.}

We begin with inference for the entries of the tensor factors of interest.
Consider the construction of $95\%$ confidence intervals (i.e.~$\alpha=0.05$).
Define the normalized estimation error as follows
\[
R_{l,k}^{\mathsf{U}}:=\frac{1}{\sqrt{\left(\bm{\Sigma}_{k}\right)_{l,l}}}\left(u_{l,k}-u_{l,k}^{\star}\right),\qquad1\leq l\leq r,\,1\leq k\leq d.
\]
For each $1\leq l\leq r$ and $1\leq k\leq d$, we denote by $\mathsf{CR}_{l,k}$
the empirical coverage rate for the tensor factor entry $u_{l,k}^{\star}$
over 100 independent trials. Let $\mathsf{Mean}(\mathsf{CR})$ (resp.~$\mathsf{Std}\left(\mathsf{CR}\right)$)
denote the average (resp.~the standard deviation) of $\{\mathsf{CR}_{l,k}\}$
over all tensor factor entries. Figure~\ref{fig:QQ-u-entry} displays
the Q-Q (quantile-quantile) plots of $R_{1,1}^{\mathsf{F}}$ vs.~a
standard Gaussian random variable, and Table~\ref{table:coverage-u-entry}
summarizes the numerical results for varying $p,r$ and $\sigma$.
Encouragingly, the empirical coverage rates are all very close to
$95\%$, and the empirical distributions of the normalized estimation
errors are all well approximated by a standard Gaussian distribution.

\begin{figure}[t]
\centering

\begin{tabular}{ccc}
\includegraphics[width=0.31\textwidth]{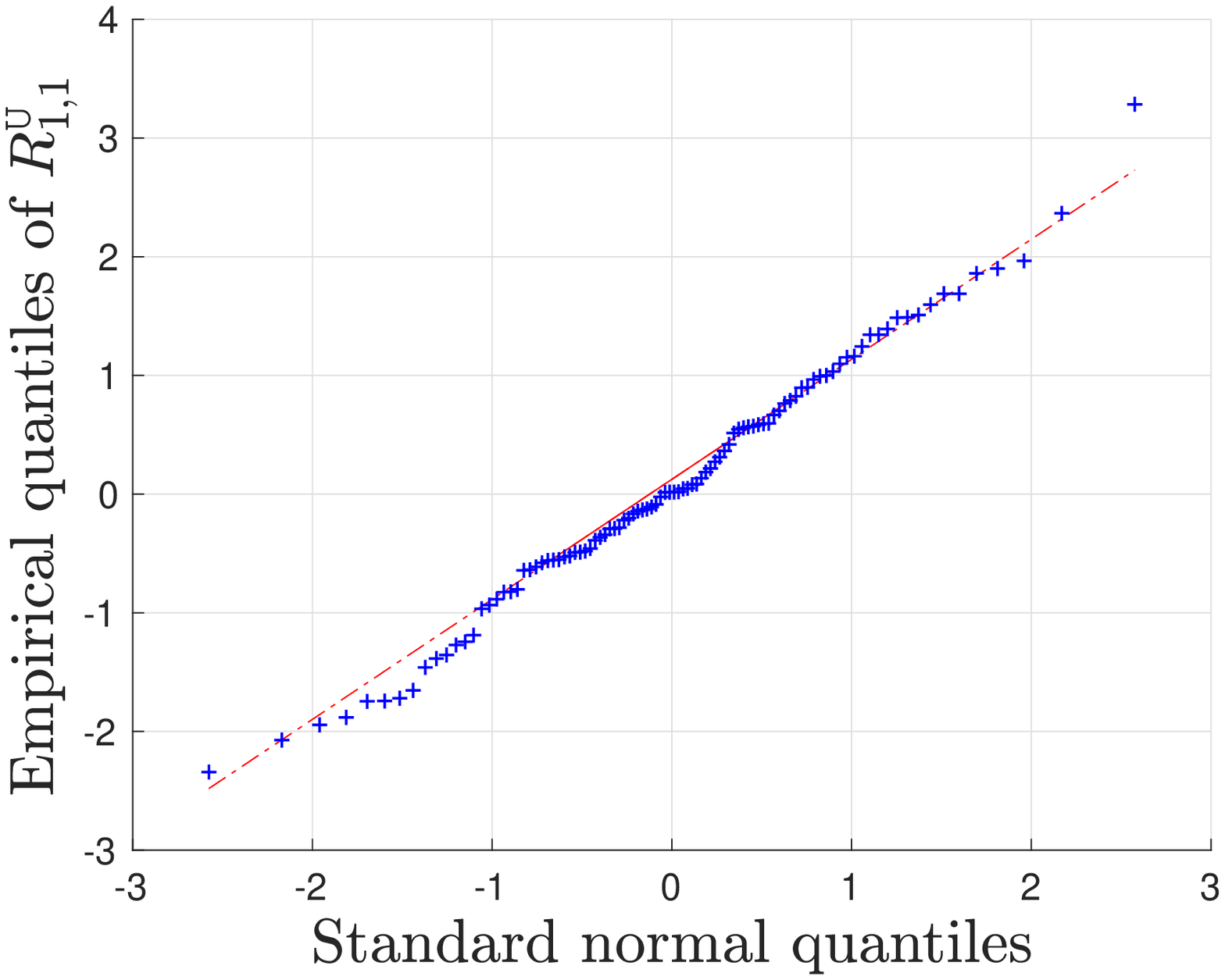} & \includegraphics[width=0.31\textwidth]{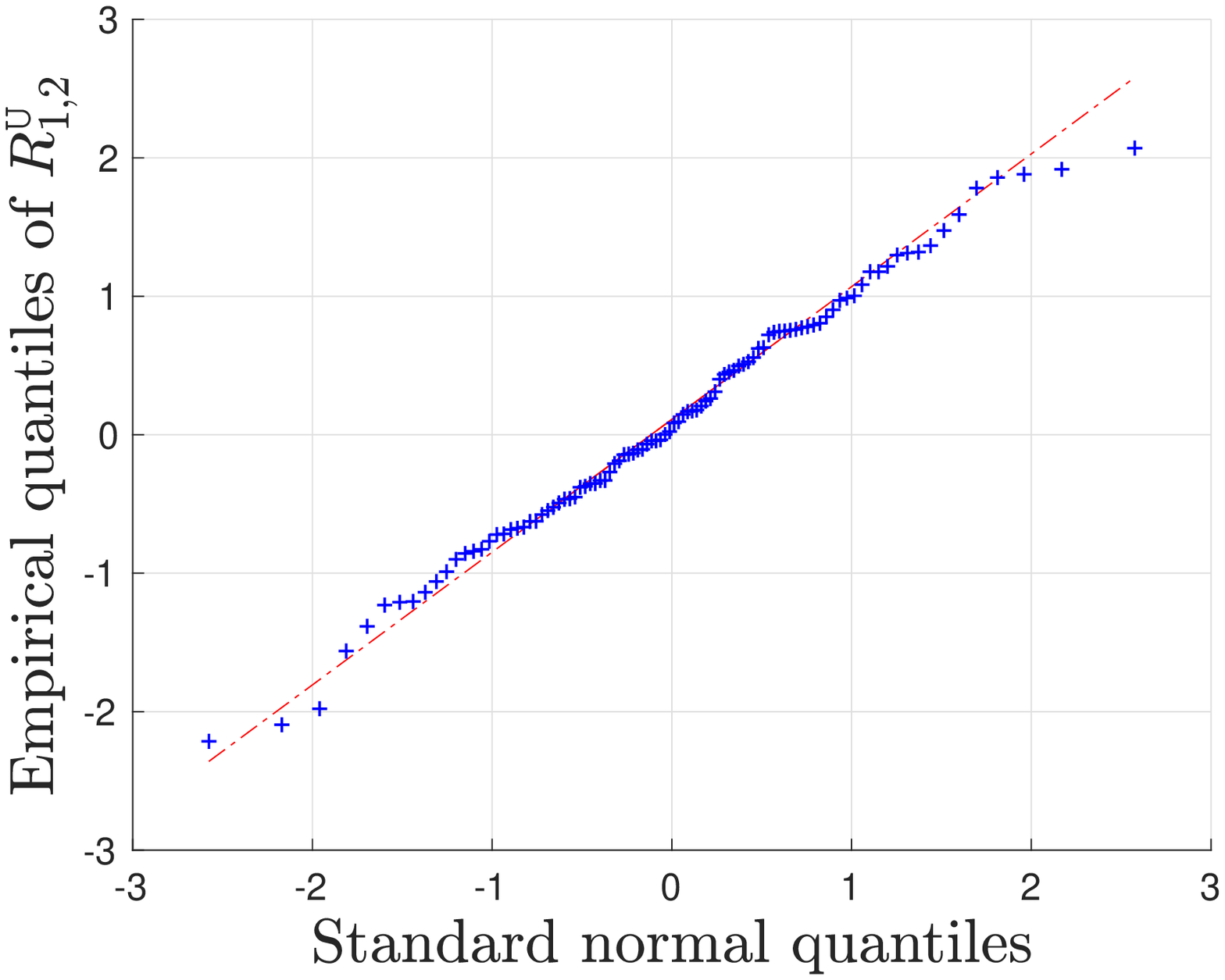} & \includegraphics[width=0.31\textwidth]{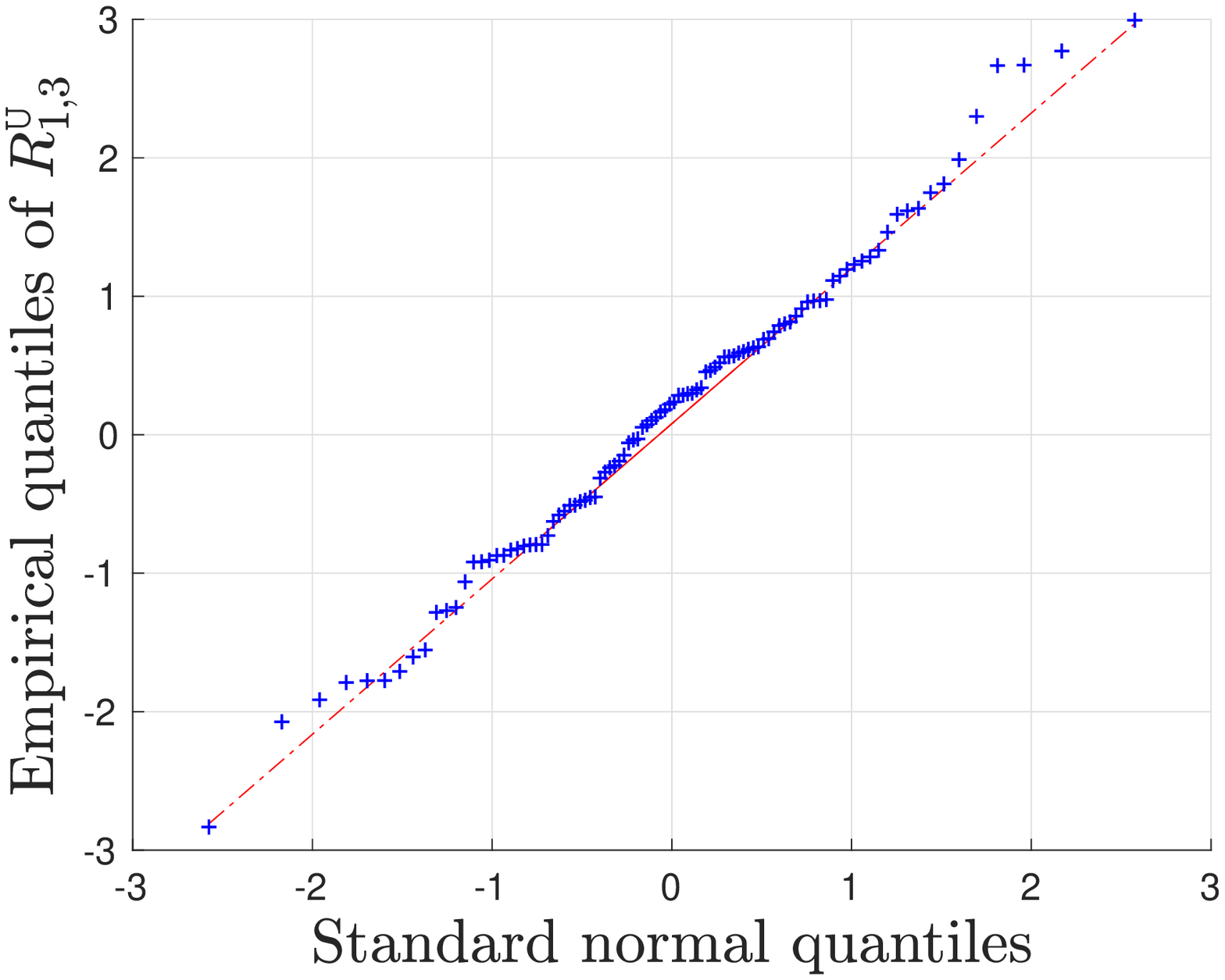}\tabularnewline
(a) & (b) & (c)\tabularnewline
\end{tabular}

\caption{Q-Q (quantile-quantile) plots of $R_{1,1}^{\mathsf{U}}$, $R_{1,2}^{\mathsf{U}}$
and $R_{1,3}^{\mathsf{U}}$ vs.~a standard Gaussian distribution
(where $r=4$, $p=0.2$, $\sigma=0.1$ and $\beta=5$).\label{fig:QQ-u-entry}}
\end{figure}

\begin{table}
\centering \caption{Empirical coverage rates of tensor factor entries for varying $r$
and $\sigma$.\label{table:coverage-u-entry}}
\begin{tabular}{c|c|c}
\hline 
$(r,\sigma)$ & $\mathsf{Mean}(\mathsf{CR})$ & $\mathsf{Std}(\mathsf{CR})$\tabularnewline
\hline 
$(2,10^{-2})$ & $0.9481$ & $0.0201$\tabularnewline
\hline 
$(2,10^{-1})$ & $0.9477$ & $0.0228$\tabularnewline
\hline 
$(2,1)$ & $0.9478$ & $0.0215$\tabularnewline
\hline 
$(4,10^{-2})$ & $0.9450$ & $0.0218$\tabularnewline
\hline 
$(4,10^{-1})$ & $0.9472$ & $0.0231$\tabularnewline
\hline 
$(4,1)$ & $0.9462$ & $0.0234$\tabularnewline
\hline 
\end{tabular}
\end{table}

\paragraph{Tensor entries.}

Next, we turn to inference for tensor entries. Similar to the above
case, we intend to construct $95\%$ confidence intervals. Define
\[
R_{i,j,k}^{\mathsf{T}}:=\frac{1}{\sqrt{v_{i,j,k}}}\left(T_{i,j,k}-T_{i,j,k}^{\star}\right),\qquad1\leq i\leq j\leq k\leq d.
\]
For each $1\leq i\leq j\leq k\leq d$, we record the empirical coverage
rate $\mathsf{CR}_{i,j,k}$ for the tensor entry $T_{i,j,k}^{\star}$
over 100 Monte Carlo trials. Denote by $\mathsf{Mean}(\mathsf{CR})$
(resp.~$\mathsf{Std}(\mathsf{CR})$) the average (resp.~the standard
deviation) of $\{\mathsf{CR}_{i,j,k}\}$ over entries $1\leq i\leq j\leq k\leq d$.
Figure~\ref{fig:QQ-T-entry} depicts the Q-Q (quantile-quantile)
plots of $R_{1,1,1}^{\mathsf{T}}$, $R_{1,1,2}^{\mathsf{T}}$ and
$R_{1,2,3}^{\mathsf{T}}$ vs.~a standard Gaussian random variable.
Table~\ref{table:coverage-T-entry} collects the numerical results
$\mathsf{Mean}(\mathsf{CR})$ and $\mathsf{Std}(\mathsf{CR})$ for
a variety of settings. Similar to previous experiments, the confidence
intervals and the Q-Q plots match our theoretical prediction in a
reasonably well manner.

\begin{figure}[t]
\centering

\begin{tabular}{ccc}
\includegraphics[width=0.31\textwidth]{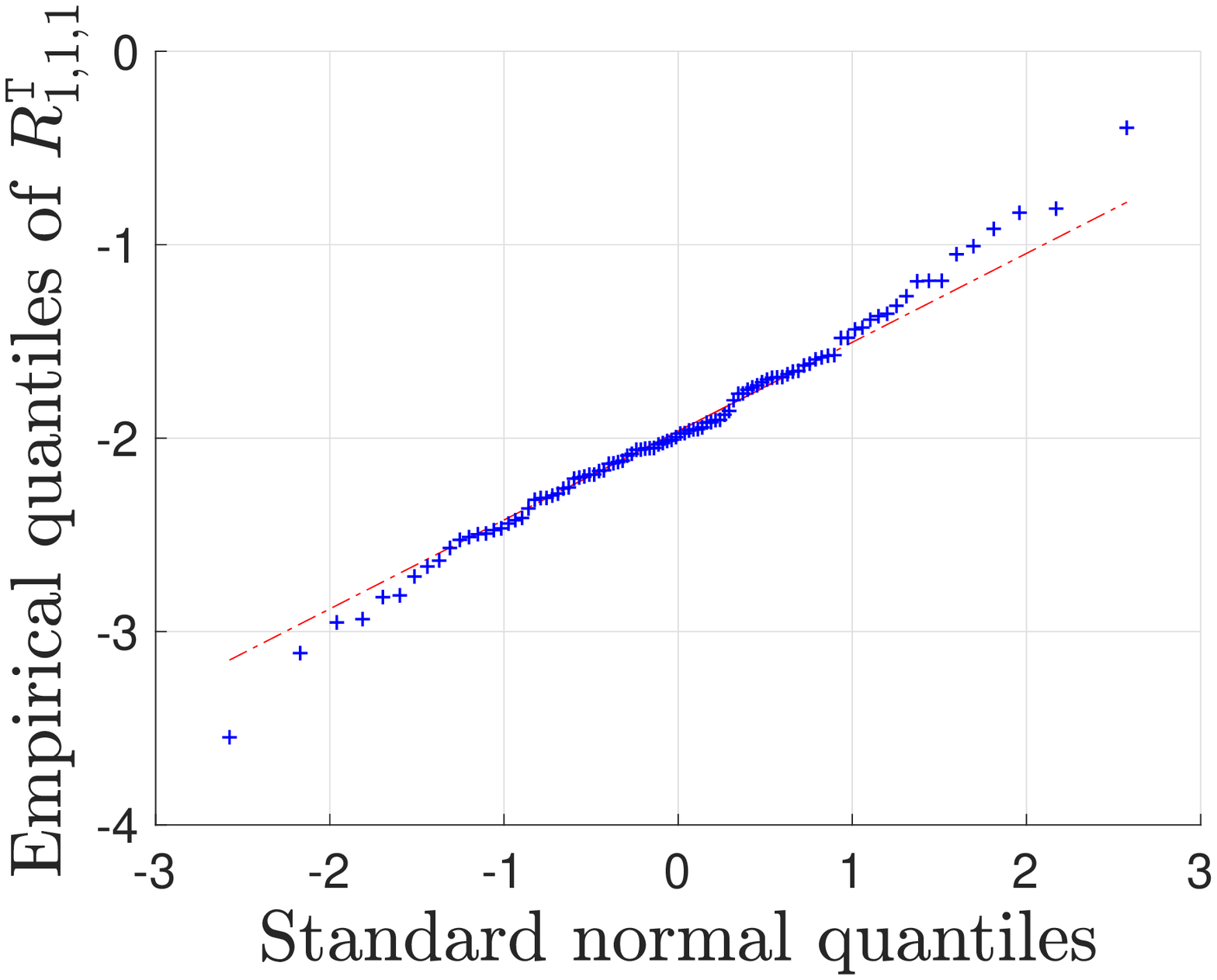} & \includegraphics[width=0.31\textwidth]{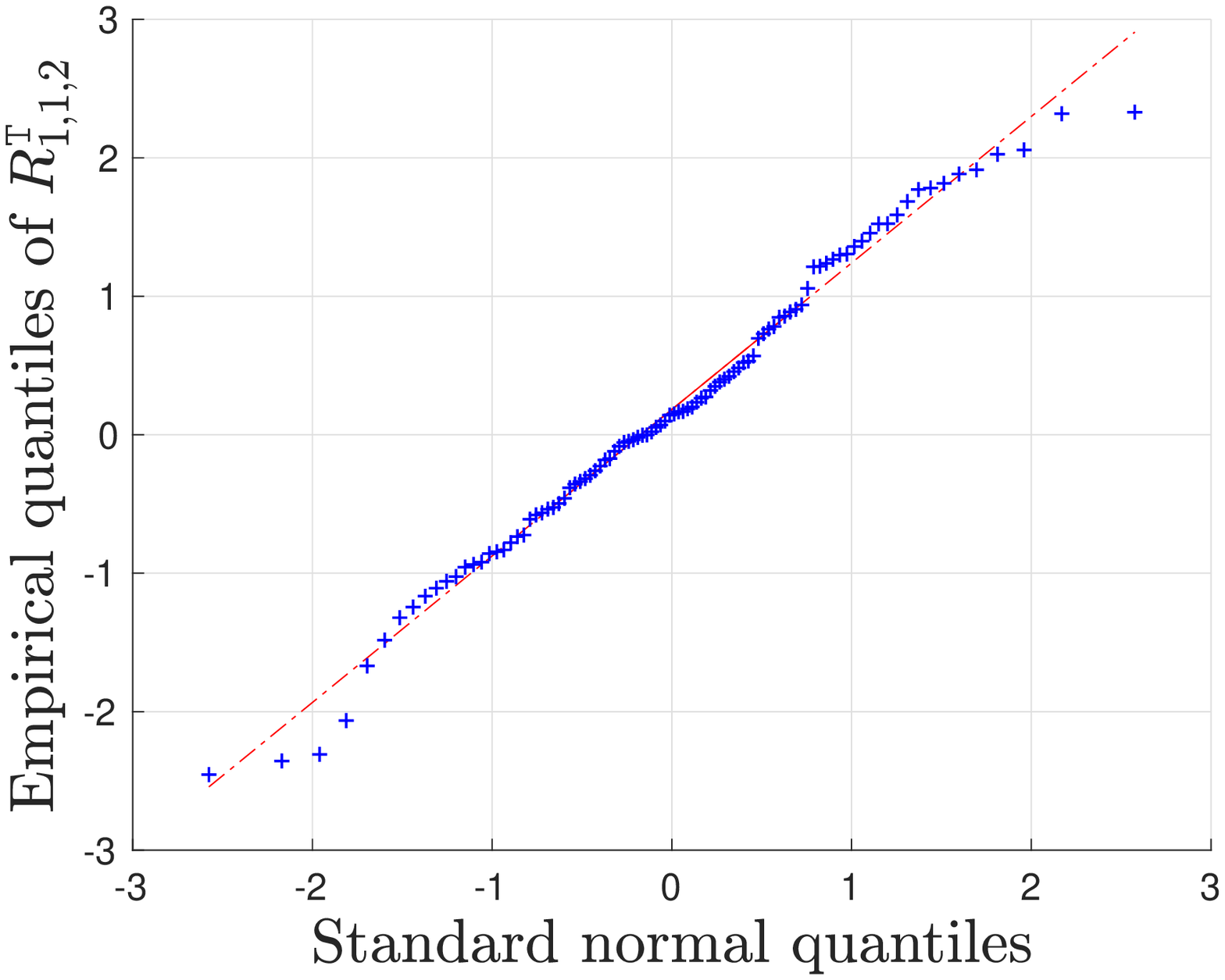} & \includegraphics[width=0.31\textwidth]{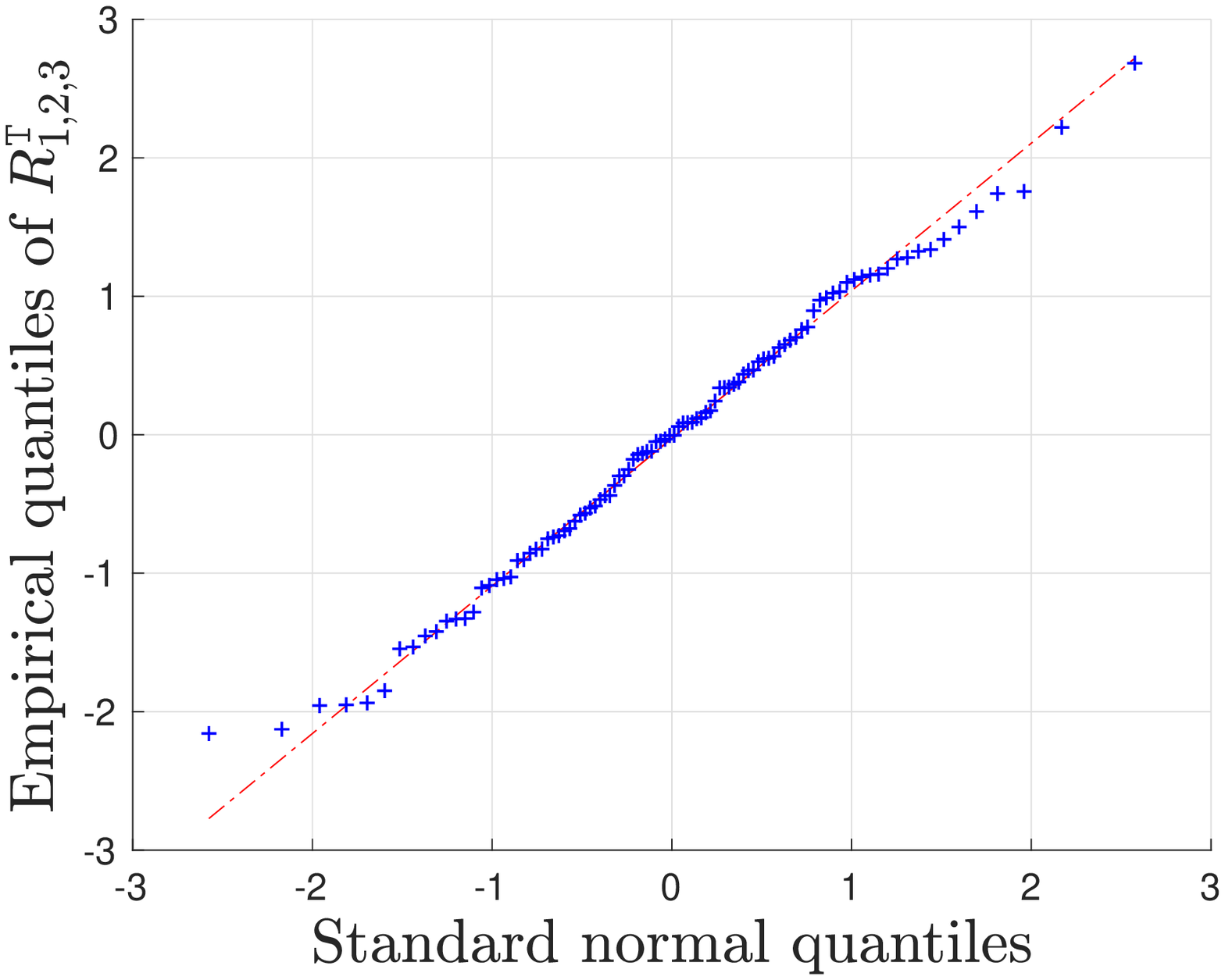}\tabularnewline
(a) & (b) & (c)\tabularnewline
\end{tabular}

\caption{Q-Q (quantile-quantile) plots of $R_{1,1,1}^{\mathsf{T}}$, $R_{1,1,2}^{\mathsf{T}}$
and $R_{1,2,3}^{\mathsf{T}}$ vs.~a standard Gaussian distribution
(where $r=4$, $p=0.2$, $\sigma=0.1$ and $\beta=5$).\label{fig:QQ-T-entry}}
\end{figure}

\begin{table}
\centering\caption{Empirical coverage rates of tensor entries for different $r$ and
$\sigma$.\label{table:coverage-T-entry}}

\vspace{0.5em}%
\begin{tabular}{c|c|c}
\hline 
$(r,\sigma)$ & $\mathsf{Mean}(\mathsf{CR})$ & $\mathsf{Std}(\mathsf{CR})$\tabularnewline
\hline 
$(2,10^{-2})$ & $0.9494$ & $0.0218$\tabularnewline
\hline 
$(2,10^{-1})$ & $0.9513$ & $0.0218$\tabularnewline
\hline 
$(2,1)$ & $0.9475$ & $0.0222$\tabularnewline
\hline 
$(4,10^{-2})$ & $0.9434$ & $0.0225$\tabularnewline
\hline 
$(4,10^{-1})$ & $0.9494$ & $0.0220$\tabularnewline
\hline 
$(4,1)$ & $0.9494$ & $0.0219$\tabularnewline
\hline 
\end{tabular}
\end{table}

\paragraph{$\ell_{2}$ estimation accuracy.}

Finally, let us verify the Euclidean estimation guarantees we develop
for Algorithm \ref{alg:gd}. Figure~\ref{fig:l2-loss} plots the
Euclidean estimation errors of the tensor factor estimate $\bm{u}_{1}$
(resp.~the tensor estimate $\bm{T}$). In this series of experiments,
we focus on the homoskedastic case, i.e.~ $\beta=0$. As one can
see, the empirical $\ell_{2}$ risks are exceedingly close to the
Cramér--Rao lower bounds supplied in Theorem \ref{thm:optimality-L2}.

\begin{figure}[t]
\centering

\begin{tabular}{cc}
\includegraphics[width=0.31\textwidth]{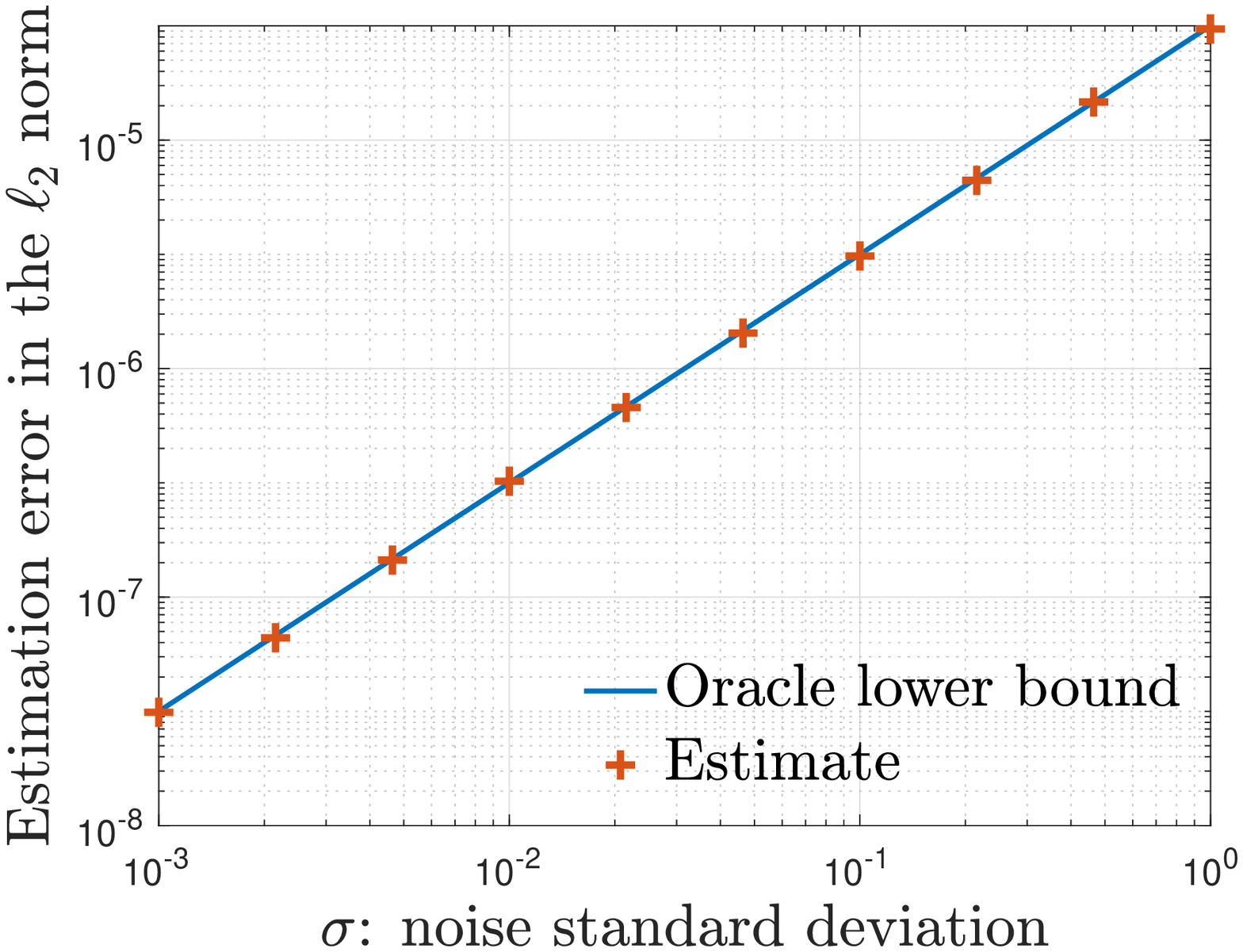} & \includegraphics[width=0.31\textwidth]{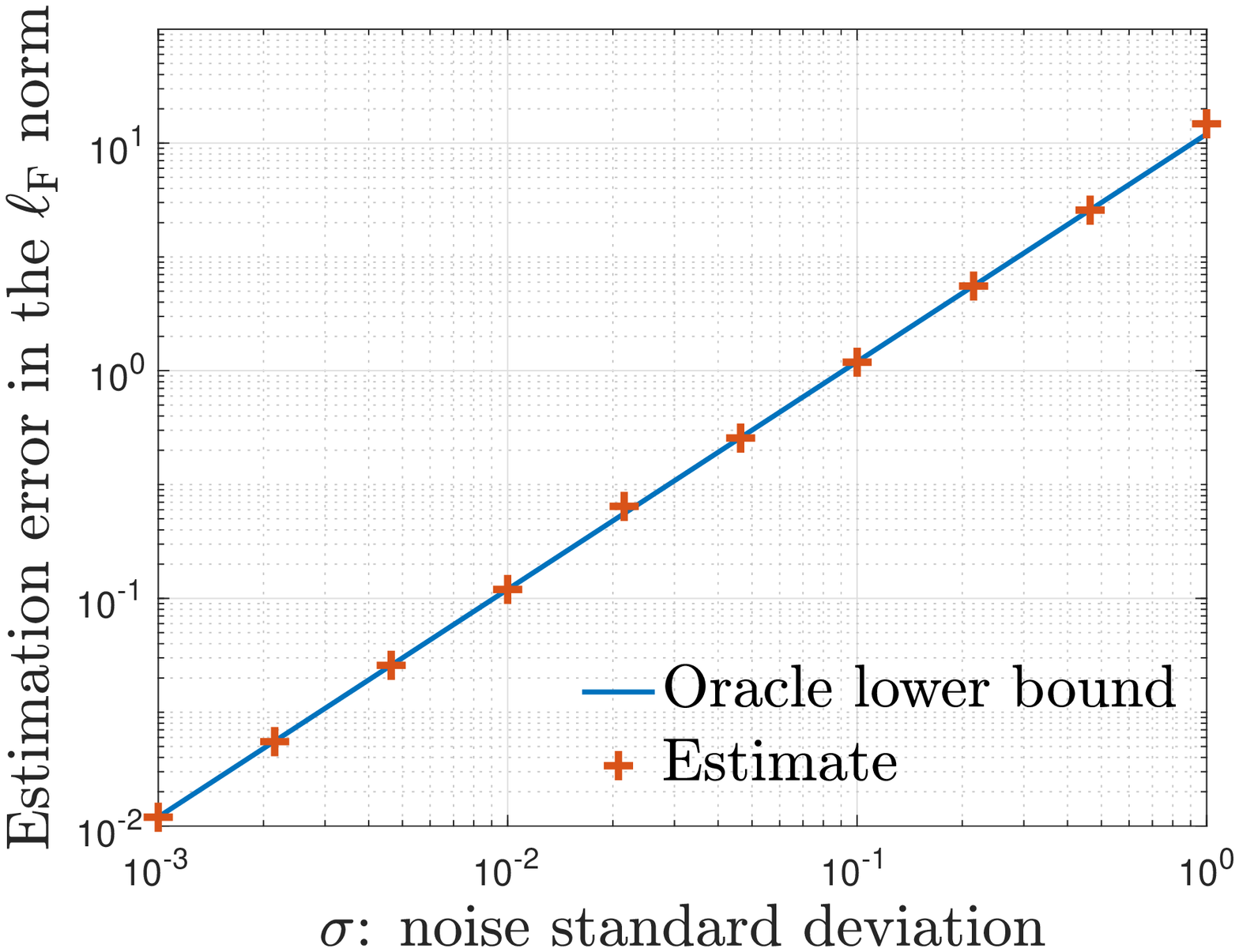}\tabularnewline
(a) & (b)\tabularnewline
\end{tabular}

\caption{(a) $\ell_{2}$ estimation error of $\bm{u}_{1}$ vs.~the Cramér--Rao
lower bound; (b) Euclidean estimation errors of $\bm{T}$ vs.~the
Cramér--Rao lower bound (where $r=4$, $p=0.2$ and $\beta=0$).\label{fig:l2-loss}}
\end{figure}

%% file: related-work.tex
\section{Prior art}

\label{sec:Related-Work}

Much progress has been made in the past few years towards understanding
and solving low-rank tensor completion. Inspired by the success of
convex relaxation for matrix completion \cite{candes2009exact,CanPla10,Gross2011recovering,chen2019noisy,li2011compressed},
an estimate based on tensor nuclear norm minimization was proposed
by \cite{yuan2016tensor,yuan2017incoherent}, which enables information-theoretically
optimal sample complexity. Unfortunately, the tensor nuclear norm
is itself NP-hard to compute and hence computationally infeasible
in practice. To allow for more economical algorithms, a widely adopted
strategy is to unfold the tensor data into a matrix \cite{tomioka2010estimation,gandy2011tensor,liu2013tensor,mu2014square},
thus transforming it into a low-rank matrix completion problem \cite{candes2009exact,KesMonSew2010,chi2018nonconvex}.
However, unfolding a third-order tensor often leads to an extremely
unbalanced matrix, thereby resulting in sub-optimal sample complexity
when directly invoking matrix completion theory. To address this issue,
a recent line of work \cite{barak2016noisy,potechin2017exact} suggested
the use of sum-of-squares (SOS) hierarchy, which performs convex relaxation
after lifting the data into higher dimension. The SOS-based algorithms
achieve a sample complexity on the order of $rd^{3/2}$ for third-order
tensors, which is widely conjectured to be optimal among all polynomial-time
algorithms. However, despite their polynomial-time complexity, the
SOS-based methods remain too expensive for solving large-scale practical
problems, primarily due to the lifting operation.

Motivated by the above computational concerns, several nonconvex approaches
have been developed, which often consist of two stages: (1) finding
a rough initialization; (2) local refinement. Existing initialization
schemes include unfolding-based spectral methods \cite{montanari2018spectral,xia2017statistically,cai2019nonconvex,cai2019subspace,xia2017polynomial,liu2020tensor}.
tensor power methods \cite{jain2014provable}, tensor SVD \cite{zhang2017exact},
and so on. To improve the estimation accuracy, the local refinement
stage invokes nonconvex optimization algorithms like alternating minimization
\cite{jain2014provable,liu2020tensor}, gradient descent \cite{cai2019nonconvex,han2020optimal},
manifold-based optimization \cite{xia2017polynomial}, block coordinate
decent \cite{ji2016tensor}, etc. These were motivated in part by
the effectiveness of nonconvex optimization in solving nonconvex low-complexity problems \cite{burer2003nonlinear,srebro2004learning,KesMonSew2010,Se2010Noisy,jain2013low,candes2014wirtinger,ChenCandes15solving,chen2015fast,ma2017implicit,chen2016projected,chen2018gradient,netrapalli2014non,hao2020sparse,zhang2016provable,cai2017fast,chen2019nonconvex,wang2016solving3,chen2020bridging,  sun2018geometric, qu2019convolutional, tong2020accelerating,zhang2017nonconvex};
see an overview of recent development in \cite{chi2018nonconvex}.
Various statistical and computational guarantees have been provided
for these algorithms, all of which have been shown to run in polynomial
time. In particular, (unfolding-based) spectral initialization followed
by gradient descent converges linearly to an accuracy that is within
a logarithmic factor from optimal \cite{cai2019nonconvex}.

However, none of the above results suggested how to evaluate the uncertainty
of the resulting estimates in a meaningful way. Despite a large body
of work on statistical estimation for noisy tensor completion, it
remains completely unclear how to exploit existing results to construct
valid yet short confidence intervals for the unknown tensor. Perhaps
the work closest to the current paper is inference and uncertainty
quantification for noisy matrix completion and matrix denoising \cite{chen2019inference,xia2019statistical,cheng2020inference},
which enables optimal construction of confidence intervals on the
basis of nonconvex matrix completion algorithms. Inference for singular
subspaces has also been investigated under both low-rank matrix regression
and denoising settings \cite{xia2018confidence,xia2019data}. While
these results might potentially be applicable here by first matricizing
the data, the resulting sample complexity, as discussed above, could
be pessimistic. Finally, we remark that construction of confidence
intervals has been extensively studied in a variety of high-dimensional
sparse estimation settings \cite{zhang2014confidence,van2014asymptotically,javanmard2014confidence,ren2015asymptotic,ning2017general,cai2017confidence,cai2016geometric,ma2017inter,sur2017likelihood,jankova2018biased,miolane2018distribution}.
Both the inferential approaches and the analysis techniques therein,
however, are drastically different from the ones employed to perform
inference for either tensor completion or matrix completion.

%% file: analysis.tex
\section{Analysis\label{sec:Analysis}}

This section outlines the proof for our main theorems. 

\subsection{A set of more general theorems}

We begin by presenting a set of more general theorems that allow both
$r$ and $\mu$ to grow with the dimension $d$. As can be straightforwardly
verified, the theorems stated below subsume as special cases the main
theorems presented in Section~\ref{sec:main-results}.

\begin{theorem}[Distributional guarantees for tensor factor estimates
(Gaussian noise, general $(r,\mu)$)]\label{thm:U-loss-dist-Gaussian-rank-r}Suppose
that the $E_{i,j,k}$'s are Gaussian, and that Assumptions \ref{assumption:random-sampling}-\ref{assumption:incoherence}
hold. Assume that $\kappa\asymp1$, and that $t_{0}=c_{0}\log d$,
\begin{align}
p & \geq c_{1}\frac{\mu^{4}r^{4}\log^{5}d}{d^{3/2}},\quad\frac{c_{2}}{d^{100}}\leq\frac{\sigma_{\max}}{\lambda_{\min}^{\star}}\leq\frac{c_{3}\sqrt{p}}{\mu r^{3/2}d^{3/4}\log^{2}d}\quad\text{and}\quad r\leq c_{4}\left(\frac{d}{\mu^{6}\log^{6}d}\right)^{1/6}\label{eq:requirement-p-sigma-rank-r}
\end{align}
hold for some sufficiently large (resp.~small) constant $c_{0},c_{1},c_{2}>0$
(resp.~$c_{3},c_{4}>0$). Then with probability at least $1-o(1)$,
one has the following decomposition:
\[
\bm{U}\bm{\Pi}-\bm{U}^{\star}=\bm{Z}+\bm{W},
\]
where $\bm{\Pi}$ is defined in (\ref{eq:defn-permutation}), $\left\Vert \bm{W}\right\Vert _{2,\infty}=o\big(\frac{\sigma_{\min}}{\lambda_{\max}^{\star2/3}\sqrt{p}}\big)$,
and for any $1\leq k\leq d$, $\bm{Z}_{k,:}\sim\mathcal{N}(\bm{0},\bm{\Sigma}_{k}^{\star})$
with covariance matrix $\bm{\Sigma}_{k}^{\star}$ defined in (\ref{eq:cov-matrix-m}).\end{theorem}\begin{remark}Theorem
\ref{thm:U-loss-dist-Gaussian-rank-r} subsumes Theorem \ref{thm:U-loss-dist-simple-Gaussian}
as a special case. \end{remark}

\begin{theorem}[Distributional guarantees for tensor factor estimates
(general noise, general $(r,\mu)$)]\label{thm:U-loss-dist-nonGaussian-rank-r}Suppose
that Assumptions \ref{assumption:random-sampling} and \ref{assumption:incoherence}
hold, and that $\{E_{i,j,k}\}$ are not necessarily Gaussian but satisfy
Assumption~\ref{assumption:random-noise}. Then under the condition
(\ref{eq:requirement-p-sigma-rank-r}), the decomposition in Theorem~\ref{thm:U-loss-dist-Gaussian-rank-r}
continues to hold, except that $\bm{Z}$ is not necessarily Gaussian
but instead obeys
\[
\left|\mathbb{P}\big\{\bm{Z}_{k,:}\in\mathcal{A}\big\}-\mathbb{P}\left\{ \bm{g}_{k}\in\mathcal{A}\right\} \right|=o\left(1\right),\qquad1\leq k\leq d
\]
for any convex set $\mathcal{A}\subset\mathbb{R}^{r}$. Here, $\bm{g}_{k}\sim\mathcal{N}\left(\bm{0},\bm{\Sigma}_{k}^{\star}\right)$
with covariance matrix $\bm{\Sigma}_{k}^{\star}$ defined in (\ref{eq:cov-matrix-m}).\end{theorem}\begin{remark}Theorem
\ref{thm:U-loss-dist-nonGaussian-rank-r} subsumes Theorem \ref{thm:U-loss-dist-simple-nonGaussian}
as a special case. \end{remark}

\begin{theorem}[Distributional guarantees for tensor entry estimates
(general $(r,\mu)$)]\label{thm:T-loss-dist-rank-r}Instate the assumptions
of Theorem~\ref{thm:U-loss-dist-nonGaussian-rank-r}. Consider any
$1\leq i\leq j\leq k\leq d$ obeying
\begin{equation}
\big\|\widetilde{\bm{U}}_{(j,k),:}^{\star}\big\|_{2}+\big\|\widetilde{\bm{U}}_{(i,j),:}^{\star}\big\|_{2}+\big\|\widetilde{\bm{U}}_{(i,k),:}^{\star}\big\|_{2}>c_{5}\frac{\sigma_{\max}}{\lambda_{\min}^{\star}}\sqrt{\frac{\mu^{5}r^{3}\log^{3}d}{dp}}\lambda_{\max}^{\star2/3}\label{eq:U-tilde-2norm-LB-rank-r}
\end{equation}
for some large constant $c_{5}>0$, with $\widetilde{\bm{U}}^{\star}$
defined in (\ref{eq:defn-Ustar-tilde}). Then the estimate $\bm{T}$
defined in (\ref{eq:defn-estimate-T}) obeys
\[
\sup_{\tau\in\mathbb{R}}\,\left|\mathbb{P}\Big\{ T_{i,j,k}\leq T_{i,j,k}^{\star}+\tau\sqrt{v_{i,j,k}^{\star}}\Big\}-\Phi(\tau)\right|=o\left(1\right),
\]
where $\Phi(\cdot)$ is the CDF of a standard Gaussian random variable.
Here, the variance parameters $\{v_{i,j,k}^{\star}\}$ are defined
in (\ref{def:T-entry-var}).\end{theorem}\begin{remark}Theorem \ref{thm:T-loss-dist-rank-r}
subsumes Theorem \ref{thm:T-loss-dist} as a special case. \end{remark}

\begin{theorem}[Validity of confidence intervals (general $(r,\mu)$]\label{thm:entry-CI-rank-r}Instate
the assumptions of Theorem~\ref{thm:U-loss-dist-nonGaussian-rank-r}.
There exists a permutation $\pi(\cdot):[d]\mapsto[d]$ such that for
any $0<\alpha<1$, the confidence interval constructed in (\ref{eq:defn-CI-u})
obeys
\[
\mathbb{P}\left\{ u_{\pi(l),k}^{\star}\in\mathsf{CI}_{u_{l,k}}^{1-\alpha}\right\} =1-\alpha+o(1),\qquad\forall1\leq l\leq r,\,1\leq k\leq d.
\]
In addition, for any $1\leq i,j,k\leq d$ obeying (\ref{eq:U-tilde-2norm-LB-rank-r})
and any $0<\alpha<1$, the confidence interval constructed in (\ref{eq:defn-CI-T})
obeys
\[
\mathbb{P}\left\{ T_{i,j,k}^{\star}\in\mathsf{CI}_{T_{i,j,k}}^{1-\alpha}\right\} =1-\alpha+o\left(1\right).
\]
\end{theorem}\begin{remark}Theorem \ref{thm:entry-CI-rank-r} subsumes
Theorem \ref{thm:entry-CI} as a special case. \end{remark}

\begin{theorem}[Entrywise lower bounds (general $(r,\mu)$]\label{thm:lower-bound-entrywise-rank-r}Consider
any unbiased estimator $\widehat{\bm{u}}_{l}$ for $\bm{u}_{l}^{\star}$
$(1\leq l\leq r)$ and any unbiased estimator $\widehat{\bm{T}}$
for $\bm{T}^{\star}$. Suppose that $\{E_{i,j,k}\}$ are i.i.d.~Gaussians
and that Assumptions \ref{assumption:random-sampling}-\ref{assumption:incoherence}
hold. Assume that $\kappa\asymp1$ and that
\[
p\geq c_{6}\frac{\mu^{2}r\log^{2}d}{d^{2}}\qquad r\leq c_{7}\sqrt{\frac{d}{\mu\log d}}
\]
hold for some sufficiently large (resp.~small) constant $c_{6}>0$
(resp.~$c_{7}>0$). Then the following holds with probability at
least $1-O(d^{-10})$:
\begin{align*}
\mathsf{Var}\big[\widehat{u}_{l,k}\big] & \geq\left(1-o\left(1\right)\right)\big(\bm{\Sigma}_{k}^{\star}\big)_{l,l},\qquad1\leq k\leq d;\\
\mathsf{Var}\big[\widehat{T}_{i,j,k}\big] & \geq\left(1-o\left(1\right)\right)v_{i,j,k}^{\star},\qquad\quad1\leq i,j,k\leq d.
\end{align*}
\end{theorem}\begin{remark}Theorem \ref{thm:lower-bound-entrywise-rank-r}
subsumes Theorem \ref{thm:lower-bound-entrywise} as a special case.
\end{remark}

\begin{theorem}[Optimality w.r.t.~$\ell_{2}$ estimation accuracy
(general $(r,\mu)$]\label{thm:optimality-L2-rank-r}Instate the
assumptions of Theorem~\ref{thm:U-loss-dist-nonGaussian-rank-r}.
With probability exceeding $1-o(1)$, the estimates returned by Algorithm~\ref{alg:gd}
obey\begin{subequations}
\begin{align*}
\big\|\bm{u}_{\pi(l)}-\bm{u}_{l}^{\star}\big\|_{2}^{2} & =\frac{\left(2+o\left(1\right)\right)\sigma_{\max}^{2}d}{p\,\big\|\bm{u}_{l}^{\star}\big\|_{2}^{4}},\qquad\forall1\leq l\leq r\\
\left\Vert \bm{T}-\bm{T}^{\star}\right\Vert _{\mathrm{F}}^{2} & =\frac{\left(6+o\left(1\right)\right)\sigma_{\max}^{2}dr}{p}
\end{align*}
\end{subequations}for some permutation $\pi(\cdot):[d]\mapsto[d]$.

\end{theorem}\begin{remark}Theorem \ref{thm:optimality-L2-rank-r}
subsumes Theorem \ref{thm:optimality-L2} as a special case. \end{remark}

\begin{theorem}[Lower bound w.r.t~$\ell_{2}$ estimation accuracy
(general $(r,\mu)$)]\label{thm:l2-error-lower-bound-rankr}Instate
the assumptions of Theorem~\ref{thm:lower-bound-entrywise-rank-r}.
The following holds with probability at least $1-O(d^{-10})$: any
unbiased estimator $\widehat{\bm{u}}_{l}$ (resp.~$\widehat{\bm{T}}$)
for $\bm{u}_{l}^{\star}$ (resp.~$\bm{T}^{\star}$) necessarily obeys
\begin{align*}
\mathbb{E}\left[\big\|\widehat{\bm{u}}_{l}-\bm{u}_{l}^{\star}\big\|_{2}^{2}\right] & \geq\frac{\left(2-o\left(1\right)\right)\sigma_{\min}^{2}d}{p\,\big\|\bm{u}_{l}^{\star}\big\|_{2}^{4}};\qquad\mathbb{E}\left[\big\|\widehat{\bm{T}}-\bm{T}^{\star}\big\|_{\mathrm{F}}^{2}\right]\geq\frac{\left(6-o\left(1\right)\right)\sigma_{\min}^{2}dr}{p}.
\end{align*}
\end{theorem}\begin{remark}Theorem \ref{thm:l2-error-lower-bound-rankr}
subsumes Theorem \ref{thm:l2-error-lower-bound} as a special case.
\end{remark}

The rest of this section is dedicated to establishing Theorems~\ref{thm:U-loss-dist-Gaussian-rank-r}-\ref{thm:entry-CI-rank-r}.
The proof of Theorems \ref{thm:lower-bound-entrywise-rank-r} and \ref{thm:l2-error-lower-bound-rankr}
(resp.~Theorem \ref{thm:optimality-L2-rank-r}) is deferred to Appendix
\ref{sec:Proof-of-lower-bounds} (resp.~Appendix \ref{sec:proof-l2-loss}).
Before continuing, we introduce several notation for simplicity of
presentation. First, we rescale the loss function as follows
\[
g(\bm{U}):=\frac{1}{6p}f(\bm{U})=\frac{1}{6p}\,\Big\|\mathcal{P}_{\Omega}\Big(\sum_{i=1}^{r}\bm{u}_{i}^{\otimes3}-\bm{T}\Big)\Big\|_{\mathrm{F}}^{2}
\]
throughout the rest of the paper. By defining $\widetilde{\bm{U}}:=[\bm{u}_{l}^{\otimes2}]_{1\leq l\leq r}\in\mathbb{R}^{d^{2}\times r}$
as before, we can express the gradient of $g(\bm{U})$ as follows
\begin{align}
\nabla g(\bm{U}) & =\Big[p^{-1}\mathcal{P}_{\Omega}\Big(\sum_{i=1}^{r}\bm{u}_{i}^{\otimes3}-\bm{T}^{\star}-\bm{E}\Big)\times_{1}\bm{u}_{l}\times_{2}\bm{u}_{l}\Big]_{1\leq l\leq d}\nonumber \\
 & =\mathsf{unfold}\Big(p^{-1}\mathcal{P}_{\Omega}\Big(\sum_{i=1}^{r}\bm{u}_{i}^{\otimes3}-\bm{T}^{\star}-\bm{E}\Big)\Big)\widetilde{\bm{U}},\label{eq:grad-f}
\end{align}
where we recall the tensor vector products $\times_{1}$ and $\times_{2}$
are both defined in Section~\ref{subsec:Notations}, and $\mathsf{unfold}\left(\cdot\right)$
denotes the mode-3 matricization of a three-order tensor. Here and
throughout, for any matrix $\bm{A}=\left[\bm{a}_{1},\dots,\bm{a}_{r}\right]\in\mathbb{R}^{d\times r}$,
we denote
\begin{equation}
\widetilde{\bm{A}}:=\left[\bm{a}_{1}\otimes\bm{a}_{1},\dots,\bm{a}_{r}\otimes\bm{a}_{r}\right]\in\mathbb{R}^{d^{2}\times r},\label{defn:widetilde-M}
\end{equation}
where for any $\bm{a},\bm{b}\in\mathbb{R}^{d}$, we let $\bm{a}\otimes\bm{b}:=\left[\begin{array}{c}
a_{1}\bm{b}\\
\vdots\\
a_{d}\bm{b}
\end{array}\right]\in\mathbb{R}^{d^{2}}.$ 

In addition, we define an event $\mathcal{E}$ on which several important
properties (\ref{eq:U-T-loss-UB})-(\ref{eq:U-tilde-property}) (which
we defer to Appendix \ref{sec:Preliminaries} to streamline presentation)
hold. In what follows, we shall primarily work with this event $\mathcal{E}$,
which happens with probability exceeding $1-o\left(1\right)$ as guaranteed
by Lemmas~\ref{lemma:U-loss-property}-\ref{lemma:U-tilde-property}
in Appendix \ref{sec:Preliminaries}. 

\subsection{Proof outline for the distributional theory}

We now outline the proof strategy for our distributional theory, namely,
Theorems \ref{thm:U-loss-dist-Gaussian-rank-r}-\ref{thm:T-loss-dist-rank-r}.

\subsubsection{Distributional theory for tensor factors}

Recall the definition $\widetilde{\bm{U}}:=[\bm{u}_{1}^{\otimes2},\cdots,\bm{u}_{r}^{\otimes2}]\in\mathbb{R}^{d^{2}\times r}$.
We start by making note of the following crucial decomposition of
$\bm{U}\bm{\Pi}$: 
\begin{align}
\bm{U}\bm{\Pi} & =\mathsf{unfold}\left(p^{-1}\mathcal{P}_{\Omega}\left(\bm{E}\right)\right)\widetilde{\bm{U}}\bm{\Pi}\big((\widetilde{\bm{U}}\bm{\Pi})^{\top}\widetilde{\bm{U}}\bm{\Pi}\big)^{-1}+\bm{U}^{\star}\widetilde{\bm{U}}^{\star\top}\widetilde{\bm{U}}\bm{\Pi}\big((\widetilde{\bm{U}}\bm{\Pi})^{\top}\widetilde{\bm{U}}\bm{\Pi}\big)^{-1}\nonumber \\
 & \quad+\mathsf{unfold}\left((\mathcal{I}-p^{-1}\mathcal{P}_{\Omega})\left(\bm{T}-\bm{T}^{\star}\right)\right)\widetilde{\bm{U}}\bm{\Pi}\big((\widetilde{\bm{U}}\bm{\Pi})^{\top}\widetilde{\bm{U}}\bm{\Pi}\big)^{-1}+\nabla g(\bm{U})\big(\widetilde{\bm{U}}^{\top}\widetilde{\bm{U}}\big)^{-1}\bm{\Pi},\label{eq:U-decomp}
\end{align}
where $\bm{\Pi}$ is defined in (\ref{eq:defn-permutation}), $\mathcal{I}$
stands for the identity operator, and $\nabla g(\bm{U})$ is given
in (\ref{eq:grad-f}). As a result, we arrive at the following key
decomposition
\begin{equation}
\bm{U}\bm{\Pi}-\bm{U}^{\star}=\underbrace{\mathsf{unfold}\big(p^{-1}\mathcal{P}_{\Omega}(\bm{E})\big)\widetilde{\bm{U}}^{\star}\big(\widetilde{\bm{U}}^{\star\top}\widetilde{\bm{U}}^{\star}\big)^{-1}}_{=:\,\bm{X}}+\sum_{1\leq i\leq4}\bm{W}_{i},\label{eq:U-loss-decomp}
\end{equation}
where the $\bm{W}_{i}$'s are given by \begin{subequations}\label{def:W1-W4}
\begin{align}
\bm{W}_{1} & :=\bm{U}^{\star}\big(\widetilde{\bm{U}}^{\star\top}\widetilde{\bm{U}}\bm{\Pi}\big((\widetilde{\bm{U}}\bm{\Pi})^{\top}\widetilde{\bm{U}}\bm{\Pi}\big)^{-1}-\bm{I}_{r}\big);\label{def:W1}\\
\bm{W}_{2} & :=\mathsf{unfold}\big(p^{-1}\mathcal{P}_{\Omega}(\bm{E})\big)\big(\widetilde{\bm{U}}\bm{\Pi}\big((\widetilde{\bm{U}}\bm{\Pi})^{\top}\widetilde{\bm{U}}\bm{\Pi}\big)^{-1}-\widetilde{\bm{U}}^{\star}(\widetilde{\bm{U}}^{\star\top}\widetilde{\bm{U}}^{\star})^{-1}\big);\label{def:W2}\\
\bm{W}_{3} & :=\mathsf{unfold}\big((\mathcal{I}-p^{-1}\mathcal{P}_{\Omega})\left(\bm{T}-\bm{T}^{\star}\right)\big)\widetilde{\bm{U}}\bm{\Pi}\big((\widetilde{\bm{U}}\bm{\Pi})^{\top}\widetilde{\bm{U}}\bm{\Pi}\big)^{-1};\label{def:W3}\\
\bm{W}_{4} & :=\nabla g(\bm{U})\big(\widetilde{\bm{U}}^{\top}\widetilde{\bm{U}}\big)^{-1}\bm{\Pi}.\label{def:W4}
\end{align}
\end{subequations}

In what follows, we shall demonstrate through a set of auxiliary lemmas
that $\bm{U}\bm{\Pi}-\bm{U}^{\star}$ is approximately characterized
by the term $\bm{X}$ defined in (\ref{eq:U-loss-decomp}). More specifically,
\begin{itemize}
\item Lemma~\ref{lemma:U-loss-dist-main-part-Gaussian} reveals that, under
Gaussian noise, each row of $\bm{X}$ is approximately a Gaussian
random vector.
\item Lemma~\ref{lemma:U-loss-dist-main-part-nonGaussian} extends the
above (approximate) normality result to the case with non-Gaussian
noise.
\item Lemmas~\ref{lemma:U-loss-dist-W1}-\ref{lemma:U-loss-dist-W4} deliver
upper bounds on the $\ell_{2,\infty}$ norms of the remaining quantities
$\bm{W}_{1}$, $\bm{W}_{2}$, $\bm{W}_{3}$ and $\bm{W}_{4}$, respectively
(in particular, they are provably negligible compared to the typical
size of each row of $\bm{X}$).
\end{itemize}
Theorems~\ref{thm:U-loss-dist-Gaussian-rank-r}-\ref{thm:U-loss-dist-nonGaussian-rank-r}
then follow immediately by combining Lemmas~\ref{lemma:U-loss-dist-main-part-Gaussian}-\ref{lemma:U-loss-dist-W4}.

\begin{lemma}\label{lemma:U-loss-dist-main-part-Gaussian}Instate
the assumptions of Theorem~\ref{thm:U-loss-dist-Gaussian-rank-r}.
Conditional on the event $\mathcal{E}$ where (\ref{eq:U-T-loss-UB})-(\ref{eq:U-tilde-property})
hold, with probability at least $1-O\left(d^{-10}\right)$ we can
decompose $\bm{X}=\bm{Z}+\bm{W}_{0}$ such that (i) for any $1\leq k\leq d$,
$\bm{Z}_{k,:}\sim\mathcal{N}(\bm{0},\bm{\Sigma}_{k}^{\star})$ with
covariance matrix $\bm{\Sigma}_{k}^{\star}$ defined in (\ref{eq:cov-matrix-m}),
and (ii)
\begin{align}
\|\bm{W}_{0}\|_{2,\infty} & \lesssim\frac{\sigma_{\max}}{\lambda_{\min}^{\star2/3}\sqrt{p}}\left\{ \frac{\mu r\log^{2}d}{d\sqrt{p}}+\sqrt{\frac{\mu r\log d}{d}}\right\} .\label{eq:W0-2inf-norm-UB}
\end{align}
\end{lemma}\begin{proof}See Appendix~\ref{subsec:U-loss-dist-main-part-Gaussian}.\end{proof}

\begin{lemma}\label{lemma:U-loss-dist-main-part-nonGaussian}Instate
the assumptions of Theorem~\ref{thm:U-loss-dist-nonGaussian-rank-r}.
Conditional on the event $\mathcal{E}$ where (\ref{eq:U-T-loss-UB})-(\ref{eq:U-tilde-property})
hold, with probability at least $1-O\left(d^{-10}\right)$, the decomposition
$\bm{X}=\bm{Z}+\bm{W}_{0}$ in Lemma~\ref{lemma:U-loss-dist-main-part-Gaussian}
and (\ref{eq:W0-2inf-norm-UB}) continues to hold, except that $\bm{Z}$
is not necessarily Gaussian but instead obeys
\[
\left|\mathbb{P}\left\{ \bm{Z}_{k,:}\in\mathcal{A}\right\} -\mathbb{P}\big\{\bm{g}_{k}\in\mathcal{A}\big\}\right|\lesssim\frac{\mu r^{3/2}}{\sqrt{d^{3/2}p}}
\]
for any convex set $A\subset\mathbb{R}^{d}$. Here, $\bm{g}_{k}\sim\mathcal{N}\left(\bm{0},\bm{\Sigma}_{k}^{\star}\right)$
with covariance matrix $\bm{\Sigma}_{k}^{\star}$ defined in (\ref{eq:cov-matrix-m}).\end{lemma}\begin{proof}See
Appendix~\ref{subsec:U-loss-dist-main-part-nonGaussian}.\end{proof}

\begin{lemma}\label{lemma:U-loss-dist-W1}Instate the assumptions
of Theorem~\ref{thm:U-loss-dist-nonGaussian-rank-r}. Conditional
on the event $\mathcal{E}$ where (\ref{eq:U-T-loss-UB})-(\ref{eq:U-tilde-property})
hold, the matrix $\bm{W}_{1}$ defined in (\ref{def:W1}) obeys
\begin{equation}
\left\Vert \bm{W}_{1}\right\Vert _{2,\infty}\lesssim\frac{\sigma_{\max}}{\lambda_{\min}^{\star2/3}\sqrt{p}}\Bigg\{\underbrace{\frac{\mu^{2}r^{2}\log^{7/2}d}{d^{3/2}p}+\frac{\mu^{2}r^{2}\log^{3}d}{d\sqrt{p}}+\sqrt{\frac{\mu^{2}r^{2}\log d}{d}}+\frac{\sigma_{\max}}{\lambda_{\min}^{\star}}\sqrt{\frac{\mu r^{3}d\log^{2}d}{p}}}_{=:\,\zeta}\Bigg\}\label{def:zeta}
\end{equation}
with probability at least $1-O(d^{-10})$. \end{lemma}\begin{proof}See
Appendix~\ref{subsec:U-loss-dist-W1}.\end{proof}

\begin{lemma}\label{lemma:U-loss-dist-W2}Instate the assumptions
of Theorem~\ref{thm:U-loss-dist-nonGaussian-rank-r}. Conditional
on the event $\mathcal{E}$ where (\ref{eq:U-T-loss-UB})-(\ref{eq:U-tilde-property})
hold, the matrix $\bm{W}_{2}$ defined in (\ref{def:W2}) obeys

\[
\left\Vert \bm{W}_{2}\right\Vert _{2,\infty}\lesssim\frac{\sigma_{\max}}{\lambda_{\min}^{\star2/3}\sqrt{p}}\frac{\sigma_{\max}}{\lambda_{\min}^{\star}}\sqrt{\frac{\mu r^{2}d\log d}{p}}
\]
with probability at least $1-O(d^{-10})$.\end{lemma}\begin{proof}See
Appendix~\ref{subsec:U-loss-dist-W2}.\end{proof}

\begin{lemma}\label{lemma:U-loss-dist-W3}Instate the assumptions
of Theorem~\ref{thm:U-loss-dist-nonGaussian-rank-r}. Conditional
on the event $\mathcal{E}$ where (\ref{eq:U-T-loss-UB})-(\ref{eq:U-tilde-property})
hold, the matrix $\bm{W}_{3}$ defined in (\ref{def:W3}) obeys
\[
\left\Vert \bm{W}_{3}\right\Vert _{2,\infty}\lesssim\frac{\sigma_{\max}}{\lambda_{\min}^{\star2/3}\sqrt{p}}\sqrt{\frac{\mu^{3}r^{2}\log^{2}d}{d^{2}p}}
\]
with probability at least $1-O(d^{-10})$.\end{lemma}\begin{proof}See
Appendix~\ref{subsec:U-loss-dist-W3}.\end{proof}

\begin{lemma}\label{lemma:U-loss-dist-W4}Instate the assumptions
of Theorem~\ref{thm:U-loss-dist-nonGaussian-rank-r}. Conditional
on the event $\mathcal{E}$ where (\ref{eq:U-T-loss-UB})-(\ref{eq:U-tilde-property})
hold, the matrix $\bm{W}_{4}$ defined in (\ref{def:W4}) obeys
\[
\left\Vert \bm{W}_{4}\right\Vert _{2,\infty}\lesssim\frac{\sigma_{\max}}{\lambda_{\min}^{\star2/3}\sqrt{p}}\frac{1}{\sqrt{d}}
\]
with probability at least $1-O(d^{-10})$.\end{lemma}\begin{proof}See
Appendix~\ref{subsec:U-loss-dist-W4}.\end{proof}

Before we move on to the distributions of the tensor entries, we make
note of the following observation that will play a useful role later.
Define
\begin{align}
\bm{W} & :=\bm{W}_{0}+\bm{W}_{1}+\bm{W}_{2}+\bm{W}_{3}+\bm{W}_{4}.\label{def:W}
\end{align}
Taking Lemmas \ref{lemma:U-loss-dist-main-part-Gaussian}-\ref{lemma:U-loss-dist-W4}
collectively, we obtain
\begin{equation}
\left\Vert \bm{W}\right\Vert _{2,\infty}\lesssim\zeta\frac{\sigma_{\max}}{\lambda_{\min}^{\star2/3}\sqrt{p}}=o\Bigg(\frac{\sigma_{\max}}{\lambda_{\min}^{\star2/3}\sqrt{p}}\Bigg),\label{eq:W-2inf-norm-UB}
\end{equation}
where $\zeta$ is defined in (\ref{def:zeta}). The last relation
holds true due to our assumptions (\ref{eq:requirement-p-sigma-rank-r})
on the sample size, the noise level, and the rank.

\subsubsection{Distributional theory for tensor entries}

As it turns out, the theoretical guarantees for the tensor factors
enable us to characterize the distribution of tensor entries. Towards
this, let us first define
\[
\bm{\Delta}:=\bm{U}\bm{\Pi}-\bm{U}^{\star},\qquad\widetilde{\bm{\Delta}}:=\big[\bm{\Delta}_{l}^{\otimes2}\big]_{1\leq l\leq r}\in\mathbb{R}^{d^{2}\times r},
\]
and recall the decomposition in (\ref{eq:U-loss-decomp}), (\ref{def:W1-W4}),
Lemma~\ref{lemma:U-loss-dist-main-part-Gaussian}, and (\ref{def:W}),
i.e.~$\bm{\Delta}=\bm{Z}+\bm{W}$. With these in mind, we can expand
\begin{align}
T_{i,j,k}-T_{i,j,k}^{\star} & =\big\langle\bm{U}_{i,:},\widetilde{\bm{U}}_{(j,k),:}\big\rangle-\big\langle\bm{U}_{i,:}^{\star},\widetilde{\bm{U}}_{(j,k),:}^{\star}\big\rangle\nonumber \\
 & =\big\langle\bm{\Delta}_{i,:},\widetilde{\bm{U}}_{(j,k),:}^{\star}\big\rangle+\big\langle\bm{\Delta}_{j,:},\widetilde{\bm{U}}_{(i,k),:}^{\star}\big\rangle+\big\langle\bm{\Delta}_{k,:},\widetilde{\bm{U}}_{(i,j),:}^{\star}\big\rangle\nonumber \\
 & \quad+\big\langle\bm{U}_{i,:}^{\star},\widetilde{\bm{\Delta}}_{(j,k),:}\big\rangle+\big\langle\bm{U}_{j,:}^{\star},\widetilde{\bm{\Delta}}_{(i,k),:}\big\rangle+\big\langle\bm{U}_{k,:}^{\star},\widetilde{\bm{\Delta}}_{(i,k),:}\big\rangle+\big\langle\bm{\Delta}_{i,:},\widetilde{\bm{\Delta}}_{(j,k),:}\big\rangle\nonumber \\
 & =\underbrace{\big\langle\bm{Z}_{i,:},\widetilde{\bm{U}}_{(j,k),:}^{\star}\big\rangle+\big\langle\bm{Z}_{j,:},\widetilde{\bm{U}}_{(i,k),:}^{\star}\big\rangle+\big\langle\bm{Z}_{k,:},\widetilde{\bm{U}}_{(i,j),:}^{\star}\big\rangle}_{=:\,Y_{i,j,k}}+R_{i,j,k}\label{eq:T-entry-loss}
\end{align}
for any $1\leq i,j,k\leq d$, with the residual term $R_{i,j,k}$
given by
\begin{align}
R_{i,j,k} & :=\big\langle\bm{W}_{i,:},\widetilde{\bm{U}}_{(j,k),:}^{\star}\big\rangle+\big\langle\bm{W}_{j,:},\widetilde{\bm{U}}_{(i,k),:}^{\star}\big\rangle+\big\langle\bm{W}_{k,:},\widetilde{\bm{U}}_{(i,j),:}^{\star}\big\rangle\nonumber \\
 & \,\quad+\big\langle\bm{U}_{i,:}^{\star},\widetilde{\bm{\Delta}}_{(j,k),:}\big\rangle+\big\langle\bm{U}_{j,:}^{\star},\widetilde{\bm{\Delta}}_{(i,k),:}\big\rangle+\big\langle\bm{U}_{k,:}^{\star},\widetilde{\bm{\Delta}}_{(i,j),:}\big\rangle+\big\langle\bm{\Delta}_{i,:},\widetilde{\bm{\Delta}}_{(j,k),:}\big\rangle.\label{eq:T-entry-loss-res}
\end{align}

Armed with the distributional characterization for $\bm{Z}$ (cf.~Lemma~\ref{lemma:U-loss-dist-main-part-nonGaussian}),
we can show that $Y_{i,j,k}$ is approximately Gaussian, as formalized
by the lemma below.

\begin{lemma}\label{lemma:T-loss-dist-main-part}Instate the assumptions
of Theorem~\ref{thm:U-loss-dist-nonGaussian-rank-r}. On the event
$\mathcal{E}$ where (\ref{eq:U-T-loss-UB})-(\ref{eq:U-tilde-property})
hold, one can decompose $Y_{i,j,k}=G_{i,j,k}+H_{i,j,k}$ for each
$1\leq i,j,k\leq d$ such that
\begin{equation}
\sup_{\tau\in\mathbb{R}}\left|\mathbb{P}\Big\{ G_{i,j,k}\leq\tau\sqrt{v_{i,j,k}^{\star}}\Big\}-\Phi(\tau)\right|\lesssim\frac{\mu\sqrt{r}}{d\sqrt{p}},\label{eq:T-entry-G-dist}
\end{equation}
where $\Phi(\cdot)$ is the CDF of a standard Gaussian random variable;
further, with probability at least $1-O\left(d^{-10}\right)$ one
has
\begin{equation}
\frac{\left|H_{i,j,k}\right|}{\sqrt{v_{i,j,k}^{\star}}}\lesssim\frac{\mu\sqrt{r}\log^{2}d}{d\sqrt{p}}+\sqrt{\frac{\mu r\log d}{d}}+\frac{\mu r^{3/2}\sqrt{\log d}}{d}.\label{eq:T-entry-H-UB}
\end{equation}
\end{lemma}\begin{proof}The key step boils down to proving that
$\bm{Z}_{i,:}$, $\bm{Z}_{j,:}$ and $\bm{Z}_{k,:}$ are nearly statistically
independent (as alluded to previously). See Appendix~\ref{subsec:T-loss-dist-main-part}.\end{proof}

In addition, given the $\ell_{2,\infty}$ bounds of the residual term
$\bm{W}$ (cf.~(\ref{eq:W-2inf-norm-UB})) and the estimation error
$\bm{\Delta}$ (cf.~(\ref{eq:U-loss-2inf})), we can demonstrate
that $R_{i,j,k}$ (cf.~(\ref{eq:T-entry-loss-res})) is negligible
in magnitude, as stated in the following lemma.

\begin{lemma}\label{lemma:T-loss-dist-neg-part}Instate the assumptions
of Theorem~\ref{thm:U-loss-dist-nonGaussian-rank-r}. Conditional
on the event $\mathcal{E}$ where (\ref{eq:U-T-loss-UB})-(\ref{eq:U-tilde-property})
hold, one has
\begin{align}
\frac{\left|R_{i,j,k}\right|}{\sqrt{v_{i,j,k}^{\star}}} & \lesssim\zeta+\frac{\sigma_{\max}}{\lambda_{\min}^{\star1/3}}\frac{\mu^{3/2}r\log d}{\sqrt{dp}}\Big(\big\|\widetilde{\bm{U}}_{(j,k),:}^{\star}\big\|_{2}+\big\|\widetilde{\bm{U}}_{(i,k),:}^{\star}\big\|_{2}+\big\|\widetilde{\bm{U}}_{(i,j),:}^{\star}\big\|_{2}\Big)^{-1}\label{eq:T-entry-R-UB}
\end{align}
for any $1\leq i,j,k\leq d$ with probability at least $1-O\left(d^{-10}\right)$,
where $\zeta$ is defined in (\ref{def:zeta}).\end{lemma}\begin{proof}See
Appendix~\ref{subsec:T-loss-dist-neg-part}.\end{proof}

\paragraph{Proof of Theorem~\ref{thm:T-loss-dist-rank-r}.}With
Lemmas \ref{lemma:T-loss-dist-main-part} and \ref{lemma:T-loss-dist-neg-part}
in place, one can readily prove Theorem~\ref{thm:T-loss-dist-rank-r}.
Applying the union bound yields that: for any $\tau\in\mathbb{R}$
and any $\varepsilon_{1},\varepsilon_{2}>0$, 
\begin{align*}
\mathbb{P}\Big\{ T_{i,j,k}-T_{i,j,k}^{\star}\leq\tau\sqrt{v_{i,j,k}^{\star}}\Big\} & \leq\mathbb{P}\Big\{ G_{i,j,k}\leq(\tau+\varepsilon_{1}+\varepsilon_{2})\sqrt{v_{i,j,k}^{\star}}\Big\}\\
 & \quad\quad+\mathbb{P}\Big\{\big|H_{i,j,k}\big|>\varepsilon_{1}\sqrt{v_{i,j,k}^{\star}}\Big\}+\mathbb{P}\Big\{\big|R_{i,j,k}\big|>\varepsilon_{2}\sqrt{v_{i,j,k}^{\star}}\Big\}\\
 & \overset{}{\leq}\Phi(\tau+\varepsilon_{1}+\varepsilon_{2})+o\left(1\right)+\mathbb{P}\Big\{\big|H_{i,j,k}\big|>\varepsilon_{1}\sqrt{v_{i,j,k}^{\star}}\Big\}+\mathbb{P}\Big\{|R_{i,j,k}|>\varepsilon_{2}\sqrt{v_{i,j,k}^{\star}}\Big\},
\end{align*}
where the last line results from (\ref{eq:T-entry-G-dist}) and the
sample size condition that $p\gg\mu^{2}rd^{-3/2}$. By setting
\begin{align*}
\varepsilon_{1} & \asymp\frac{\mu\sqrt{r}\log^{2}d}{d\sqrt{p}}+\sqrt{\frac{\mu r\log d}{d}}+\frac{\mu r^{3/2}\sqrt{\log d}}{d},\\
\varepsilon_{2} & \asymp\zeta+\frac{\sigma_{\max}}{\lambda_{\min}^{\star1/3}}\frac{\mu^{3/2}r\log d}{\sqrt{dp}}\Big(\big\|\widetilde{\bm{U}}_{(n,l),:}^{\star}\big\|_{2}+\big\|\widetilde{\bm{U}}_{(m,l),:}^{\star}\big\|_{2}+\big\|\widetilde{\bm{U}}_{(m,n),:}^{\star}\big\|_{2}\Big)^{-1},
\end{align*}
one sees from (\ref{eq:T-entry-H-UB}) and (\ref{eq:T-entry-R-UB})
that 
\[
\mathbb{P}\{|H_{i,j,k}|>\varepsilon_{1}\sqrt{v_{i,j,k}^{\star}}\}\lesssim d^{-10}\quad\text{and}\quad\mathbb{P}\{|R_{i,j,k}|>\varepsilon_{2}\sqrt{v_{i,j,k}^{\star}}\}\lesssim d^{-10}.
\]
In particular, in view of the assumptions (\ref{eq:requirement-p-sigma-rank-r})
and (\ref{eq:U-tilde-2norm-LB-rank-r}), one has $\max\{\varepsilon_{1},\varepsilon_{2}\}=o\left(1\right)$.
Consequently, we can obtain
\begin{align*}
\mathbb{P}\Big\{ T_{i,j,k}-T_{i,j,k}^{\star}\leq\tau\sqrt{v_{i,j,k}^{\star}}\Big\}-\Phi(\tau) & \leq\Phi(\tau+\varepsilon_{1}+\varepsilon_{2})-\Phi(\tau)+o\left(1\right)\\
 & \leq\varepsilon_{1}+\varepsilon_{2}+o\left(1\right)=o\left(1\right)
\end{align*}
for any $\tau\in\mathbb{R}$, where the last step arises from the
property of the CDF of a standard Gaussian. The lower bound on $\mathbb{P}\big\{ T_{i,j,k}-T_{i,j,k}^{\star}\leq\tau\sqrt{v_{i,j,k}^{\star}}\big\}-\Phi(\tau)$
can be obtained analogously. These taken together lead to the advertised
claim
\[
\sup_{\tau\in\mathbb{R}}\left|\mathbb{P}\Big\{ T_{i,j,k}-T_{i,j,k}^{\star}\leq\tau\sqrt{v_{i,j,k}^{\star}}\Big\}-\Phi(\tau)\right|=o\left(1\right).
\]

\subsection{Proof outline for the validity of confidence intervals}

With the above distributional guarantees in place, the validity of
our confidence intervals can be established as long as the proposed
variance$\,$/$\,$covariance estimates are sufficiently accurate.
Before proceeding, we make note of the following crucial observation:
\begin{align}
\max_{(i,j,k)\in\Omega}\big|\widehat{E}_{i,j,k}-E_{i,j,k}\big| & =\max_{(i,j,k)\in\Omega}\left|T_{i,j,k}^{\mathsf{obs}}-T_{i,j,k}-E_{i,j,k}\right|\leq\left\Vert \bm{T}-\bm{T}^{\star}\right\Vert _{\infty}\nonumber \\
 & \lesssim\frac{\sigma_{\max}}{\lambda_{\min}^{\star}}\sqrt{\frac{\mu^{3}r^{2}\log d}{d^{2}p}}\,\lambda_{\max}^{\star},\label{eq:noise-est-inf-UB}
\end{align}
where $\widehat{E}_{i,j,k}$ is defined in (\ref{def:noise-est}).
Here, we have used the relation (\ref{eq:T-loss-inf}) provided in
Appendix~\ref{sec:Preliminaries}. As we shall see momentarily, this
simple fact plays a crucial role in ensuring that our procedure returns
faithful variance estimates.

\subsubsection{Confidence intervals for tensor factors}

We start with the tensor factors. For each $1\leq l\leq r$ and $1\leq k\leq d$,
we can decompose
\begin{equation}
\frac{u_{l,k}-u_{l,k}^{\star}}{\sqrt{(\bm{\Sigma}_{k})_{l,l}}}=\frac{u_{l,k}-u_{l,k}^{\star}}{\sqrt{(\bm{\Sigma}_{k}^{\star})_{l,l}}}+\underbrace{\frac{u_{l,k}-u_{l,k}^{\star}}{\sqrt{(\bm{\Sigma}_{k})_{l,l}}}-\frac{u_{l,k}-u_{l,k}^{\star}}{\sqrt{(\bm{\Sigma}_{k}^{\star})_{l,l}}}}_{=:\,J_{l,k}}.\label{eq:u-entry-loss-decomp-CI}
\end{equation}
As it turns out, the approximation error term $J_{l,k}$ is quite
small, as formalized in Lemma \ref{lemma:u-entry-var-est-loss-neg}
below. The proof is postponed to Appendix~\ref{subsec:u-entry-var-est-loss-neg}.

\begin{lemma}\label{lemma:u-entry-var-est-loss-neg}Instate the assumptions
of Theorem~\ref{thm:entry-CI-rank-r}. Conditional on the event $\mathcal{E}$
where (\ref{eq:U-T-loss-UB})-(\ref{eq:U-tilde-property}) hold, one
has
\begin{equation}
\left|J_{l,k}\right|\lesssim\sqrt{\frac{\mu^{4}r^{3}\log^{3}d}{d^{2}p}}+\frac{\sigma_{\max}}{\lambda_{\min}^{\star}}\sqrt{\frac{\mu^{3}r^{2}d\log^{2}d}{p}},\qquad\forall1\leq l\leq r,1\leq k\leq d\label{claim:u-entry-loss-res-small}
\end{equation}
with probability at least $1-O\left(d^{-10}\right)$. \end{lemma}

\paragraph{Proof of Theorem~\ref{thm:entry-CI-rank-r} (the part
w.r.t.~$u_{l,k}^{\star}$).} Fix arbitrary $1\leq l\leq r$ and
$1\leq k\leq d$. By virtue of Theorem~\ref{thm:U-loss-dist-nonGaussian-rank-r}
and the continuous mapping theorem, we know that
\begin{equation}
\sup_{\tau\in\mathbb{R}}\Big|\mathbb{P}\Big\{ u_{l,k}-u_{l,k}^{\star}\leq\tau\sqrt{(\bm{\Sigma}_{k}^{\star})_{l,l}}\Big\}-\Phi(\tau)\Big|=o\left(1\right).\label{eq:u-entry-loss-dist}
\end{equation}
Given the decomposition in (\ref{eq:u-entry-loss-decomp-CI}), one
can use the union bound to find that for any $\tau\in\mathbb{R}$
and any $\varepsilon>0$,
\begin{align*}
\mathbb{P}\Big\{ u_{l,k}-u_{l,k}^{\star}\leq\tau\sqrt{(\bm{\Sigma}_{k})_{l,l}}\Big\}-\Phi(\tau) & \leq\mathbb{P}\Big\{ u_{l,k}-u_{l,k}^{\star}\leq(\tau+\varepsilon)\sqrt{(\bm{\Sigma}_{k}^{\star})_{l,l}}\Big\}+\mathbb{P}\left\{ \left|J_{l,k}\right|>\varepsilon\right\} -\Phi(\tau)\\
 & \overset{(\mathrm{i})}{\leq}\Phi(\tau+\varepsilon)-\Phi(\tau)+o\left(1\right)+\mathbb{P}\left\{ \left|J_{l,k}\right|>\varepsilon\right\} \\
 & \overset{(\mathrm{ii})}{\leq}\varepsilon+o\left(1\right)+\mathbb{P}\left\{ \left|J_{l,k}\right|>\varepsilon\right\} ,
\end{align*}
where (i) follows from (\ref{eq:u-entry-loss-dist}), and (ii) arises
from the property of the CDF of $\mathcal{N}(0,1)$. Set
\[
\varepsilon\asymp\sqrt{\frac{\mu^{4}r^{3}\log^{3}d}{d^{2}p}}+\frac{\sigma_{\max}}{\lambda_{\min}^{\star}}\sqrt{\frac{\mu^{3}r^{2}d\log^{2}d}{p}}=o\left(1\right),
\]
where the last identity is valid as long as $p\gg\mu^{4}r^{3}d^{-2}\log^{4}d$
and $\sigma_{\max}/\lambda_{\min}^{\star}\ll\sqrt{p/(\mu^{3}r^{2}d\log^{3}d)}$.
By Lemma~\ref{lemma:u-entry-var-est-loss-neg}, we have $\mathbb{P}\left\{ \left|J_{l,k}\right|>\varepsilon\right\} \lesssim d^{-10}$.
Applying a similar argument for the lower bound, one arrives at
\begin{align*}
\sup_{\tau\in\mathbb{R}}\Big|\mathbb{P}\Big\{ u_{l,k}-u_{l,k}^{\star}\leq\tau\sqrt{(\bm{\Sigma}_{k})_{l,l}}\Big\}-\Phi(\tau)\Big| & \leq\varepsilon+o\left(1\right)+\mathbb{P}\left\{ \left|J_{l,k}\right|>\varepsilon\right\} =o\left(1\right)
\end{align*}
as claimed.

\subsubsection{Confidence intervals for tensor entries}

Next, we turn to the constructed confidence intervals for tensor entries.
As before, let us decompose
\begin{align}
\frac{T_{i,j,k}-T_{i,j,k}^{\star}}{\sqrt{v_{i,j,k}}} & =\frac{T_{i,j,k}-T_{i,j,k}^{\star}}{\sqrt{v_{i,j,k}^{\star}}}+\underbrace{\frac{T_{i,j,k}-T_{i,j,k}^{\star}}{\sqrt{v_{i,j,k}}}-\frac{T_{i,j,k}-T_{i,j,k}^{\star}}{\sqrt{v_{i,j,k}^{\star}}}}_{=:\,K_{i,j,k}}\label{eq:T-entry-loss-decomp-CI}
\end{align}
for each $1\leq i,j,k\leq d$. The following lemma reveals that the
residual term $K_{i,j,k}$ is considerably small; the proof is deferred
to Appendix~\ref{lemma:T-entry-var-est-loss-neg}.

\begin{lemma}\label{lemma:T-entry-var-est-loss-neg}Instate the assumptions
and notation of Theorem~\ref{thm:entry-CI-rank-r}. Conditional on
the event $\mathcal{E}$ where (\ref{eq:U-T-loss-UB})-(\ref{eq:U-tilde-property})
hold, one has
\[
\left|K_{i,j,k}\right|\lesssim\sqrt{\frac{\mu^{4}r^{3}\log^{3}d}{d^{2}p}}+\Big(\big\|\widetilde{\bm{U}}_{(i,j),:}^{\star}\big\|_{2}+\big\|\widetilde{\bm{U}}_{(i,k),:}^{\star}\big\|_{2}+\big\|\widetilde{\bm{U}}_{(j,k),:}^{\star}\big\|_{2}\Big)^{-1}\frac{\sigma_{\max}}{\lambda_{\min}^{\star1/3}}\sqrt{\frac{\mu^{5}r^{3}\log^{2}d}{dp}},\qquad1\leq i,j,k\leq d
\]
with probability at least $1-O\left(d^{-10}\right)$. \end{lemma}

\paragraph{Proof of Theorem~\ref{thm:entry-CI-rank-r} (the part
w.r.t.~$T_{i,j,k}^{\star}$).}Fix arbitrary $1\leq i\leq j\leq k\leq d$.
Recalling the decomposition in (\ref{eq:T-entry-loss-decomp-CI}),
we can apply the union bound to show that: for any $\tau\in\mathbb{R}$
and any $\varepsilon>0$,
\begin{align*}
\mathbb{P}\big\{ T_{i,j,k}-T_{i,j,k}^{\star}\leq\tau\sqrt{v_{i,j,k}}\big\}-\Phi(\tau) & \leq\mathbb{P}\Big\{ T_{i,j,k}-T_{i,j,k}^{\star}\leq(\tau+\varepsilon)\sqrt{v_{i,j,k}^{\star}}\Big\}+\mathbb{P}\big\{|K_{i,j,k}|>\varepsilon\big\}-\Phi(\tau)\\
 & \overset{(\mathrm{i})}{\leq}\Phi(\tau+\varepsilon)-\Phi(\tau)+o\left(1\right)+\mathbb{P}\big\{|K_{i,j,k}|>\varepsilon\big\}\\
 & \overset{(\mathrm{ii})}{\leq}\varepsilon+o\left(1\right)+\mathbb{P}\big\{|K_{i,j,k}|>\varepsilon\big\},
\end{align*}
where (i) follows from Theorem~\ref{thm:T-loss-dist-rank-r}, and
(ii) arises from the property of the CDF of a standard Gaussian. Set
\[
\varepsilon\asymp\sqrt{\frac{\mu^{4}r^{3}\log^{3}d}{d^{2}p}}+\Big(\big\|\widetilde{\bm{U}}_{(i,j),:}^{\star}\big\|_{2}+\big\|\widetilde{\bm{U}}_{(i,k),:}^{\star}\big\|_{2}+\big\|\widetilde{\bm{U}}_{(j,k),:}^{\star}\big\|_{2}\Big)^{-1}\frac{\sigma_{\max}}{\lambda_{\min}^{\star1/3}}\sqrt{\frac{\mu^{5}r^{3}\log^{2}d}{dp}}=o\left(1\right),
\]
where the last equality holds due to our conditions $p\gg\mu^{4}r^{3}d^{-2}\log^{4}d$
and (\ref{eq:U-tilde-2norm-LB-rank-r}). Then Lemma~\ref{lemma:T-entry-var-est-loss-neg}
guarantees that $\mathbb{P}\big\{|K_{i,j,k}|>\varepsilon\big\}\lesssim d^{-10}$,
allowing us to reach
\begin{align*}
\mathbb{P}\left\{ T_{i,j,k}-T_{i,j,k}^{\star}\leq\tau\sqrt{v_{i,j,k}}\right\} -\Phi(\tau) & \leq\varepsilon+o\left(1\right)+\mathbb{P}\big\{|K_{i,j,k}|>\varepsilon\big\}=o\left(1\right).
\end{align*}
The lower bound can be obtained analogously. The proof is thus complete.

%% file: discussion.tex
\section{Discussions}

\label{sec:Discussion}

This paper has explored the problem of uncertainty quantification
for nonconvex tensor completion. The main contributions lie in establishing
(nearly) precise distributional guarantees for the nonconvex estimates
down to an entrywise level. Our distributional representation enables
data-driven construction of confidence intervals for both the unknown
tensor and its underlying tensor factors. Our inferential procedure
and the accompanying theory are model-agnostic, which do not require
prior knowledge about the noise distributions and are automatically
adaptive to location-varying noise levels. Our results uncover the
unreasonable effectiveness of nonconvex optimization, which is statistically
optimal for both estimation and confidence interval construction.

The findings of the current paper further suggest numerous possible
extensions that are worth pursuing. To begin with, our current results
are only optimal when both the rank $r$ and the condition number
$\kappa$are constants independent of the ambient dimension $d$.
Can we further refine the analysis to enable optimal inference for
more general settings? In addition, our theory falls short of providing
valid confidence intervals for tensor entries with very small ``strength''.
This calls for further investigation in order to complete the picture.
It would also be interesting to go beyond uniform random sampling
by considering the type of sampling patterns with a heterogeneous
missingness mechanism.

%% file: init-algorithms.tex
\section{More details about Algorithm \ref{alg:gd}}

\subsection{The initialization scheme\label{sec:Initialization-scheme}}

For self-completeness, we record in this section the detailed initialization
procedure employed in the two-stage nonconvex algorithm proposed in
\cite{cai2019nonconvex} (namely, Algorithm~\ref{alg:gd}). This
is summarized in Algorithm~\ref{alg:init}, with auxiliary procedures
detailed in Algorithm~\ref{alg:localization}. As a high-level interpretation,
Algorithm~\ref{alg:init} estimates the subspace spanned by the tensor
factor $\{\bm{u}_{l}^{\star}\}_{1\leq l\leq r}$ via a spectral method (similar to PCA-type methods \cite{montanari2018spectral,zhang2018heteroskedastic,cai2019subspace}),
whereas Algorithm~\ref{alg:localization} attempts to retrieve estimates
for individual tensor factors from this subspace estimate $\bm{U}_{\mathsf{space}}$.
Here and throughout, we denote $\bm{T}^{\mathsf{obs}}:=[T_{i,j,k}^{\mathsf{obs}}]_{1\leq i,j,k\leq d}$,
where we set $T_{i,j,k}^{\mathsf{obs}}=0$ for any $(i,j,k)\notin\Omega$.

\begin{algorithm}[H]
\caption{Spectral initialization for nonconvex tensor completion}
\label{alg:init} \begin{algorithmic}[1] \State Let $\bm{U}_{\mathsf{space}}\bm{\Lambda}\bm{U}_{\mathsf{space}}^{\top}$
be the rank-$r$ eigen-decomposition of 
\begin{align}
\bm{B}:=\mathcal{P}_{\mathsf{off}\text{-}\mathsf{diag}}(\bm{A}\bm{A}^{\top}),\label{B_alg}
\end{align}
where $\bm{A}=\mathsf{unfold}\big(p^{-1}\bm{T}^{\mathsf{obs}}\big)$
is the mode-1 matricization of $p^{-1}\bm{T}^{\mathsf{obs}}$, and
$\mathcal{P}_{\mathsf{off}\text{-}\mathsf{diag}}(\bm{Z})$ extracts
out the off-diagonal entries of $\bm{Z}$. \State \textbf{Output:}
an initial estimate $\bm{U}^{0}\in\mathbb{R}^{d\times r}$ on the
basis of $\bm{U}_{\mathsf{space}}\in\mathbb{R}^{d\times r}$ using
Algorithm~\ref{alg:localization}. \end{algorithmic}
\end{algorithm}

\begin{algorithm}[H]
\caption{Retrieval of low-rank tensor factors from a given subspace estimate.}
\label{alg:localization} \begin{algorithmic}[1] \State \textbf{Input:}
number of restarts $L$, pruning threshold $\epsilon_{\mathsf{th}}$,
subspace estimate $\bm{U}_{\mathsf{space}}\in\mathbb{R}^{d\times r}$
given by Algorithm~\ref{alg:init}. \For{$\tau=1,\dots,L$} \State
Generate an independent Gaussian vector $\bm{g}^{\tau}\sim\mathcal{N}(0,\bm{I}_{d})$.
\State $\big(\bm{\nu}^{\tau},\lambda_{\tau},\mathsf{spec}\text{-}\mathsf{gap}_{\tau}\big)\gets\Call{Retrieve-one-tensor-factor}{\bm{T}^{\mathsf{obs}},p,\bm{U}_{\mathsf{space}},\bm{g}^{\tau}}$.
\EndFor \State Generate tensor factor estimates $\big\{(\bm{w}^{1},\lambda_{1}),\dots,(\bm{w}^{r},\lambda_{r})\big\}\gets\Call{Prune}{\big\{\big(\bm{\nu}^{\tau},\lambda_{\tau},\mathsf{spec}\text{-}\mathsf{gap}_{\tau}\big)\big\}_{\tau=1}^{L},\epsilon_{\mathsf{th}}}$.
\State \textbf{Output:} initial estimate $\bm{U}^{0}=\big[\lambda_{1}^{1/3}\bm{w}^{1},\dots,\lambda_{r}^{1/3}\bm{w}^{r}\big].$
\end{algorithmic}
\end{algorithm}

\begin{algorithm}[htb]
\label{alg:one-factor} \begin{algorithmic}[1] \Function{Retrieve-one-tensor-factor}{$\bm{T},p,\bm{U}_{\mathsf{space}},\bm{g}$}
\State Compute \begin{subequations} 
\begin{align}
\bm{\theta} & =\bm{U}_{\mathsf{space}}\bm{U}_{\mathsf{space}}^{\top}\bm{g}=:\mathcal{P}_{\bm{U}_{\mathsf{space}}}(\bm{g}),\\
\bm{M} & =p^{-1}\bm{T}^{\mathsf{obs}}\times_{3}\bm{\theta},\label{eq:defn-Mtau}
\end{align}
\end{subequations} where $\times_{3}$ is defined in Section~\ref{subsec:Notations}.
\State Let $\bm{\nu}$ be the leading singular vector of $\bm{M}$
obeying $\langle\bm{T}^{\mathsf{obs}},\bm{\nu}^{\otimes3}\rangle\geq0$,
and set $\lambda=\langle p^{-1}\bm{T}^{\mathsf{obs}},\bm{\nu}^{\otimes3}\rangle$.
\State \Return $\big(\bm{\nu},\lambda,\sigma_{1}(\bm{M})-\sigma_{2}(\bm{M})\big)$.
\EndFunction \end{algorithmic}
\end{algorithm}

\begin{algorithm}[htb]
\label{alg:prune} \begin{algorithmic}[1] \Function{Prune}{$\big\{\big(\bm{\nu}^{\tau},\lambda_{\tau},\mathsf{spec}\text{-}\mathsf{gap}_{\tau}\big)\big\}_{\tau=1}^{L},\epsilon_{\mathsf{th}}$}
\State Set $\Theta=\big\{\big(\bm{\nu}^{\tau},\lambda_{\tau},\mathsf{spec}\text{-}\mathsf{gap}_{\tau}\big)\big\}_{\tau=1}^{L}.$
\For{$i=1,\dots,r$} \State Choose $(\bm{\nu}^{\tau},\lambda_{\tau},\mathsf{spec}\text{-}\mathsf{gap}_{\tau})$
from $\Theta$ with the largest $\mathsf{spec}\text{-}\mathsf{gap}_{\tau}$;
set $\bm{w}^{i}=\bm{\nu}^{\tau}$ and $\lambda_{i}=\lambda_{\tau}$.
\State Update $\Theta\gets\Theta\setminus\left\{ \big(\bm{\nu}^{\tau},\lambda_{\tau},\mathsf{spec}\text{-}\mathsf{gap}_{\tau}\big)\in\Theta:|\langle\bm{\nu}^{\tau},\bm{w}^{i}\rangle|>1-\epsilon_{\mathsf{th}}\right\} $.
\EndFor \State \Return $\big\{(\bm{w}^{1},\lambda_{1}),\dots,(\bm{w}^{r},\lambda_{r})\big\}.$
\EndFunction \end{algorithmic}
\end{algorithm}

\subsection{Choices of algorithmic parameters\label{subsec:Choices-of-algorithmic-pars}}

To guarantee fast convergence of Algorithm~\ref{alg:gd}, there are
a couple of algorithmic parameters --- namely, the number of restart
attempts $L$, the pruning threshold $\epsilon_{\mathsf{th}}$ in
Algorithm~\ref{alg:localization}, as well as the learning rates
$\eta_{t}$ --- that need to be properly chosen. Unless otherwise
noted, this paper adopts the following choices suggested by \cite{cai2019nonconvex}:
\begin{equation}
L=c_{4}r^{2\kappa^{2}}\log^{3/2}d,\quad \eta_{t}\equiv\frac{c_{5}\lambda_{\min}^{\star4/3}}{p\lambda_{\max}^{\star8/3}} \quad\text{and}\quad \epsilon_{\mathsf{th}}=c_{6}\bigg(\frac{\mu r\log d}{d\sqrt{p}}+\frac{\sigma_{\min}}{\lambda_{\min}^{\star}}\sqrt{\frac{rd\log^{2}d}{p}}+\sqrt{\frac{\mu r\log d}{d}}\bigg),\label{eq:choice-L-epsilon}
\end{equation}
where $c_{4}>0$ is some sufficiently large constant, and $c_{5},c_{6}>0$
are some sufficiently small constants. The interested reader is referred
to \cite{cai2019nonconvex} for justification.

%% file: preliminaries.tex
\section{Preliminary facts}

\label{sec:Preliminaries}

In this section, we gather a few preliminary facts that prove useful
throughout the analysis.

\subsection{Leave-one-out sequences}

\label{subsec:Leave-one-out-sequences}

To facilitate the analysis and decouple statistical dependency, we
introduce the following set of auxiliary tensors and loss functions
for all $1\leq m\leq d$:
\begin{align*}
\bm{T}^{\mathsf{obs},(m)} & :=\mathcal{P}_{\Omega_{-m}}(\bm{T}^{\mathsf{obs}})+p\,\mathcal{P}_{m}(\bm{T}^{\star}),\\
f^{(m)}(\bm{U}) & :=\Big\|\mathcal{P}_{\Omega_{-m}}\Big(\sum_{i=1}^{r}\bm{u}_{i}^{\otimes3}-\bm{T}^{\star}-\bm{E}\Big)\Big\|_{\mathrm{F}}^{2}+p\,\Big\|\mathcal{P}_{m}\Big(\sum_{i=1}^{r}\bm{u}_{i}^{\otimes3}-\bm{T}^{\star}\Big)\Big\|_{\mathrm{F}}^{2},
\end{align*}
where $\mathcal{P}_{\Omega_{-m}}$ (resp.~$\mathcal{P}_{m}$) is
the Euclidean projection onto the subspace of tensors supported on
$\{(i,j,k)\in\Omega\colon i\neq m\text{ and }j\neq m\text{ and }k\neq m\}$
(resp.~$\{(i,j,k)\in[d]^{3}\colon i=m\text{ or }j=m\text{ or }k=m\}$).
We shall denote by $\bm{U}^{(m)}$ the leave-one-out estimate returned
by Algorithm~\ref{alg:gd_loo}.

While the algorithms look somewhat complex, the idea is very simple.
In words, the new estimate $\bm{U}^{(m)}$ is obtained by dropping
all randomness from the $m$-th slice (namely, those data from the
index set $\{(i,j,k)\in[d]^{3}\colon i=m\text{ or }j=m\text{ or }k=m\}$).
This means that $\bm{U}^{(m)}$ is statistically independent from
the data coming from the $m$-th slice. These leave-one-out sequences
enjoy several useful properties that have been established in \cite{cai2019nonconvex},
which we shall present in the next subsection.

\begin{algorithm}[H]
\caption{The $m$-th leave-one-out estimate}

\label{alg:gd_loo}\begin{algorithmic}

\State \textbf{{Initialize}} $\bm{U}^{0,(m)}=\big[\bm{u}_{1}^{0,(m)},\cdots,\bm{u}_{r}^{0,(m)}\big]$
via Algorithm \ref{alg:init_loo}.

\State \textbf{{Gradient updates}}: \textbf{for }$t=0,1,\ldots,t_{0}-1$
\textbf{do}

\State \vspace{-2.5em}
 
\begin{align}
\bm{U}^{t+1,(m)} & =\bm{U}^{t,(m)}-\eta_{t}\nabla f^{(m)}(\bm{U}^{t,(m)}).\label{eq:gradient_update_ncvx-TC-1}
\end{align}

\State \textbf{{Output} }$\bm{U}^{(m)}=\big[\bm{u}_{1}^{(m)},\cdots,\bm{u}_{r}^{(m)}\big]:=\bm{U}^{t_{0},(m)}$.

\end{algorithmic}
\end{algorithm}

\begin{algorithm}[H]
\caption{The $m$-th leave-one-out sequence for spectral initialization}
\label{alg:init_loo}

\begin{algorithmic}[1] \State Let $\bm{U}_{\mathsf{space}}^{(m)}\bm{\Lambda}^{(m)}\bm{U}_{\mathsf{space}}^{(m)\top}$
be the rank-$r$ eigen-decomposition of 
\begin{align}
\bm{B}^{(m)}:=\mathcal{P}_{\mathsf{off}\text{-}\mathsf{diag}}(\bm{A}^{(m)}\bm{A}^{(m)\top}),\label{eq:B-alg-loo}
\end{align}
where $\bm{A}^{(m)}=\mathsf{unfold}\big(p^{-1}\bm{T}^{\mathsf{obs},(m)}\big)$
is the mode-1 matricization of $p^{-1}\bm{T}^{\mathsf{obs},(m)}$,
and $\mathcal{P}_{\mathsf{off}\text{-}\mathsf{diag}}(\bm{Z})$ extracts
out the off-diagonal entries of $\bm{Z}$.

\State \textbf{Output:} an estimate $\bm{U}^{0,(m)}\in\mathbb{R}^{d\times r}$
on the basis of $\bm{U}_{\mathsf{space}}^{(m)}\in\mathbb{R}^{d\times r}$
using Algorithm~\ref{alg:localization-loo}. \end{algorithmic}
\end{algorithm}

\begin{algorithm}[H]
\caption{The $m$-th leave-one-out sequence for retrieving individual tensor
components}

\label{alg:localization-loo}

\begin{algorithmic}[1] \State \textbf{Input:} number of restarts
$L$, pruning threshold $\epsilon_{\mathsf{th}}$, subspace estimate
$\bm{U}_{\mathsf{space}}^{(m)}\in\mathbb{R}^{d\times r}$ given by
Algorithm~\ref{alg:init_loo}. \For{$\tau=1,\dots,L$}

\State Recall the Gaussian vector $\bm{g}^{\tau}\sim\mathcal{N}(0,\bm{I}_{d})$
generated in Algorithm~\ref{alg:localization}.

\State $\big(\bm{\nu}^{\tau,(m)},\lambda_{\tau}^{(m)},\mathsf{spec}\text{-}\mathsf{gap}_{\tau}^{(m)}\big)\gets\Call{Retrieve-one-tensor-factor}{\bm{T}^{(m)},p,\bm{U}_{\mathsf{space}}^{(m)},\bm{g}^{\tau}}$.
\EndFor \State Generate tensor factor estimates 
\[
\big\{\big(\bm{w}^{1,(m)},\lambda_{1}^{(m)}\big),\dots,\big(\bm{w}^{r,(m)},\lambda_{r}^{(m)}\big)\big\}\gets\Call{Prune}{\big\{\big(\bm{\nu}^{\tau,(m)},\lambda_{\tau}^{(m)},\mathsf{spec}\text{-}\mathsf{gap}_{\tau}^{(m)}\big)\big\}_{\tau=1}^{L},\epsilon_{\mathsf{th}}}.
\]
\State \textbf{Output:} an estimate $\bm{U}^{0,(m)}=\big[\big(\lambda_{1}^{(m)}\big)^{1/3}\bm{w}^{1,(m)},\dots,\big(\lambda_{r}^{(m)}\big)^{1/3}\bm{w}^{r,(m)}\big].$
\end{algorithmic}
\end{algorithm}

\subsection{Properties of the nonconvex estimates}

We now collect several important properties of our tensor estimates
as well as the associated leave-one-out estimates, most of which have
been established in \cite{cai2019nonconvex}. To begin with, Lemma~\ref{lemma:U-loss-property}
quantifies the estimation error of $\bm{U}$ and $\bm{T}$.

\begin{lemma}\label{lemma:U-loss-property}Instate the assumptions
and notations of Theorem~\ref{thm:U-loss-dist-nonGaussian-rank-r}.
With probability at least $1-o\left(1\right)$, \begin{subequations}\label{eq:U-T-loss-UB}
\begin{align}
\big\|\bm{U}\bm{\Pi}-\bm{U}^{\star}\big\|_{\mathrm{F}} & \lesssim\frac{\sigma_{\max}}{\lambda_{\min}^{\star}}\sqrt{\frac{rd\log d}{p}}\,\lambda_{\max}^{\star1/3};\label{eq:U-loss-fro}\\
\big\|\bm{U}\bm{\Pi}-\bm{U}^{\star}\big\|_{2,\infty} & \lesssim\frac{\sigma_{\max}}{\lambda_{\min}^{\star}}\sqrt{\frac{\mu r\log d}{p}}\,\lambda_{\max}^{\star1/3};\label{eq:U-loss-2inf}\\
\big\|\bm{T}-\bm{T}^{\star}\big\|_{\infty} & \lesssim\frac{\sigma_{\max}}{\lambda_{\min}^{\star}}\sqrt{\frac{\mu^{3}r^{2}\log d}{d^{2}p}}\,\lambda_{\max}^{\star}.\label{eq:T-loss-inf}
\end{align}
\end{subequations}\end{lemma}

The next lemma demonstrates that the leave-one-out sequences $\big\{\bm{U}^{(m)}\big\}_{1\leq m\leq d}$
constructed in Algorithm~\ref{alg:gd_loo} are sufficiently close
to the true estimate $\bm{U}$. As a result, $\bm{U}^{(m)}$ (resp.~$\bm{T}^{(m)}$)
also serves as a faithful estimate of the ground truth $\bm{U}^{\star}$
(resp.~$\bm{T}^{\star}$), where
\begin{equation}
\bm{T}^{(m)}:=\sum_{1\leq l\leq r}\big(\bm{u}_{l}^{(m)}\big)^{\otimes3}.\label{eq:Tm-definition}
\end{equation}
The results are summarized as follows.

\begin{lemma}\label{lemma:U-loo-property}Instate the assumptions
and notation of Theorem~\ref{thm:U-loss-dist-nonGaussian-rank-r}.
With probability at least $1-o\left(1\right)$, for all $1\leq m\leq d$
one has\begin{subequations}\label{eq:U-loo-property}
\begin{align}
\big\|\bm{U}-\bm{U}^{\left(m\right)}\big\|_{\mathrm{F}} & \lesssim\frac{\sigma_{\max}}{\lambda_{\min}^{\star}}\sqrt{\frac{\mu r\log d}{p}}\,\lambda_{\max}^{\star1/3}.\label{eq:U-U-loss-diff-fro}\\
\big\|\bm{U}^{\left(m\right)}\bm{\Pi}-\bm{U}^{\star}\big\|_{\mathrm{F}} & \lesssim\frac{\sigma_{\max}}{\lambda_{\min}^{\star}}\sqrt{\frac{rd\log d}{p}}\,\lambda_{\max}^{\star1/3};\label{eq:U-loo-loss-fro}\\
\big\|\bm{U}^{\left(m\right)}\bm{\Pi}-\bm{U}^{\star}\big\|_{2,\infty} & \lesssim\frac{\sigma_{\max}}{\lambda_{\min}^{\star}}\sqrt{\frac{\mu r\log d}{p}}\,\lambda_{\max}^{\star1/3};\label{eq:U-loo-loss-2inf}\\
\big\|\bm{T}^{(m)}-\bm{T}^{\star}\big\|_{\infty} & \lesssim\frac{\sigma_{\max}}{\lambda_{\min}^{\star}}\sqrt{\frac{\mu^{3}r^{2}\log d}{d^{2}p}}\,\lambda_{\max}^{\star}.\label{eq:T-loo-loss-inf}
\end{align}
\end{subequations}\end{lemma}

In addition, Lemma~\ref{lemma:U-property} collects several simple
properties about the true tensor factors and their corresponding estimates.
The proof can be found in Appendix~\ref{subsec:U-property}.

\begin{lemma}\label{lemma:U-property}Instate the assumptions and
notation of Theorem~\ref{thm:U-loss-dist-nonGaussian-rank-r}. With
probability at least $1-o\left(1\right)$, there is a permutation
$\pi(\cdot):[d]\mapsto[d]$ such that \begin{subequations}\label{eq:U-property}
\begin{align}
 & \left\Vert \bm{U}^{\star}\right\Vert _{\mathrm{F}}\leq\sqrt{r}\,\lambda_{\max}^{\star1/3},\qquad\qquad\left\Vert \bm{U}^{\star}\right\Vert _{2,\infty}\leq\sqrt{\frac{\mu r}{d}}\,\lambda_{\max}^{\star1/3};\label{eq:U-true-norm}\\
 & \sigma_{1}(\bm{U}^{\star})=\lambda_{\max}^{\star1/3}\left(1+o\left(1\right)\right),\qquad\qquad\sigma_{r}(\bm{U}^{\star})=\lambda_{\min}^{\star}\left(1+o\left(1\right)\right);\label{eq:U-true-spectrum}\\
 & \left\Vert \bm{u}_{\pi(i)}-\bm{u}_{i}^{\star}\right\Vert _{2}=o\left(1\right)\left\Vert \bm{u}_{i}^{\star}\right\Vert _{2},\qquad\qquad\left\Vert \bm{u}_{\pi(i)}-\bm{u}_{i}^{\star}\right\Vert _{\infty}=o\left(1\right)\left\Vert \bm{u}_{i}^{\star}\right\Vert _{\infty},\qquad1\leq i\leq r;\label{eq:u-loss-u-relation}\\
 & \lambda_{\min}^{\star1/3}\lesssim\left\Vert \bm{u}_{i}\right\Vert _{2}\lesssim\lambda_{\max}^{\star1/3};\qquad\qquad\sqrt{\frac{1}{d}}\,\lambda_{\min}^{\star1/3}\lesssim\left\Vert \bm{u}_{i}\right\Vert _{\infty}\lesssim\sqrt{\frac{\mu}{d}}\,\lambda_{\max}^{\star1/3},\qquad1\leq i\leq r;\label{eq:u-norm}\\
 & \max_{1\leq i\neq j\leq r}\left|\left\langle \bm{u}_{i},\bm{u}_{j}\right\rangle \right|\lesssim\left\{ \sqrt{\frac{\mu}{d}}+\frac{\sigma_{\max}}{\lambda_{\min}^{\star}}\sqrt{\frac{rd\log d}{p}}\right\} \,\lambda_{\max}^{\star2/3},\label{eq:u-inner-prod}\\
 & \sigma_{1}(\bm{U})=\lambda_{\max}^{\star1/3}\left(1+o\left(1\right)\right),\qquad\qquad\sigma_{r}(\bm{U})=\lambda_{\min}^{\star}\left(1+o\left(1\right)\right).\label{eq:U-spectrum}
\end{align}
In addition, these results hold unchanged if we replace $\bm{u}_{i}$
with $\bm{u}_{i}^{\left(m\right)}$ for all $1\leq m\leq d$.\end{subequations}\end{lemma}

Finally, Lemma~\ref{lemma:U-tilde-property} summarizes several useful
bounds regarding $\widetilde{\bm{U}}^{\star}:=[\bm{u}_{l}^{\star\otimes2}]_{1\leq l\leq r}\in\mathbb{R}^{d^{2}\times r}$
and $\widetilde{\bm{U}}:=[\bm{u}_{l}^{\otimes2}]_{1\leq l\leq r}\in\mathbb{R}^{d^{2}\times r}$.
The proof is deferred to Appendix~\ref{subsec:U-tilde-property}.

\begin{lemma}\label{lemma:U-tilde-property}Instate the assumptions
and notation of Theorem~\ref{thm:U-loss-dist-nonGaussian-rank-r}.
With probability at least $1-o\left(1\right)$,\begin{subequations}\label{eq:U-tilde-property}
\begin{align}
 & \big\|\widetilde{\bm{U}}^{\star}\big\|_{2,\infty}\leq\frac{\mu\sqrt{r}}{d}\lambda_{\max}^{\star2/3},\qquad\qquad\big\|\widetilde{\bm{U}}^{\star}\big\|_{\mathrm{F}}\leq\sqrt{r}\,\lambda_{\max}^{\star2/3};\label{eq:U-true-tilde-norm}\\
 & \sigma_{1}\big(\widetilde{\bm{U}}^{\star}\big)=\lambda_{\max}^{\star2/3}\left(1+o\left(1\right)\right),\qquad\qquad\sigma_{r}\big(\widetilde{\bm{U}}^{\star}\big)=\lambda_{\min}^{\star2/3}\left(1+o\left(1\right)\right);\label{eq:U-true-tilde-spectrum}\\
 & \sigma_{1}\big(\widetilde{\bm{U}}\big)=\lambda_{\max}^{\star2/3}\left(1+o\left(1\right)\right),\qquad\qquad\sigma_{r}\big(\widetilde{\bm{U}}\big)=\lambda_{\min}^{\star2/3}\left(1+o\left(1\right)\right);\label{eq:U-tilde-spectrum}\\
 & \big\|\widetilde{\bm{U}}\bm{\Pi}-\widetilde{\bm{U}}^{\star}\big\|_{\mathrm{F}}\lesssim\frac{\sigma_{\max}}{\lambda_{\min}^{\star}}\sqrt{\frac{rd\log d}{p}}\,\lambda_{\max}^{\star2/3};\label{eq:U-tilde-loss-fro}\\
 & \big\|\widetilde{\bm{U}}\bm{\Pi}-\widetilde{\bm{U}}^{\star}\big\|_{2,\infty}\lesssim\frac{\sigma_{\max}}{\lambda_{\min}^{\star}}\sqrt{\frac{\mu^{2}r\log d}{dp}}\,\lambda_{\max}^{\star2/3}.\label{eq:U-tilde-loss-2inf}
\end{align}
In addition, the above results continue to hold if $\widetilde{\bm{U}}$
is replaced by $\widetilde{\bm{U}}^{(m)}=\big[\bm{u}_{l}^{(m)}\otimes\bm{u}_{l}^{(m)}\big]_{1\leq l\leq r}$
for all $1\leq m\leq d$.\end{subequations}\end{lemma}

\subsection{A Berry-Esseen-type theorem}

\label{subsec:Berry-Esseen-theorem}

The distributional guarantees are built upon the Berry-Esseen-type
inequality \cite[Theorem 1.1]{bentkus2005lyapunov}, which will be
used multiple times in the analysis.

\begin{theorem}\label{thm:Berry-Esseen}Let $\left\{ \bm{x}_{i}\right\} _{1\leq i\leq n}$
be a sequence of independent zero-mean random vectors in $\mathbb{R}^{d}$.
Denote by $\bm{\Sigma}$ the covariance matrix of $\sum_{1\leq i\leq n}\bm{x}_{i}$,
and let $\bm{z}\sim\mathcal{N}(\bm{0},\bm{\Sigma})$ be a Gaussian
vector in $\mathbb{R}^{d}$. Then one has
\begin{equation}
\sup_{\mathcal{A}\in\mathcal{C}}\Big|\mathbb{P}\Big\{\sum_{1\leq i\leq n}\bm{x}_{i}\in\mathcal{A}\Big\}-\mathbb{P}\left\{ \bm{z}\in\mathcal{A}\right\} \Big|\lesssim d^{1/4}\rho,\label{eq:Berry-Esseen}
\end{equation}
where $\mathcal{C}$ is the set of all convex subsets of $\mathbb{R}^{d}$,
and $\rho$ is defined as follows
\begin{equation}
\rho:=\sum_{1\leq i\leq n}\mathbb{E}\left[\big\|\bm{\Sigma}^{-1/2}\bm{x}_{i}\big\|_{2}^{3}\right].\label{def:Berry-Esseen-rho}
\end{equation}
\end{theorem}

%% file: proof-U-dist.tex
\section{Proof of auxiliary lemmas: distributional theory for tensor factors}

\label{sec:Analysis-of-U}

Given a symmetric random tensor, we can always partition it into six
sub-tensors such that the entries within each sub-tensor are independent.
Therefore, whenever orderwise bounds are sufficient, we shall treat
$\{E_{i,j,k}\}_{1\leq i,j,k\leq d}$ and $\{\chi_{i,j,k}\}_{1\leq i,j,k\leq d}$
(see the definition in (\ref{eq:def:chi})) as independent random
variables in order to simplify presentation.

\subsection{Proof of Lemma \ref{lemma:U-loss-dist-main-part-Gaussian}}

\label{subsec:U-loss-dist-main-part-Gaussian}

Fix an arbitrary $m\in[d]$. Let us define a sequence of random vectors
$\{\bm{z}_{i,j,k}\}_{1\leq i,j,k\leq d}$ in $\mathbb{R}^{r}$ as
follows:
\begin{align}
\bm{z}_{i,j,k} & :=p^{-1}E_{i,j,k}\chi_{i,j,k}\widetilde{\bm{U}}_{(i,j),:}^{\star}(\widetilde{\bm{U}}^{\star\top}\widetilde{\bm{U}}^{\star})^{-1},\qquad1\leq i,j,k\leq d,\label{def:z}
\end{align}
where we recall the notation $(i,j):=(i-1)d+j$ defined in Section~\ref{subsec:Notations}.
Then we can express
\[
\bm{X}_{m,:}=\sum_{1\leq i\leq d}\bm{z}_{i,i,m}+2\sum_{1\leq i<j\leq d}\bm{z}_{i,j,m}
\]
as a sum of independent zero-mean random vectors in $\mathbb{R}^{r}$.
Let us further define a matrix $\bm{Z}\in\mathbb{R}^{d\times r}$
whose $k$-th row is given by
\begin{align}
\bm{Z}_{k,:} & =\sqrt{2}\sum_{1\leq i\leq d}\bm{z}_{i,i,k}+2\sum_{1\leq i<j\leq d}\bm{z}_{i,j,k},\label{def:Z}
\end{align}
and let $\bm{W}_{0}:=\bm{X}-\bm{Z}$. Straightforward calculation
gives
\begin{align*}
\mathbb{E}\big[(\bm{X}_{m,:})^{\top}\bm{X}_{m,:}\big] & =\frac{1}{p}(\widetilde{\bm{U}}^{\star\top}\widetilde{\bm{U}}^{\star})^{-1}\widetilde{\bm{U}}^{\star\top}(2\bm{D}_{m}^{\star}-\bm{C}_{m}^{\star})\widetilde{\bm{U}}^{\star}(\widetilde{\bm{U}}^{\star\top}\widetilde{\bm{U}}^{\star})^{-1},\\
\mathbb{E}\big[(\bm{Z}_{m,:})^{\top}\bm{Z}_{m,:}\big] & =\frac{2}{p}(\widetilde{\bm{U}}^{\star\top}\widetilde{\bm{U}}^{\star})^{-1}\widetilde{\bm{U}}^{\star\top}\bm{D}_{m}^{\star}\widetilde{\bm{U}}^{\star}(\widetilde{\bm{U}}^{\star\top}\widetilde{\bm{U}}^{\star})^{-1}=\bm{\Sigma}_{m}^{\star},
\end{align*}
where we recall that $\bm{D}_{m}^{\star}$ (resp.~$\bm{\Sigma}_{m}^{\star}$)
is defined in (\ref{eq:cov-matrix-m-diag}) (resp.~(\ref{eq:cov-matrix-m})),
and $\bm{C}_{m}^{\star}$ is a diagonal matrix in $\mathbb{R}^{d^{2}\times d^{2}}$
with entries
\[
\left(\bm{C}_{m}^{\star}\right)_{(i,j),(i,j)}=\begin{cases}
\sigma_{i,j,m}^{2}, & \text{if}\quad i=j,\\
0, & \text{if}\quad i\neq j.
\end{cases}
\]
In what follows, we will prove (i) $\bm{X}_{m,:}$ and $\bm{Z}_{m,:}$
are sufficiently close, i.e.~the $\ell_{2,\infty}$ norm of $\bm{W}_{0}$
is considerably small; (ii) $\bm{Z}_{m,:}$ is a Gaussian vector with
mean zero and covariance matrix $\bm{\Sigma}_{m}^{\star}$ with high
probability.

We begin with the first claim. Observe that
\[
(\bm{W}_{0})_{m,:}=(\sqrt{2}-1)\sum_{1\leq i\leq d}p^{-1}E_{i,i,m}\chi_{i,i,m}\widetilde{\bm{U}}_{(i,i),:}^{\star}(\widetilde{\bm{U}}^{\star\top}\widetilde{\bm{U}}^{\star})^{-1}
\]
 is a sum of independent zero-mean random vectors in $\mathbb{R}^{r}$.
By (\ref{eq:U-true-tilde-norm}) and (\ref{eq:U-true-tilde-spectrum}),
it is straightforward to compute
\begin{align*}
B_{1} & :=\max_{1\leq i\leq d}\big\| p^{-1}E_{i,i,m}\chi_{i,i,m}\widetilde{\bm{U}}_{(i,i),:}^{\star}(\widetilde{\bm{U}}^{\star\top}\widetilde{\bm{U}}^{\star})^{-1}\big\|_{\psi_{1}}\lesssim\frac{\sigma_{\max}}{p}\big\|\widetilde{\bm{U}}^{\star}\big\|_{2,\infty}\big\|(\widetilde{\bm{U}}^{\star\top}\widetilde{\bm{U}}^{\star})^{-1}\big\|\lesssim\frac{\sigma_{\max}}{p}\cdot\frac{\mu\sqrt{r}\lambda_{\max}^{\star2/3}}{d\lambda_{\min}^{\star4/3}},\\
V_{1} & :=\sum_{1\leq i\leq d}\frac{1}{p^{2}}\mathbb{E}\big[E_{i,i,m}^{2}\chi_{i,i,m}\big]\big\|\widetilde{\bm{U}}_{(i,i),:}^{\star}(\widetilde{\bm{U}}^{\star\top}\widetilde{\bm{U}}^{\star})^{-1}\big\|_{2}^{2}\lesssim\frac{\sigma_{\max}^{2}}{p}\sum_{1\leq i\leq d}\big\|\widetilde{\bm{U}}_{(i,i),:}^{\star}\big\|_{2}^{2}\big\|(\widetilde{\bm{U}}^{\star\top}\widetilde{\bm{U}}^{\star})^{-1}\big\|^{2}\lesssim\frac{\sigma_{\max}^{2}}{p}\cdot\frac{\mu r\lambda_{\max}^{\star4/3}}{d\lambda_{\min}^{\star8/3}},
\end{align*}
where $\|\cdot\|_{\psi_{1}}$ denotes the sub-exponential norm, and
we use the following bound in the second line:
\[
\sum_{1\leq i\leq d}\big\|\widetilde{\bm{U}}_{(i,i),:}^{\star}\big\|_{2}^{2}=\sum_{1\leq i\leq d}\sum_{1\leq l\leq r}u_{l,i}^{\star4}\leq\max_{1\leq l\leq r}\left\Vert \bm{u}_{l}^{\star}\right\Vert _{\infty}^{2}\left\Vert \bm{U}^{\star}\right\Vert _{\mathrm{F}}^{2}\lesssim\frac{\mu r}{d}\lambda_{\max}^{\star4/3}.
\]
We then invoke the matrix Bernstein inequality \cite[Corollary 2.1]{Koltchinskii2011oracle}
to find that with probability exceeding $1-O\left(d^{-11}\right)$,
\begin{equation}
\left\Vert (\bm{W}_{0})_{m,:}\right\Vert _{2}\lesssim B_{1}\log^{2}d+\sqrt{V_{1}\log d}\lesssim\frac{\sigma_{\max}}{\lambda_{\min}^{\star2/3}\sqrt{p}}\left\{ \frac{\mu\sqrt{r}\log^{2}d}{d\sqrt{p}}+\sqrt{\frac{\mu r\log d}{d}}\right\} ,\label{eq:Z-row-Gaussian-diff-norm-part1}
\end{equation}
where we have used the assumption $\kappa\asymp1$.

We move on to consider the distribution of $\bm{Z}_{m,:}$. Conditional
on $\{\chi_{i,j,m}\}_{1\leq i,j\leq d}$, the vector $\bm{Z}_{m,:}$
is zero-mean Gaussian with covariance matrix 
\[
\bm{S}_{m}^{\star}:=\frac{2}{p^{2}}\sum_{1\leq i,j\leq d}\sigma_{i,j,m}^{2}\chi_{i,j,m}(\widetilde{\bm{U}}^{\star\top}\widetilde{\bm{U}}^{\star})^{-1}(\widetilde{\bm{U}}_{(i,j),:}^{\star})^{\top}\widetilde{\bm{U}}_{(i,j),:}^{\star}(\widetilde{\bm{U}}^{\star\top}\widetilde{\bm{U}}^{\star})^{-1},
\]
which satisfies
\[
\mathbb{E}[\bm{S}_{m}^{\star}]=\bm{\Sigma}_{m}^{\star}.
\]
Lemma~\ref{lemma:Z-row-cond-cov} below demonstrates that $\bm{S}_{m}^{\star}$
and $\bm{\Sigma}_{m}^{\star}$ are, with high probability, sufficiently
close in the spectral norm; the proof is deferred to the end of the
section.

\begin{lemma}\label{lemma:Z-row-cond-cov}Instate the assumptions
of Lemma~\ref{lemma:U-loss-dist-main-part-Gaussian}. With probability
exceeding $1-O\left(d^{-10}\right)$,

\begin{equation}
\max_{1\leq m\leq d}\left\Vert \bm{S}_{m}^{\star}-\bm{\Sigma}_{m}^{\star}\right\Vert \lesssim\frac{\sigma_{\max}^{2}}{\lambda_{\min}^{\star4/3}p}\sqrt{\frac{\mu^{2}r\log d}{d^{2}p}}=o\left(\frac{\sigma_{\max}^{2}}{\lambda_{\min}^{\star4/3}p}\right).\label{claim:Z-row-cond-cov}
\end{equation}
\end{lemma}In what follows, we shall work on the high-probability
event where (\ref{claim:Z-row-cond-cov}) holds. From the definition
of the covariance matrix $\bm{\Sigma}_{m}^{\star}$, it is easily
seen that
\[
\bm{\Sigma}_{m}^{\star}\succeq\frac{2\sigma_{\min}^{2}}{p}(\widetilde{\bm{U}}^{\star\top}\widetilde{\bm{U}}^{\star})^{-1}\qquad\text{and}\qquad\bm{\Sigma}_{m}^{\star}\preceq\frac{2\sigma_{\max}^{2}}{p}(\widetilde{\bm{U}}^{\star\top}\widetilde{\bm{U}}^{\star})^{-1}.
\]
Additionally, it follows from (\ref{eq:U-true-tilde-spectrum}) that
\begin{align}
\lambda_{\min}(\bm{\Sigma}_{m}^{\star}) & \gtrsim\frac{\sigma_{\min}^{2}}{\lambda_{\max}^{\star4/3}p}\qquad\text{and}\qquad\lambda_{\max}(\bm{\Sigma}_{m}^{\star})\lesssim\frac{\sigma_{\max}^{2}}{\lambda_{\max}^{\star4/3}p}.\label{eq:cov-matrix-eigval}
\end{align}
We then know from Weyl's inequality and the conditions $\sigma_{\max}/\sigma_{\min}\asymp1$
and $\kappa\asymp1$ that
\begin{align}
\lambda_{\min}(\bm{S}_{m}^{\star}) & \geq\lambda_{\min}(\bm{\Sigma}_{m}^{\star})-\left\Vert \bm{S}_{m}^{\star}-\bm{\Sigma}_{m}^{\star}\right\Vert \gtrsim\frac{\sigma_{\min}^{2}}{\lambda_{\max}^{\star4/3}p}>0.\label{eq:emp-cov-matrix-Gaussian-eigval}
\end{align}
This implies that $\bm{S}_{m}^{\star}$ is positive semidefinite and,
therefore, $\bm{S}_{m}^{\star-1/2}$ is well-defined. As a result,
$\bm{Z}_{m,:}\bm{S}_{m}^{\star-1/2}\bm{\Sigma}_{m}^{\star1/2}$ is
a zero-mean Gaussian vector with covariance matrix $\bm{\Sigma}_{m}^{\star}$.

With slight abuse of notation, we will treat $\bm{Z}_{m,:}-\bm{Z}_{m,:}\bm{S}_{m}^{\star-1/2}\bm{\Sigma}_{m}^{\star1/2}+\bm{W}_{0}$
as the residual term for the Gaussian approximation. Hence, it remains
to show that $\bm{Z}_{m,:}\bm{S}_{m}^{\star-1/2}\bm{\Sigma}_{m}^{\star1/2}$
and $\bm{Z}_{m,:}$ are exceedingly close in the $\ell_{2}$ norm.
To this end, we observe an upper bound
\begin{align*}
\big\|\bm{Z}_{m,:}-\bm{Z}_{m,:}\bm{S}_{m}^{\star-1/2}\bm{\Sigma}_{m}^{\star1/2}\big\|_{2} & =\big\|\bm{Z}_{m,:}\bm{S}_{m}^{\star-1/2}\big(\bm{S}_{m}^{\star1/2}-\bm{\Sigma}_{m}^{\star1/2}\big)\big\|_{2}\lesssim\left\Vert \bm{Z}_{m,:}\right\Vert _{2}\big\|\bm{S}_{m}^{\star-1/2}\big\|\big\|\bm{S}_{m}^{\star1/2}-\bm{\Sigma}_{m}^{\star1/2}\big\|.
\end{align*}
By the perturbation bounds for matrix square roots \cite[Lemma 2.1]{MR1176461},
one knows from (\ref{claim:Z-row-cond-cov}), (\ref{eq:cov-matrix-eigval})
and (\ref{eq:U-true-tilde-spectrum}) that
\[
\big\|\bm{S}_{m}^{\star1/2}-\bm{\Sigma}_{m}^{\star1/2}\big\|\leq\frac{1}{\lambda_{\min}\big(\bm{S}_{m}^{\star1/2}\big)+\lambda_{\min}\big(\bm{\Sigma}_{m}^{\star1/2}\big)}\left\Vert \bm{S}_{m}^{\star}-\bm{\Sigma}_{m}^{\star}\right\Vert \lesssim\frac{\sigma_{\max}}{\lambda_{\min}^{\star2/3}\sqrt{p}}\sqrt{\frac{\mu^{2}r\log d}{d^{2}p}},
\]
where we use the conditions $\sigma_{\max}/\sigma_{\min}\asymp1$
and $\kappa\asymp1$. In addition, $\bm{Z}_{m,:}$ is a sum of independent
zero-mean random vectors with bounds
\begin{align*}
B_{2} & :=\max_{1\leq i,j\leq d}\big\| p^{-1}E_{i,j,m}\chi_{i,j,m}\widetilde{\bm{U}}_{(i,j),:}^{\star}(\widetilde{\bm{U}}^{\star\top}\widetilde{\bm{U}}^{\star})^{-1}\big\|_{\psi_{1}}\lesssim\frac{\sigma_{\max}}{p}\big\|\widetilde{\bm{U}}^{\star}\big\|_{2,\infty}\big\|(\widetilde{\bm{U}}^{\star\top}\widetilde{\bm{U}}^{\star})^{-1}\big\|\lesssim\frac{\sigma_{\max}}{p}\cdot\frac{\mu\sqrt{r}\lambda_{\max}^{\star2/3}}{d\lambda_{\min}^{\star4/3}},\\
V_{2} & :=\sum_{1\leq i,j\leq d}\frac{1}{p^{2}}\mathbb{E}\big[E_{i,j,m}^{2}\chi_{i,j,m}\big]\big\|\widetilde{\bm{U}}_{(i,j),:}^{\star}(\widetilde{\bm{U}}^{\star\top}\widetilde{\bm{U}}^{\star})^{-1}\big\|_{2}^{2}\lesssim\frac{\sigma_{\max}^{2}}{p}\big\|\widetilde{\bm{U}}^{\star}\big\|_{\mathrm{F}}^{2}\big\|(\widetilde{\bm{U}}^{\star\top}\widetilde{\bm{U}}^{\star})^{-1}\big\|^{2}\lesssim\frac{\sigma_{\max}^{2}}{p}\cdot\frac{r\lambda_{\max}^{\star4/3}}{\lambda_{\min}^{\star8/3}},
\end{align*}
which rely on (\ref{eq:U-true-tilde-norm}) and (\ref{eq:U-true-tilde-spectrum}).
It then follows from the matrix Bernstein inequality that
\begin{align}
\left\Vert \bm{Z}_{m,:}\right\Vert _{2} & \lesssim B_{2}\log^{2}d+\sqrt{V_{2}\log d}\lesssim\frac{\sigma_{\max}\lambda_{\max}^{\star2/3}}{\lambda_{\min}^{\star4/3}\sqrt{p}}\left\{ \frac{\mu\sqrt{r}\,\log^{2}d}{d\sqrt{p}}+\sqrt{r\log d}\right\} \asymp\frac{\sigma_{\max}\sqrt{r\log d}}{\lambda_{\min}^{\star2/3}\sqrt{p}}\label{eq:Z-row-l2-norm-UB}
\end{align}
with probability at least $1-O\left(d^{-11}\right)$, as long as $p\gtrsim\mu^{2}d^{-2}\log^{3}$
and $\kappa\asymp1$. Therefore, we arrive at
\begin{equation}
\big\|\bm{Z}_{m,:}-\bm{Z}_{m,:}\bm{S}_{m}^{\star-1/2}\bm{\Sigma}_{m}^{\star1/2}\big\|_{2}\lesssim\frac{\sigma_{\max}\sqrt{r\log d}}{\lambda_{\min}^{\star2/3}\sqrt{p}}\cdot\sqrt{\frac{\lambda_{\max}^{\star4/3}p}{\sigma_{\min}^{2}}}\cdot\frac{\sigma_{\max}}{\lambda_{\min}^{\star2/3}\sqrt{p}}\sqrt{\frac{\mu^{2}r\log d}{d^{2}p}}\lesssim\frac{\sigma_{\max}}{\lambda_{\min}^{\star2/3}\sqrt{p}}\frac{\mu r\log d}{d\sqrt{p}}.\label{eq:Z-row-Gaussian-diff-norm-part2}
\end{equation}
Combining (\ref{eq:Z-row-Gaussian-diff-norm-part1}) and (\ref{eq:Z-row-Gaussian-diff-norm-part2})
and then taking a union bound over $1\leq m\leq d$, we reach the
advertised bound on the $\ell_{2,\infty}$ norm of the residual term
$\bm{W}_{0}$.

\subsubsection{Proof of Lemma~\ref{lemma:Z-row-cond-cov}}

Recalling the definitions of $\bm{S}_{m}^{\star}$ and $\bm{\Sigma}_{m}^{\star}$,
we can express 
\[
\bm{S}_{m}^{\star}-\bm{\Sigma}_{m}^{\star}=\frac{2}{p}(\widetilde{\bm{U}}^{\star\top}\widetilde{\bm{U}}^{\star})^{-1}\sum_{1\leq i,j\leq d}\sigma_{i,j,m}^{2}(p^{-1}\chi_{i,j,m}-1)(\widetilde{\bm{U}}_{(i,j),:}^{\star})^{\top}\widetilde{\bm{U}}_{(i,j),:}^{\star}(\widetilde{\bm{U}}^{\star\top}\widetilde{\bm{U}}^{\star})^{-1}
\]
as a sum of independent zero-mean random matrices in $\mathbb{R}^{d\times d}$.
By (\ref{eq:U-tilde-property}), it is straightforward to bound
\begin{align*}
B & \coloneqq\max_{1\leq i,j\leq d}\Big\|\sigma_{i,j,m}^{2}(p^{-1}\chi_{i,j,m}-1)(\widetilde{\bm{U}}_{(i,j),:}^{\star})^{\top}\widetilde{\bm{U}}_{(i,j),:}^{\star}\Big\|\\
 & \leq\frac{\sigma_{\max}^{2}}{p}\big\|\widetilde{\bm{U}}^{\star}\big\|_{2,\infty}^{2}\lesssim\frac{\sigma_{\max}^{2}}{p}\cdot\frac{\mu^{2}r}{d^{2}}\lambda_{\max}^{\star4/3};
\end{align*}
and
\begin{align*}
V & \coloneqq\Big\|\sum_{1\leq i,j\leq d}\sigma_{i,j,m}^{2}\mathbb{E}\big[(p^{-1}\chi_{i,j,m}-1)^{2}\big](\widetilde{\bm{U}}_{(i,j),:}^{\star})^{\top}\widetilde{\bm{U}}_{(i,j),:}^{\star}(\widetilde{\bm{U}}_{(i,j),:}^{\star})^{\top}\widetilde{\bm{U}}_{(i,j),:}^{\star}\Big\|\\
 & \leq\frac{\sigma_{\max}^{2}}{p}\big\|\widetilde{\bm{U}}^{\star}\big\|_{2,\infty}^{2}\big\|\widetilde{\bm{U}}^{\star\top}\widetilde{\bm{U}}^{\star}\big\|\lesssim\frac{\sigma_{\max}^{2}}{p}\cdot\frac{\mu^{2}r}{d^{2}}\lambda_{\max}^{\star4/3}\cdot\lambda_{\max}^{\star4/3}.
\end{align*}
Invoke the matrix Bernstein inequality to reveal: with probability
at least $1-O\left(d^{-11}\right)$,
\begin{align*}
\Big\|\sum_{1\leq i,j\leq d}\sigma_{i,j,m}^{2}(p^{-1}\chi_{i,j,m}-1)(\widetilde{\bm{U}}_{(i,j),:}^{\star})^{\top}\widetilde{\bm{U}}_{(i,j),:}^{\star}\Big\| & \lesssim B\log d+\sqrt{V\log d}\\
 & \lesssim\sigma_{\max}^{2}\lambda_{\max}^{\star4/3}\left\{ \frac{\mu^{2}r\log d}{d^{2}p}+\sqrt{\frac{\mu^{2}r\log d}{d^{2}p}}\right\} \\
 & \asymp\sigma_{\max}^{2}\lambda_{\max}^{\star4/3}\sqrt{\frac{\mu^{2}r\log d}{d^{2}p}},
\end{align*}
where the last line holds as long as $p\gtrsim\mu^{2}rd^{-2}\log d$.
Combined with (\ref{eq:U-true-tilde-spectrum}) and the condition
$\kappa\asymp1$, we conclude that
\[
\left\Vert \bm{S}_{m}^{\star}-\bm{\Sigma}_{m}^{\star}\right\Vert \lesssim\frac{\sigma_{\max}^{2}\lambda_{\max}^{\star4/3}}{p}\sqrt{\frac{\mu^{2}r\log d}{d^{2}p}}\big\|(\widetilde{\bm{U}}^{\star\top}\widetilde{\bm{U}}^{\star})^{-1}\big\|^{2}\lesssim\frac{\sigma_{\max}^{2}}{\lambda_{\min}^{\star4/3}p}\sqrt{\frac{\mu^{2}r\log d}{d^{2}p}}=o\left(\frac{\sigma_{\max}^{2}}{\lambda_{\min}^{\star4/3}p}\right),
\]
where the last step arises from the assumption that $p\gg\mu^{2}rd^{-2}\log^{2}d$.

\subsection{Proof of Lemma \ref{lemma:U-loss-dist-main-part-nonGaussian}}

\label{subsec:U-loss-dist-main-part-nonGaussian}

As before, let us use the notation $\bm{Z}$ and $\bm{W}_{0}:=\bm{X}-\bm{Z}$
as defined in (\ref{def:Z}) in Lemma \ref{lemma:U-loss-dist-main-part-Gaussian}.
It is easily seen that (\ref{eq:Z-row-Gaussian-diff-norm-part1})
continues to hold in the non-Gaussian noise case (using the same proof).
Therefore, it suffices to show that $\bm{Z}_{m,:}$ converges in distribution
to a Gaussian random vector $\bm{g}_{m}\sim\mathcal{N}(\bm{0},\bm{\Sigma}_{m}^{\star})$
in $\mathbb{R}^{r}$, towards which we resort to the Berry--Esseen-type
theorem in Appendix~\ref{subsec:Berry-Esseen-theorem}. In order
to do so, we need to upper bound the quantity $\rho$ defined in (\ref{def:Berry-Esseen-rho}),
which we proceed as follows
\begin{align*}
\rho & \lesssim\sum_{1\leq i,j\leq d}\mathbb{E}\left[\big\| p^{-1}E_{i,j,m}\chi_{i,j,m}\bm{\Sigma}_{m}^{\star-1/2}\widetilde{\bm{U}}_{(i,j),:}^{\star}\big(\widetilde{\bm{U}}^{\star\top}\widetilde{\bm{U}}^{\star}\big)^{-1}\big\|_{2}^{3}\right]\\
 & =\sum_{1\leq i,j\leq d}\frac{1}{p^{3}}\mathbb{E}\big[|E_{i,j,m}|^{3}\chi_{i,j,m}\big]\big\|\widetilde{\bm{U}}_{(i,j),:}^{\star}(\widetilde{\bm{U}}^{\star\top}\widetilde{\bm{U}}^{\star})^{-1}\bm{\Sigma}_{m}^{\star-1/2}\big\|_{2}^{3}\\
 & \overset{(\mathrm{i})}{\lesssim}\frac{\sigma_{\max}^{3}}{p^{2}}\sum_{1\leq i,j\leq d}\big\|\widetilde{\bm{U}}_{(i,j),:}^{\star}\big\|_{2}^{3}\big\|(\widetilde{\bm{U}}^{\star\top}\widetilde{\bm{U}}^{\star})^{-1}\big\|^{3}\big\|\bm{\Sigma}_{m}^{\star-1/2}\big\|^{3}\\
 & \lesssim\frac{\sigma_{\max}^{3}}{p^{2}}\big\|\widetilde{\bm{U}}^{\star}\big\|_{2,\infty}\big\|\widetilde{\bm{U}}^{\star}\big\|_{\mathrm{F}}^{2}\big\|(\widetilde{\bm{U}}^{\star\top}\widetilde{\bm{U}}^{\star})^{-1}\big\|^{3}\big\|\bm{\Sigma}_{m}^{\star-1/2}\big\|^{3}\\
 & \overset{(\mathrm{ii})}{\lesssim}\frac{\sigma_{\max}^{3}}{p^{2}}\cdot\frac{\mu\sqrt{r}}{d}\lambda_{\max}^{\star2/3}\cdot r\lambda_{\max}^{\star4/3}\cdot\frac{1}{\lambda_{\min}^{\star4}}\cdot\frac{\lambda_{\max}^{\star2}p^{3/2}}{\sigma_{\min}^{3}}\overset{(\mathrm{iii})}{\lesssim}\frac{\mu r^{3/2}}{d\sqrt{p}}.
\end{align*}
Here, (i) follows from the property of sub-Gaussian random variables,
(ii) arises from (\ref{eq:U-true-tilde-norm}), (\ref{eq:U-true-tilde-spectrum})
and (\ref{eq:cov-matrix-eigval}), whereas (iii) results from the
assumptions $\sigma_{\max}/\sigma_{\min}\asymp1$ and $\kappa\asymp1$.
Therefore, invoke the Berry-Esseen theorem in Appendix~\ref{subsec:Berry-Esseen-theorem}
to conclude that: for any convex set $\mathcal{A}\subset\mathbb{R}^{d}$,
\[
\left|\mathbb{P}\left\{ \bm{Z}_{m,:}\in\mathcal{A}\right\} -\mathbb{P}\big\{\bm{g}_{m}\in\mathcal{A}\big\}\right|\lesssim\frac{\mu r^{3/2}}{\sqrt{d^{3/2}p}},
\]
where $\bm{g}_{m}\sim\mathcal{N}(\bm{0},\bm{\Sigma}_{m}^{\star})$
is a Gaussian random vector in $\mathbb{R}^{r}$.

\subsection{Proof of Lemma \ref{lemma:U-loss-dist-W1}}

\label{subsec:U-loss-dist-W1}

Without loss of generality, assume that $\bm{\Pi}=\bm{I}_{r}$ to
simplify presentation. Fix an arbitrary $m\in\left[d\right]$. One
can use (\ref{eq:U-true-tilde-norm}) and (\ref{eq:U-true-tilde-spectrum})
to upper bound
\begin{align*}
\big\|\bm{U}_{m,:}^{\star}\big(\widetilde{\bm{U}}^{\star\top}\widetilde{\bm{U}}(\widetilde{\bm{U}}^{\top}\widetilde{\bm{U}})^{-1}-\bm{I}_{r}\big)\big\|_{2} & =\big\|\bm{U}_{m,:}^{\star}(\widetilde{\bm{U}}-\widetilde{\bm{U}}^{\star})^{\top}\widetilde{\bm{U}}(\widetilde{\bm{U}}^{\top}\widetilde{\bm{U}})^{-1}\big\|_{2}\leq\left\Vert \bm{U}^{\star}\right\Vert _{2,\infty}\big\|(\widetilde{\bm{U}}-\widetilde{\bm{U}}^{\star})^{\top}\widetilde{\bm{U}}\big\|_{2}\big\|(\widetilde{\bm{U}}^{\top}\widetilde{\bm{U}})^{-1}\big\|\\
 & \lesssim\frac{1}{\lambda_{\min}^{\star4/3}}\sqrt{\frac{\mu r}{d}}\,\lambda_{\max}^{\star1/3}\,\big\|(\widetilde{\bm{U}}-\widetilde{\bm{U}}^{\star})^{\top}\widetilde{\bm{U}}\big\|.
\end{align*}
It then suffices to bound the spectral norm of $(\widetilde{\bm{U}}-\widetilde{\bm{U}}^{\star})^{\top}\widetilde{\bm{U}}$.
For notational convenience, we define\begin{subequations}
\begin{align}
\bm{\Delta}_{s} & :=\bm{u}_{s}-\bm{u}_{s}^{\star},\qquad1\leq s\leq r;\label{eq:defn-Deltas-1234}\\
\bm{\Delta} & :=\bm{U}-\bm{U}^{\star}.\label{eq:defn-Delta-1234}
\end{align}
\end{subequations}Let us decompose
\begin{equation}
(\widetilde{\bm{U}}-\widetilde{\bm{U}}^{\star})^{\top}\widetilde{\bm{U}}=(\widetilde{\bm{U}}-\widetilde{\bm{U}}^{\star})^{\top}\widetilde{\bm{U}}^{\star}+(\widetilde{\bm{U}}-\widetilde{\bm{U}}^{\star})^{\top}(\widetilde{\bm{U}}-\widetilde{\bm{U}}^{\star}),\label{eq:U-tilde-loss-U-tilde-decomp}
\end{equation}
and look at these two matrices separately.

\medskip\noindent 1. We begin with the first term $(\widetilde{\bm{U}}-\widetilde{\bm{U}}^{\star})^{\top}\widetilde{\bm{U}}^{\star}$
in (\ref{eq:U-tilde-loss-U-tilde-decomp}), whose entries are given
by
\begin{align}
\big((\widetilde{\bm{U}}-\widetilde{\bm{U}}^{\star})^{\top}\widetilde{\bm{U}}^{\star}\big)_{i,j} & =\left\langle \bm{u}_{i},\bm{u}_{j}^{\star}\right\rangle ^{2}-\left\langle \bm{u}_{i}^{\star},\bm{u}_{j}^{\star}\right\rangle ^{2}=\left\langle \bm{u}_{i}^{\star}+\bm{\Delta}_{i},\bm{u}_{j}^{\star}\right\rangle ^{2}-\left\langle \bm{u}_{i}^{\star},\bm{u}_{j}^{\star}\right\rangle ^{2}\nonumber \\
 & =2\left\langle \bm{u}_{i}^{\star},\bm{u}_{j}^{\star}\right\rangle \left\langle \bm{\Delta}_{i},\bm{u}_{j}^{\star}\right\rangle +\left\langle \bm{\Delta}_{i},\bm{u}_{j}^{\star}\right\rangle ^{2}\label{eq:U-tilde-loss-U-entry}
\end{align}
for all $1\leq i,j\leq r$. Here, we have used the fact that $\left\langle \bm{a}^{\otimes2},\bm{b}^{\otimes2}\right\rangle =\left\langle \bm{a},\bm{b}\right\rangle ^{2}$
for any $\bm{a},\bm{b}\in\mathbb{R}^{d}$. Therefore, one can express
\begin{equation}
(\widetilde{\bm{U}}-\widetilde{\bm{U}}^{\star})^{\top}\widetilde{\bm{U}}^{\star}=2\,(\bm{U}^{\star\top}\bm{U}^{\star})\odot(\bm{\Delta}^{\top}\bm{U}^{\star})+(\bm{\Delta}^{\top}\bm{U}^{\star})\odot(\bm{\Delta}^{\top}\bm{U}^{\star}),\label{eq:U-tilde-loss-U-true-tilde}
\end{equation}
where we recall that $\odot$ is the Hadamard (entrywise) product.
In the sequel, we shall treat these two terms individually.
\begin{itemize}
\item With regards to $(\bm{\Delta}^{\top}\bm{U}^{\star})\odot(\bm{\Delta}^{\top}\bm{U}^{\star})$,
we can simply bound
\begin{align}
\left\Vert (\bm{\Delta}^{\top}\bm{U}^{\star})\odot(\bm{\Delta}^{\top}\bm{U}^{\star})\right\Vert  & \leq\left\Vert (\bm{\Delta}^{\top}\bm{U}^{\star})\odot(\bm{\Delta}^{\top}\bm{U}^{\star})\right\Vert _{\mathrm{F}}\overset{(\mathrm{i})}{\leq}\left\Vert \bm{\Delta}^{\top}\bm{U}^{\star}\right\Vert _{\infty}\left\Vert \bm{\Delta}^{\top}\bm{U}^{\star}\right\Vert _{\mathrm{F}}\nonumber \\
 & \overset{(\mathrm{ii})}{\leq}\max_{1\leq i\leq r}\left\Vert \bm{\Delta}_{i}\right\Vert _{2}\max_{1\leq i\leq r}\left\Vert \bm{u}_{i}^{\star}\right\Vert _{2}\left\Vert \bm{\Delta}\right\Vert _{\mathrm{F}}\left\Vert \bm{U}^{\star}\right\Vert \nonumber \\
 & \overset{(\mathrm{iii})}{\lesssim}\frac{\sigma_{\max}}{\lambda_{\min}^{\star}}\sqrt{\frac{rd\log d}{p}}\,\lambda_{\max}^{\star1/3}\cdot\lambda_{\max}^{\star1/3}\cdot\frac{\sigma_{\max}}{\lambda_{\min}^{\star}}\sqrt{\frac{rd\log d}{p}}\,\lambda_{\max}^{\star1/3}\cdot\lambda_{\max}^{\star1/3}\nonumber \\
 & \overset{(\mathrm{iv})}{\lesssim}\frac{\sigma_{\max}^{2}}{\lambda_{\min}^{\star2/3}}\frac{rd\log d}{p}.\label{eq:Delta-t-U-true-odot-op-UB}
\end{align}
Here, (i) is due to $\|\bm{A}\odot\bm{A}\|_{\mathrm{F}}^{2}=\sum_{i,j}A_{i.j}^{4}\leq\max_{i,j}A_{i,j}^{2}\sum_{i,j}A_{i,j}^{2}\leq\|\bm{A}\|_{\infty}^{2}\|\bm{A}\|_{\mathrm{F}}^{2}$
for any matrix $\bm{A}$; (ii) arises from the inequality that $\|\bm{A}\bm{B}\|_{\mathrm{F}}\leq\|\bm{A}\|\|\bm{B}\|_{\mathrm{F}}$
for any matrices $\bm{A},\bm{B}$; (iii) uses (\ref{eq:U-T-loss-UB});
and (iv) arises from the condition that $\kappa\asymp1$.
\item Bounding the term $(\bm{U}^{\star\top}\bm{U}^{\star})\odot(\bm{\Delta}^{\top}\bm{U}^{\star})$
turns out to be more challenging. Towards this, we shall look at its
diagonal and off-diagonal parts separately. For the off-diagonal part,
by the incoherence condition (\ref{assumption:u-inner-prod}), one
can bound 
\begin{align*}
\big\|\mathcal{P}_{\mathsf{off}\text{-}\mathsf{diag}}\big((\bm{U}^{\star\top}\bm{U}^{\star})\odot(\bm{\Delta}^{\top}\bm{U}^{\star})\big)\big\| & \leq\big\|\mathcal{P}_{\mathsf{off}\text{-}\mathsf{diag}}\big((\bm{U}^{\star\top}\bm{U}^{\star})\odot(\bm{\Delta}^{\top}\bm{U}^{\star})\big)\big\|_{\mathrm{F}}\\
 & \leq\left\Vert \mathcal{P}_{\mathsf{off}\text{-}\mathsf{diag}}\left(\bm{U}^{\star\top}\bm{U}^{\star}\right)\right\Vert _{\infty}\left\Vert \bm{\Delta}^{\top}\bm{U}^{\star}\right\Vert _{\mathrm{F}}\\
 & \leq\max_{1\leq i\neq j\leq r}\left|\big\langle\bm{u}_{i}^{\star},\bm{u}_{j}^{\star}\big\rangle\right|\left\Vert \bm{\Delta}\right\Vert _{\mathrm{F}}\left\Vert \bm{U}^{\star}\right\Vert \\
 & \lesssim\sqrt{\frac{\mu}{d}}\,\lambda_{\max}^{\star2/3}\cdot\frac{\sigma_{\max}}{\lambda_{\min}^{\star}}\sqrt{\frac{rd\log d}{p}}\,\lambda_{\max}^{\star1/3}\cdot\lambda_{\max}^{\star1/3}\\
 & \lesssim\sigma_{\max}\lambda_{\max}^{\star1/3}\sqrt{\frac{\mu r\log d}{p}},
\end{align*}
where we also use (\ref{eq:U-T-loss-UB}) and $\kappa\asymp1$. Turning
to the diagonal part, one observes that
\[
\big\|\mathcal{P}_{\mathsf{diag}}\big((\bm{U}^{\star\top}\bm{U}^{\star})\odot(\bm{\Delta}^{\top}\bm{U}^{\star})\big)\big\|=\max_{1\leq i\leq r}\left\Vert \bm{u}_{i}\right\Vert _{2}^{2}\left|\left\langle \bm{\Delta}_{i},\bm{u}_{i}^{\star}\right\rangle \right|\leq\max_{1\leq i\leq r}\left|\left\langle \bm{\Delta}_{i},\bm{u}_{i}^{\star}\right\rangle \right|\lambda_{\max}^{\star2/3}.
\]
As result, the key step lies in upper bounding $\max_{1\leq i\leq r}\left|\left\langle \bm{\Delta}_{i},\bm{u}_{i}^{\star}\right\rangle \right|$,
which will be accomplished in the lemma below. The proof is deferred
to the end of this section. \begin{lemma}\label{lemma:u-loss-u-true-inner-product-UB}Instate
the assumptions of Lemma~\ref{lemma:U-loss-dist-W1}. With probability
at least $1-O\left(d^{-10}\right)$, one has
\begin{align}
\max_{1\leq i\leq d}\left|\left\langle \bm{\Delta}_{i},\bm{u}_{i}^{\star}\right\rangle \right| & \lesssim\frac{\sigma_{\max}}{\lambda_{\min}^{\star1/3}}\sqrt{\frac{d}{p}}\frac{\zeta}{\sqrt{\mu r}},\label{def:u-loss-u-inner-prod-UB}
\end{align}
where we recall the definition of $\zeta$ in (\ref{def:zeta}) and
the definition of $\bm{\Delta}_{i}$ in (\ref{eq:defn-Deltas-1234}).
\end{lemma}With the above results in place, we conclude that
\begin{align*}
\left\Vert (\bm{U}^{\star\top}\bm{U}^{\star})\odot(\bm{\Delta}^{\top}\bm{U}^{\star})\right\Vert  & \leq\big\|\mathcal{P}_{\mathsf{off}\text{-}\mathsf{diag}}\big((\bm{U}^{\star\top}\bm{U}^{\star})\odot(\bm{\Delta}^{\top}\bm{U}^{\star})\big)\big\|+\big\|\mathcal{P}_{\mathsf{diag}}\big((\bm{U}^{\star\top}\bm{U}^{\star})\odot(\bm{\Delta}^{\top}\bm{U}^{\star})\big)\big\|\\
 & \lesssim\sigma_{\max}\lambda_{\max}^{\star1/3}\sqrt{\frac{\mu r\log d}{p}}+\frac{\sigma_{\max}}{\lambda_{\min}^{\star1/3}}\sqrt{\frac{d}{p}}\frac{\zeta}{\sqrt{\mu r}}\lambda_{\max}^{\star2/3}\\
 & \asymp\sigma_{\max}\lambda_{\max}^{\star1/3}\sqrt{\frac{d}{p}}\frac{\zeta}{\sqrt{\mu r}},
\end{align*}
where we use the condition $\kappa\asymp1$, as well as the definition
of $\zeta$ in (\ref{def:zeta}) (which indicates that $\zeta\gtrsim\sqrt{\frac{\mu^{2}r^{2}\log d}{d}}$)
in the last step.
\item Combining the bounds above demonstrates that
\begin{align}
\big\|(\widetilde{\bm{U}}-\widetilde{\bm{U}}^{\star})^{\top}\widetilde{\bm{U}}^{\star}\big\| & \lesssim\left\Vert (\bm{U}^{\star\top}\bm{U}^{\star})\odot(\bm{\Delta}^{\top}\bm{U}^{\star})\right\Vert +\left\Vert (\bm{\Delta}^{\top}\bm{U}^{\star})\odot(\bm{\Delta}^{\top}\bm{U}^{\star})\right\Vert \nonumber \\
 & \lesssim\sigma_{\max}\lambda_{\max}^{\star1/3}\sqrt{\frac{d}{p}}\frac{\zeta}{\sqrt{\mu r}}+\frac{\sigma_{\max}^{2}}{\lambda_{\min}^{\star2/3}}\frac{rd\log d}{p}\asymp\sigma_{\max}\lambda_{\max}^{\star1/3}\sqrt{\frac{d}{p}}\frac{\zeta}{\sqrt{\mu r}},\label{eq:U-tilde-loss-U-tilde-true-op-UB}
\end{align}
where we have used the fact that $\zeta\gtrsim\frac{\sigma_{\max}}{\lambda_{\min}^{\star}}\sqrt{\frac{\mu r^{3}d\log^{2}d}{p}}.$
In particular, we obtain the following upper bound for the spectral
norm of $\widetilde{\bm{U}}-\widetilde{\bm{U}}^{\star}$:
\begin{equation}
\big\|\widetilde{\bm{U}}-\widetilde{\bm{U}}^{\star}\big\|=\big\|(\widetilde{\bm{U}}-\widetilde{\bm{U}}^{\star})^{\top}\widetilde{\bm{U}}^{\star}\widetilde{\bm{U}}^{\star-1}\big\|\leq\big\|\widetilde{\bm{U}}^{\star-1}\big\|\big\|(\widetilde{\bm{U}}-\widetilde{\bm{U}}^{\star})^{\top}\widetilde{\bm{U}}^{\star}\big\|\lesssim\frac{\sigma_{\max}}{\lambda_{\min}^{\star}}\sqrt{\frac{d}{p}}\,\lambda_{\max}^{\star2/3},\label{eq:U-tilde-op-loss}
\end{equation}
where we use the conditions that $p\gg\mu^{3}r^{3}d^{-3/2}\log^{3}d$,
$\sigma_{\max}/\lambda_{\min}^{\star}\ll\sqrt{p/(r^{2}d^{3/2})}$
and $r\ll d/(\mu\log d)$.
\end{itemize}
2. Turning to the second term $(\widetilde{\bm{U}}-\widetilde{\bm{U}}^{\star})^{\top}(\widetilde{\bm{U}}-\widetilde{\bm{U}}^{\star})$
in (\ref{eq:U-tilde-loss-U-tilde-decomp}), one can use (\ref{eq:U-tilde-op-loss})
and $\kappa\asymp1$ to upper bound
\begin{align}
\big\|(\widetilde{\bm{U}}-\widetilde{\bm{U}}^{\star})^{\top}(\widetilde{\bm{U}}-\widetilde{\bm{U}}^{\star})\big\| & \leq\left(\frac{\sigma_{\max}}{\lambda_{\min}^{\star}}\sqrt{\frac{d}{p}}\,\lambda_{\max}^{\star2/3}\right)^{2}.\label{eq:U-tilde-square-loss-op-UB}
\end{align}

\medskip\noindent 3. Taking (\ref{eq:Delta-t-U-true-odot-op-UB})
and (\ref{eq:U-tilde-loss-U-tilde-true-op-UB}) together leads to
\begin{align*}
\big\|(\widetilde{\bm{U}}-\widetilde{\bm{U}}^{\star})^{\top}\widetilde{\bm{U}}\big\| & \leq2\,\big\|(\widetilde{\bm{U}}-\widetilde{\bm{U}}^{\star})^{\top}\widetilde{\bm{U}}^{\star}\big\|+\big\|(\widetilde{\bm{U}}-\widetilde{\bm{U}}^{\star})^{\top}\big(\widetilde{\bm{U}}-\widetilde{\bm{U}}^{\star}\big)\big\|\\
 & \lesssim\sigma_{\max}\lambda_{\max}^{\star1/3}\sqrt{\frac{d}{p}}\frac{\zeta}{\sqrt{\mu r}}+\left(\frac{\sigma_{\max}}{\lambda_{\min}^{\star}}\sqrt{\frac{d}{p}}\,\lambda_{\max}^{\star2/3}\right)^{2}\asymp\sigma_{\max}\lambda_{\max}^{\star1/3}\sqrt{\frac{d}{p}}\frac{\zeta}{\sqrt{\mu r}},
\end{align*}
where the last step arises from the definition of $\zeta$ (cf.~(\ref{def:zeta}))
that $\zeta\gtrsim\frac{\sigma_{\max}}{\lambda_{\min}^{\star}}\sqrt{\frac{\mu rd}{p}}$.
Therefore, one can use the condition $\kappa\asymp1$ to establish
that
\[
\big\|\bm{U}_{m,:}^{\star}\big(\widetilde{\bm{U}}^{\star\top}\widetilde{\bm{U}}\big(\widetilde{\bm{U}}^{\top}\widetilde{\bm{U}}\big)^{-1}-\bm{I}_{r}\big)\big\|_{2}\lesssim\frac{1}{\lambda_{\min}^{\star4/3}}\sqrt{\frac{\mu r}{d}}\,\lambda_{\max}^{\star1/3}\cdot\sigma_{\max}\lambda_{\max}^{\star1/3}\sqrt{\frac{d}{p}}\frac{\zeta}{\sqrt{\mu r}}\lesssim\frac{\sigma_{\max}}{\lambda_{\min}^{\star2/3}\sqrt{p}}\zeta
\]
as claimed.

\subsubsection{Proof of Lemma \ref{lemma:u-loss-u-true-inner-product-UB}}

Fix any $1\leq i\leq r$. Recall the decomposition in (\ref{eq:U-loss-decomp})
and (\ref{def:W1-W4}) as well as the assumption $\bm{\Pi}=\bm{I}_{r}$
(without loss of generality). Left multiplying it by $\bm{U}^{\star}$
and right multiplying it by $\widetilde{\bm{U}}^{\top}\widetilde{\bm{U}}$,
we arrive at
\begin{align}
\bm{U}^{\star\top}(\bm{U}-\bm{U}^{\star})\widetilde{\bm{U}}^{\top}\widetilde{\bm{U}} & =-\bm{U}^{\star\top}\bm{U}^{\star}(\widetilde{\bm{U}}-\widetilde{\bm{U}}^{\star})^{\top}\widetilde{\bm{U}}^{\star}+\bm{B},\label{eq:beta4-decomp}
\end{align}
where we use the following fact:
\begin{align*}
\bm{U}^{\star\top}\bm{U}^{\star}\big(\widetilde{\bm{U}}^{\star\top}\widetilde{\bm{U}}(\widetilde{\bm{U}}^{\top}\widetilde{\bm{U}})^{-1}-\bm{I}_{r}\big) & =\bm{U}^{\star\top}\bm{U}^{\star}\big(\widetilde{\bm{U}}^{\star\top}\widetilde{\bm{U}}(\widetilde{\bm{U}}^{\top}\widetilde{\bm{U}})^{-1}-\widetilde{\bm{U}}^{\top}\widetilde{\bm{U}}(\widetilde{\bm{U}}^{\top}\widetilde{\bm{U}})^{-1}\big)\\
 & =-\bm{U}^{\star\top}\bm{U}^{\star}(\widetilde{\bm{U}}-\widetilde{\bm{U}}^{\star})^{\top}\widetilde{\bm{U}}(\widetilde{\bm{U}}^{\top}\widetilde{\bm{U}})^{-1},
\end{align*}
and $\bm{B}$ is given by
\begin{align}
\bm{B} & :=\underbrace{-\bm{U}^{\star\top}\bm{U}^{\star}(\widetilde{\bm{U}}-\widetilde{\bm{U}}^{\star})^{\top}(\widetilde{\bm{U}}-\widetilde{\bm{U}}^{\star})}_{=:\,\bm{B}_{1}}+\underbrace{\bm{U}^{\star\top}\mathsf{unfold}\left(p^{-1}\mathcal{P}_{\Omega}(\bm{E})\right)\widetilde{\bm{U}}}_{=:\,\bm{B}_{2}}\nonumber \\
 & \,\quad+\underbrace{\bm{U}^{\star\top}\mathsf{unfold}\left((\mathcal{I}-p^{-1}\mathcal{P}_{\Omega})\left(\bm{T}-\bm{T}^{\star}\right)\right)\widetilde{\bm{U}}}_{=:\,\bm{B}_{3}}+\underbrace{\bm{U}^{\star\top}\nabla g(\bm{U})}_{=:\,\bm{B}_{4}}.\label{eq:B-decomp}
\end{align}
One can compute the $(i,i)$-th entry of $\bm{U}^{\star\top}(\bm{U}-\bm{U}^{\star})\widetilde{\bm{U}}^{\top}\widetilde{\bm{U}}$
on the left-hand side of (\ref{eq:beta4-decomp}) as follows
\begin{align*}
\big(\bm{U}^{\star\top}(\bm{U}-\bm{U}^{\star})\widetilde{\bm{U}}^{\top}\widetilde{\bm{U}}\big)_{i,i} & =\bm{u}_{i}^{\star\top}(\bm{U}-\bm{U}^{\star})\widetilde{\bm{U}}^{\top}\widetilde{\bm{U}}_{:,i}=\sum_{1\leq s\leq r}\left\langle \bm{u}_{i}^{\star},\bm{\Delta}_{s}\right\rangle \left\langle \bm{u}_{s},\bm{u}_{i}\right\rangle ^{2},
\end{align*}
where we recall that $\bm{\Delta}_{s}:=\bm{u}_{s}-\bm{u}_{s}^{\star}$.
In view of (\ref{eq:U-tilde-loss-U-entry}), the $(i,i)$-th entry
of $\bm{U}^{\star\top}\bm{U}^{\star}(\widetilde{\bm{U}}-\widetilde{\bm{U}}^{\star})^{\top}\widetilde{\bm{U}}^{\star}$
on the right-hand side of (\ref{eq:beta4-decomp}) is given by
\begin{align*}
\big(\bm{U}^{\star\top}\bm{U}^{\star}(\widetilde{\bm{U}}-\widetilde{\bm{U}}^{\star})^{\top}\widetilde{\bm{U}}^{\star}\big)_{i,i} & =(\bm{U}^{\star\top}\bm{U}^{\star})_{i,:}(\widetilde{\bm{U}}-\widetilde{\bm{U}}^{\star})^{\top}\widetilde{\bm{U}}_{:,i}^{\star}\\
 & =\sum_{1\leq s\leq r}\left\langle \bm{u}_{s}^{\star},\bm{u}_{i}^{\star}\right\rangle \big(2\left\langle \bm{u}_{s}^{\star},\bm{u}_{i}^{\star}\right\rangle \left\langle \bm{\Delta}_{s},\bm{u}_{i}^{\star}\right\rangle +\left\langle \bm{\Delta}_{s},\bm{u}_{i}^{\star}\right\rangle ^{2}\big),
\end{align*}
Therefore, substituting these into (\ref{eq:beta4-decomp}) and rearranging
terms lead to
\begin{align}
 & \big(\left\Vert \bm{u}_{i}\right\Vert _{2}^{4}+2\left\Vert \bm{u}_{i}^{\star}\right\Vert _{2}^{4}\big)\left\langle \bm{\Delta}_{i},\bm{u}_{i}^{\star}\right\rangle \nonumber \\
 & \quad=-\sum_{s:s\neq i}\big(\left\langle \bm{u}_{s},\bm{u}_{i}\right\rangle ^{2}+2\left\langle \bm{u}_{s}^{\star},\bm{u}_{i}^{\star}\right\rangle ^{2}\big)\left\langle \bm{\Delta}_{s},\bm{u}_{i}^{\star}\right\rangle -\sum_{1\leq s\leq r}\left\langle \bm{u}_{s}^{\star},\bm{u}_{i}^{\star}\right\rangle \left\langle \bm{\Delta}_{s},\bm{u}_{i}^{\star}\right\rangle ^{2}+B_{i,i}.\label{eq:identity-12345}
\end{align}
It then suffices to control the quantities on the right-hand side
of (\ref{eq:identity-12345}). 

For the first term of (\ref{eq:identity-12345}), apply the Cauchy-Schwartz
inequality to yield
\begin{align*}
\Big|\sum_{s:s\neq i}\big(\left\langle \bm{u}_{s},\bm{u}_{i}\right\rangle ^{2}+2\left\langle \bm{u}_{s}^{\star},\bm{u}_{i}^{\star}\right\rangle ^{2}\big)\left\langle \bm{\Delta}_{s},\bm{u}_{i}^{\star}\right\rangle \Big| & \leq\max_{s:s\neq i}\big\{\langle\bm{u}_{s},\bm{u}_{i}\rangle^{2}+\langle\bm{u}_{s}^{\star},\bm{u}_{i}^{\star}\rangle^{2}\big\}\left\Vert \bm{u}_{i}^{\star}\right\Vert _{2}\sum_{s\neq i}\left\Vert \bm{\Delta}_{s}\right\Vert _{2}\\
 & \leq\max_{s:s\neq i}\big\{\langle\bm{u}_{s},\bm{u}_{i}\rangle^{2}+\langle\bm{u}_{s}^{\star},\bm{u}_{i}^{\star}\rangle^{2}\big\}\left\Vert \bm{u}_{i}^{\star}\right\Vert _{2}\sqrt{r}\left\Vert \bm{U}-\bm{U}^{\star}\right\Vert _{\mathrm{F}}\\
 & \overset{(\mathrm{i})}{\lesssim}\left\{ \frac{\mu}{d}+\frac{\sigma_{\max}^{2}}{\lambda_{\min}^{\star2}}\frac{rd\log d}{p}\right\} \lambda_{\max}^{\star4/3}\cdot\lambda_{\max}^{\star1/3}\cdot\sqrt{r}\cdot\frac{\sigma_{\max}}{\lambda_{\min}^{\star}}\sqrt{\frac{rd\log d}{p}}\lambda_{\max}^{\star1/3}\\
 & \overset{(\mathrm{ii})}{\lesssim}\sigma_{\max}\lambda_{\max}^{\star}\sqrt{\frac{\mu^{2}r^{2}\log d}{dp}}+\frac{\sigma_{\max}^{2}rd\log d}{p}.
\end{align*}
Here, (i) is due to the incoherence condition (\ref{assumption:u-inf-norm})
as well as (\ref{eq:U-loss-fro}) and (\ref{eq:u-inner-prod}), whereas
(ii) arises from the noise condition $\sigma_{\max}/\lambda_{\min}^{\star}\ll\sqrt{p/(r^{2}d\log d)}$.
Turning to the second term of (\ref{eq:identity-12345}), the Cauchy-Schwartz
inequality tells us that
\begin{align*}
\Big|\sum_{1\leq s\leq r}\left\langle \bm{u}_{s}^{\star},\bm{u}_{i}^{\star}\right\rangle \left\langle \bm{\Delta}_{s},\bm{u}_{i}^{\star}\right\rangle ^{2}\Big| & \leq\left\Vert \bm{u}_{i}^{\star}\right\Vert _{2}^{2}\left\langle \bm{\Delta}_{i},\bm{u}_{i}^{\star}\right\rangle ^{2}+\max_{s:s\neq i}\left|\left\langle \bm{u}_{s}^{\star},\bm{u}_{i}^{\star}\right\rangle \right|\left\Vert \bm{u}_{i}^{\star}\right\Vert _{2}^{2}\sum_{s:s\neq i}\left\Vert \bm{\Delta}_{i}\right\Vert _{2}^{2}\\
 & \leq\left\Vert \bm{\Delta}_{i}\right\Vert _{2}\left\Vert \bm{u}_{i}^{\star}\right\Vert _{2}^{3}\left|\left\langle \bm{\Delta}_{i},\bm{u}_{i}^{\star}\right\rangle \right|+\max_{1\leq i\leq r}\left\Vert \bm{u}_{i}^{\star}\right\Vert _{2}^{4}\left\Vert \bm{U}-\bm{U}^{\star}\right\Vert _{\mathrm{F}}^{2}\\
 & \lesssim o\left(1\right)\left\Vert \bm{u}_{i}^{\star}\right\Vert _{2}^{4}\left|\left\langle \bm{\Delta}_{i},\bm{u}_{i}^{\star}\right\rangle \right|+\lambda_{\max}^{\star4/3}\cdot\frac{\sigma_{\max}^{2}}{\lambda_{\min}^{\star2}}\frac{rd\log d}{p}\lambda_{\max}^{\star2/3},
\end{align*}
where we use (\ref{eq:U-loss-fro}) and (\ref{eq:u-loss-u-relation}).
Substituting these into (\ref{eq:identity-12345}) and using (\ref{eq:u-norm})
and $\kappa\asymp1$, we arrive at
\begin{equation}
\lambda_{\min}^{\star4/3}\left|\left\langle \bm{\Delta}_{i},\bm{u}_{i}^{\star}\right\rangle \right|\lesssim\sigma_{\max}\lambda_{\max}^{\star}\sqrt{\frac{\mu^{2}r^{2}\log d}{dp}}+\frac{\sigma_{\max}^{2}rd\log d}{p}+\left|B_{i,i}\right|.\label{eq:u-loss-u-true-inner-product-UB-temp}
\end{equation}

It remains to bound $|B_{i,i}|$. Towards this end, we claim for the
moment that
\begin{equation}
\left|B_{i,i}\right|\lesssim\sigma_{\max}\lambda_{\max}^{\star}\sqrt{\frac{d}{p}}\frac{\zeta}{\sqrt{\mu r}},\label{claim:B-diag-UB}
\end{equation}
where $\zeta$ is defined in (\ref{def:zeta}). If this were true,
then one could use (\ref{eq:u-loss-u-true-inner-product-UB-temp})
and the condition $\kappa\asymp1$ to obtain the advertised bound
\begin{align*}
\left|\left\langle \bm{\Delta}_{i},\bm{u}_{i}^{\star}\right\rangle \right| & \lesssim\frac{\sigma_{\max}}{\lambda_{\min}^{\star1/3}}\sqrt{\frac{d}{p}}\Bigg\{\frac{\mu r\sqrt{\log d}}{d}+\frac{\sigma_{\max}}{\lambda_{\min}^{\star}}\sqrt{\frac{r^{2}d\log^{2}d}{p}}+\frac{\zeta}{\sqrt{\mu r}}\Bigg\}\asymp\frac{\sigma_{\max}}{\lambda_{\min}^{\star1/3}}\sqrt{\frac{d}{p}}\frac{\zeta}{\sqrt{\mu r}},
\end{align*}
where the last step holds due to the fact that $\zeta\gtrsim\frac{\mu^{3/2}r^{3/2}\sqrt{\log d}}{d}+\frac{\sigma_{\max}}{\lambda_{\min}^{\star}}\sqrt{\frac{\mu r^{3}d\log^{2}d}{p}}$.

The remainder of the proof is thus devoted to proving the claim (\ref{claim:B-diag-UB}).
Recalling the decomposition in (\ref{eq:B-decomp}), we shall control
$(\bm{B}_{j})_{i,i},\,1\leq j\leq4$ separately. 
\begin{itemize}
\item For $\bm{B}_{1}$, using (\ref{eq:U-true-spectrum}), (\ref{eq:U-tilde-loss-fro})
and $\kappa\asymp1$, we can simply upper bound
\begin{align}
\big|\left(\bm{B}_{1}\right)_{i,i}\big| & \leq\big\|\bm{U}^{\star\top}\bm{U}^{\star}(\widetilde{\bm{U}}-\widetilde{\bm{U}}^{\star})^{\top}(\widetilde{\bm{U}}-\widetilde{\bm{U}}^{\star})\big\|\leq\big\|\bm{U}^{\star}\big\|^{2}\big\|\widetilde{\bm{U}}-\widetilde{\bm{U}}^{\star}\big\|^{2}\leq\big\|\bm{U}^{\star}\big\|^{2}\big\|\widetilde{\bm{U}}-\widetilde{\bm{U}}^{\star}\big\|_{\mathrm{F}}^{2}\nonumber \\
 & \lesssim\left(\lambda_{\max}^{\star1/3}\cdot\frac{\sigma_{\max}}{\lambda_{\min}^{\star1/3}}\sqrt{\frac{rd\log d}{p}}\right)^{2}\lesssim\frac{\sigma_{\max}^{2}rd\log d}{p}.\label{eq:B1-m-entry-UB}
\end{align}
\item Regarding $\bm{B}_{2}$, we can decompose
\begin{align*}
\left(\bm{B}_{2}\right)_{i,i} & =\left\langle p^{-1}\mathcal{P}_{\Omega}(\bm{E}),\bm{u}_{i}^{\star}\otimes\bm{u}_{i}\otimes\bm{u}_{i}\right\rangle \\
 & =\underbrace{\left\langle p^{-1}\mathcal{P}_{\Omega}(\bm{E}),\bm{u}_{i}^{\star\otimes3}\right\rangle }_{=:\,\gamma_{1}}+\underbrace{\left\langle p^{-1}\mathcal{P}_{\Omega}(\bm{E}),\bm{u}_{i}^{\star}\otimes\bm{\Delta}_{i}\otimes\bm{u}_{i}\right\rangle }_{=:\,\gamma_{2}}+\underbrace{\left\langle p^{-1}\mathcal{P}_{\Omega}(\bm{E}),\bm{u}_{i}^{\star}\otimes\bm{u}_{i}^{\star}\otimes\bm{\Delta}_{i}\right\rangle }_{=:\,\gamma_{3}},
\end{align*}
leaving us with three terms to control. 
\begin{itemize}
\item For $\gamma_{1}$, observe that $\gamma_{1}=\sum_{1\leq j,k,l\leq d}p^{-1}E_{j,k,l}\chi_{jkl}u_{i,j}^{\star}u_{i,k}^{\star}u_{i,l}^{\star}$
is a sum of independent zero-mean random variables. Applying the Bernstein
inequality shows that with probability at least $1-O\left(d^{-20}\right)$,
\begin{align*}
\left|\gamma_{1}\right| & \lesssim\sigma_{\max}\left\{ \frac{\log^{2}d}{p}\left\Vert \bm{u}_{i}^{\star}\right\Vert _{\infty}^{3}+\sqrt{\frac{\log d}{p}}\left\Vert \bm{u}_{i}^{\star}\right\Vert _{2}^{3}\right\} \lesssim\sigma_{\max}\lambda_{\max}^{\star}\left\{ \frac{\mu^{3/2}\log^{2}d}{d^{3/2}p}+\sqrt{\frac{\log d}{p}}\right\} \\
 & \asymp\sigma_{\max}\lambda_{\max}^{\star}\sqrt{\frac{\log d}{p}},
\end{align*}
where we use the incoherence condition (\ref{assumption:u-inf-norm}),
and the last step holds true as long as $p\gtrsim\mu^{3}d^{-3}\log^{3}d$.
\item Regarding $\gamma_{2}$ and $\gamma_{3}$, we know from \cite[Lemma~D.4]{cai2019nonconvex}
that with probability at least $1-O\left(d^{-20}\right)$,
\begin{align*}
\left\Vert p^{-1}\mathcal{P}_{\Omega}(\bm{E})\times_{1}\bm{u}_{i}^{\star}\right\Vert  & \lesssim\sigma_{\max}\left\{ \frac{\log^{5/2}d}{p}\left\Vert \bm{u}_{i}^{\star}\right\Vert _{\infty}+\sqrt{\frac{d\log d}{p}}\left\Vert \bm{u}_{i}^{\star}\right\Vert _{2}\right\} \lesssim\sigma_{\max}\lambda_{\max}^{\star1/3}\left\{ \frac{\log^{5/2}d}{p}\sqrt{\frac{\mu}{d}}+\sqrt{\frac{d\log d}{p}}\right\} \\
 & \asymp\sigma_{\max}\lambda_{\max}^{\star1/3}\sqrt{\frac{d\log d}{p}},
\end{align*}
where we use the incoherence condition (\ref{assumption:u-inf-norm})
and the assumption that $p\gtrsim\mu d^{-2}\log^{4}d$. Consequently,
we can use (\ref{eq:u-norm}) and (\ref{eq:U-loss-fro}) to obtain
\begin{align*}
\left|\gamma_{2}\right|+\left|\gamma_{3}\right| & \leq\left\Vert p^{-1}\mathcal{P}_{\Omega}(\bm{E})\times_{1}\bm{u}_{i}^{\star}\right\Vert \big(\left\Vert \bm{u}_{i}\right\Vert _{2}+\left\Vert \bm{u}_{i}^{\star}\right\Vert _{2}\big)\left\Vert \bm{\Delta}_{i}\right\Vert _{2}\\
 & \lesssim\sigma_{\max}\lambda_{\max}^{\star1/3}\sqrt{\frac{d\log d}{p}}\cdot\lambda_{\max}^{\star1/3}\cdot\frac{\sigma_{\max}}{\lambda_{\min}^{\star}}\sqrt{\frac{rd\log d}{p}}\lambda_{\max}^{\star1/3}.
\end{align*}
\item Combining the bounds above and using the condition $\kappa\asymp1$,
we find that
\begin{equation}
\big|\left(\bm{B}_{2}\right)_{i,i}\big|\lesssim\sigma_{\max}\lambda_{\max}^{\star}\sqrt{\frac{\log d}{p}}+\frac{\sigma_{\max}^{2}\sqrt{r}\,d\log d}{p}.\label{eq:B2-m-entry-UB}
\end{equation}
\end{itemize}
\item Turning to $\bm{B}_{3}$, we can upper bound
\begin{align*}
\big|\left(\bm{B}_{3}\right)_{i,i}\big| & =\left|\big\langle(\mathcal{I}-p^{-1}\mathcal{P}_{\Omega})\left(\bm{T}-\bm{T}^{\star}\right),\bm{u}_{i}^{\star}\otimes\bm{u}_{i}^{\otimes2}\big\rangle\right|\leq\left\Vert (p^{-1}\mathcal{P}_{\Omega}-\mathcal{I})\left(\bm{T}-\bm{T}^{\star}\right)\right\Vert \left\Vert \bm{u}_{i}^{\star}\right\Vert _{2}\left\Vert \bm{u}_{i}\right\Vert _{2}^{2}.
\end{align*}
Hence, it boils down to upper bounding the spectral norm of $(p^{-1}\mathcal{P}_{\Omega}-\mathcal{I})\left(\bm{T}-\bm{T}^{\star}\right)$.
Towards this, we decompose
\[
\bm{T}-\bm{T}^{\star}=\sum_{1\leq s\leq r}\bm{\Delta}_{s}\otimes\bm{u}_{s}^{\otimes2}+\bm{u}_{s}^{\star}\otimes\bm{\Delta}_{s}\otimes\bm{u}_{s}+\bm{u}_{s}^{\star\otimes2}\otimes\bm{\Delta}_{s}.
\]
Applying Lemma~\ref{lemma:Omega-I-T-op-UB} in Appendix~\ref{sec:Auxiliary-lemmas},
one obtains
\begin{align*}
\left\Vert (p^{-1}\mathcal{P}_{\Omega}-\mathcal{I})\left(\bm{T}-\bm{T}^{\star}\right)\right\Vert  & \leq\left\Vert p^{-1}\mathcal{P}_{\Omega}\left(\bm{1}^{\otimes3}\right)-\bm{1}^{\otimes3}\right\Vert \sum_{1\leq s\leq r}\left\Vert \bm{\Delta}_{s}\right\Vert _{\infty}\big(\left\Vert \bm{u}_{s}\right\Vert _{\infty}^{2}+\left\Vert \bm{u}_{s}^{\star}\right\Vert _{\infty}\left\Vert \bm{u}_{s}\right\Vert _{\infty}+\left\Vert \bm{u}_{s}\right\Vert _{\infty}^{2}\big)\\
 & \lesssim\left\Vert p^{-1}\mathcal{P}_{\Omega}\left(\bm{1}^{\otimes3}\right)-\bm{1}^{\otimes3}\right\Vert \cdot r\max_{1\leq s\leq r}\left\Vert \bm{\Delta}_{s}\right\Vert _{\infty}\max_{1\leq s\leq r}\left\Vert \bm{u}_{s}^{\star}\right\Vert _{\infty}^{2},
\end{align*}
where we use (\ref{eq:u-loss-u-relation}) that $\left\Vert \bm{\Delta}_{i}\right\Vert _{\infty}\ll\left\Vert \bm{u}_{i}^{\star}\right\Vert _{\infty}$.
In addition, from \cite[Lemma D.2]{cai2019nonconvex}, we know that
\[
\left\Vert \mathcal{P}_{\Omega}\left(\bm{1}^{\otimes3}\right)-p\bm{1}^{\otimes3}\right\Vert \lesssim\log^{3}d+\sqrt{dp}\,\log^{5/2}d
\]
with probability at least $1-O\left(d^{-20}\right)$. Consequently,
we obtain
\begin{align*}
\left\Vert (p^{-1}\mathcal{P}_{\Omega}-\mathcal{I})\left(\bm{T}-\bm{T}^{\star}\right)\right\Vert  & \lesssim\frac{r}{p}\big(\log^{3}d+\sqrt{dp}\,\log^{5/2}d\big)\frac{\sigma_{\max}}{\lambda_{\min}^{\star}}\sqrt{\frac{\mu r\log d}{p}}\lambda_{\max}^{\star1/3}\cdot\frac{\mu}{d}\lambda_{\max}^{\star2/3}\\
 & \lesssim\sigma_{\max}\sqrt{\frac{d}{p}}\left\{ \frac{\mu^{3/2}r^{3/2}\log^{7/2}d}{d^{3/2}p}+\frac{\mu^{3/2}r^{3/2}\log^{3}d}{d\sqrt{p}}\right\} .
\end{align*}
Together with (\ref{eq:u-norm}), this enables us to conclude that
\begin{align}
\big|\left(\bm{B}_{3}\right)_{i,i}\big| & \lesssim\lambda_{\max}^{\star}\left\Vert (p^{-1}\mathcal{P}_{\Omega}-\mathcal{I})\left(\bm{T}-\bm{T}^{\star}\right)\right\Vert \nonumber \\
 & \lesssim\sigma_{\max}\lambda_{\max}^{\star}\sqrt{\frac{d}{p}}\left\{ \frac{\mu^{3/2}r^{3/2}\log^{7/2}d}{d^{3/2}p}+\frac{\mu^{3/2}r^{3/2}\log^{3}d}{d\sqrt{p}}\right\} .\label{eq:B3-m-entry-UB}
\end{align}
\item It remains to look at $\bm{B}_{4}$. As shown in \cite{cai2019nonconvex},
the loss function $g(\bm{U})$ is locally strong convex and smooth
with respect to the initial estimate $\bm{U}^{0}$. By the standard
result of convex optimization, we know that the Euclidean norm of
the gradient undergoes contraction at each iteration, in the sense
that
\[
\left\Vert \nabla g\big(\bm{U}^{t+1}\big)\right\Vert _{\mathrm{F}}\leq\rho\left\Vert \nabla g\big(\bm{U}^{t}\big)\right\Vert _{\mathrm{F}}
\]
for some constant $0<\rho<1$. By the construction of our estimate
$\bm{U}:=\bm{U}^{t_{0}}$ and the assumptions $\sigma_{\max}/\lambda_{\min}^{\star}\gtrsim d^{-100}$
and $t_{0}\lesssim\log d$, one has
\begin{equation}
\left\Vert \nabla g(\bm{U})\right\Vert _{\mathrm{F}}\lesssim\sigma_{\max}\lambda_{\max}^{\star2/3}\sqrt{\frac{d}{p}}\,\frac{1}{d}.\label{eq:grad-U-fro-UB}
\end{equation}
Consequently, we can upper bound
\begin{align}
\big|\left(\bm{B}_{4}\right)_{i,i}\big| & =\big|\bm{u}_{i}^{\star\top}\nabla_{\bm{u}_{i}}g(\bm{U})\big|\leq\left\Vert \bm{u}_{i}^{\star}\right\Vert _{2}\left\Vert \nabla g(\bm{U})\right\Vert _{\mathrm{F}}\lesssim\sigma_{\max}\lambda_{\max}^{\star}\sqrt{\frac{d}{p}}\frac{1}{d}.\label{eq:B4-m-entry-UB}
\end{align}
\item Taking collectively (\ref{eq:B1-m-entry-UB}), (\ref{eq:B2-m-entry-UB}),
(\ref{eq:B3-m-entry-UB}) and (\ref{eq:B4-m-entry-UB}) yields that
\begin{align*}
\left|B_{i,i}\right| & \leq\big|\left(\bm{B}_{1}\right)_{i,i}\big|+\big|\left(\bm{B}_{2}\right)_{i,i}\big|+\big|\left(\bm{B}_{3}\right)_{i,i}\big|+\big|\left(\bm{B}_{4}\right)_{i,i}\big|\\
 & \lesssim\sigma_{\max}\lambda_{\max}^{\star}\sqrt{\frac{d}{p}}\Bigg\{\frac{\mu^{3/2}r^{3/2}\log^{7/2}d}{d^{3/2}p}+\frac{\mu^{3/2}r^{3/2}\log^{3}d}{d\sqrt{p}}+\sqrt{\frac{\log d}{d}}+\frac{\sigma_{\max}}{\lambda_{\max}^{\star}}\sqrt{\frac{r^{2}d\log^{2}d}{p}}+\frac{1}{d}\Bigg\}\\
 & \lesssim\sigma_{\max}\lambda_{\max}^{\star}\sqrt{\frac{d}{p}}\frac{\zeta}{\sqrt{\mu r}},
\end{align*}
where we recall the definition of $\zeta$ in (\ref{def:zeta}).
\end{itemize}

\subsection{Proof of Lemma \ref{lemma:U-loss-dist-W2}}

\label{subsec:U-loss-dist-W2}

Without loss of generality, we assume that $\bm{\Pi}=\bm{I}_{r}$
for simplicity of presentation. For any fixed $m\in\left[d\right]$,
it is straightforward to decompose
\begin{align*}
 & \bm{e}_{m}^{\top}\mathsf{unfold}\left(\mathcal{P}_{\Omega}\left(\bm{E}\right)\right)\big(\widetilde{\bm{U}}(\widetilde{\bm{U}}^{\top}\widetilde{\bm{U}})^{-1}-\widetilde{\bm{U}}^{\star}(\widetilde{\bm{U}}^{\star\top}\widetilde{\bm{U}}^{\star})^{-1}\big)\\
 & =\underbrace{\bm{e}_{m}^{\top}\mathsf{unfold}\left(\mathcal{P}_{\Omega}\left(\bm{E}\right)\right)\big(\widetilde{\bm{U}}-\widetilde{\bm{U}}^{\star}\big)(\widetilde{\bm{U}}^{\top}\widetilde{\bm{U}})^{-1}}_{=:\,\beta_{1}}+\underbrace{\bm{e}_{m}^{\top}\mathsf{unfold}\left(\mathcal{P}_{\Omega}\left(\bm{E}\right)\right)\widetilde{\bm{U}}^{\star}\big((\widetilde{\bm{U}}^{\top}\widetilde{\bm{U}})^{-1}-(\widetilde{\bm{U}}^{\star\top}\widetilde{\bm{U}}^{\star})^{-1}\big)}_{=:\,\beta_{2}}.
\end{align*}
In the sequel, we shall upper bound $\beta_{1}$ and $\beta_{2}$
separately.

\subsubsection{Controlling $\beta_{1}$}

From the fact (\ref{eq:U-tilde-spectrum}) that $\lambda_{\min}(\widetilde{\bm{U}}^{\top}\widetilde{\bm{U}})\asymp\lambda_{\min}^{\star4/3}$,
one has
\begin{align}
\big\|\bm{e}_{m}^{\top}\mathsf{unfold}\big(\mathcal{P}_{\Omega}(\bm{E})\big)(\widetilde{\bm{U}}-\widetilde{\bm{U}}^{\star})(\widetilde{\bm{U}}^{\top}\widetilde{\bm{U}})^{-1}\big\|_{2} & \leq\big\|\bm{e}_{m}^{\top}\mathsf{unfold}\big(\mathcal{P}_{\Omega}(\bm{E})\big)(\widetilde{\bm{U}}-\widetilde{\bm{U}}^{\star})\big\|_{2}\big\|(\widetilde{\bm{U}}^{\top}\widetilde{\bm{U}})^{-1}\big\|\nonumber \\
 & \asymp\frac{1}{\lambda_{\min}^{\star4/3}}\big\|\bm{e}_{m}^{\top}\mathsf{unfold}\big(\mathcal{P}_{\Omega}(\bm{E})\big)(\widetilde{\bm{U}}-\widetilde{\bm{U}}^{\star})\big\|_{2}.\label{eq:beta1-term1}
\end{align}
 Hence, it suffices to control the $\ell_{2}$ norm of the $m$-th
row of 
\begin{align*}
\mathsf{unfold}\big(\mathcal{P}_{\Omega}(\bm{E})\big)(\widetilde{\bm{U}}-\widetilde{\bm{U}}^{\star}) & =\Big[\mathcal{P}_{\Omega}(\bm{E})\times_{1}\bm{u}_{s}\times_{2}\bm{u}_{s}-\mathcal{P}_{\Omega}(\bm{E})\times_{1}\bm{u}_{s}^{\star}\times_{2}\bm{u}_{s}^{\star}\Big]_{1\leq s\leq r},
\end{align*}
which admits the following decomposition
\begin{align*}
\Big[\mathcal{P}_{\Omega}(\bm{E})\times_{1}\bm{u}_{s}\times_{2}\bm{u}_{s}-\mathcal{P}_{\Omega}(\bm{E})\times_{1}\bm{u}_{s}^{\star}\times_{2}\bm{u}_{s}^{\star}\Big]_{1\leq s\leq r} & =\underbrace{\Big[\mathcal{P}_{\Omega}(\bm{E})\times_{1}\bm{u}_{s}^{\left(m\right)}\times_{2}\bm{u}_{s}^{\left(m\right)}-\mathcal{P}_{\Omega}(\bm{E})\times_{1}\bm{u}_{s}^{\star}\times_{2}\bm{u}_{s}^{\star}\Big]_{1\leq s\leq r}}_{=:\,\gamma_{1}}\\
 & \quad+\underbrace{\Big[\mathcal{P}_{\Omega}(\bm{E})\times_{1}\bm{u}_{s}\times_{2}\bm{u}_{s}-\mathcal{P}_{\Omega}(\bm{E})\times_{1}\bm{u}_{s}^{\left(m\right)}\times_{2}\bm{u}_{s}^{\left(m\right)}\Big]_{1\leq s\leq r}}_{=:\,\gamma_{2}}.
\end{align*}
Here, we recall the leave-one-out matrix $\bm{U}^{(m)}=\big[\bm{u}_{s}^{(m)}\big]_{1\leq s\leq r}\in\mathbb{R}^{d\times r}$
returned by Algorithm~\ref{alg:gd_loo}. In what follow, we shall
control $\gamma_{1}$ and $\gamma_{2}$ separately.
\begin{itemize}
\item Let us start with $\gamma_{1}$. For notational convenience, we denote
$\bm{\Delta}_{s}:=\bm{u}_{s}-\bm{u}_{s}^{\star}$, $\bm{\Delta}_{s}^{\left(m\right)}:=\bm{u}_{s}^{\left(m\right)}-\bm{u}_{s}^{\star}$,
$\Delta_{s,i}=\left(\bm{\Delta}_{s}\right)_{i}$ and $\Delta_{s,i}^{(m)}=\big(\bm{\Delta}_{s}^{(m)}\big)_{i}$
for each $1\leq s\leq r$ and $1\leq i\leq d$. With this notation
in place, for each $1\leq s\leq r$ we can expand
\begin{align*}
 & \big(\mathcal{P}_{\Omega}\left(\bm{E}\right)\times_{1}\bm{u}_{s}^{\left(m\right)}\times_{2}\bm{u}_{s}^{\left(m\right)}-\mathcal{P}_{\Omega}\left(\bm{E}\right)\times_{1}\bm{u}_{s}^{\star}\times_{2}\bm{u}_{s}^{\star}\big)_{m}\\
 & \qquad=\bm{u}_{s}^{\left(m\right)\top}(\mathcal{P}_{\Omega}\left(\bm{E}\right))_{:,:,m}\bm{u}_{s}^{\left(m\right)}-\bm{u}_{s}^{\star\top}(\mathcal{P}_{\Omega}\left(\bm{E}\right))_{:,:,m}\bm{u}_{s}^{\star}\\
 & \qquad=2\,\bm{\Delta}_{s}^{\left(m\right)\top}(\mathcal{P}_{\Omega}\left(\bm{E}\right))_{:,:,m}\bm{u}_{s}^{\star}+\bm{\Delta}_{s}^{\left(m\right)\top}(\mathcal{P}_{\Omega}\left(\bm{E}\right))_{:,:,m}\bm{\Delta}_{s}^{\left(m\right)},
\end{align*}
where $\bm{R}_{:,:,m}\in\mathbb{R}^{d\times d}$ denotes the $m$-th
mode-$3$ slice of a tensor $\bm{R}\in\mathbb{R}^{d\times d\times d}$
as defined in Section~\ref{subsec:Notations}.

We first look at $\bm{\Delta}_{s}^{\left(m\right)\top}(\mathcal{P}_{\Omega}\left(\bm{E}\right))_{:,:,m}\bm{u}_{s}^{\star}$.
By construction, $\bm{\Delta}_{s}^{\left(m\right)}$ is independent
of the $m$-th mode-3 slice of $\mathcal{P}_{\Omega}\left(\bm{E}\right)$.
Consequently, $\bm{\Delta}_{s}^{\left(m\right)\top}\left(\mathcal{P}_{\Omega}\left(\bm{E}\right)\right)_{:,:,m}\bm{u}_{s}^{\star}=\sum_{1\leq i,j\leq d}E_{i,j,m}\chi_{i,j,m}\Delta_{s,i}^{\left(m\right)}u_{s,j}^{\star}$
is a sum of independent zero-mean random variables (conditional on
$\mathcal{\mathcal{P}}_{\Omega_{-m}}\left(\bm{E}\right)$). Using
the incoherence assumption (\ref{assumption:u-inf-norm}), straightforward
calculation gives
\begin{align*}
B_{1} & :=\max_{1\leq i,j\leq d}\big\| E_{i,j,m}\chi_{i,j,m}\Delta_{s,i}^{\left(m\right)}u_{s,j}^{\star}\big\|_{\psi_{1}}\lesssim\sigma_{\max}\,\big\|\bm{\Delta}_{s}^{\left(m\right)}\big\|_{\infty}\big\|\bm{u}_{s}^{\star}\big\|_{\infty}\leq\sigma_{\max}\lambda_{\max}^{\star1/3}\sqrt{\frac{\mu}{d}}\,\big\|\bm{\Delta}_{s}^{\left(m\right)}\big\|_{\infty};\\
V_{1} & :=\sum_{1\leq i,j\leq d}\mathbb{E}\big[E_{i,j,m}^{2}\chi_{i,j,m}^{2}\big]\big(\Delta_{s,i}^{\left(m\right)}\big)_{i}^{2}u_{s,j}^{\star2}\lesssim\sigma_{\max}^{2}p\,\big\|\bm{\Delta}_{s}^{\left(m\right)}\big\|_{2}^{2}\big\|\bm{u}_{s}^{\star}\big\|_{2}^{2}\leq\sigma_{\max}^{2}\lambda_{\max}^{\star2/3}p\,\big\|\bm{\Delta}_{s}^{\left(m\right)}\big\|_{2}^{2}.
\end{align*}

We then apply the Bernstein inequality to find that with probability
at least $1-O\left(d^{-20}\right)$,
\begin{align}
\big|\bm{\Delta}_{s}^{\left(m\right)\top}\left(\mathcal{P}_{\Omega}\left(\bm{E}\right)\right)_{:,:,m}\bm{u}_{s}^{\star}\big| & \lesssim B_{1}\log^{2}d+\sqrt{V_{1}\log d}\nonumber \\
 & \lesssim\sigma_{\max}\lambda_{\max}^{\star1/3}\left\{ \sqrt{\frac{\mu}{d}}\log^{2}d\,\big\|\bm{\Delta}_{s}^{\left(m\right)}\big\|_{\infty}+\sqrt{p\log d}\,\big\|\bm{\Delta}_{s}^{\left(m\right)}\big\|_{2}\right\} .\label{eq:u-loo-loss-noise-u-true-term1}
\end{align}
Applying a similar argument, we can also upper bound
\begin{align}
\big|\bm{\Delta}_{s}^{\left(m\right)\top}\left(\mathcal{P}_{\Omega}\left(\bm{E}\right)\right)_{:,:,m}\bm{\Delta}_{s}^{\left(m\right)}\big| & \lesssim\sigma_{\max}\left\{ \log^{2}d\,\big\|\bm{\Delta}_{s}^{\left(m\right)}\big\|_{\infty}^{2}+\sqrt{p\log d}\,\big\|\bm{\Delta}_{s}^{\left(m\right)}\big\|_{2}^{2}\right\} \nonumber \\
 & \ll\sigma_{\max}\lambda_{\max}^{\star1/3}\left\{ \sqrt{\frac{\mu}{d}}\log^{2}d\,\big\|\bm{\Delta}_{s}^{\left(m\right)}\big\|_{\infty}+\sqrt{p\log d}\,\big\|\bm{\Delta}_{s}^{\left(m\right)}\big\|_{2}\right\} ,\label{eq:u-loo-loss-noise-u-true-term2}
\end{align}
where we utilize (\ref{eq:U-property}) in Lemma~\ref{lemma:U-property}
that $\big\|\bm{\Delta}_{s}^{\left(m\right)}\big\|_{\infty}\ll\sqrt{\mu/d}\,\lambda_{\max}^{\star1/3}$
and $\big\|\bm{\Delta}_{s}^{\left(m\right)}\big\|_{2}\ll\lambda_{\max}^{\star1/3}$
in the last inequality.

Combining (\ref{eq:u-loo-loss-noise-u-true-term1}) and (\ref{eq:u-loo-loss-noise-u-true-term2}),
and summing over $s\in\left[r\right]$, we obtain
\begin{align}
 & \Big\|\bm{e}_{m}^{\top}\big[\mathcal{P}_{\Omega}\left(\bm{E}\right)\times_{1}\bm{u}_{s}^{\left(m\right)}\times_{2}\bm{u}_{s}^{\left(m\right)}-\mathcal{P}_{\Omega}\left(\bm{E}\right)\times_{1}\bm{u}_{s}^{\star}\times_{2}\bm{u}_{s}^{\star}\big]_{1\leq s\leq r}\Big\|_{2}\nonumber \\
 & \qquad\lesssim\sigma_{\max}\lambda_{\max}^{\star1/3}\left\{ \sqrt{\frac{\mu}{d}}\log^{2}d\,\sqrt{r}\,\big\|\bm{\Delta}_{s}^{\left(m\right)}\big\|_{\infty}+\sqrt{p\log d}\,\big\|\bm{U}^{\left(m\right)}-\bm{U}^{\star}\big\|_{\mathrm{F}}\right\} \nonumber \\
 & \qquad\lesssim\sigma_{\max}\lambda_{\max}^{\star1/3}\left\{ \sqrt{\frac{\mu}{d}}\log^{2}d\cdot\frac{\sigma_{\max}}{\lambda_{\min}^{\star}}\sqrt{\frac{\mu r^{2}\log d}{p}}\lambda_{\max}^{\star1/3}+\sqrt{p\log d}\cdot\frac{\sigma_{\max}}{\lambda_{\min}^{\star}}\sqrt{\frac{rd\log d}{p}}\lambda_{\max}^{\star1/3}\right\} \nonumber \\
 & \qquad\asymp\frac{\sigma_{\max}\sqrt{p}}{\lambda_{\min}^{\star1/3}}\sigma_{\max}\sqrt{\frac{rd\log^{2}d}{p}},\label{eq:u-loo-quad-term-loss}
\end{align}
where the last step holds as long as $p\gtrsim\mu^{2}rd^{-2}\log^{3}d$
and $\kappa\asymp1$.
\item Turning to $\gamma_{2}$, we can decompose
\begin{align}
 & \big(\mathcal{P}_{\Omega}\left(\bm{E}\right)\times_{1}\bm{u}_{s}\times_{2}\bm{u}_{s}-\mathcal{P}_{\Omega}\left(\bm{E}\right)\times_{1}\bm{u}_{s}^{\left(m\right)}\times_{2}\bm{u}_{s}^{\left(m\right)}\big)_{m}\nonumber \\
 & \qquad=2\,\big(\bm{u}_{s}-\bm{u}_{s}^{\left(m\right)}\big)^{\top}\left(\mathcal{P}_{\Omega}\left(\bm{E}\right)\right)_{:,:,m}\bm{u}_{s}^{\left(m\right)}+\big(\bm{u}_{s}-\bm{u}_{s}^{\left(m\right)}\big)^{\top}\left(\mathcal{P}_{\Omega}\left(\bm{E}\right)\right)_{:,:,m}\big(\bm{u}_{s}-\bm{u}_{s}^{\left(m\right)}\big).\label{eq:UB12345}
\end{align}
For the first term, we use the Cauchy-Schwartz to derive
\begin{align*}
\big|\big(\bm{u}_{s}-\bm{u}_{s}^{\left(m\right)}\big)^{\top}\left(\mathcal{P}_{\Omega}\left(\bm{E}\right)\right)_{:,:,m}\bm{u}_{s}^{\left(m\right)}\big| & \leq\big\|\bm{u}_{s}-\bm{u}_{s}^{\left(m\right)}\big\|_{2}\big\|\left(\mathcal{P}_{\Omega}\left(\bm{E}\right)\right)_{:,:,m}\bm{u}_{s}^{\left(m\right)}\big\|_{2}.
\end{align*}
This motivates us to bound the $\ell_{2}$ norm of $\left(\mathcal{P}_{\Omega}\left(\bm{E}\right)\right)_{:,:,m}\bm{u}_{s}^{\left(m\right)}=\sum_{1\leq i,j\leq d}E_{i,j,m}\chi_{i,j,m}\bm{e}_{i}\bm{e}_{j}^{\top}\bm{u}_{s}^{\left(m\right)}$
--- which is a sum of independent random zero-mean matrices. By Lemma~\ref{lemma:U-property},
it is straightforward to calculate
\begin{align*}
B_{2} & :=\max_{1\leq i,j\leq d}\big\| E_{i,j,m}\chi_{i,j,m}\bm{e}_{i}\bm{e}_{j}^{\top}\bm{u}_{s}^{\left(m\right)}\big\|_{\psi_{1}}\lesssim\sigma_{\max}\,\big\|\bm{u}_{s}^{\left(m\right)}\big\|_{\infty}\leq\sigma_{\max}\lambda_{\max}^{\star1/3}\sqrt{\frac{\mu}{d}};\\
V_{2} & :=\sum_{1\leq i,j\leq d}\mathbb{E}\big[E_{i,j,m}^{2}\chi_{i,j,m}^{2}\big]\big(\bm{u}_{s,j}^{\left(m\right)}\big)^{2}\leq\sigma_{\max}^{2}dp\,\big\|\bm{u}_{s}^{\left(m\right)}\big\|_{2}^{2}\lesssim\sigma_{\max}^{2}\lambda_{\max}^{\star2/3}dp.
\end{align*}
In view of the matrix Bernstein inequality, we show that with probability
exceeding $1-O\left(d^{-20}\right)$,
\[
\big\|\left(\mathcal{P}_{\Omega}\left(\bm{E}\right)\right)_{:,:,m}\bm{u}_{s}^{\left(m\right)}\big\|_{2}\lesssim B_{2}\log^{2}d+\sqrt{V_{2}\log d}\lesssim\sigma_{\max}\lambda_{\max}^{\star1/3}\left\{ \sqrt{\frac{\mu}{d}}\log^{2}d+\sqrt{dp\log d}\right\} \asymp\sigma_{\max}\lambda_{\max}^{\star1/3}\sqrt{dp\log d},
\]
where the last step holds as long as $p\gg\mu d^{-2}\log^{3}d$. Consequently,
we reach
\begin{equation}
\big|\big(\bm{u}_{s}-\bm{u}_{s}^{\left(m\right)}\big)^{\top}\left(\mathcal{P}_{\Omega}\left(\bm{E}\right)\right)_{:,:,m}\bm{u}_{s}^{\left(m\right)}\big|\lesssim\sigma_{\max}\lambda_{\max}^{\star1/3}\sqrt{dp\log d}\,\big\|\bm{u}_{s}-\bm{u}_{s}^{\left(m\right)}\big\|_{2}.\label{eq:u-loo-diff-noise-u-loo}
\end{equation}
For the second term of (\ref{eq:UB12345}), invoke \cite[Lemma 11]{chen2015fast}
to demonstrate that: with probability at least $1-O\left(d^{-20}\right)$,
\begin{equation}
\max_{1\leq m\leq d}\big\|\big(\mathcal{P}_{\Omega}(\bm{E})\big)_{:,:,m}\big\|\lesssim\sigma_{\max}\big(\sqrt{dp}+\log d\big).\label{eq:noise-slice-op-UB}
\end{equation}
This enables us to bound
\begin{align}
\big|\big(\bm{u}_{s}-\bm{u}_{s}^{\left(m\right)}\big)^{\top}\left(\mathcal{P}_{\Omega}\left(\bm{E}\right)\right)_{:,:,m}\big(\bm{u}_{s}-\bm{u}_{s}^{\left(m\right)}\big)\big| & \leq\big\|\big(\mathcal{P}_{\Omega}(\bm{E})\big)_{:,:,m}\big\|\big\|\bm{u}_{s}-\bm{u}_{s}^{\left(m\right)}\big\|_{2}^{2}\nonumber \\
 & \lesssim\sigma_{\max}\big(\sqrt{dp}+\log d\big)\big\|\bm{u}_{s}-\bm{u}_{s}^{\left(m\right)}\big\|_{2}^{2}.\label{eq:u-loo-diff-noise-u-loo-diff}
\end{align}
Taking together the bounds (\ref{eq:u-loo-diff-noise-u-loo}), (\ref{eq:u-loo-diff-noise-u-loo-diff})
and summing over $s\in\left[r\right]$, we obtain
\begin{align}
 & \Big\|\bm{e}_{m}^{\top}\big[\mathcal{P}_{\Omega}\left(\bm{E}\right)\times_{1}\bm{u}_{s}\times_{2}\bm{u}_{s}-\mathcal{P}_{\Omega}\left(\bm{E}\right)\times_{1}\bm{u}_{s}^{\left(m\right)}\times_{2}\bm{u}_{s}^{\left(m\right)}\big]_{1\leq s\leq r}\Big\|_{2}\nonumber \\
 & \qquad\lesssim\sigma_{\max}\lambda_{\max}^{\star1/3}\sqrt{dp\log d}\,\big\|\bm{U}-\bm{U}^{\left(m\right)}\big\|_{\mathrm{F}}+\sigma_{\max}\big(\sqrt{dp}+\log d\big)\big\|\bm{U}-\bm{U}^{\left(m\right)}\big\|_{\mathrm{F}}^{2}\nonumber \\
 & \qquad\lesssim\sigma_{\max}\lambda_{\max}^{\star1/3}\sqrt{dp\log d}\cdot\frac{\sigma_{\max}}{\lambda_{\min}^{\star}}\sqrt{\frac{\mu r\log d}{p}}\lambda_{\max}^{\star1/3}+\sigma_{\max}\big(\sqrt{dp}+\log d\big)\cdot\frac{\sigma_{\max}^{2}}{\lambda_{\min}^{\star2}}\frac{\mu r\log d}{p}\lambda_{\max}^{\star2/3}\nonumber \\
 & \qquad\asymp\frac{\sigma_{\max}\sqrt{p}}{\lambda_{\min}^{\star1/3}}\sigma_{\max}\sqrt{\frac{\mu rd\log^{2}d}{p}},\label{eq:u-loo-quad-term-diff}
\end{align}
where we have used the conditions that $\sigma_{\max}/\lambda_{\min}^{\star}\ll\sqrt{p}/d^{3/4}$,
$p\gg\mu rd^{-3/2}\log^{2}d$ and $\kappa\asymp1$.
\item Putting (\ref{eq:u-loo-diff-noise-u-loo-diff}) and (\ref{eq:u-loo-quad-term-diff})
together and substituting them into (\ref{eq:beta1-term1}) yield
\begin{align}
\big\|\bm{e}_{m}^{\top}\mathsf{unfold}\big(p^{-1}\mathcal{P}_{\Omega}(\bm{E})\big)\big(\widetilde{\bm{U}}-\widetilde{\bm{U}}^{\star}\big)\big(\widetilde{\bm{U}}^{\top}\widetilde{\bm{U}}\big)^{-1}\big\|_{2} & \lesssim\frac{\sigma_{\max}\sqrt{p}}{\lambda_{\min}^{\star2/3}}\frac{\sigma_{\max}}{\lambda_{\min}^{\star}}\sqrt{\frac{\mu rd\log^{2}d}{p}}.\label{eq:beta1-term1-UB}
\end{align}
\end{itemize}

\subsubsection{Controlling $\beta_{2}$}

Recognizing that
\[
\big\|\bm{e}_{m}^{\top}\mathsf{unfold}\big(p^{-1}\mathcal{P}_{\Omega}(\bm{E})\big)\widetilde{\bm{U}}^{\star}\big(\big(\widetilde{\bm{U}}^{\top}\widetilde{\bm{U}}\big)^{-1}-\big(\widetilde{\bm{U}}^{\star\top}\widetilde{\bm{U}}^{\star}\big)^{-1}\big)\big\|_{2}\leq\big\|\bm{e}_{m}^{\top}\mathsf{unfold}\big(p^{-1}\mathcal{P}_{\Omega}(\bm{E})\big)\widetilde{\bm{U}}^{\star}\big\|_{2}\big\|\big(\widetilde{\bm{U}}^{\top}\widetilde{\bm{U}}\big)^{-1}-\big(\widetilde{\bm{U}}^{\star\top}\widetilde{\bm{U}}^{\star}\big)^{-1}\big\|,
\]
it suffices to control the $\ell_{2}$ norm of $\bm{e}_{m}^{\top}\mathsf{unfold}\big(\mathcal{P}_{\Omega}(\bm{E})\widetilde{\bm{U}}^{\star}\big)$.
Let us express
\[
\bm{e}_{m}^{\top}\mathsf{unfold}\left(\mathcal{P}_{\Omega}\left(\bm{E}\right)\right)\widetilde{\bm{U}}^{\star}=\sum_{1\leq i,j\leq d}E_{i,j,m}\chi_{i,j,m}\widetilde{\bm{U}}_{(i,j),:}^{\star}
\]
as a sum of independent zero-mean random vectors in $\mathbb{R}^{r}$.
By (\ref{eq:U-true-tilde-norm}), it is straightforward to compute
that
\begin{align*}
B_{3} & :=\max_{1\leq i,j\leq d}\big\| E_{i,j,m}\chi_{i,j,m}\widetilde{\bm{U}}_{(i,j),:}^{\star}\big\|_{\psi_{1}}\lesssim\sigma_{\max}\,\big\|\widetilde{\bm{U}}^{\star}\big\|_{2,\infty}\lesssim\sigma_{\max}\lambda_{\max}^{\star2/3}\frac{\mu\sqrt{r}}{d},\\
V_{3} & :=\sum_{1\leq i,j\leq d}\mathbb{E}\big[E_{i,j,m}^{2}\chi_{i,j,m}\big]\big\|\widetilde{\bm{U}}_{(i,j),:}^{\star}\big\|_{2}^{2}\lesssim\sigma_{\max}^{2}p\,\big\|\widetilde{\bm{U}}^{\star}\big\|_{\mathrm{F}}^{2}\lesssim\sigma_{\max}^{2}\lambda_{\max}^{\star4/3}rp,
\end{align*}
where $\|\cdot\|_{\psi_{1}}$ denotes the sub-exponential norm. Applying
the matrix Bernstein inequality yields that
\begin{align*}
\big\|\bm{e}_{m}^{\top}\mathsf{unfold}\big(p^{-1}\mathcal{P}_{\Omega}(\bm{E})\big)\widetilde{\bm{U}}^{\star}\big\| & \lesssim B_{3}\log^{2}d+\sqrt{V_{3}\log d}\lesssim\sigma_{\max}\lambda_{\max}^{\star2/3}\left\{ \frac{\mu\sqrt{r}\,\log^{2}d}{d}+\sqrt{rp\log d}\right\} \\
 & \asymp\sigma_{\max}\lambda_{\max}^{\star2/3}\sqrt{rp\log d}
\end{align*}
with probability at least $1-O\left(d^{-20}\right)$, where the last
line holds as long as $p\gtrsim\mu^{2}d^{-2}\log^{3}d$.

In addition, we can use (\ref{eq:U-true-tilde-spectrum}) and (\ref{eq:U-tilde-spectrum})
to upper bound
\begin{align}
\big\|(\widetilde{\bm{U}}^{\top}\widetilde{\bm{U}})^{-1}-(\widetilde{\bm{U}}^{\star\top}\widetilde{\bm{U}}^{\star})^{-1}\big\| & \leq\big\|(\widetilde{\bm{U}}^{\top}\widetilde{\bm{U}})^{-1}\big\|\big\|\widetilde{\bm{U}}^{\top}\big(\widetilde{\bm{U}}-\widetilde{\bm{U}}^{\star}\big)+\big(\widetilde{\bm{U}}-\widetilde{\bm{U}}^{\star}\big)^{\top}\widetilde{\bm{U}}^{\star}\big\|\big\|(\widetilde{\bm{U}}^{\star\top}\widetilde{\bm{U}}^{\star})^{-1}\big\|\nonumber \\
 & \leq\big\|(\widetilde{\bm{U}}^{\top}\widetilde{\bm{U}})^{-1}\big\|\big\|(\widetilde{\bm{U}}^{\star\top}\widetilde{\bm{U}}^{\star})^{-1}\big\|\big\|\widetilde{\bm{U}}-\widetilde{\bm{U}}^{\star}\big\|\big(\|\widetilde{\bm{U}}^{\star}\|+\|\widetilde{\bm{U}}\|\big)\nonumber \\
 & \lesssim\frac{1}{\lambda_{\min}^{\star2}}\big\|\widetilde{\bm{U}}-\widetilde{\bm{U}}^{\star}\big\|\lesssim\frac{\sigma_{\max}}{\lambda_{\min}^{\star3}}\sqrt{\frac{d}{p}}\,\lambda_{\max}^{\star2/3},\label{eq:U-tilde-Gram-inv-op-loss}
\end{align}
where we use (\ref{eq:U-tilde-op-loss}) in the last step.

Taking together the above bounds, we arrive at
\begin{equation}
\Big\|\bm{e}_{m}^{\top}\mathsf{unfold}\big(p^{-1}\mathcal{P}_{\Omega}(\bm{E})\big)\widetilde{\bm{U}}^{\star}\big(\big(\widetilde{\bm{U}}^{\top}\widetilde{\bm{U}}\big)^{-1}-\big(\widetilde{\bm{U}}^{\star\top}\widetilde{\bm{U}}^{\star}\big)^{-1}\big)\Big\|_{2}\lesssim\frac{\sigma_{\max}\sqrt{p}}{\lambda_{\min}^{\star2/3}}\frac{\sigma_{\max}}{\lambda_{\min}^{\star}}\sqrt{\frac{rd\log d}{p}}.\label{eq:beta1-term2-UB}
\end{equation}

\subsubsection{Combining $\beta_{1}$ and $\beta_{2}$}

Putting (\ref{eq:beta1-term1-UB}) and (\ref{eq:beta1-term2-UB})
together, we conclude that with probability at least $1-O\left(d^{-20}\right)$,
one has
\[
\big\|\mathsf{unfold}\big(p^{-1}\mathcal{P}_{\Omega}(\bm{E})\big)\big(\widetilde{\bm{U}}\big(\widetilde{\bm{U}}^{\top}\widetilde{\bm{U}}\big)^{-1}-\widetilde{\bm{U}}^{\star}\big(\widetilde{\bm{U}}^{\star\top}\widetilde{\bm{U}}^{\star}\big)^{-1}\big)\big\|_{2,\infty}\lesssim\frac{\sigma_{\max}}{\lambda_{\min}^{\star2/3}\sqrt{p}}\frac{\sigma_{\max}}{\lambda_{\min}^{\star}}\sqrt{\frac{\mu rd\log d}{p}}.
\]

\subsection{Proof of Lemma~\ref{lemma:U-loss-dist-W3}}

\label{subsec:U-loss-dist-W3}

Without loss of generality, assume that $\bm{\Pi}=\bm{I}_{r}$ for
simplicity of presentation. Fix an arbitrary $1\leq m\leq d$. From
(\ref{eq:U-tilde-spectrum}), we can upper bound
\begin{align}
\big\|\bm{e}_{m}^{\top}\mathsf{unfold}\big((p^{-1}\mathcal{P}_{\Omega}-\mathcal{I})(\bm{T}-\bm{T}^{\star})\big)\widetilde{\bm{U}}\big(\widetilde{\bm{U}}^{\top}\widetilde{\bm{U}}\big)^{-1}\big\|_{2} & \leq\big\|\bm{e}_{m}^{\top}\mathsf{unfold}\big((p^{-1}\mathcal{P}_{\Omega}-\mathcal{I})(\bm{T}-\bm{T}^{\star})\big)\widetilde{\bm{U}}\big\|_{2}\big\|\big(\widetilde{\bm{U}}^{\top}\widetilde{\bm{U}}\big)^{-1}\big\|\nonumber \\
 & \asymp\frac{1}{\lambda_{\min}^{\star4/3}}\big\|\bm{e}_{m}^{\top}\mathsf{unfold}\big((p^{-1}\mathcal{P}_{\Omega}-\mathcal{I})(\bm{T}-\bm{T}^{\star})\big)\widetilde{\bm{U}}\big\|_{2}.\label{eq:beta3-UB-temp}
\end{align}
As a result, it suffices to upper bound the $\ell_{2}$ norm of the
$m$-th row of $\mathsf{unfold}\big((p^{-1}\mathcal{P}_{\Omega}-\mathcal{I})(\bm{T}-\bm{T}^{\star})\big)\widetilde{\bm{U}}$.
Observe that this matrix can be decomposed as follows
\begin{align*}
\mathsf{unfold}\big((p^{-1}\mathcal{P}_{\Omega}-\mathcal{I})(\bm{T}-\bm{T}^{\star})\big)\widetilde{\bm{U}} & =\underbrace{\mathsf{unfold}\big((p^{-1}\mathcal{P}_{\Omega}-\mathcal{I})(\bm{T}^{\left(m\right)}-\bm{T}^{\star})\big)\widetilde{\bm{U}}^{\left(m\right)}}_{=:\,\beta_{1}}\\
 & \quad+\underbrace{\mathsf{unfold}\big((p^{-1}\mathcal{P}_{\Omega}-\mathcal{I})(\bm{T}-\bm{T}^{\star})\big)\big(\widetilde{\bm{U}}-\widetilde{\bm{U}}^{\left(m\right)}\big)}_{=:\,\beta_{2}}\\
 & \quad+\underbrace{\mathsf{unfold}\big((p^{-1}\mathcal{P}_{\Omega}-\mathcal{I})(\bm{T}-\bm{T}^{\left(m\right)})\big)\widetilde{\bm{U}}^{\left(m\right)}}_{=:\,\beta_{3}},
\end{align*}
where $\widetilde{\bm{U}}^{\left(m\right)}:=\big[\bm{u}_{l}^{(m)}\otimes\bm{u}_{l}^{(m)}\big]_{1\leq l\leq r}$.
In what follows, we shall control these three terms separately. To
simplify presentation, let us define 
\begin{align}
\bm{\xi}_{s}^{\left(m\right)} & :=\bm{u}_{s}-\bm{u}_{s}^{\left(m\right)},\qquad1\leq s\leq r.\label{def:xi}
\end{align}

\subsubsection{Controlling $\beta_{1}$}

By construction, we can express
\[
\bm{e}_{m}^{\top}\mathsf{unfold}\big((p^{-1}\mathcal{P}_{\Omega}-\mathcal{I})(\bm{T}^{\left(m\right)}-\bm{T}^{\star})\big)\widetilde{\bm{U}}^{\left(m\right)}=\sum_{1\leq i,j\leq d}\big(\bm{T}^{\left(m\right)}-\bm{T}^{\star}\big)_{i,j,m}(p^{-1}\chi_{i,j,m}-1)\widetilde{\bm{U}}_{(i,j),:}^{\left(m\right)}
\]
as a sum of independent zero-mean random vectors. We claim for the
moment that 
\begin{align}
\sum_{1\leq i,j\leq d}\big(\bm{T}^{\left(m\right)}-\bm{T}^{\star}\big)_{i,j,m}^{2} & \lesssim\frac{\sigma_{\max}^{2}}{\lambda_{\min}^{\star2}}\frac{\mu r\log d}{p}\lambda_{\max}^{\star2}.\label{eq:W3-beta1-temp-claim}
\end{align}
Combined with (\ref{eq:T-loo-loss-inf}) and Lemma~\ref{lemma:U-tilde-property},
it is straightforward to compute that
\begin{align*}
B:=\max_{1\leq i,j\leq d}\big\|\big(\bm{T}^{\left(m\right)}-\bm{T}^{\star}\big)_{i,j,m}(p^{-1}\chi_{ijm}-1)\widetilde{\bm{U}}_{(i,j),:}^{\left(m\right)}\big\|_{\psi_{1}} & \lesssim\frac{1}{p}\,\big\|\bm{T}^{\left(m\right)}-\bm{T}^{\star}\big\|_{\infty}\big\|\widetilde{\bm{U}}^{\left(m\right)}\big\|_{2,\infty}\\
 & \lesssim\frac{1}{p}\cdot\frac{\sigma_{\max}}{\lambda_{\min}^{\star}}\sqrt{\frac{\mu^{3}r^{2}\log d}{d^{2}p}}\lambda_{\max}^{\star}\cdot\frac{\mu\sqrt{r}\,\lambda_{\max}^{\star2/3}}{d},
\end{align*}
and
\begin{align*}
V:=\sum_{1\leq i,j\leq d}\mathbb{E}\big[(\bm{T}^{\left(m\right)}-\bm{T}^{\star})_{i,j,m}^{2}(p^{-1}\chi_{i,j,m}-1)^{2}\big]\big\|\widetilde{\bm{U}}_{(i,j),:}^{\left(m\right)}\big\|_{2}^{2} & \leq\frac{1}{p}\,\big\|\widetilde{\bm{U}}^{\left(m\right)}\big\|_{\mathrm{2,\infty}}^{2}\sum_{1\leq i,j\leq d}\big(\bm{T}^{\left(m\right)}-\bm{T}^{\star}\big)_{i,j,m}^{2}\\
 & \lesssim\frac{1}{p}\cdot\frac{\mu^{2}r\lambda_{\max}^{\star4/3}}{d^{2}}\cdot\frac{\sigma_{\max}^{2}}{\lambda_{\min}^{\star2}}\frac{\mu r\log d}{p}\lambda_{\max}^{\star2}.
\end{align*}
In view of the matrix Bernstein inequality, one has with probability
at least $1-O\left(d^{-20}\right)$,
\begin{align}
\big\|\bm{e}_{m}^{\top}\mathsf{unfold}\big(p^{-1}\mathcal{P}_{\Omega}(\bm{T}^{\left(m\right)}-\bm{T}^{\star})-(\bm{T}^{\left(m\right)}-\bm{T}^{\star})\big)\widetilde{\bm{U}}^{\left(m\right)}\big\|_{2} & \lesssim B\log^{2}d+\sqrt{V\log d}\nonumber \\
 & \lesssim\frac{\sigma_{\max}\lambda_{\max}^{\star5/3}}{\lambda_{\min}^{\star}\sqrt{p}}\left\{ \sqrt{\frac{\mu^{3}r^{2}}{d^{2}p}}\frac{\mu\sqrt{r}\,\log^{5/2}d}{d\sqrt{p}}+\sqrt{\frac{\mu^{3}r^{2}\log^{2}d}{d^{2}p}}\right\} \nonumber \\
 & \asymp\frac{\sigma_{\max}\lambda_{\max}^{\star2/3}}{\sqrt{p}}\sqrt{\frac{\mu^{3}r^{2}\log^{2}d}{d^{2}p}},\label{eq:beta3-gamma1-UB}
\end{align}
where the last step holds as long as $p\gtrsim\mu^{2}rd^{-2}\log^{3}d$
and $\kappa\asymp1$.

Now we are left with justifying the claim (\ref{eq:W3-beta1-temp-claim}).
Observe that 
\[
\sum_{1\leq i,j\leq d}\big(\bm{T}^{\left(m\right)}-\bm{T}^{\star}\big)_{i,j,m}^{2}=\big\|\bm{U}^{(m)}\bm{F}^{(m)}\bm{U}^{(m)\top}-\bm{U}^{\star}\bm{F}^{\star}\bm{U}^{\star\top}\big\|_{\mathrm{F}}^{2},
\]
where $\bm{F}^{(m)}$ and $\bm{F}^{\star}$ are diagonal matrices
in $\mathbb{R}^{r\times r}$ with entries $F_{i,i}^{(m)}=\big(\bm{u}_{i}^{(m)}\big)_{m}$
and $F_{i,i}^{\star}=\left(\bm{u}_{i}^{\star}\right)_{m}.$ Note that
$\|\bm{F}^{(m)}\|\leq\max_{1\leq i\leq r}\|\bm{u}_{i}^{(m)}\|_{\infty}$,
$\left\Vert \bm{F}^{\star}\right\Vert \leq\max_{1\leq i\leq r}\left\Vert \bm{u}_{i}^{\star}\right\Vert _{\infty}$and
$\big\|\bm{F}^{(m)}-\bm{F}^{\star}\big\|_{\mathrm{F}}\leq\big\|\bm{U}^{(m)}-\bm{U}^{\star}\big\|_{2,\infty}$.
By the triangle inequality, it is straightforward to bound
\begin{align*}
\big\|\bm{U}^{(m)}\bm{F}^{(m)}\bm{U}^{(m)\top}-\bm{U}^{\star}\bm{F}^{\star}\bm{U}^{\star\top}\big\|_{\mathrm{F}} & \leq\big\|(\bm{U}^{(m)}-\bm{U}^{\star})\bm{F}^{(m)}\bm{U}^{(m)\top}\big\|_{\mathrm{F}}+\big\|\bm{U}^{\star}(\bm{F}^{(m)}-\bm{F}^{\star})\bm{U}^{(m)\top}\big\|_{\mathrm{F}}\\
 & \quad+\big\|\bm{U}^{\star}\bm{F}^{\star}(\bm{U}^{(m)}-\bm{U}^{\star})^{\top}\big\|_{\mathrm{F}}\\
 & \leq\big\|\bm{U}^{(m)}-\bm{U}^{\star}\big\|_{\mathrm{F}}\big\|\bm{F}^{(m)}\big\|\big\|\bm{U}^{(m)}\big\|+\left\Vert \bm{U}^{\star}\right\Vert \big\|\bm{F}^{(m)}-\bm{F}^{\star}\big\|_{\mathrm{F}}\big\|\bm{U}^{(m)}\big\|\\
 & \quad+\left\Vert \bm{U}^{\star}\right\Vert \left\Vert \bm{F}^{\star}\right\Vert \big\|\bm{U}^{(m)}-\bm{U}^{\star}\big\|_{\mathrm{F}}\\
 & \leq\big\|\bm{U}^{(m)}-\bm{U}^{\star}\big\|_{\mathrm{F}}\max_{1\leq i\leq r}\|\bm{u}_{i}^{(m)}\|_{\infty}\big\|\bm{U}^{(m)}\big\|+\left\Vert \bm{U}^{\star}\right\Vert \big\|\bm{U}^{(m)}-\bm{U}^{\star}\big\|_{2,\infty}\big\|\bm{U}^{(m)}\big\|\\
 & \quad+\left\Vert \bm{U}^{\star}\right\Vert \max_{1\leq i\leq r}\left\Vert \bm{u}_{i}^{\star}\right\Vert _{\infty}\big\|\bm{U}^{(m)}-\bm{U}^{\star}\big\|_{\mathrm{F}}\\
 & \overset{(\mathrm{i})}{\lesssim}\frac{\sigma_{\max}}{\lambda_{\min}^{\star}}\sqrt{\frac{rd\log d}{p}}\,\lambda_{\max}^{\star1/3}\cdot\sqrt{\frac{\mu}{d}}\,\lambda_{\max}^{\star1/3}\cdot\lambda_{\max}^{\star1/3}+\lambda_{\max}^{\star2/3}\cdot\frac{\sigma_{\max}}{\lambda_{\min}^{\star}}\sqrt{\frac{\mu r\log d}{p}}\,\lambda_{\max}^{\star1/3}\\
 & \quad+\lambda_{\max}^{\star1/3}\cdot\sqrt{\frac{\mu}{d}}\,\lambda_{\max}^{\star1/3}\cdot\frac{\sigma_{\max}}{\lambda_{\min}^{\star}}\sqrt{\frac{rd\log d}{p}}\,\lambda_{\max}^{\star1/3}\\
 & \lesssim\frac{\sigma_{\max}}{\lambda_{\min}^{\star}}\sqrt{\frac{\mu r\log d}{p}}\,\lambda_{\max}^{\star},
\end{align*}
where (i) results from Lemmas~\ref{lemma:U-loss-property}-\ref{lemma:U-property}.

\subsubsection{Controlling $\beta_{2}$}

For each $s\in\left[r\right]$, we can further decompose
\begin{align*}
 & \big(\mathsf{unfold}\big((p^{-1}\mathcal{P}_{\Omega}-\mathcal{I})(\bm{T}-\bm{T}^{\star})\big)\big(\widetilde{\bm{U}}-\widetilde{\bm{U}}^{\left(m\right)}\big)\big)_{m,s}\\
 & \qquad=\bm{u}_{s}^{\top}\big((p^{-1}\mathcal{P}_{\Omega}-\mathcal{I})(\bm{T}-\bm{T}^{\star})\big)_{:,:,m}\bm{u}_{s}-\bm{u}_{s}^{\left(m\right)\top}\big((p^{-1}\mathcal{P}_{\Omega}-\mathcal{I})(\bm{T}-\bm{T}^{\star})\big)_{:,:,m}\bm{u}_{s}^{\left(m\right)}\\
 & \qquad=\bm{\xi}_{s}^{\left(m\right)\top}\big((p^{-1}\mathcal{P}_{\Omega}-\mathcal{I})(\bm{T}-\bm{T}^{\star})\big)_{:,:,m}\bm{u}_{s}+\bm{u}_{s}^{\left(m\right)\top}\big((p^{-1}\mathcal{P}_{\Omega}-\mathcal{I})(\bm{T}-\bm{T}^{\star})\big)_{:,:,m}\bm{\xi}_{s}^{\left(m\right)},
\end{align*}
where we recall that $\bm{\xi}_{s}^{\left(m\right)}:=\bm{u}_{s}-\bm{u}_{s}^{\left(m\right)}$.
Let us first upper bound the spectral norm of the $m$-th mode-$3$
slice of $(p^{-1}\mathcal{P}_{\Omega}-\mathcal{I})(\bm{T}-\bm{T}^{\star})$.
Recalling the notation $\bm{\Delta}:=\bm{U}-\bm{U}^{\star}$ and $\bm{\Delta}_{s}=\bm{u}_{s}-\bm{u}_{s}^{\star},1\leq s\leq r$,
one can express 
\[
\bm{T}-\bm{T}^{\star}=\sum_{1\leq s\leq r}\bm{\Delta}_{s}\otimes\bm{u}_{s}\otimes\bm{u}_{s}+\bm{u}_{s}^{\star}\otimes\bm{\Delta}_{s}\otimes\bm{u}_{s}+\bm{u}_{s}^{\star}\otimes\bm{u}_{s}^{\star}\otimes\bm{\Delta}_{s},
\]
and consequently,
\[
(\bm{T}-\bm{T}^{\star})_{:,:,m}=\bm{\Delta}\bm{F}\bm{U}^{\top}+\bm{U}^{\star}\bm{F}\bm{\Delta}^{\top}+\bm{U}^{\star}\bm{F}^{\star}\bm{\Delta}^{\top},
\]
where $\bm{F}$ and $\bm{F}^{\star}$ are diagonal matrices in $\mathbb{R}^{r\times r}$
with entries $F_{i,i}=\left(\bm{u}_{i}\right){}_{m}$ and $F_{i,i}^{\star}=\left(\bm{u}_{i}^{\star}\right)_{m}.$
Applying \cite[Lemma 4.5]{chen2017memory} yields
\[
\big\|\big((p^{-1}\mathcal{P}_{\Omega}-\mathcal{I})(\bm{T}-\bm{T}^{\star})\big)_{:,:,m}\big\|\leq\big\|\big(p^{-1}\mathcal{P}_{\Omega}(\bm{1}^{\otimes3})-\bm{1}^{\otimes3}\big)_{:,:,m}\big\|\left\Vert \bm{\Delta}\right\Vert _{2,\infty}\big(\left\Vert \bm{U}\bm{F}\right\Vert _{2,\infty}+\left\Vert \bm{U}^{\star}\bm{F}\right\Vert _{2,\infty}+\left\Vert \bm{U}^{\star}\bm{F}^{\star}\right\Vert _{2,\infty}\big).
\]
The matrix Bernstein inequality then reveals that with probability
at least $1-O\left(d^{-20}\right)$,
\[
\big\|\big(p^{-1}\mathcal{P}_{\Omega}(\bm{1}^{\otimes3})-\bm{1}^{\otimes3}\big)_{:,:,m}\big\|\lesssim\frac{\log d}{p}+\sqrt{\frac{d\log d}{p}}.
\]
Moreover, we know from (\ref{assumption:u-inf-norm}), (\ref{eq:U-true-norm})
and (\ref{eq:u-norm}) that 
\[
\left\Vert \bm{U}\bm{F}\right\Vert _{2,\infty}+\left\Vert \bm{U}^{\star}\bm{F}\right\Vert _{2,\infty}+\left\Vert \bm{U}^{\star}\bm{F}^{\star}\right\Vert _{2,\infty}\lesssim\sqrt{\frac{\mu}{d}}\,\lambda_{\max}^{\star1/3}\cdot\sqrt{\frac{\mu r}{d}}\,\lambda_{\max}^{\star1/3}.
\]
Combining this with (\ref{eq:U-loss-2inf}), we demonstrate that
\begin{align*}
\big\|\big((p^{-1}\mathcal{P}_{\Omega}-\mathcal{I})(\bm{T}-\bm{T}^{\star})\big)_{:,:,m}\big\| & \lesssim\left\{ \frac{\log d}{p}+\sqrt{\frac{d\log d}{p}}\right\} \cdot\frac{\sigma_{\max}}{\lambda_{\min}^{\star}}\sqrt{\frac{\mu r\log d}{p}}\,\lambda_{\max}^{\star1/3}\cdot\sqrt{\frac{\mu}{d}}\,\lambda_{\max}^{\star1/3}\cdot\sqrt{\frac{\mu r}{d}}\,\lambda_{\max}^{\star1/3}\\
 & \lesssim\frac{\sigma_{\max}}{\lambda_{\min}^{\star}}\left\{ \frac{\log d}{p}+\sqrt{\frac{d\log d}{p}}\right\} \sqrt{\frac{\mu^{3}r^{2}\log d}{d^{2}p}}\,\lambda_{\max}^{\star}.
\end{align*}
As a result, one can use (\ref{eq:U-U-loss-diff-fro}) and (\ref{eq:u-norm})
to bound
\begin{align*}
\big|\bm{\xi}_{s}^{\left(m\right)\top}\big((p^{-1}\mathcal{P}_{\Omega}-\mathcal{I})(\bm{T}-\bm{T}^{\star})\big)_{:,:,m}\bm{u}_{s}\big| & \leq\big\|\big((p^{-1}\mathcal{P}_{\Omega}-\mathcal{I})(\bm{T}-\bm{T}^{\star})\big)_{:,:,m}\big\|\big\|\bm{\xi}_{s}^{\left(m\right)}\big\|_{2}\big\|\bm{u}_{s}\big\|_{2}\\
 & \lesssim\sigma_{\max}\lambda_{\max}^{\star1/3}\sqrt{\frac{\mu^{3}r^{2}\log d}{d^{2}p}}\left\{ \frac{\log d}{p}+\sqrt{\frac{d\log d}{p}}\right\} \big\|\bm{\xi}_{s}^{\left(m\right)}\big\|_{2}.
\end{align*}
Clearly, the upper bound also holds for $\bm{u}_{s}^{\left(m\right)\top}\big(\big(p^{-1}\mathcal{P}_{\Omega}-\mathcal{I}\big)\big(\bm{T}-\bm{T}^{\star}\big)\big)_{:,:,m}\bm{\xi}_{s}^{\left(m\right)}$.
Summing over $s\in\left[r\right]$, we conclude that
\begin{align}
 & \big\|\bm{e}_{m}^{\top}\mathsf{unfold}\big((p^{-1}\mathcal{P}_{\Omega}-\mathcal{I})(\bm{T}-\bm{T}^{\star})\big)\big(\widetilde{\bm{U}}-\widetilde{\bm{U}}^{\left(m\right)}\big)\big\|_{2}\nonumber \\
 & \qquad\lesssim\sigma_{\max}\lambda_{\max}^{\star1/3}\sqrt{\frac{\mu^{3}r^{2}\log d}{d^{2}p}}\left\{ \frac{\log d}{p}+\sqrt{\frac{d\log d}{p}}\right\} \big\|\bm{U}-\bm{U}^{\left(m\right)}\big\|_{\mathrm{F}}\nonumber \\
 & \qquad\lesssim\sigma_{\max}\lambda_{\max}^{\star1/3}\sqrt{\frac{\mu^{3}r^{2}\log d}{d^{2}p}}\left\{ \frac{\log d}{p}+\sqrt{\frac{d\log d}{p}}\right\} \frac{\sigma_{\max}}{\lambda_{\min}^{\star}}\sqrt{\frac{\mu r\log d}{p}}\lambda_{\max}^{\star1/3}\nonumber \\
 & \qquad\ll\frac{\sigma_{\max}\lambda_{\max}^{\star2/3}}{\sqrt{p}}\sqrt{\frac{\mu^{3}r^{2}\log d}{d^{2}p}}.\label{eq:beta3-gamma2-UB}
\end{align}
Here, the last inequality holds due to our assumptions $p\gg\mu^{2}d^{-3/2}\log^{3}d$
and $\sigma_{\max}/\lambda_{\min}^{\star}\ll\sqrt{p}/d^{3/4}$ and
$\kappa\asymp1$.

\subsubsection{Controlling $\beta_{3}$}

For each $s\in\left[r\right]$, we have
\[
\big(\mathsf{unfold}\big((p^{-1}\mathcal{P}_{\Omega}-\mathcal{I})\big(\bm{T}-\bm{T}^{\left(m\right)}\big)\big)\widetilde{\bm{U}}^{\left(m\right)}\big)_{m,s}=\bm{u}_{s}^{\left(m\right)\top}\big((p^{-1}\mathcal{P}_{\Omega}-\mathcal{I})\big(\bm{T}-\bm{T}^{\left(m\right)}\big)\big)_{:,:,m}\bm{u}_{s}^{\left(m\right)}.
\]
Recall the definition of $\bm{\xi}_{s}^{\left(m\right)}$(cf.~\ref{def:xi},
we decompose
\begin{align}
\bm{T}-\bm{T}^{\left(m\right)} & =\sum_{1\leq s\leq r}\big(\bm{u}_{s}^{\left(m\right)}+\bm{\xi}_{s}^{\left(m\right)}\big)^{\otimes3}-\big(\bm{u}_{s}^{\left(m\right)}\big)^{\otimes3}\nonumber \\
 & =\sum_{1\leq s\leq r}\bm{\xi}_{s}^{\left(m\right)}\otimes\bm{u}_{s}^{\left(m\right)}\otimes\bm{u}_{s}^{\left(m\right)}+\bm{u}_{s}^{\left(m\right)}\otimes\bm{\xi}_{s}^{\left(m\right)}\otimes\bm{u}_{s}^{\left(m\right)}+\bm{u}_{s}^{\left(m\right)}\otimes\bm{u}_{s}^{\left(m\right)}\otimes\bm{\xi}_{s}^{\left(m\right)}\nonumber \\
 & \qquad\quad+\bm{\xi}_{s}^{\left(m\right)}\otimes\bm{\xi}_{s}^{\left(m\right)}\otimes\bm{u}_{s}^{\left(m\right)}+\bm{\xi}_{s}^{\left(m\right)}\otimes\bm{u}_{s}^{\left(m\right)}\otimes\bm{\xi}_{s}^{\left(m\right)}+\bm{u}_{s}^{\left(m\right)}\otimes\bm{\xi}_{s}^{\left(m\right)}\otimes\bm{\xi}_{s}^{\left(m\right)}+\bm{\xi}_{s}^{\left(m\right)}\otimes\bm{\xi}_{s}^{\left(m\right)}\otimes\bm{\xi}_{s}^{\left(m\right)}.\label{eq:T-T-loo-diff-decomp}
\end{align}
In view of the triangle inequality, it suffices to control these terms
separately.
\begin{itemize}
\item Let us first consider the terms which are linear in terms of $\bm{\xi}_{s}^{\left(m\right)}$.
For $\sum_{1\leq s\leq r}\bm{\xi}_{s}^{\left(m\right)}\otimes\bm{u}_{s}^{\left(m\right)}\otimes\bm{u}_{s}^{\left(m\right)}$,
one can write
\begin{align*}
 & \bm{u}_{s}^{\left(m\right)\top}\Big\{(p^{-1}\mathcal{P}_{\Omega}-\mathcal{I})\Big(\sum_{1\leq s\leq r}\bm{\xi}_{s}^{\left(m\right)}\otimes\bm{u}_{s}^{\left(m\right)}\otimes\bm{u}_{s}^{\left(m\right)}\Big)\Big\}_{:,:,m}\bm{u}_{s}^{\left(m\right)}\\
 & \qquad=\sum_{1\leq i,j\leq d}\sum_{1\leq s\leq r}\big(p^{-1}\chi_{i,j,m}-1\big)\big(\bm{u}_{s}^{\left(m\right)}\big)_{m}\big(\bm{u}_{s}^{\left(m\right)}\big)_{i}\big(\bm{\xi}_{s}^{\left(m\right)}\big)_{i}\big(\bm{u}_{s}^{\left(m\right)}\big)_{j}^{2}\\
 & \qquad=\sum_{1\leq s\leq r}\big(\bm{u}_{s}^{\left(m\right)}\big)_{m}\sum_{1\leq i\leq d}\big(\bm{\xi}_{s}^{\left(m\right)}\big)_{i}\sum_{1\leq j\leq d}\big(\bm{u}_{s}^{\left(m\right)}\big)_{i}\big(\bm{u}_{s}^{\left(m\right)}\big)_{j}^{2}\big(p^{-1}\chi_{i,j,m}-1\big).
\end{align*}
We shall use the Cauchy-Schwartz inequality to upper bound the absolute
value of the quantity above. By construction, $\bm{u}_{s}^{\left(m\right)}$
is independent of $\{\chi_{i,j,m}\}_{1\leq i,j\leq d}$. Applying
a similar argument in \cite[Lemma D.9]{cai2019nonconvex}, we know
that with probability at least $1-O\left(d^{-20}\right)$,
\[
\sum_{1\leq i\leq d}\Big|\sum_{1\leq j\leq d}\big(\bm{u}_{s}^{\left(m\right)}\big)_{i}\big(\bm{u}_{s}^{\left(m\right)}\big)_{j}^{2}\big(p^{-1}\chi_{i,j,m}-1\big)\Big|^{2}\lesssim\frac{1}{p}\,\big\|\bm{u}_{s}^{\left(m\right)}\big\|_{\infty}^{2}\big\|\bm{u}_{s}^{\left(m\right)}\big\|_{2}^{4}.
\]
This leads to the following upper bound
\begin{align*}
 & \Big|\sum_{1\leq i\leq d}\big(\bm{\xi}_{s}^{\left(m\right)}\big)_{i}\sum_{1\leq j\leq d}\big(\bm{u}_{s}^{\left(m\right)}\big)_{i}\big(\bm{u}_{s}^{\left(m\right)}\big)_{j}^{2}\big(p^{-1}\chi_{i,j,m}-1\big)\Big|^{2}\\
 & \qquad\leq\sum_{1\leq i\leq d}\big(\bm{\xi}_{s}^{\left(m\right)}\big)_{i}^{2}\sum_{1\leq i\leq d}\Big|\sum_{1\leq j\leq d}\big(\bm{u}_{s}^{\left(m\right)}\big)_{i}^{2}\big(\bm{u}_{s}^{\left(m\right)}\big)_{j}^{2}\big(p^{-1}\chi_{i,j,m}-1\big)\Big|^{2}\\
 & \qquad\lesssim\frac{1}{p}\,\big\|\bm{\xi}_{s}^{\left(m\right)}\big\|_{2}^{2}\big\|\bm{u}_{s}^{\left(m\right)}\big\|_{\infty}^{2}\big\|\bm{u}_{s}^{\left(m\right)}\big\|_{2}^{4}.
\end{align*}
It then follows that
\begin{align}
 & \Big\|\Big[\bm{u}_{s}^{\left(m\right)\top}\Big\{\big(p^{-1}\mathcal{P}_{\Omega}-\mathcal{I}\big)\Big(\sum_{1\leq s\leq r}\bm{\xi}_{s}^{\left(m\right)}\otimes\bm{u}_{s}^{\left(m\right)}\otimes\bm{u}_{s}^{\left(m\right)}\Big)\Big\}_{:,:,m}\bm{u}_{s}^{\left(m\right)}\Big]_{1\leq s\leq r}\Big\|_{2}^{2}\nonumber \\
 & \qquad\leq\sum_{1\leq s\leq r}\big(\bm{u}_{s}^{\left(m\right)}\big)_{m}^{2}\sum_{1\leq s\leq r}\Big|\sum_{1\leq i\leq d}\big(\bm{\xi}_{s}^{\left(m\right)}\big)_{i}\sum_{1\leq j\leq d}\big(\bm{u}_{s}^{\left(m\right)}\big)_{i}\big(\bm{u}_{s}^{\left(m\right)}\big)_{j}^{2}\big(p^{-1}\chi_{i,j,m}-1\big)\Big|^{2}\nonumber \\
 & \qquad\lesssim\frac{1}{p}\sum_{1\leq s\leq r}\big(\bm{u}_{s}^{\left(m\right)}\big)_{m}^{2}\sum_{1\leq s\leq r}\big\|\bm{\xi}_{s}^{\left(m\right)}\big\|_{2}^{2}\big\|\bm{u}_{s}^{\left(m\right)}\big\|_{\infty}^{2}\big\|\bm{u}_{s}^{\left(m\right)}\big\|_{2}^{4}\nonumber \\
 & \qquad\leq\frac{1}{p}\,\big\|\bm{U}^{\left(m\right)}\big\|_{2,\infty}^{2}\max_{1\leq s\leq r}\big\|\bm{u}_{s}^{\left(m\right)}\big\|_{\infty}^{2}\max_{1\leq s\leq r}\big\|\bm{u}_{s}^{\left(m\right)}\big\|_{2}^{4}\big\|\bm{U}-\bm{U}^{\left(m\right)}\big\|_{\mathrm{F}}^{2}\nonumber \\
 & \qquad\overset{}{\lesssim}\frac{1}{p}\cdot\frac{\mu r\lambda_{\max}^{\star2/3}}{d}\cdot\frac{\mu\lambda_{\max}^{\star2/3}}{d}\cdot\lambda_{\max}^{\star4/3}\cdot\frac{\sigma_{\max}^{2}}{\lambda_{\min}^{\star2}}\frac{\mu r\log d\lambda_{\max}^{\star2/3}}{p}\lesssim\frac{\sigma_{\max}^{2}\lambda_{\max}^{\star10/3}}{\lambda_{\min}^{\star2}p}\frac{\mu^{3}r^{2}\log d}{d^{2}p},\label{eq:P-Omega-I-T-loss-quad-term-term1}
\end{align}
where we use Lemmas~\ref{lemma:U-loo-property} and \ref{lemma:U-property}
in (\ref{assumption:u-inf-norm}). Clearly, the upper bound is also
valid for $\sum_{1\leq s\leq r}\bm{u}_{s}^{\left(m\right)}\otimes\bm{\xi}_{s}^{\left(m\right)}\otimes\bm{u}_{s}^{\left(m\right)}$.
In an analogous manner, one can show that with probability exceeding
$1-O\left(d^{-20}\right)$,
\begin{align}
 & \Big\|\Big[\bm{u}_{s}^{\left(m\right)\top}\Big\{\big(p^{-1}\mathcal{P}_{\Omega}-\mathcal{I}\big)\Big(\sum_{1\leq s\leq r}\bm{u}_{s}^{\left(m\right)}\otimes\bm{u}_{s}^{\left(m\right)}\otimes\bm{\xi}_{s}^{\left(m\right)}\Big)\Big\}_{:,:,m}\bm{u}_{s}^{\left(m\right)}\Big]_{1\leq s\leq r}\Big\|_{2}^{2}\nonumber \\
 & \qquad\lesssim\sum_{1\leq s\leq r}\big(\bm{\xi}_{s}^{\left(m\right)}\big)_{m}^{2}\sum_{1\leq s\leq r}\Big|\sum_{1\leq i,j\leq d}\big(\bm{u}_{s}^{\left(m\right)}\big)_{i}^{2}\big(\bm{u}_{s}^{\left(m\right)}\big)_{j}^{2}\big(p^{-1}\chi_{i,j,m}-1\big)\Big|^{2}\nonumber \\
 & \qquad\lesssim\frac{1}{p}\sum_{1\leq s\leq r}\big(\bm{\xi}_{s}^{\left(m\right)}\big)_{m}^{2}\sum_{1\leq s\leq r}\big\|\bm{u}_{s}^{\left(m\right)}\big\|_{\infty}^{4}\big\|\bm{u}_{s}^{\left(m\right)}\big\|_{2}^{4}\nonumber \\
 & \qquad\leq\frac{1}{p}\big\|\bm{U}-\bm{U}^{\left(m\right)}\big\|_{\mathrm{\mathrm{F}}}^{2}\max_{1\leq s\leq r}\big\|\bm{u}_{s}^{\left(m\right)}\big\|_{\infty}^{4}\max_{1\leq s\leq r}\big\|\bm{u}_{s}^{\left(m\right)}\big\|_{2}^{2}\big\|\bm{U}^{(m)}\big\|_{\mathrm{F}}^{2}\nonumber \\
 & \qquad\lesssim\frac{1}{p}\cdot\frac{\sigma_{\max}^{2}}{\lambda_{\min}^{\star2}}\frac{\mu r\log d\lambda_{\max}^{\star2/3}}{p}\cdot\frac{\mu^{2}\lambda_{\max}^{\star4/3}}{d^{2}}\cdot\lambda_{\max}^{\star2/3}\cdot r\lambda_{\max}^{\star4/3}\lesssim\frac{\sigma_{\max}^{2}\lambda_{\max}^{\star10/3}}{\lambda_{\min}^{\star2}p}\frac{\mu^{3}r^{2}\log d}{d^{2}p},\label{eq:P-Omega-I-T-loss-quad-term-term2}
\end{align}
where the last step arises from use (\ref{eq:U-loo-property}) and
(\ref{eq:U-property}).
\item Next, we turn to the quadratic terms with respect to $\bm{\xi}_{s}^{\left(m\right)}$
in (\ref{eq:T-T-loo-diff-decomp}). For $\sum_{1\leq s\leq r}\bm{\xi}_{s}^{\left(m\right)}\otimes\bm{\xi}_{s}^{\left(m\right)}\otimes\bm{u}_{s}^{\left(m\right)}$,
we can expand
\begin{align*}
 & \bm{u}_{s}^{\left(m\right)\top}\Big\{\big(p^{-1}\mathcal{P}_{\Omega}-\mathcal{I}\big)\Big(\sum_{1\leq s\leq r}\bm{\xi}_{s}^{\left(m\right)}\otimes\bm{\xi}_{s}^{\left(m\right)}\otimes\bm{u}_{s}^{\left(m\right)}\Big)\Big\}_{:,:,m}\bm{u}_{s}^{\left(m\right)}\\
 & \quad=\sum_{1\leq s\leq r}\big(\bm{u}_{s}^{\left(m\right)}\big)_{m}\sum_{1\leq i,j\leq d}\big(\bm{\xi}_{s}^{\left(m\right)}\big)_{i}\big(\bm{\xi}_{s}^{\left(m\right)}\big)_{j}\big(\bm{u}_{s}^{\left(m\right)}\big)_{i}\big(\bm{u}_{s}^{\left(m\right)}\big)_{j}\big(p^{-1}\chi_{i,j,m}-1\big).
\end{align*}
Use the Cauchy-Schwartz inequality and the Bernstein inequality again
to yield
\begin{align*}
 & \Big|\sum_{1\leq i,j\leq d}\big(\bm{\xi}_{s}^{\left(m\right)}\big)_{i}\big(\bm{\xi}_{s}^{\left(m\right)}\big)_{j}\big(\bm{u}_{s}^{\left(m\right)}\big)_{i}\big(\bm{u}_{s}^{\left(m\right)}\big)_{j}\big(p^{-1}\chi_{i,j,m}-1\big)\Big|^{2}\\
 & \qquad\leq\sum_{1\leq i,j\leq d}\big(\bm{\xi}_{s}^{\left(m\right)}\big)_{i}^{2}\big(\bm{\xi}_{s}^{\left(m\right)}\big)_{j}^{2}\sum_{1\leq i,j\leq d}\big(\bm{u}_{s}^{\left(m\right)}\big)_{i}^{2}\big(\bm{u}_{s}^{\left(m\right)}\big)_{j}^{2}\big(p^{-1}\chi_{i,j,m}-1\big)^{2}\\
 & \qquad\lesssim\frac{1}{p}\,\big\|\bm{\xi}_{s}^{\left(m\right)}\big\|_{2}^{4}\,\big\|\bm{u}_{s}^{\left(m\right)}\big\|_{2}^{4}
\end{align*}
with probability at least $1-O\left(d^{-20}\right)$. Consequently,
we find that
\begin{align}
 & \Big\|\Big[\bm{u}_{s}^{\left(m\right)\top}\Big\{\big(p^{-1}\mathcal{P}_{\Omega}-\mathcal{I}\big)\Big(\sum_{1\leq s\leq r}\bm{\xi}_{s}^{\left(m\right)}\otimes\bm{\xi}_{s}^{\left(m\right)}\otimes\bm{u}_{s}^{\left(m\right)}\Big)\Big\}_{:,:,m}\bm{u}_{s}^{\left(m\right)}\Big]_{1\leq s\leq r}\Big\|_{2}^{2}\nonumber \\
 & \qquad\lesssim\frac{1}{p}\sum_{1\leq s\leq r}\big(\bm{u}_{s}^{\left(m\right)}\big)_{m}^{2}\sum_{1\leq s\leq r}\big\|\bm{\xi}_{s}^{\left(m\right)}\big\|_{2}^{4}\big\|\bm{u}_{s}^{\left(m\right)}\big\|_{2}^{4}\nonumber \\
 & \qquad\lesssim\frac{1}{p}\big\|\bm{U}^{\left(m\right)}\big\|_{2,\infty}^{2}\max_{1\leq s\leq r}\big\|\bm{u}_{s}^{\left(m\right)}\big\|_{2}^{4}\max_{1\leq s\leq r}\big\|\bm{u}_{s}-\bm{u}_{s}^{\left(m\right)}\big\|_{\infty}^{2}\big\|\bm{U}-\bm{U}^{\left(m\right)}\big\|_{\mathrm{F}}^{2}\nonumber \\
 & \qquad\overset{}{\lesssim}\frac{1}{p}\cdot\frac{\mu r\lambda_{\max}^{\star2/3}}{d}\cdot\lambda_{\max}^{\star4/3}\cdot\frac{\mu\lambda_{\max}^{\star2/3}}{d}\cdot\frac{\sigma_{\max}^{2}}{\lambda_{\min}^{\star2}}\frac{\mu r\log d\lambda_{\max}^{\star2/3}}{p}\lesssim\frac{\sigma_{\max}^{2}\lambda_{\max}^{\star10/3}}{\lambda_{\min}^{\star2}p}\frac{\mu^{3}r^{2}\log d}{d^{2}p},\label{eq:P-Omega-I-T-loss-quad-term-term3}
\end{align}
where the last line holds true due to Lemmas~\ref{lemma:U-loo-property}
and \ref{lemma:U-property}, and $\big\|\bm{u}_{s}-\bm{u}_{s}^{\left(m\right)}\big\|_{\mathrm{\infty}}\leq\|\bm{u}_{s}\|_{\mathrm{\infty}}+\|\bm{u}_{s}^{\left(m\right)}\|_{\mathrm{\infty}}\lesssim\sqrt{\mu/d}\,\lambda_{\max}^{\star1/3}$.
Using a similar argument for (\ref{eq:P-Omega-I-T-loss-quad-term-term1}),
one can verify that with probability at least $1-O\left(d^{-20}\right),$
\begin{align}
 & \Big\|\Big[\bm{u}_{s}^{\left(m\right)\top}\Big\{(p^{-1}\mathcal{P}_{\Omega}-\mathcal{I})\Big(\sum_{1\leq s\leq r}\bm{\xi}_{s}^{\left(m\right)}\otimes\bm{u}_{s}^{\left(m\right)}\otimes\bm{\xi}_{s}^{\left(m\right)}\Big)\Big\}_{:,:,m}\bm{u}_{s}^{\left(m\right)}\Big]_{1\leq s\leq r}\Big\|_{2}^{2}\nonumber \\
 & \qquad\lesssim\frac{1}{p}\big\|\bm{U}-\bm{U}^{\left(m\right)}\big\|_{2,\infty}^{2}\max_{1\leq s\leq r}\big\|\bm{u}_{s}^{\left(m\right)}\big\|_{\infty}^{2}\max_{1\leq s\leq r}\big\|\bm{u}_{s}^{\left(m\right)}\big\|_{2}^{4}\big\|\bm{U}-\bm{U}^{\left(m\right)}\big\|_{\mathrm{F}}^{2}\nonumber \\
 & \qquad\overset{}{\lesssim}\frac{1}{p}\cdot\frac{\mu r\lambda_{\max}^{\star2/3}}{d}\cdot\frac{\mu\lambda_{\max}^{\star2/3}}{d}\cdot\lambda_{\max}^{\star4/3}\cdot\frac{\sigma_{\max}^{2}}{\lambda_{\min}^{\star2}}\frac{\mu r\log d\lambda_{\max}^{\star2/3}}{p}\lesssim\frac{\sigma_{\max}^{2}\lambda_{\max}^{\star10/3}}{\lambda_{\min}^{\star2}p}\frac{\mu^{3}r^{2}\log d}{d^{2}p},\label{eq:P-Omega-I-T-loss-quad-term-term4}
\end{align}
where we use (\ref{eq:U-loo-property}), (\ref{eq:U-property}) and
$\|\bm{U}-\bm{U}^{\left(m\right)}\|_{2,\infty}\leq\|\bm{U}\|_{2,\infty}+\|\bm{U}^{\left(m\right)}\|_{2,\infty}\lesssim\sqrt{\mu r/d}\lambda_{\max}^{\star1/3}$
in the last step. Note that the same bound also holds for $\sum_{1\leq s\leq r}\bm{u}_{s}^{\left(m\right)}\otimes\bm{\xi}_{s}^{\left(m\right)}\otimes\bm{\xi}_{s}^{\left(m\right)}$.
\item As for the cubic term in (\ref{eq:T-T-loo-diff-decomp}), arguing
similarly as in (\ref{eq:P-Omega-I-T-loss-quad-term-term3}), we know
that with probability greater than $1-O\left(d^{-20}\right),$
\begin{align}
 & \Big\|\big[\bm{u}_{s}^{\left(m\right)\top}\Big\{\big(p^{-1}\mathcal{P}_{\Omega}-\mathcal{I}\big)\Big(\sum_{1\leq s\leq r}\bm{\xi}_{s}^{\left(m\right)}\otimes\bm{\xi}_{s}^{\left(m\right)}\otimes\bm{\xi}_{s}^{\left(m\right)}\Big)\Big\}_{:,:,m}\bm{u}_{s}^{\left(m\right)}\big]_{s}\Big\|_{2}^{2}\nonumber \\
 & \qquad\lesssim\frac{1}{p}\sum_{1\leq s\leq r}\big(\bm{\xi}_{s}^{\left(m\right)}\big)_{m}^{2}\sum_{1\leq s\leq r}\big\|\bm{\xi}_{s}^{\left(m\right)}\big\|_{2}^{4}\big\|\bm{u}_{s}^{\left(m\right)}\big\|_{2}^{4}\nonumber \\
 & \qquad\lesssim\frac{1}{p}\big\|\bm{U}-\bm{U}^{\left(m\right)}\big\|_{2,\infty}^{2}\max_{1\leq s\leq r}\big\|\bm{u}_{s}^{\left(m\right)}\big\|_{2}^{4}\max_{1\leq s\leq r}\big\|\bm{\xi}_{s}^{\left(m\right)}\big\|_{\infty}^{2}\big\|\bm{U}-\bm{U}^{\left(m\right)}\big\|_{\mathrm{F}}^{2}\nonumber \\
 & \qquad\lesssim\frac{1}{p}\cdot\frac{\mu r\lambda_{\max}^{\star2/3}}{d}\cdot\lambda_{\max}^{\star4/3}\cdot\frac{\mu\lambda_{\max}^{\star2/3}}{d}\cdot\frac{\sigma_{\max}^{2}}{\lambda_{\min}^{\star2}}\frac{\mu r\log d\lambda_{\max}^{\star2/3}}{p}\lesssim\frac{\sigma_{\max}^{2}\lambda_{\max}^{\star10/3}}{\lambda_{\min}^{\star2}p}\frac{\mu^{3}r^{2}\log d}{d^{2}p},\label{eq:P-Omega-I-T-loss-quad-term-term5}
\end{align}
where we use (\ref{eq:U-loo-property}), (\ref{eq:U-property}) and
$\|\bm{U}-\bm{U}^{\left(m\right)}\|_{2,\infty}\leq\|\bm{U}\|_{2,\infty}+\|\bm{U}^{\left(m\right)}\|_{2,\infty}\lesssim\sqrt{\mu r/d}\lambda_{\max}^{\star1/3}$
in (\ref{eq:P-Omega-I-T-loss-quad-term-term5}).
\item Putting the results (\ref{eq:P-Omega-I-T-loss-quad-term-term1})-(\ref{eq:P-Omega-I-T-loss-quad-term-term5})
together with the condition $\kappa\asymp1$ reveals that: with probability
at least $1-O\left(d^{-20}\right)$,
\begin{equation}
\big\|\bm{e}_{m}^{\top}\mathsf{unfold}\big((p^{-1}\mathcal{P}_{\Omega}-\mathcal{I})(\bm{T}-\bm{T}^{\left(m\right)})\big)\widetilde{\bm{U}}^{\left(m\right)}\big\|_{2}\lesssim\frac{\sigma_{\max}\lambda_{\max}^{\star2/3}}{\sqrt{p}}\sqrt{\frac{\mu^{3}r^{2}\log^{2}d}{d^{2}p}}.\label{eq:beta3-gamma3-UB}
\end{equation}
\end{itemize}

\subsubsection{Combining the bounds on $\beta_{1}$, $\beta_{2}$ and $\beta_{3}$}

Substituting (\ref{eq:beta3-gamma1-UB}), (\ref{eq:beta3-gamma2-UB})
and (\ref{eq:beta3-gamma3-UB}) into (\ref{eq:beta3-UB-temp}), and
taking the union bound over $m\in\left[d\right]$, we conclude that
\[
\big\|\mathsf{unfold}\left(\big(p^{-1}\mathcal{P}_{\Omega}-\mathcal{I}\big)\left(\bm{T}-\bm{T}^{\star}\right)\right)\widetilde{\bm{U}}\big(\widetilde{\bm{U}}^{\top}\widetilde{\bm{U}}\big)^{-1}\big\|_{2,\infty}\lesssim\frac{\sigma_{\max}}{\lambda_{\min}^{\star2/3}\sqrt{p}}\sqrt{\frac{\mu^{3}r^{2}\log^{2}d}{d^{2}p}}
\]
with probability at least $1-O\left(d^{-10}\right)$, provided that
$\kappa\asymp1$.

\subsection{Proof of Lemma~\ref{lemma:U-loss-dist-W4}}

\label{subsec:U-loss-dist-W4}

Without loss of generality, assume that $\bm{\Pi}=\bm{I}_{r}$. By
(\ref{eq:grad-U-fro-UB}), it is straightforward to invoke (\ref{eq:U-tilde-spectrum})
to obtain
\begin{align*}
\big\|\nabla g(\bm{U})(\widetilde{\bm{U}}^{\top}\widetilde{\bm{U}})^{-1}\big\|_{2,\infty} & \leq\big\|\nabla g(\bm{U})\big\|_{2,\infty}\big\|(\widetilde{\bm{U}}^{\top}\widetilde{\bm{U}})^{-1}\big\|\leq\big\|\nabla g(\bm{U})\big\|_{\mathrm{F}}\big\|(\widetilde{\bm{U}}^{\top}\widetilde{\bm{U}})^{-1}\big\|\\
 & \lesssim\sigma_{\max}\lambda_{\max}^{\star2/3}\sqrt{\frac{d}{p}}\,\frac{1}{d}\cdot\frac{1}{\lambda_{\min}^{\star4/3}}\lesssim\frac{\sigma_{\max}}{\lambda_{\min}^{\star2/3}}\sqrt{\frac{1}{p}}\frac{1}{\sqrt{d}},
\end{align*}
where the last step follows from the condition $\kappa\asymp1$.

%% file: proof-T-dist.tex
\section{Proof of auxiliary lemmas: distributional theory for tensor entries}

\label{sec:Analysis-of-T}

\begin{comment}
@@@

\subsection{Proof of Theorem~\ref{thm:T-loss-dist}}

\label{subsec:T-loss-dist}

With the help of Lemmas \ref{lemma:T-loss-dist-main-part} and \ref{lemma:T-loss-dist-neg-part},
we know that for any $1\leq i\leq j\leq k\leq d$, $T_{i,j,k}-T_{i,j,k}^{\star}=G_{i,j,k}+H_{i,j,k}+R_{i,j,k}$
where $G_{i,j,k}$ converges to a normal distribution, and $H_{i,j,k}$,
$R_{i,j,k}$ are exceedingly small in magnitude with high probability.
\end{comment}

\subsection{Proof of Lemma \ref{lemma:T-loss-dist-main-part}}

\label{subsec:T-loss-dist-main-part}

Fix arbitrary $1\leq m\leq n\leq l\leq d$. In what follows, we shall
focus on the case $m<n<l$. The analysis naturally extends to the
case where $m=n<l,m<n=l$ or $m=n=l$.

Before we embark on the proof, we remind the readers of the definitions
of $\bm{z}$ (resp.~$\bm{Z}$) in (\ref{def:z}) (resp.~(\ref{def:Z})).
While $\bm{Z}_{m,:}$, $\bm{Z}_{n,:}$ and $\bm{Z}_{l,:}$ are not
mutually independent due to the symmetric sampling, we can show that
the dependence between them are extremely weak. This in turn allows
us to invoke the Berry-Esseen theorem to prove the advertised distributional
guarantees.

We now begin to present our analysis. To decouple the weak dependence,
we define the following auxiliary random vector:
\begin{equation}
\widehat{\bm{Z}}_{m,:}=\breve{\bm{Z}}_{m,:}\breve{\bm{\Sigma}}_{m}^{-1/2}\bm{\Sigma}_{m}^{\star1/2}
\end{equation}
with
\begin{align*}
\breve{\bm{Z}}_{m,:} & :=\sqrt{2}\sum_{i:i\neq n,l}\bm{z}_{i,i,m}+2\sum_{i:i\neq m,n,l}\bm{z}_{i,m,m}+2\sum_{\substack{(i,j):i,j\neq m,n,l\\
1\leq i<j\leq d
}
}\bm{z}_{i,j,m},\qquad\text{and}\qquad\breve{\bm{\Sigma}}_{m}:=\mathbb{E}\big[(\breve{\bm{Z}}_{m,:})^{\top}\breve{\bm{Z}}_{m,:}\big].
\end{align*}
The vectors $\widehat{\bm{Z}}_{n,:}$ and $\widehat{\bm{Z}}_{l,:}$
are defined in a similar manner. By construction, it is easy to verify
that $\widehat{\bm{Z}}_{m,:},\widehat{\bm{Z}}_{n,:}$ and $\widehat{\bm{Z}}_{l,:}$
are mutually independent. Moreover, Lemma~\ref{lemma:Z-row-near-indep}
as stated below reveals that the constructed auxiliary vectors are
sufficiently close to the original ones.

\begin{lemma}\label{lemma:Z-row-near-indep} Instate the assumptions
and notation of Lemma~\ref{lemma:T-loss-dist-main-part}. With probability
at least $1-O\left(d^{-13}\right)$, one has
\[
\big\|(\bm{Z}-\widehat{\bm{Z}})_{m,:}\big\|_{2}\lesssim\frac{\sigma_{\max}}{\lambda_{\min}^{\star2/3}\sqrt{p}}\left\{ \frac{\mu\sqrt{r}\log^{2}d}{d\sqrt{p}}+\sqrt{\frac{\mu r\log d}{d}}+\frac{\mu r^{3/2}\sqrt{\log d}}{d}\right\} .
\]
In addition, the upper bound continues to hold for $(\bm{Z}-\widehat{\bm{Z}})_{n,:}$
and $(\bm{Z}-\widehat{\bm{Z}})_{l,:}$.\end{lemma}

With these in place, we then define random variables
\begin{align*}
G_{m,n,l} & :=\big\langle\widehat{\bm{Z}}_{m,:},\widetilde{\bm{U}}_{(n,l),:}^{\star}\big\rangle+\big\langle\widehat{\bm{Z}}_{n,:},\widetilde{\bm{U}}_{(m,l),:}^{\star}\big\rangle+\big\langle\widehat{\bm{Z}}_{l,:},\widetilde{\bm{U}}_{(m,n),:}^{\star}\big\rangle;\\
H_{m,n,l} & :=Y_{m,n,l}-G_{m,n,l}=\big\langle(\bm{Z}-\widehat{\bm{Z}})_{m,:},\widetilde{\bm{U}}_{(n,l),:}^{\star}\big\rangle+\big\langle(\bm{Z}-\widehat{\bm{Z}})_{n,:},\widetilde{\bm{U}}_{(m,l),:}^{\star}\big\rangle+\big\langle(\bm{Z}-\widehat{\bm{Z}})_{l,:},\widetilde{\bm{U}}_{(m,n),:}^{\star}\big\rangle.
\end{align*}
By construction, $G_{m,n,l}$ is a sum of independent zero-mean random
variables with variance $v_{m,n,l}^{\star}$ defined in (\ref{def:T-entry-var}).
In the sequel, we shall apply the Berry-Esseen theorem \cite[Theorem 1.1]{bentkus2005lyapunov}
(cf.~Appendix~\ref{subsec:Berry-Esseen-theorem}) to show that $G_{m,n,l}$
is close in distribution to a Gaussian random variable. As before,
we need to control the quantity $\rho$ defined in (\ref{def:Berry-Esseen-rho}).
From (\ref{eq:U-true-tilde-norm}) and (\ref{eq:U-true-tilde-spectrum}),
it is straightforward to upper bound
\begin{align*}
 & \frac{1}{p^{3}}\sum_{1\leq i,j\leq d}\mathbb{E}\big[|E_{i,j,k}|^{3}\chi_{i,j,k}\big]\big|\widetilde{\bm{U}}_{(i,j),:}^{\star}(\widetilde{\bm{U}}^{\star\top}\widetilde{\bm{U}}^{\star})^{-1}\big(\widetilde{\bm{U}}_{(n,l),:}^{\star}\big)^{\top}\big|^{3}\\
 & \qquad\lesssim\frac{\sigma_{\max}^{3}}{p^{2}}\max_{1\leq i,j\leq d}\big|\widetilde{\bm{U}}_{(i,j),:}^{\star}(\widetilde{\bm{U}}^{\star\top}\widetilde{\bm{U}}^{\star})^{-1}\big(\widetilde{\bm{U}}_{(n,l),:}^{\star}\big)^{\top}\big|\widetilde{\bm{U}}_{(n,l),:}^{\star}(\widetilde{\bm{U}}^{\star\top}\widetilde{\bm{U}}^{\star})^{-1}\big(\widetilde{\bm{U}}_{(n,l),:}^{\star}\big)^{\top}\\
 & \qquad\leq\frac{\sigma_{\max}^{3}}{p^{2}}\big\|\widetilde{\bm{U}}^{\star}\big\|_{2,\infty}\big\|(\widetilde{\bm{U}}^{\star\top}\widetilde{\bm{U}}^{\star})^{-1}\big\|^{2}\big\|\big(\widetilde{\bm{U}}_{(n,l),:}^{\star}\big)\big\|_{2}^{3}\\
 & \qquad\lesssim\frac{\sigma_{\max}^{3}}{p^{2}}\cdot\frac{\mu\sqrt{r}}{d}\lambda_{\max}^{\star2/3}\cdot\Big(\lambda_{\min}^{\star-4/3}\Big)^{2}\big\|\big(\widetilde{\bm{U}}_{(n,l),:}^{\star}\big)\big\|_{2}^{3}\lesssim\frac{\sigma_{\max}^{3}\lambda_{\max}^{\star2/3}}{\lambda_{\min}^{\star8/3}}\frac{\mu\sqrt{r}}{dp^{2}}\big\|\big(\widetilde{\bm{U}}_{(n,l),:}^{\star}\big)\big\|_{2}^{3}.
\end{align*}
In addition, we can use (\ref{eq:cov-matrix-eigval}) to lower bound
the variance as follows
\begin{align}
v_{m,n,l}^{\star} & \geq\lambda_{\min}\big(\bm{\Sigma}_{m}^{\star}\big)\big\|\widetilde{\bm{U}}_{(n,l),:}^{\star}\big\|_{2}^{2}+\lambda_{\min}\big(\bm{\Sigma}_{n}^{\star}\big)\big\|\widetilde{\bm{U}}_{(m,l),:}^{\star}\big\|_{2}^{2}+\lambda_{\min}\big(\bm{\Sigma}_{l}^{\star}\big)\big\|\widetilde{\bm{U}}_{(m,n),:}^{\star}\big\|_{2}^{2}\nonumber \\
 & \gtrsim\frac{\sigma_{\min}^{2}}{\lambda_{\max}^{\star4/3}p}\Big(\big\|\widetilde{\bm{U}}_{(n,l),:}^{\star}\big\|_{2}^{2}+\big\|\widetilde{\bm{U}}_{(m,l),:}^{\star}\big\|_{2}^{2}+\big\|\widetilde{\bm{U}}_{(m,n),:}^{\star}\big\|_{2}^{2}\Big).\label{eq:T-entry-var-LB}
\end{align}
Combining these two bounds, we arrive at
\[
\rho\lesssim(v_{m,n,l}^{\star})^{-3/2}\frac{\sigma_{\max}^{3}\lambda_{\max}^{\star2/3}}{\lambda_{\min}^{\star8/3}}\frac{\mu\sqrt{r}}{dp^{2}}\Big(\big\|\widetilde{\bm{U}}_{(n,l),:}^{\star}\big\|_{2}^{3}+\big\|\widetilde{\bm{U}}_{(m,l),:}^{\star}\big\|_{2}^{3}+\big\|\widetilde{\bm{U}}_{(m,n),:}^{\star}\big\|_{2}^{3}\Big)\lesssim\frac{\mu\sqrt{r}}{d\sqrt{p}},
\]
using the conditions $\sigma_{\max}/\sigma_{\min}\asymp1$ and $\kappa\asymp1$.
As a consequence, invoke the Berry-Esseen-type theorem in Appendix~\ref{subsec:Berry-Esseen-theorem}
to derive
\begin{align*}
\sup_{\tau\in\mathbb{R}}\left|\mathbb{P}\Big\{ G_{m,n,l}\leq\tau\sqrt{v_{m,n,l}^{\star}}\Big\}-\Phi(\tau)\right| & \lesssim\frac{\mu\sqrt{r}}{d\sqrt{p}},
\end{align*}
where $\Phi(\cdot)$ is the CDF of a standard Gaussian random variable.

We then move on to the residual term $H_{m,n,l}$. By Lemma~\ref{lemma:Z-row-near-indep}
and the Cauchy-Schwartz inequality, one can easily upper bound
\begin{align*}
\left|H_{m,n,l}\right| & \lesssim\Big(\big\|(\bm{Z}-\widehat{\bm{Z}})_{m,:}\big\|_{2}+\big\|(\bm{Z}-\widehat{\bm{Z}})_{n,:}\big\|_{2}+\big\|(\bm{Z}-\widehat{\bm{Z}})_{l,:}\big\|_{2}\Big)\Big(\big\|\widetilde{\bm{U}}_{(n,l),:}^{\star}\big\|_{2}+\big\|\widetilde{\bm{U}}_{(m,l),:}^{\star}\big\|_{2}+\big\|\widetilde{\bm{U}}_{(m,n),:}^{\star}\big\|_{2}\Big)\\
 & \lesssim\left\{ \frac{\mu\sqrt{r}\log^{2}d}{d\sqrt{p}}+\sqrt{\frac{\mu r\log d}{d}}+\frac{\mu r^{3/2}\sqrt{\log d}}{d}\right\} \frac{\sigma_{\max}}{\lambda_{\min}^{\star2/3}\sqrt{p}}\Big(\big\|\widetilde{\bm{U}}_{(n,l),:}^{\star}\big\|_{2}+\big\|\widetilde{\bm{U}}_{(m,n),:}^{\star}\big\|_{2}+\big\|\widetilde{\bm{U}}_{(m,n),:}^{\star}\big\|_{2}\Big)\\
 & \lesssim\left\{ \frac{\mu\sqrt{r}\log^{2}d}{d\sqrt{p}}+\sqrt{\frac{\mu r\log d}{d}}+\frac{\mu r^{3/2}\sqrt{\log d}}{d}\right\} \sqrt{v_{m,n,l}^{\star}},
\end{align*}
where the last step arises from (\ref{eq:T-entry-var-LB}) and the
conditions $\sigma_{\max}/\sigma_{\min}\asymp1$ and $\kappa\asymp1$. 

\subsubsection{Proof of Lemma \ref{lemma:Z-row-near-indep}}

By the triangle inequality, we have
\[
\big\|(\bm{Z}-\widehat{\bm{Z}})_{m,:}\big\|_{2}\leq\underbrace{\big\|\big(\bm{Z}-\breve{\bm{Z}}\big)_{m,:}\big\|_{2}}_{=:\,\beta_{1}}+\underbrace{\big\|\breve{\bm{Z}}_{m,:}\breve{\bm{\Sigma}}_{m}^{-1/2}\big(\breve{\bm{\Sigma}}_{m}^{1/2}-\bm{\Sigma}_{m}^{\star1/2}\big)\big\|_{2}}_{=:\,\beta_{2}}.
\]
We shall upper bound these two terms separately.
\begin{enumerate}
\item For $\beta_{1}$, observe that $(\bm{Z}-\breve{\bm{Z}})_{m,:}=\sqrt{2}\,\bm{z}_{n,n,m}+\sqrt{2}\,\bm{z}_{l,l,m}+2\sum_{i:i\neq n,l}\left(\bm{z}_{i,n,m}+\bm{z}_{i,l,m}\right)$
is a sum of independent zero-mean random vectors. From (\ref{eq:U-true-tilde-norm})
and (\ref{eq:U-true-tilde-spectrum}), straightforward computation
gives
\begin{align*}
B:=\max_{1\leq i,j\leq d}\big\| p^{-1}E_{i,j,k}\chi_{i,j,k}\widetilde{\bm{U}}_{(i,j),:}^{\star}(\widetilde{\bm{U}}^{\star\top}\widetilde{\bm{U}}^{\star})^{-1}\big\|_{\psi_{1}} & \lesssim\frac{\sigma_{\max}}{p}\big\|\widetilde{\bm{U}}^{\star}\big\|_{2,\infty}\big\|(\widetilde{\bm{U}}^{\star\top}\widetilde{\bm{U}}^{\star})^{-1}\big\|\\
 & \lesssim\frac{\sigma_{\max}}{p}\cdot\frac{\mu\sqrt{r}}{d}\lambda_{\max}^{\star2/3}\cdot\frac{1}{\lambda_{\min}^{\star4/3}}
\end{align*}
and 
\begin{align*}
V:=\sum_{1\leq i\leq d}p^{-2}\mathbb{E}\big[E_{i,j,k}^{2}\chi_{i,j,k}\big]\big\|\widetilde{\bm{U}}_{(i,j),:}^{\star}(\widetilde{\bm{U}}^{\star\top}\widetilde{\bm{U}}^{\star})^{-1}\big\|_{2}^{2} & \leq\frac{\sigma_{\max}^{2}}{p}\big\|(\widetilde{\bm{U}}^{\star\top}\widetilde{\bm{U}}^{\star})^{-1}\big\|^{2}\sum_{1\leq i\leq d}\big\|\widetilde{\bm{U}}_{(i,j),:}^{\star}\big\|_{2}^{2}\\
 & \leq\frac{\sigma_{\max}^{2}}{p}\big\|(\widetilde{\bm{U}}^{\star\top}\widetilde{\bm{U}}^{\star})^{-1}\big\|^{2}\max_{1\leq s\leq r}\left\Vert \bm{u}_{s}^{\star}\right\Vert _{2}^{2}\left\Vert \bm{U}^{\star}\right\Vert _{2,\infty}^{2}\\
 & \lesssim\frac{\sigma_{\max}^{2}}{p}\cdot\frac{1}{\lambda_{\min}^{\star8/3}}\cdot\lambda_{\max}^{\star2/3}\cdot\frac{\mu r}{d}\lambda_{\max}^{\star2/3}.
\end{align*}
Applying the matrix Bernstein inequality reveals that with probability
exceeding $1-O\left(d^{-20}\right)$,
\[
\big\|\big(\bm{Z}-\breve{\bm{Z}}\big)_{m,:}\big\|_{2}\lesssim B\log^{2}d+\sqrt{V\log d}\lesssim\frac{\sigma_{\max}\lambda_{\max}^{\star2/3}}{\lambda_{\min}^{\star4/3}\sqrt{p}}\left\{ \frac{\mu\sqrt{r}\log^{2}d}{d\sqrt{p}}+\sqrt{\frac{\mu r\log d}{d}}\right\} .
\]
In particular, we can combine it with (\ref{eq:Z-row-l2-norm-UB})
and $\kappa\asymp1$ to obtain
\begin{align}
\big\|\breve{\bm{Z}}_{m,:}\big\|_{2} & \leq\left\Vert \bm{Z}_{m,:}\right\Vert _{2}+\big\|\big(\bm{Z}-\breve{\bm{Z}}\big)_{m,:}\big\|_{2}\nonumber \\
 & \lesssim\frac{\sigma_{\max}}{\lambda_{\min}^{\star2/3}\sqrt{p}}\left\{ \frac{\mu\sqrt{r}\log^{2}d}{d\sqrt{p}}+\sqrt{\frac{\mu r\log d}{d}}+\sqrt{r\log d}\right\} \asymp\frac{\sigma_{\max}\sqrt{r\log d}}{\lambda_{\min}^{\star2/3}\sqrt{p}},\label{eq:Z-breve-l2-norm-UB}
\end{align}
with the proviso that $p\gtrsim\mu^{2}d^{-2}\log^{3}d$ and $\mu\lesssim d$.
\item Turning to $\beta_{2}$, we invoke the independence between $(\bm{Z}-\breve{\bm{Z}})_{m,:}$
and $\breve{\bm{Z}}_{m,:}$ to derive
\begin{align*}
\bm{\Sigma}_{m}^{\star}-\breve{\bm{\Sigma}}_{m} & =\mathbb{E}\big[\bm{Z}_{m,:}\bm{Z}_{m,:}^{\top}\big]-\mathbb{E}\big[(\breve{\bm{Z}}_{m,:})^{\top}\breve{\bm{Z}}_{m,:}\big]=\mathbb{E}\big[(\bm{Z}_{m,:}-\breve{\bm{Z}}_{m,:})^{\top}(\bm{Z}_{m,:}-\breve{\bm{Z}}_{m,:})\big]\\
 & =\frac{2}{p}\sigma_{n,n,m}^{2}(\widetilde{\bm{U}}^{\star\top}\widetilde{\bm{U}}^{\star})^{-1}(\widetilde{\bm{U}}_{(n,n),:}^{\star})^{\top}\widetilde{\bm{U}}_{(n,n),:}^{\star}(\widetilde{\bm{U}}^{\star\top}\widetilde{\bm{U}}^{\star})^{-1}\\
 & \quad+\frac{2}{p}\sigma_{l,l,m}^{2}(\widetilde{\bm{U}}^{\star\top}\widetilde{\bm{U}}^{\star})^{-1}(\widetilde{\bm{U}}_{(l,l),:}^{\star})^{\top}\widetilde{\bm{U}}_{(l,l),:}^{\star}(\widetilde{\bm{U}}^{\star\top}\widetilde{\bm{U}}^{\star})^{-1}\\
 & \quad+\frac{4}{p}\sum_{i:i\neq n,l}\sigma_{i,n,m}^{2}(\widetilde{\bm{U}}^{\star\top}\widetilde{\bm{U}}^{\star})^{-1}(\widetilde{\bm{U}}_{(i,n),:}^{\star})^{\top}\widetilde{\bm{U}}_{(i,n),:}^{\star}(\widetilde{\bm{U}}^{\star\top}\widetilde{\bm{U}}^{\star})^{-1}\\
 & \quad+\frac{4}{p}\sum_{i:i\neq n,l}\sigma_{i,l,m}^{2}(\widetilde{\bm{U}}^{\star\top}\widetilde{\bm{U}}^{\star})^{-1}(\widetilde{\bm{U}}_{(i,l),:}^{\star})^{\top}\widetilde{\bm{U}}_{(i,l),:}^{\star}(\widetilde{\bm{U}}^{\star\top}\widetilde{\bm{U}}^{\star})^{-1}.
\end{align*}
One can then use (\ref{eq:U-true-tilde-spectrum}) to upper bound
\begin{align*}
\big\|\bm{\Sigma}_{m}^{\star}-\breve{\bm{\Sigma}}_{m}\big\| & \lesssim\frac{\sigma_{\max}^{2}}{p}\big\|(\widetilde{\bm{U}}^{\star\top}\widetilde{\bm{U}}^{\star})^{-1}\big\|^{2}\Big\|\sum_{i:i\neq l}(\widetilde{\bm{U}}_{(i,n),:}^{\star})^{\top}\widetilde{\bm{U}}_{(i,n),:}^{\star}+\sum_{i:i\neq n}(\widetilde{\bm{U}}_{(i,l),:}^{\star})^{\top}\widetilde{\bm{U}}_{(i,l),:}^{\star}\Big\|\\
 & \leq\frac{\sigma_{\max}^{2}}{p}\big\|(\widetilde{\bm{U}}^{\star\top}\widetilde{\bm{U}}^{\star})^{-1}\big\|^{2}\sum_{1\leq i\leq d}\big(\big\|\widetilde{\bm{U}}_{(i,n),:}^{\star}\big\|_{2}^{2}+\big\|\widetilde{\bm{U}}_{(i,l),:}^{\star}\big\|_{2}^{2}\big)\\
 & \lesssim\frac{\sigma_{\max}^{2}}{p}\cdot\frac{1}{\lambda_{\min}^{\star8/3}}\cdot\frac{\mu r}{d}\lambda_{\max}^{\star4/3}\lesssim\frac{\sigma_{\max}^{2}}{\lambda_{\min}^{\star4/3}p}\frac{\mu r}{d}\ll\frac{\sigma_{\min}^{2}}{\lambda_{\max}^{\star4/3}p},
\end{align*}
where the last line arises from the assumption that $\sigma_{\max}/\sigma_{\min}\asymp1,$
$\kappa\asymp1$ and $r\ll d/\mu$. Combining this with (\ref{eq:cov-matrix-eigval})
and Weyl's inequality, one arrives at
\[
\lambda_{\max}\big(\breve{\bm{\Sigma}}_{m}\big)\lesssim\frac{\sigma_{\max}^{2}}{\lambda_{\min}^{\star4/3}p}\qquad\text{and}\qquad\lambda_{\min}\big(\breve{\bm{\Sigma}}_{m}\big)\gtrsim\frac{\sigma_{\min}^{2}}{\lambda_{\max}^{\star4/3}p}.
\]
Applying the perturbation bound for matrix square roots \cite[Lemma 2.1]{MR1176461}
yields
\begin{align*}
\big\|\breve{\bm{\Sigma}}_{m}^{1/2}-\bm{\Sigma}_{m}^{\star1/2}\big\| & \lesssim\frac{1}{\lambda_{\min}^{1/2}(\bm{\Sigma}_{m}^{\star})+\lambda_{\min}^{1/2}(\breve{\bm{\Sigma}}_{m})}\big\|\breve{\bm{\Sigma}}_{m}-\bm{\Sigma}_{m}^{\star}\big\|\lesssim\frac{\lambda_{\max}^{\star2/3}\sqrt{p}}{\sigma_{\min}}\cdot\frac{\sigma_{\max}^{2}}{\lambda_{\min}^{\star4/3}p}\frac{\mu r}{d}.
\end{align*}
This taken collectively with (\ref{eq:Z-breve-l2-norm-UB}) implies
that
\begin{align*}
\big\|\breve{\bm{Z}}_{m,:}\breve{\bm{\Sigma}}_{m}^{-1/2}\big(\breve{\bm{\Sigma}}_{m}^{1/2}-\bm{\Sigma}_{m}^{\star1/2}\big)\big\|_{2} & \lesssim\big\|\breve{\bm{Z}}_{m,:}\big\|_{2}\big\|\breve{\bm{\Sigma}}_{m}^{-1/2}\big\|\big\|\breve{\bm{\Sigma}}_{m}^{1/2}-\bm{\Sigma}_{m}^{\star1/2}\big\|\\
 & \lesssim\frac{\sigma_{\max}\sqrt{r\log d}}{\lambda_{\min}^{\star2/3}\sqrt{p}}\cdot\frac{\lambda_{\max}^{\star2/3}\sqrt{p}}{\sigma_{\min}}\cdot\frac{\lambda_{\max}^{\star2/3}\sqrt{p}}{\sigma_{\min}}\frac{\sigma_{\max}^{2}}{\lambda_{\min}^{\star4/3}p}\frac{\mu r}{d}\\
 & \lesssim\frac{\sigma_{\max}}{\lambda_{\min}^{\star2/3}\sqrt{p}}\frac{\mu r^{3/2}\sqrt{\log d}}{d},
\end{align*}
where the last line follows from the conditions $\sigma_{\max}/\sigma_{\min}\asymp1$
and $\kappa\asymp1$.
\item Putting the above bounds together allows us to conclude that
\[
\big\|(\bm{Z}-\widehat{\bm{Z}})_{m,:}\big\|_{2}\lesssim\frac{\sigma_{\max}}{\lambda_{\min}^{\star2/3}\sqrt{p}}\left\{ \frac{\mu\sqrt{r}\log^{2}d}{d\sqrt{p}}+\sqrt{\frac{\mu r\log d}{d}}+\frac{\mu r^{3/2}\sqrt{\log d}}{d}\right\} .
\]
\end{enumerate}

\subsection{Proof of Lemma \ref{lemma:T-loss-dist-neg-part}}

\label{subsec:T-loss-dist-neg-part}

We shall bound the terms in (\ref{eq:T-entry-loss-res}) separately,
followed by the triangle inequality. In what follows, we denote $u_{s,i}^{\star}:=\left(\bm{u}_{s}\right)_{i}$
and $\Delta_{s,i}:=\left(\bm{\Delta}_{s}\right)_{i}$ for any $1\leq s\leq r$
and $1\leq i\leq d$. 
\begin{enumerate}
\item Regarding the first three terms involving $\bm{W}$, combine the Cauchy-Schwartz
with (\ref{eq:W-2inf-norm-UB}) to show that
\begin{align}
 & \big|\langle\bm{W}_{m,:},\widetilde{\bm{U}}_{(n,l),:}^{\star}\rangle+\langle\bm{W}_{n,:},\widetilde{\bm{U}}_{(m,l),:}^{\star}\rangle+\langle\bm{W}_{l,:},\widetilde{\bm{U}}_{(m,n),:}^{\star}\rangle\big|\nonumber \\
 & \qquad\leq\left\Vert \bm{W}\right\Vert _{2,\infty}\Big(\big\|\widetilde{\bm{U}}_{(n,l),:}^{\star}\big\|_{2}+\big\|\widetilde{\bm{U}}_{(m,l),:}^{\star}\big\|_{2}+\big\|\widetilde{\bm{U}}_{(m,n),:}^{\star}\big\|_{2}\Big)\nonumber \\
 & \qquad\lesssim\zeta\frac{\sigma_{\max}}{\lambda_{\min}^{\star2/3}\sqrt{p}}\Big(\big\|\widetilde{\bm{U}}_{(n,l),:}^{\star}\big\|_{2}+\big\|\widetilde{\bm{U}}_{(m,l),:}^{\star}\big\|_{2}+\big\|\widetilde{\bm{U}}_{(m,n),:}^{\star}\big\|_{2}\Big),\label{eq:T-loss-dist-neg-part-part1}
\end{align}
where we recall the definition of $\zeta$ in (\ref{def:zeta}).
\item We now turn to $\langle\bm{U}_{m,:}^{\star},\widetilde{\bm{\Delta}}_{(n,l),:}\rangle=\sum_{1\leq s\leq r}u_{s,m}^{\star}\Delta_{s,n}\Delta_{s,l}$.
By virtue of (\ref{assumption:u-inf-norm}) and (\ref{eq:U-loss-2inf}),
one can invoke Cauchy-Schwartz to bound
\begin{align}
\Big|\sum_{1\leq s\leq r}u_{s,m}^{\star}\Delta_{s,n}\Delta_{s,l}\Big| & \leq\sqrt{\sum_{1\leq s\leq r}\Delta_{s,l}^{2}}\sqrt{\sum_{1\leq s\leq r}u_{s,m}^{\star2}\Delta_{s,n}^{2}}\leq\left\Vert \bm{\Delta}\right\Vert _{2,\infty}^{2}\max_{1\leq s\leq r}\left\Vert \bm{u}_{s}^{\star}\right\Vert _{\infty}\nonumber \\
 & \lesssim\left(\frac{\sigma_{\max}}{\lambda_{\min}^{\star}}\sqrt{\frac{\mu r\log d}{p}}\lambda_{\max}^{\star1/3}\right)^{2}\sqrt{\frac{\mu}{d}}\,\lambda_{\max}^{\star1/3}.\label{eq:T-loss-dist-neg-part-part2}
\end{align}
Clearly, this upper bound also holds for both $\langle\bm{U}_{n,:}^{\star},\widetilde{\bm{\Delta}}_{(m,l),:}\rangle$
and $\langle\bm{U}_{l,:}^{\star},\widetilde{\bm{\Delta}}_{(m,n),:}\rangle$.
\item As for the last term $\langle\bm{\Delta}_{m,:},\widetilde{\bm{\Delta}}_{(n,l),:}\rangle$,
we know from (\ref{eq:U-loss-2inf}) and (\ref{eq:u-loss-u-relation})
that
\begin{align}
\Big|\sum_{1\leq s\leq r}\Delta_{s,m}\Delta_{s,n}\Delta_{s,l}\Big| & \leq\left\Vert \bm{\Delta}\right\Vert _{2,\infty}^{2}\max_{1\leq s\leq r}\left\Vert \bm{\Delta}_{s}\right\Vert _{\infty}\ll\left(\frac{\sigma_{\max}}{\lambda_{\min}^{\star}}\sqrt{\frac{\mu r\log d}{p}}\lambda_{\max}^{\star1/3}\right)^{2}\sqrt{\frac{\mu}{d}}\,\lambda_{\max}^{\star1/3}.\label{eq:T-loss-dist-neg-part-part3}
\end{align}
\item Combining (\ref{eq:T-loss-dist-neg-part-part1}), (\ref{eq:T-loss-dist-neg-part-part2})
and (\ref{eq:T-loss-dist-neg-part-part3}), we arrive at the advertised
bound
\begin{align*}
\left|R_{mnl}\right| & \lesssim\zeta\frac{\sigma_{\max}}{\lambda_{\min}^{\star2/3}\sqrt{p}}\Big(\big\|\widetilde{\bm{U}}_{(n,l),:}^{\star}\big\|_{2}+\big\|\widetilde{\bm{U}}_{(m,l),:}^{\star}\big\|_{2}+\big\|\widetilde{\bm{U}}_{(m,n),:}^{\star}\big\|_{2}\Big)+\left(\frac{\sigma_{\max}}{\lambda_{\min}^{\star}}\sqrt{\frac{\mu r\log d}{p}}\lambda_{\max}^{\star1/3}\right)^{2}\sqrt{\frac{\mu}{d}}\,\lambda_{\max}^{\star1/3}\\
 & \overset{(\mathrm{i})}{\lesssim}\left(\zeta+\frac{\sigma_{\max}}{\lambda_{\min}^{\star1/3}}\frac{\mu^{3/2}r\log d}{\sqrt{dp}}\Big(\big\|\widetilde{\bm{U}}_{(n,l),:}^{\star}\big\|_{2}+\big\|\widetilde{\bm{U}}_{(m,l),:}^{\star}\big\|_{2}+\big\|\widetilde{\bm{U}}_{(m,n),:}^{\star}\big\|_{2}\Big)^{-1}\right)\sqrt{v_{mnl}^{\star}}\overset{(\mathrm{ii})}{=}o\left(1\right)\sqrt{v_{mnl}^{\star}}.
\end{align*}
Here, (i) arises from the lower bound on $v_{m,n,l}^{\star}$ (cf.~(\ref{eq:T-entry-var-LB}))
and the conditions $\sigma_{\max}/\sigma_{\min}\asymp1$ and $\kappa\asymp1$,
whereas (ii) makes use of the assumptions (\ref{eq:requirement-p-sigma-rank-r})
and (\ref{eq:U-tilde-2norm-LB-rank-r}).
\end{enumerate}

%% file: proof-CI.tex
\section{Proof of auxiliary lemmas: confidence intervals}

\label{sec:Analysis-of-CI}

\begin{comment}
@@@

\subsection{Tensor factors}

\label{subsec:Tensor-factors-CI}
\end{comment}

\subsection{Proof of Lemma \ref{lemma:u-entry-var-est-loss-neg}}

\label{subsec:u-entry-var-est-loss-neg}

Fix arbitrary $1\leq l\leq r$ and $1\leq k\leq d$. Before proceeding,
we pause to introduce some notation for simplicity of presentation.
Recalling the notation $\widetilde{\bm{U}}^{\star}:=\big[\bm{u}_{l}^{\star\otimes2}\big]_{1\leq l\leq r}\in\mathbb{R}^{d^{2}\times r}$
and $\widetilde{\bm{U}}:=\big[\bm{u}_{l}^{\otimes2}\big]_{1\leq l\leq r}\in\mathbb{R}^{d^{2}\times r}$,
we define two $d^{2}\times r$ matrices as follows
\begin{align}
\bm{V}^{\star} & :=\widetilde{\bm{U}}^{\star}(\widetilde{\bm{U}}^{\star\top}\widetilde{\bm{U}}^{\star})^{-1},\qquad\bm{V}:=\widetilde{\bm{U}}(\widetilde{\bm{U}}^{\top}\widetilde{\bm{U}})^{-1}.\label{def:V-true}
\end{align}
These allow us to express the covariance matrix as $\bm{\Sigma}_{k}^{\star}=\bm{V}^{\star}\bm{D}_{k}^{\star}\bm{V}^{\star}$
(resp.~$\bm{\Sigma}_{k}=\bm{V}\bm{D}_{k}\bm{V}$), where $\bm{D}_{k}^{\star}$
(resp.~$\bm{D}_{k}$) is defined in (\ref{eq:cov-matrix-m-diag})
(resp.~(\ref{eq:defn-Dk})). In addition, let us define 
\[
s_{l,k}^{\star}:=\sqrt{(\bm{\Sigma}_{k}^{\star})_{l,l}}\qquad\text{and}\qquad s_{l,k}:=\sqrt{(\bm{\Sigma}_{k})_{l,l}}.
\]
Lemma~\ref{lemma:V-property} below collects several useful properties
regrading $\bm{V}_{:,l}$ and $\bm{V}_{:,l}^{\star}$; the proof is
deferred to the end of this section. 

\begin{lemma}\label{lemma:V-property}Instate the assumptions and
notation of Lemma \ref{lemma:u-entry-var-est-loss-neg}. For each
$1\leq l\leq r$, one has
\begin{align}
 & \left\Vert \bm{V}_{:,l}^{\star}\right\Vert _{2}=\frac{1+o\left(1\right)}{\left\Vert \bm{u}_{l}^{\star}\right\Vert _{2}^{2}},\qquad\big\|\bm{V}_{:,l}^{\star}\big\|_{\infty}\lesssim\frac{\mu\sqrt{r}}{d}\frac{1}{\lambda_{\min}^{\star2/3}},\qquad\sum_{1\leq k\leq d}V_{(i,k),l}^{\star2}\lesssim\frac{\mu r}{d}\frac{1}{\lambda_{\min}^{\star4/3}};\label{eq:V-true-col-norm}\\
 & \left\Vert \bm{V}_{:,l}-\bm{V}_{:,l}^{\star}\right\Vert _{2}\lesssim\frac{\sigma_{\max}}{\lambda_{\min}^{\star}}\sqrt{\frac{d}{p}}\,\frac{1}{\lambda_{\min}^{\star2/3}},\qquad\left\Vert \bm{V}_{:,l}-\bm{V}_{:,l}^{\star}\right\Vert _{\infty}\lesssim\frac{\sigma_{\max}}{\lambda_{\min}^{\star}}\sqrt{\frac{\mu^{2}r\log d}{dp}}\,\frac{1}{\lambda_{\min}^{\star2/3}}.\label{eq:V-col-loss}
\end{align}
\end{lemma}

With these in place, we are ready to control $J_{l,k}$, which can
be expressed as
\[
J_{l,k}=\frac{u_{l,k}-u_{l,k}^{\star}}{s_{l,k}}-\frac{u_{l,k}-u_{l,k}^{\star}}{s_{l,k}^{\star}}=\left(u_{l,k}-u_{l,k}^{\star}\right)\frac{s_{l,k}^{\star2}-s_{l,k}^{2}}{s_{l,k}s_{l,k}^{\star}}\frac{1}{s_{l,k}^{\star}+s_{l,k}}.
\]
This suggest that we control both $u_{l,k}-u_{l,k}^{\star}$ and $s_{l,k}^{\star2}-s_{l,k}^{2}$.
\begin{itemize}
\item Regarding the estimation error of $u_{l,k}$, combining (\ref{eq:V-true-col-norm})
with the assumptions $\sigma_{\max}/\sigma_{\min}\asymp1$ and $\kappa\asymp1$
allows us to lower bound
\begin{equation}
s_{l,k}^{\star2}=\frac{1}{p}(\bm{V}_{:,l}^{\star})^{\top}\bm{D}_{k}^{\star}\bm{V}_{:,l}^{\star}\geq\frac{1}{p}\lambda_{\min}(\bm{D}_{k}^{\star})\left\Vert \bm{V}_{:,l}^{\star}\right\Vert _{2}^{2}\gtrsim\frac{\sigma_{\min}^{2}}{p\left\Vert \bm{u}_{l}^{\star}\right\Vert _{2}^{4}}.\label{eq:u-entry-var-LB}
\end{equation}
Hence, we know from (\ref{eq:U-loss-2inf}) and the conditions $\sigma_{\max}/\sigma_{\min}\asymp1$
and $\kappa\asymp1$ that
\[
\left|u_{l,k}-u_{l,k}^{\star}\right|\leq\big\|\bm{U}-\bm{U}^{\star}\big\|_{2,\infty}\lesssim\frac{\sigma_{\max}}{\lambda_{\min}^{\star}}\sqrt{\frac{\mu r\log d}{p}}\,\lambda_{\max}^{\star1/3}\lesssim s_{l,k}^{\star}\sqrt{\mu r\log d}.
\]
\item Next, we claim that
\begin{equation}
\left|s_{l,k}^{\star2}-s_{l,k}^{2}\right|\lesssim\Bigg\{\sqrt{\frac{\mu^{3}r^{2}\log^{2}d}{d^{2}p}}+\frac{\sigma_{\max}}{\lambda_{\min}^{\star}}\sqrt{\frac{\mu^{2}rd\log d}{p}}\Bigg\} s_{l,k}^{\star2}\ll s_{l,k}^{\star2};\label{claim:u-entry-var-loss}
\end{equation}
if this were true, then one would further obtain 
\[
s_{l,k}\geq s_{l,k}^{\star}-|s_{l,k}-s_{l,k}^{\star}|\gtrsim s_{l,k}^{\star}.
\]
\item Putting the above bounds together reveals that
\[
\left|J_{l,k}\right|\lesssim\frac{\big|u_{l,k}-u_{l,k}^{\star}\big|\,\big|s_{l,k}^{\star2}-s_{l,k}^{2}\big|}{s_{l,k}^{\star3}}\lesssim\sqrt{\mu r\log d}\,\Bigg\{\sqrt{\frac{\mu^{3}r^{2}\log^{2}d}{d^{2}p}}+\frac{\sigma_{\max}}{\lambda_{\min}^{\star}p}\sqrt{\frac{\mu^{2}rd\log d}{p}}\Bigg\}
\]
as claimed.
\end{itemize}
Hence, the remainder of the proof boils down to establishing the claim
(\ref{claim:u-entry-var-loss}). Towards this, our starting point
is the following decomposition
\begin{align}
\frac{p}{2}\left(s_{l,k}^{2}-s_{l,k}^{\star2}\right) & =(\bm{V}_{:,l})^{\top}\bm{D}_{k}\bm{V}_{:,l}-(\bm{V}_{:,l}^{\star})^{\top}\bm{D}_{k}^{\star}\bm{V}_{:,l}^{\star}\nonumber \\
 & =\underbrace{(\bm{V}_{:,l})^{\top}(\bm{D}_{k}-\widehat{\bm{D}}_{k})\bm{V}_{:,l}}_{=:\,\beta_{1}}+\underbrace{(\bm{V}_{:,l})^{\top}\widehat{\bm{D}}_{k}\bm{V}_{:,l}-(\bm{V}_{:,l}^{\star})^{\top}\bm{D}_{k}^{\star}\bm{V}_{:,l}^{\star}}_{=:\,\beta_{2}},\label{eq:u-entry-var-est-loss-decomp}
\end{align}
where $\widehat{\bm{D}}_{k}\in\mathbb{R}^{d^{2}\times d^{2}}$ is
a diagonal matrix with entries given by
\begin{equation}
(\widehat{\bm{D}}_{k}){}_{(i,j),(i,j)}=p^{-1}E_{i,j,k}^{2}\chi_{i,j,k},\qquad1\leq i,j\leq d.\label{def:D-hat}
\end{equation}
In what follows, we shall control $\beta_{1}$ and $\beta_{2}$ separately. 

\subsubsection{Bounding $\beta_{1}$}

To begin with, let us decompose
\begin{align*}
\beta_{1} & =\underbrace{(\bm{V}_{:,l}^{\star})^{\top}(\bm{D}_{k}-\widehat{\bm{D}}_{k})\bm{V}_{:,l}^{\star}}_{=:\,\gamma_{1}}+2\underbrace{(\bm{V}_{:,l}^{\star})^{\top}(\bm{D}_{k}-\widehat{\bm{D}}_{k})(\bm{V}_{:,l}-\bm{V}_{:,l}^{\star})}_{=:\,\gamma_{2}}\\
 & \quad+\underbrace{(\bm{V}_{:,l}-\bm{V}_{:,l}^{\star})^{\top}(\bm{D}_{k}-\widehat{\bm{D}}_{k})(\bm{V}_{:,l}-\bm{V}_{:,l}^{\star})}_{=:\,\gamma_{3}}.
\end{align*}

\begin{itemize}
\item With respect to $\gamma_{1}$, the triangle inequality yields
\[
\left|\gamma_{1}\right|=\Big|\sum_{1\leq i,j\leq d}\big(\widehat{E}_{i,j,k}^{2}-E_{i,j,k}^{2}\big)p^{-1}\chi_{i,j,k}V_{(i,j),l}^{\star2}\Big|\leq\max_{(i,j,k)\in\Omega}\big|\widehat{E}_{i,j,k}^{2}-E_{i,j,k}^{2}\big|\sum_{1\leq i,j\leq d}p^{-1}\chi_{i,j,k}V_{(i,j),l}^{\star2}.
\]
From (\ref{eq:noise-est-inf-UB}), we know that
\[
\max_{(i,j,k)\in\Omega}\big|\widehat{E}_{ijk}-E_{ijk}\big|\lesssim\frac{\sigma_{\max}}{\lambda_{\min}^{\star}}\sqrt{\frac{\mu^{3}r^{2}\log d}{d^{2}p}}\,\lambda_{\max}^{\star}\ll\sigma_{\max}\sqrt{\log d},
\]
where the last inequality arises from the conditions $p\gg\mu^{3}r^{2}d^{-2}$
and $\kappa\asymp1$. By the standard results of the sub-Gaussian
random variables, one also has 
\begin{align}
\left\Vert \bm{E}\right\Vert _{\infty} & \lesssim\sigma_{\max}\sqrt{\log d}\label{eq:noise-inf-UB}
\end{align}
with probability at least $1-O\left(d^{-20}\right)$. This reveals
that
\begin{equation}
\max_{(i,j,k)\in\Omega}\big|\widehat{E}_{i,j,k}^{2}-E_{i,j,k}^{2}\big|\leq\max_{(i,j,k)\in\Omega}\big|(\widehat{E}_{i,j,k}-E_{i,j,k})(\widehat{E}_{i,j,k}+E_{i,j,k})\big|\lesssim\sigma_{\max}^{2}\sqrt{\frac{\mu^{3}r^{2}\log^{2}d}{d^{2}p}}.\label{eq:noise-square-est-loss-UB}
\end{equation}
In addition, apply the Bernstein inequality to find that with probability
at least $1-O\left(d^{-20}\right)$,
\begin{align}
\sum_{1\leq i,j\leq d}p^{-1}\chi_{i,j,k}V_{(i,j),l}^{\star2} & \lesssim\big\|\bm{V}_{:,l}^{\star}\big\|_{2}^{2}+p^{-1}\log d\,\big\|\bm{V}_{:,l}^{\star}\big\|_{\infty}^{2}+\sqrt{p^{-1}\log d}\big\|\bm{V}_{:,l}^{\star}\big\|_{\infty}^{2}\big\|\bm{V}_{:,l}^{\star}\big\|_{2}^{2}\nonumber \\
 & \overset{(\mathrm{i})}{\asymp}\big\|\bm{V}_{:,l}^{\star}\big\|_{2}^{2}+p^{-1}\log d\,\big\|\bm{V}_{:,l}^{\star}\big\|_{\infty}^{2}\overset{(\mathrm{ii})}{\lesssim}\frac{1}{\lambda_{\min}^{\star4/3}}\left\{ 1+\frac{\mu^{2}r\log d}{d^{2}p}\right\} \overset{(\mathrm{iii})}{\asymp}\frac{1}{\lambda_{\min}^{\star4/3}},\label{eq:V-col-concentration}
\end{align}
where (i) follows from the AM-GM inequality, (ii) makes use of (\ref{eq:V-true-col-norm}),
and (iii) holds true as long as $p\gg\mu^{2}rd^{-2}\log d$. Combining
the above bounds, we arrive at
\begin{align*}
\left|\gamma_{1}\right| & \lesssim\frac{\sigma_{\max}^{2}}{\lambda_{\min}^{\star4/3}}\sqrt{\frac{\mu^{3}r^{2}\log^{2}d}{d^{2}p}}
\end{align*}
with probability at least $1-O\left(d^{-20}\right)$. 
\item Turning to $\gamma_{2}$, one can use the triangle inequality and
Cauchy-Schwartz to obtain
\begin{align*}
\left|\gamma_{2}\right| & =\Big|\sum_{1\leq i,j\leq d}\big(\widehat{E}_{i,j,k}^{2}-E_{i,j,k}^{2}\big)p^{-1}\chi_{i,j,k}V_{(i,j),l}^{\star}\big(V_{(i,j),l}-V_{(i,j),l}^{\star}\big)\Big|\\
 & \leq\max_{(i,j,k)\in\Omega}\big|\widehat{E}_{i,j,k}^{2}-E_{i,j,k}^{2}\big|\left\Vert \bm{V}_{:,l}-\bm{V}_{:,l}^{\star}\right\Vert _{\infty}\sum_{1\leq i,j\leq d}\Big|\,p^{-1}\chi_{i,j,k}V_{(i,j),l}^{\star}\Big|\\
 & \leq\max_{(i,j,k)\in\Omega}\big|\widehat{E}_{i,j,k}^{2}-E_{i,j,k}^{2}\big|\left\Vert \bm{V}_{:,l}-\bm{V}_{:,l}^{\star}\right\Vert _{\infty}\sqrt{\sum_{1\leq i,j\leq d}p^{-1}\chi_{i,j,k}V_{(i,j),l}^{\star2}}\sqrt{\sum_{1\leq i,j\leq d}p^{-1}\chi_{i,j,k}}.
\end{align*}
It is straightforward to apply the Bernstein inequality to find that,
with probability exceeding $1-O\left(d^{-12}\right)$,
\begin{equation}
\sum_{1\leq i,j\leq d}p^{-1}\chi_{i,j,k}\lesssim d^{2}+p^{-1}\log d+\sqrt{d^{2}p^{-1}\log d}\asymp d^{2},\label{eq:obs-card-concentration}
\end{equation}
with the proviso that $p\gg d^{-2}\log d$. Taking this with (\ref{eq:V-col-loss}),
(\ref{eq:noise-square-est-loss-UB}) and (\ref{eq:V-col-concentration})
collectively, we conclude that
\begin{align*}
\left|\gamma_{2}\right| & \lesssim\sigma_{\max}^{2}\sqrt{\frac{\mu^{3}r^{2}\log^{2}d}{d^{2}p}}\cdot\frac{\sigma_{\max}}{\lambda_{\min}^{\star}}\sqrt{\frac{\mu^{2}r\log d}{dp}}\,\frac{1}{\lambda_{\min}^{\star2/3}}\cdot\frac{1}{\lambda_{\min}^{\star2/3}}\cdot d=\frac{\sigma_{\max}^{2}}{\lambda_{\min}^{\star4/3}}\frac{\sigma_{\max}}{\lambda_{\min}^{\star}}\sqrt{\frac{\mu^{5}r^{3}\log^{3}d}{dp^{2}}}.
\end{align*}
\item Regarding $\gamma_{3}$, we can develop an upper bound in an analogous
manner:
\begin{align*}
\left|\gamma_{3}\right| & =\Big|\sum_{1\leq i,j\leq d}\big(\widehat{E}_{i,j,k}^{2}-E_{i,j,k}^{2}\big)p^{-1}\chi_{i,j,k}\big(V_{(i,j),l}-V_{(i,j),l}^{\star}\big)^{2}\Big|\\
 & \leq\max_{(i,j,k)\in\Omega}\big|\widehat{E}_{i,j,k}^{2}-E_{i,j,k}^{2}\big|\left\Vert \bm{V}_{:,l}-\bm{V}_{:,l}^{\star}\right\Vert _{\infty}^{2}\sum_{1\leq i,j\leq d}p^{-1}\chi_{i,j,k}\\
 & \overset{(\mathrm{i})}{\lesssim}\sigma_{\max}^{2}\sqrt{\frac{\mu^{3}r^{2}\log^{2}d}{d^{2}p}}\left(\frac{\sigma_{\max}}{\lambda_{\min}^{\star}}\sqrt{\frac{\mu^{2}r\log d}{dp}}\,\frac{1}{\lambda_{\min}^{\star2/3}}\right)^{2}d^{2}\\
 & \overset{(\mathrm{ii})}{\ll}\frac{\sigma_{\max}^{2}}{\lambda_{\min}^{\star4/3}}\frac{\sigma_{\max}}{\lambda_{\min}^{\star}}\sqrt{\frac{\mu^{5}r^{3}\log^{3}d}{dp^{2}}}.
\end{align*}
Here, (i) uses (\ref{eq:V-col-loss}), (\ref{eq:noise-square-est-loss-UB})
and (\ref{eq:obs-card-concentration}), whereas (ii) holds as long
as $\sigma_{\max}/\lambda_{\min}^{\star}\ll\sqrt{p/(\mu^{2}rd\log d)}$.
\item Taking these bounds together, we demonstrate that with probability
at least $1-O\left(d^{-12}\right)$,
\begin{align*}
\left|\beta_{1}\right| & \lesssim\frac{\sigma_{\max}^{2}}{\lambda_{\min}^{\star4/3}}\left\{ \sqrt{\frac{\mu^{3}r^{2}\log^{2}d}{d^{2}p}}+\frac{\sigma_{\max}}{\lambda_{\min}^{\star}}\sqrt{\frac{\mu^{5}r^{3}\log^{3}d}{dp^{2}}}\right\} \asymp\frac{\sigma_{\max}^{2}}{\lambda_{\min}^{\star4/3}}\sqrt{\frac{\mu^{3}r^{2}\log^{2}d}{d^{2}p}},
\end{align*}
where the last step relies on the noise condition $\sigma_{\max}/\lambda_{\min}^{\star}\ll\sqrt{p/(\mu^{2}rd^{3/2}\log d)}$.
\end{itemize}

\subsubsection{Bounding $\beta_{2}$}

Next, we move on to the term $\beta_{2}$ defined in (\ref{eq:u-entry-var-est-loss-decomp}),
which admits the following decomposition
\begin{align*}
\beta_{2} & =(\bm{V}_{:,l})^{\top}\widehat{\bm{D}}_{k}\bm{V}_{:,l}-(\bm{V}_{:,l}^{\star})^{\top}\widehat{\bm{D}}_{k}\bm{V}_{:,l}^{\star}+(\bm{V}_{:,l}^{\star})^{\top}\widehat{\bm{D}}_{k}\bm{V}_{:,l}^{\star}-(\bm{V}_{:,l}^{\star})^{\top}\bm{D}_{k}^{\star}\bm{V}_{:,l}^{\star}\\
 & =2\underbrace{(\bm{V}_{:,l}^{\star})^{\top}\widehat{\bm{D}}_{k}(\bm{V}_{:,l}-\bm{V}_{:,l}^{\star})}_{=:\,\gamma_{4}}+\underbrace{(\bm{V}_{:,l}-\bm{V}_{:,l}^{\star})^{\top}\widehat{\bm{D}}_{k}(\bm{V}_{:,l}-\bm{V}_{:,l}^{\star})}_{=:\,\gamma_{5}}+\underbrace{(\bm{V}_{:,l}^{\star})^{\top}\big(\widehat{\bm{D}}_{k}-\bm{D}_{k}^{\star}\big)\bm{V}_{:,l}^{\star}}_{=:\,\gamma_{6}}.
\end{align*}
In the sequel, we shall upper bound each of these terms individually.
\begin{itemize}
\item To begin with, invoke the Cauchy-Schwartz inequality to bound
\begin{align*}
\left|\gamma_{4}\right| & =\Big|\sum_{1\leq i,j\leq d}E_{i,j,k}^{2}p^{-1}\chi_{i,j,k}V_{(i,j),l}^{\star}\big(V_{(i,j),l}-V_{(i,j),l}^{\star}\big)\Big|\\
 & \leq\left\Vert \bm{V}_{:,l}-\bm{V}_{:,l}^{\star}\right\Vert _{\infty}\sqrt{\sum_{1\leq i,j\leq d}E_{i,j,k}^{4}p^{-1}\chi_{i,j,k}V_{(i,j),l}^{\star2}}\sqrt{\sum_{1\leq i,j\leq d}p^{-1}\chi_{i,j,k}}.
\end{align*}
Applying the Bernstein inequality yields that, with probability at
least $1-O\left(d^{-13}\right)$,
\begin{align*}
\sum_{1\leq i,j\leq d}E_{i,j,k}^{4}p^{-1}\chi_{i,j,k}V_{(i,j),l}^{\star2} & \lesssim\sigma_{\max}^{4}\left\{ \big\|\bm{V}_{:,l}^{\star}\big\|_{2}^{2}+p^{-1}\log^{2}d\,\big\|\bm{V}_{:,l}^{\star}\big\|_{\infty}^{2}+\sqrt{p^{-1}\log d}\,\big\|\bm{V}_{:,l}^{\star}\big\|_{\infty}\big\|\bm{V}_{:,l}^{\star}\big\|_{2}\right\} \\
 & \overset{(\mathrm{i})}{\asymp}\sigma_{\max}^{4}\left\{ \big\|\bm{V}_{:,l}^{\star}\big\|_{2}^{2}+p^{-1}\log^{2}d\,\big\|\bm{V}_{:,l}^{\star}\big\|_{\infty}^{2}\right\} \\
 & \overset{(\mathrm{ii})}{\lesssim}\frac{\sigma_{\max}^{4}}{\lambda_{\min}^{\star4/3}}\left\{ 1+\frac{\mu^{2}r\log^{2}d}{d^{2}p}\right\} \overset{(\mathrm{iii})}{\asymp}\frac{\sigma_{\max}^{4}}{\lambda_{\min}^{\star4/3}},
\end{align*}
where (i) is due to the AM-GM inequality, (ii) uses (\ref{eq:V-true-col-norm}),
and (iii) holds as long as $p\gtrsim\mu^{2}rd^{-2}\log^{2}d$. This
combined with (\ref{eq:V-true-col-norm}) and (\ref{eq:obs-card-concentration})
further leads to
\[
\left|\gamma_{4}\right|\lesssim\frac{\sigma_{\max}}{\lambda_{\min}^{\star}}\sqrt{\frac{\mu^{2}r\log d}{dp}}\,\frac{1}{\lambda_{\min}^{\star2/3}}\cdot d\cdot\frac{\sigma_{\max}^{2}}{\lambda_{\min}^{\star2/3}}=\frac{\sigma_{\max}^{2}}{\lambda_{\min}^{\star4/3}}\frac{\sigma_{\max}}{\lambda_{\min}^{\star}}\sqrt{\frac{\mu^{2}rd\log d}{p}}.
\]
\item Next, we turn to the term $\gamma_{5}$. By (\ref{eq:V-col-loss}),
(\ref{eq:noise-inf-UB}) and (\ref{eq:obs-card-concentration}), the
following holds with probability at least $1-O\left(d^{-13}\right)$,
\begin{align*}
\gamma_{5} & =\sum_{1\leq i,j\leq d}E_{i,j,k}^{2}p^{-1}\chi_{i,j,k}\big(V_{(i,j),l}-V_{(i,j),l}^{\star}\big)^{2}\leq\left\Vert \bm{E}\right\Vert _{\infty}^{2}\left\Vert \bm{V}_{:,l}-\bm{V}_{:,l}^{\star}\right\Vert _{\infty}^{2}\sum_{1\leq i,j\leq d}p^{-1}\chi_{i,j,k}\\
 & \lesssim\sigma_{\max}^{2}\log d\cdot\left(\frac{\sigma_{\max}}{\lambda_{\min}^{\star}}\sqrt{\frac{\mu^{2}r\log d}{dp}}\,\frac{1}{\lambda_{\min}^{\star2/3}}\right)^{2}\cdot d^{2}\ll\frac{\sigma_{\max}^{2}}{\lambda_{\min}^{\star4/3}}\frac{\sigma_{\max}}{\lambda_{\min}^{\star}}\sqrt{\frac{\mu^{2}rd\log d}{p}},
\end{align*}
where the last step holds as long as $\sigma_{\max}/\lambda_{\min}^{\star}\ll\sqrt{p/(\mu^{2}rd\log^{3}d)}$.
\item As for $\gamma_{6}$, we can express $\gamma_{6}=\sum_{1\leq i,j\leq d}V_{(i,j),l}^{\star2}(p^{-1}E_{i,j,k}^{2}\chi_{i,j,k}-\sigma_{i,j,k}^{2})$
as a sum of independent zero-mean random variables. With the assistance
of (\ref{eq:V-true-col-norm}), one derives
\begin{align*}
B & :=\max_{1\leq i,j\leq d}\big\| V_{(i,j),l}^{\star2}(p^{-1}E_{i,j,k}^{2}\chi_{i,j,k}-\sigma_{i,j,k}^{2})\big\|_{\psi_{1}}\leq\frac{\sigma_{\max}^{2}}{p}\big\|\bm{V}_{:,l}^{\star}\big\|_{\infty}^{2}\lesssim\frac{\sigma_{\max}^{2}}{\lambda_{\min}^{\star4/3}}\frac{\mu^{2}r}{d^{2}p}.\\
V & :=\sum_{1\leq i,j\leq d}V_{(i,j),l}^{\star4}\mathbb{E}\big[(p^{-1}E_{i,j,k}^{2}\chi_{i,j,k}-\sigma_{i,j,k}^{2})^{2}\big]\lesssim\frac{\sigma_{\max}^{4}}{p}\big\|\bm{V}_{:,l}^{\star}\big\|_{\infty}^{2}\big\|\bm{V}_{:,l}^{\star}\big\|_{2}^{2}\lesssim\frac{\sigma_{\max}^{4}}{\lambda_{\min}^{\star8/3}}\frac{\mu^{2}r}{d^{2}p}.
\end{align*}
Applying the matrix Bernstein inequality, one has with probability
at least $1-O\left(d^{-20}\right)$,
\begin{align*}
\left|\gamma_{6}\right| & \lesssim B\log^{2}d+\sqrt{V\log d}\lesssim\frac{\sigma_{\max}^{2}}{\lambda_{\min}^{\star4/3}}\left\{ \frac{\mu^{2}r\log^{2}d}{d^{2}p}+\sqrt{\frac{\mu^{2}r\log d}{d^{2}p}}\right\} \asymp\frac{\sigma_{\max}^{2}}{\lambda_{\min}^{\star4/3}}\sqrt{\frac{\mu^{2}r\log d}{d^{2}p}},
\end{align*}
where the last step holds as long as $p\gtrsim\mu^{2}rd^{-2}\log^{3}d$.
\item Putting the bounds above together, we reach
\[
\left|\beta_{2}\right|\lesssim\frac{\sigma_{\max}^{2}}{\lambda_{\min}^{\star4/3}}\left\{ \sqrt{\frac{\mu^{2}r\log d}{d^{2}p}}+\frac{\sigma_{\max}}{\lambda_{\min}^{\star}}\sqrt{\frac{\mu^{2}rd\log d}{p}}\right\} .
\]
\end{itemize}

\subsubsection{Combining $\beta_{1}$ and $\beta_{2}$ to establish the claim (\ref{claim:u-entry-var-loss})}

Taking the bounds on $\beta_{1}$ and $\beta_{2}$ collectively yields
that, with probability exceeding $1-O\left(d^{-10}\right)$,
\begin{align*}
\left|s_{l,k}^{2}-s_{l,k}^{\star2}\right| & \lesssim\frac{\sigma_{\max}^{2}}{\lambda_{\min}^{\star4/3}p}\Bigg\{\sqrt{\frac{\mu^{3}r^{2}\log^{2}d}{d^{2}p}}+\frac{\sigma_{\max}}{\lambda_{\min}^{\star}}\sqrt{\frac{\mu^{2}rd\log d}{p}}\Bigg\}\lesssim\Bigg\{\sqrt{\frac{\mu^{3}r^{2}\log^{2}d}{d^{2}p}}+\frac{\sigma_{\max}}{\lambda_{\min}^{\star}}\sqrt{\frac{\mu^{2}rd\log d}{p}}\Bigg\} s_{l,k}^{\star2},
\end{align*}
where we have used the lower bound on $s_{l,k}^{\star2}$ (cf.~(\ref{eq:u-entry-var-LB}))
as well as the conditions $\sigma_{\max}/\sigma_{\min}\asymp1$ and
$\kappa\asymp1$ in the last step.

\subsubsection{Proof of Lemma \ref{lemma:V-property}}
\begin{enumerate}
\item We first consider the $\ell_{2}$ norm of $\bm{V}_{:,l}^{\star}$.
Let $\bm{\Lambda}^{\star}\in\mathbb{R}^{r\times r}$ be a diagonal
matrix with entries $\Lambda_{i,i}^{\star}=\left\Vert \bm{u}_{i}^{\star}\right\Vert _{2}^{3}$
for all $1\leq i\leq r$. We can then decompose
\begin{align*}
\bm{V}_{:,l}^{\star} & =\big(\widetilde{\bm{U}}^{\star}\bm{\Lambda}^{\star-4/3}\big)_{:,l}+\widetilde{\bm{U}}^{\star}\big((\widetilde{\bm{U}}^{\star\top}\widetilde{\bm{U}}^{\star})^{-1}-\bm{\Lambda}^{\star-4/3}\big)_{:,l}=\left\Vert \bm{u}_{l}^{\star}\right\Vert _{2}^{-4}\bm{u}_{l}^{\star\otimes2}+\widetilde{\bm{U}}^{\star}\big((\widetilde{\bm{U}}^{\star\top}\widetilde{\bm{U}}^{\star})^{-1}-\bm{\Lambda}^{\star-4/3}\big)_{:,l}.
\end{align*}
One can use the assumption (\ref{assumption:u-inf-norm}), as well
as the conditions (\ref{eq:U-true-tilde-spectrum}) and $\kappa\asymp1$,
to bound the second term
\begin{align}
\big\|\widetilde{\bm{U}}^{\star}\big((\widetilde{\bm{U}}^{\star\top}\widetilde{\bm{U}}^{\star})^{-1}-\bm{\Lambda}^{\star-4/3}\big)_{:,l}\big\|_{2} & \leq\big\|\widetilde{\bm{U}}^{\star}\big\|\big\|(\widetilde{\bm{U}}^{\star\top}\widetilde{\bm{U}}^{\star})^{-1}\big\|\big\|\widetilde{\bm{U}}^{\star\top}\widetilde{\bm{U}}^{\star}-\bm{\Lambda}^{\star4/3}\big\|\big\|\bm{\Lambda}^{\star-4/3}\big\|\nonumber \\
 & \lesssim\lambda_{\max}^{\star2/3}\cdot\frac{1}{\lambda_{\min}^{\star4/3}}\cdot r\max_{i\neq j}\left|\left\langle \bm{u}_{i}^{\star},\bm{u}_{j}^{\star}\right\rangle \right|^{2}\cdot\frac{1}{\lambda_{\min}^{\star4/3}}\lesssim\frac{1}{\lambda_{\min}^{\star2/3}}\frac{\mu r}{d}=\frac{o\left(1\right)}{\left\Vert \bm{u}_{l}^{\star}\right\Vert _{2}^{2}},\label{eq:U-true-tilde-Gram-Ind-diff}
\end{align}
where the last step arises from the condition $r=o\left(d/\mu\right)$.
Therefore, we obtain that $\big\|\bm{V}_{:,l}^{\star}\big\|_{2}=\left(1+o\left(1\right)\right)\left\Vert \bm{u}_{l}^{\star}\right\Vert _{2}^{-2}.$
\item Regarding the $\ell_{\infty}$ norm of $\bm{V}_{:,l}^{\star}$, we
can use (\ref{eq:U-true-tilde-norm}) and (\ref{eq:U-true-tilde-spectrum})
to upper bound
\[
\big\|\bm{V}_{:,l}^{\star}\big\|_{\infty}\leq\big\|\widetilde{\bm{U}}^{\star}\big\|_{2,\infty}\big\|(\widetilde{\bm{U}}^{\star\top}\widetilde{\bm{U}}^{\star})^{-1}\big\|\lesssim\frac{\mu\sqrt{r}\,\lambda_{\max}^{\star2/3}}{d}\frac{1}{\lambda_{\min}^{\star4/3}}\lesssim\frac{\mu\sqrt{r}}{d}\frac{1}{\lambda_{\min}^{\star2/3}}.
\]
\item Moreover, for any $1\leq i\leq d$, one can apply (\ref{eq:U-true-tilde-norm})
and (\ref{eq:U-true-tilde-spectrum}) again to demonstrate that
\begin{align*}
\sum_{1\leq k\leq d}V_{(i,k),l}^{\star2} & =\sum_{1\leq k\leq d}\big(\widetilde{\bm{U}}_{(i,k),:}^{\star}(\widetilde{\bm{U}}^{\star\top}\widetilde{\bm{U}}^{\star})_{:,l}^{-1}\big)^{2}\leq\sum_{1\leq k\leq d}\big\|\widetilde{\bm{U}}_{(i,k),:}^{\star}\big\|_{2}^{2}\big\|(\widetilde{\bm{U}}^{\star\top}\widetilde{\bm{U}}^{\star})^{-1}\big\|^{2}\\
 & \leq\frac{1}{\lambda_{\min}^{\star8/3}}\sum_{1\leq k\leq d}\sum_{1\leq s\leq r}\left(\bm{u}_{s}^{\star}\right)_{i}^{2}\left(\bm{u}_{s}^{\star}\right)_{k}^{2}\leq\frac{1}{\lambda_{\min}^{\star8/3}}\left\Vert \bm{U}^{\star}\right\Vert _{2,\infty}^{2}\max_{1\leq s\leq r}\left\Vert \bm{u}_{s}^{\star}\right\Vert _{2}^{2}\\
 & \lesssim\frac{1}{\lambda_{\min}^{\star8/3}}\cdot\lambda_{\max}^{\star2/3}\cdot\frac{\mu r}{d}\lambda_{\max}^{\star2/3}\lesssim\frac{\mu r}{d}\frac{1}{\lambda_{\min}^{\star4/3}},
\end{align*}
where the last line holds due to $\kappa\asymp1$.
\item Regarding the $\ell_{2}$ loss, invoke (\ref{eq:U-tilde-spectrum}),
(\ref{eq:U-tilde-op-loss}) and (\ref{eq:U-tilde-Gram-inv-op-loss})
to upper bound
\begin{align*}
\left\Vert (\bm{V}_{:,l}-\bm{V}_{:,l}^{\star})\right\Vert _{2} & \leq\big\|\widetilde{\bm{U}}(\widetilde{\bm{U}}^{\top}\widetilde{\bm{U}})^{-1}-\widetilde{\bm{U}}^{\star}(\widetilde{\bm{U}}^{\star\top}\widetilde{\bm{U}}^{\star})^{-1}\big\|\\
 & \leq\big\|\widetilde{\bm{U}}-\widetilde{\bm{U}}^{\star}\big\|\big\|(\widetilde{\bm{U}}^{\top}\widetilde{\bm{U}})^{-1}\big\|+\big\|\widetilde{\bm{U}}^{\star}\big\|\big\|(\widetilde{\bm{U}}^{\top}\widetilde{\bm{U}})^{-1}-(\widetilde{\bm{U}}^{\star\top}\widetilde{\bm{U}}^{\star})^{-1}\big\|\\
 & \lesssim\frac{\sigma_{\max}}{\lambda_{\min}^{\star}}\sqrt{\frac{d}{p}}\,\lambda_{\max}^{\star2/3}\cdot\frac{1}{\lambda_{\min}^{\star4/3}}+\lambda_{\max}^{\star2/3}\cdot\frac{\sigma_{\max}}{\lambda_{\min}^{\star3}}\sqrt{\frac{d}{p}}\,\lambda_{\max}^{\star2/3}\lesssim\frac{\sigma_{\max}}{\lambda_{\min}^{\star}}\sqrt{\frac{d}{p}}\,\frac{1}{\lambda_{\min}^{\star2/3}},
\end{align*}
which holds as long as $\kappa\asymp1$.
\item Finally, combining (\ref{eq:U-tilde-loss-2inf}), (\ref{eq:U-tilde-spectrum}),
(\ref{eq:U-true-tilde-norm}) and (\ref{eq:U-tilde-Gram-inv-op-loss})
allows us to upper bound the $\ell_{\infty}$ loss by
\begin{align*}
\left\Vert (\bm{V}_{:,l}-\bm{V}_{:,l}^{\star})\right\Vert _{\infty} & \leq\big\|\widetilde{\bm{U}}-\widetilde{\bm{U}}^{\star}\big\|_{2,\infty}\big\|(\widetilde{\bm{U}}^{\top}\widetilde{\bm{U}})^{-1}\big\|+\big\|\widetilde{\bm{U}}^{\star}\big\|_{2,\infty}\big\|(\widetilde{\bm{U}}^{\top}\widetilde{\bm{U}})^{-1}-(\widetilde{\bm{U}}^{\star\top}\widetilde{\bm{U}}^{\star})^{-1}\big\|\\
 & \lesssim\frac{\sigma_{\max}}{\lambda_{\min}^{\star}}\sqrt{\frac{\mu^{2}r\log d}{dp}}\,\lambda_{\max}^{\star2/3}\cdot\frac{1}{\lambda_{\min}^{\star4/3}}+\frac{\mu\sqrt{r}}{d}\lambda_{\max}^{\star2/3}\cdot\frac{\sigma_{\max}}{\lambda_{\min}^{\star3}}\sqrt{\frac{d}{p}}\,\lambda_{\max}^{\star2/3}\\
 & \asymp\frac{\sigma_{\max}}{\lambda_{\min}^{\star}}\sqrt{\frac{\mu^{2}r\log d}{dp}}\,\frac{1}{\lambda_{\min}^{\star2/3}},
\end{align*}
where the last step arises from $\kappa\asymp1$.
\end{enumerate}

\subsection{Proof of Lemma \ref{lemma:T-entry-var-est-loss-neg}}

\label{subsec:T-entry-var-est-loss-neg}

Fix any $1\leq i\leq j\leq k\leq d$. In order to control $K_{i,j,k}$,
we will apply an almost identical argument as in Appendix~\ref{subsec:u-entry-var-est-loss-neg}
for Lemma~\ref{lemma:u-entry-var-est-loss-neg}. We omit some details
of proof for the sake of conciseness. 

By definition, one can express
\[
K_{i,j,k}=\frac{T_{i,j,k}-T_{i,j,k}^{\star}}{v_{i,j,k}^{1/2}}-\frac{T_{i,j,k}-T_{i,j,k}^{\star}}{v_{i,j,k}^{\star1/2}}=\left(T_{i,j,k}-T_{i,j,k}^{\star}\right)\frac{v_{i,j,k}^{\star}-v_{i,j,k}}{\sqrt{v_{i,j,k}v_{i,j,k}^{\star}}}\frac{1}{\sqrt{v_{i,j,k}^{\star}}+\sqrt{v_{i,j,k}}}.
\]
Recall the definitions $\bm{\Delta}:=\bm{U}\bm{\Pi}-\bm{U}^{\star}$
and $\widetilde{\bm{\Delta}}=\big[\bm{\Delta}_{l}^{\otimes2}\big]_{1\leq l\leq r}\in\mathbb{R}^{d^{2}\times r}$.
In view of the decomposition in (\ref{eq:T-entry-loss}), it is straightforward
to bound
\begin{align*}
\big|T_{i,j,k}-T_{i,j,k}^{\star}\big| & \lesssim\left\Vert \bm{\Delta}\right\Vert _{2,\infty}\Big(\big\|\widetilde{\bm{U}}_{(i,j),:}^{\star}\big\|_{2}+\big\|\widetilde{\bm{U}}_{(i,k),:}^{\star}\big\|_{2}+\big\|\widetilde{\bm{U}}_{(j,k),:}^{\star}\big\|_{2}\Big)\\
 & \quad+\Big|\big\langle\bm{U}_{i,:}^{\star},\widetilde{\bm{\Delta}}_{(j,k),:}\big\rangle+\big\langle\bm{U}_{j,:}^{\star},\widetilde{\bm{\Delta}}_{(i,k),:}\big\rangle+\big\langle\bm{U}_{k,:}^{\star},\widetilde{\bm{\Delta}}_{(i,k),:}\big\rangle+\big\langle\bm{\Delta}_{i,:},\widetilde{\bm{\Delta}}_{(j,k),:}\big\rangle\Big|\\
 & \overset{(\mathrm{i})}{\lesssim}\frac{\sigma_{\max}}{\lambda_{\min}^{\star}}\sqrt{\frac{\mu r\log d}{p}}\,\lambda_{\max}^{\star1/3}\Big(\big\|\widetilde{\bm{U}}_{(i,j),:}^{\star}\big\|_{2}+\big\|\widetilde{\bm{U}}_{(i,k),:}^{\star}\big\|_{2}+\big\|\widetilde{\bm{U}}_{(j,k),:}^{\star}\big\|_{2}\Big)\\
 & \quad+\left(\frac{\sigma_{\max}}{\lambda_{\min}^{\star}}\sqrt{\frac{\mu r\log d}{p}}\lambda_{\max}^{\star1/3}\right)^{2}\sqrt{\frac{\mu}{d}}\,\lambda_{\max}^{\star1/3}\\
 & \overset{(\mathrm{ii})}{\lesssim}\left\{ \sqrt{\mu r\log d}+\frac{\sigma_{\max}}{\lambda_{\min}^{\star1/3}}\frac{\mu^{3/2}r\log d}{\sqrt{dp}}\Big(\big\|\widetilde{\bm{U}}_{(n,l),:}^{\star}\big\|_{2}+\big\|\widetilde{\bm{U}}_{(m,l),:}^{\star}\big\|_{2}+\big\|\widetilde{\bm{U}}_{(m,n),:}^{\star}\big\|_{2}\Big)^{-1}\right\} \sqrt{v_{i,j,k}^{\star}}\\
 & \overset{(\mathrm{iii})}{\lesssim}\sqrt{\mu r\log d}\sqrt{v_{i,j,k}^{\star}},
\end{align*}
where (i) uses (\ref{eq:U-loss-2inf}), (\ref{eq:T-loss-dist-neg-part-part2})
and (\ref{eq:T-loss-dist-neg-part-part3}); (ii) follows from the
lower bound of $v_{i,j,k}^{\star}$ in (\ref{eq:T-entry-var-LB})
and conditions $\sigma_{\max}/\sigma_{\min},\kappa\asymp1$; and (iii)
arises from the assumption (\ref{eq:U-tilde-2norm-LB-rank-r}).

Then the claim (\ref{claim:T-entry-var-loss}) would immediately follow
as long as we could show that
\begin{align}
\frac{\big|v_{i,j,k}-v_{i,j,k}^{\star}\big|}{v_{i,j,k}^{\star}} & \lesssim\sqrt{\frac{\mu^{3}r^{2}\log d}{d^{2}p}}+\Big(\big\|\widetilde{\bm{U}}_{(i,j),:}^{\star}\big\|_{2}+\big\|\widetilde{\bm{U}}_{(i,k),:}^{\star}\big\|_{2}+\big\|\widetilde{\bm{U}}_{(j,k),:}^{\star}\big\|_{2}\Big)^{-1}\frac{\sigma_{\max}}{\lambda_{\min}^{\star1/3}}\sqrt{\frac{\mu^{4}r^{2}\log d}{dp}}=o\left(1\right).\label{claim:T-entry-var-loss}
\end{align}
Indeed, one can apply the triangle inequality to show that $v_{i,j,k}\asymp v_{i,j,k}^{\star}$,
and consequently obtain
\begin{align*}
\left|K_{i,j,k}\right| & \lesssim\frac{1}{v_{i,j,k}^{\star3/2}}\big|T_{i,j,k}-T_{i,j,k}^{\star}\big|\big|v_{i,j,k}-v_{i,j,k}^{\star}\big|\\
 & \lesssim\sqrt{\frac{\mu^{4}r^{3}\log^{2}d}{d^{2}p}}+\Big(\big\|\widetilde{\bm{U}}_{(i,j),:}^{\star}\big\|_{2}+\big\|\widetilde{\bm{U}}_{(i,k),:}^{\star}\big\|_{2}+\big\|\widetilde{\bm{U}}_{(j,k),:}^{\star}\big\|_{2}\Big)^{-1}\frac{\sigma_{\max}}{\lambda_{\min}^{\star1/3}}\sqrt{\frac{\mu^{5}r^{3}\log^{2}d}{dp}}
\end{align*}
as claimed.

Therefore, it remains to justify (\ref{claim:T-entry-var-loss}).
For notational convenience, we define the following $d^{2}\times d^{2}$
matrices:
\begin{align}
\bm{P} & :=\widetilde{\bm{U}}(\widetilde{\bm{U}}^{\top}\widetilde{\bm{U}})^{-1}\widetilde{\bm{U}}^{\top},\qquad\bm{P}^{\star}:=\widetilde{\bm{U}}^{\star}(\widetilde{\bm{U}}^{\star\top}\widetilde{\bm{U}}^{\star})^{-1}\widetilde{\bm{U}}^{\star\top}.\label{def:P}
\end{align}
We can then express
\begin{align}
v_{i,j,k}^{\star} & =\frac{2}{p}\,\Big(\bm{P}_{(i,j),:}^{\star}\bm{D}_{k}^{\star}\bm{P}_{:,(i,j)}^{\star}+\bm{P}_{(i,k),:}^{\star}\bm{D}_{j}^{\star}\bm{P}_{:,(i,k)}^{\star}+\bm{P}_{(j,k),:}^{\star}\bm{D}_{i}^{\star}\bm{P}_{:,(j,k)}^{\star}\Big),\label{eq:T-entry-var-expression}\\
v_{i,j,k} & =\frac{2}{p}\,\Big(\bm{P}_{(i,j),:}\bm{D}_{k}\bm{P}_{:,(i,j)}+\bm{P}_{(i,k),:}\bm{D}_{j}\bm{P}_{:,(i,k)}+\bm{P}_{(j,k),:}\bm{D}_{i}\bm{P}_{:,(j,k)}\Big),\label{eq:T-entry-var-est-expression}
\end{align}
where $\bm{D}_{k}^{\star}$ (resp.~$\bm{D}_{k}$) is defined in (\ref{eq:cov-matrix-m-diag})
(resp.~(\ref{eq:defn-Dk})) for each $1\leq k\leq d$. Lemma~\ref{lemma:P-property}
summarizes several bounds regarding $\bm{P}$ and $\bm{P}^{\star}$,
whose proof can be found at the end of the section.

\begin{lemma}\label{lemma:P-property}Instate the assumptions and
notations of Lemma \ref{lemma:T-entry-var-est-loss-neg}. For any
$1\leq i,j\leq d$, one has
\begin{align}
 & \big\|\bm{P}_{(i,j),:}^{\star}\big\|_{2}\lesssim\frac{1}{\lambda_{\min}^{\star2/3}}\big\|\widetilde{\bm{U}}_{(i,j),:}^{\star}\big\|_{2},\qquad\big\|\bm{P}_{(i,j),:}^{\star}\big\|_{\infty}\lesssim\frac{\mu\sqrt{r}}{d}\frac{1}{\lambda_{\min}^{\star2/3}}\big\|\widetilde{\bm{U}}_{(i,j),:}^{\star}\big\|_{2};\label{eq:P-true-norm}\\
 & \big\|(\bm{P}-\bm{P}^{\star})_{(i,j),:}\big\|_{2}\lesssim\frac{\sigma_{\max}}{\lambda_{\min}^{\star}}\sqrt{\frac{d\log d}{p}}\frac{\mu\sqrt{r}}{d},\qquad\big\|(\bm{P}-\bm{P}^{\star})_{(i,j),:}\big\|_{\infty}\lesssim\frac{\sigma_{\max}}{\lambda_{\min}^{\star}}\sqrt{\frac{d\log d}{p}}\frac{\mu^{2}r}{d^{2}}.\label{eq:P-loss}
\end{align}
\end{lemma}With these in mind, we are positioned to upper bound $v_{i,j,k}-v_{i,j,k}^{\star}$.
By (\ref{eq:T-entry-var-expression}), (\ref{eq:T-entry-var-est-expression})
and the triangle inequality, we will show below how to upper bound
$\bm{P}_{(i,j),:}\bm{D}_{k}\bm{P}_{:,(i,j)}-\bm{P}_{(i,j),:}^{\star}\bm{D}_{k}^{\star}\bm{P}_{:,(i,j)}^{\star}$.
The other two terms can be controlled analogously.

Recall the auxiliary matrix $\widehat{\bm{D}}_{k}$ (cf.~\ref{def:D-hat}).
One can then expand
\begin{align*}
\bm{P}_{(i,j),:}\bm{D}_{k}\bm{P}_{:,(i,j)}-\bm{P}_{(i,j),:}^{\star}\bm{D}_{k}^{\star}\bm{P}_{:,(i,j)}^{\star} & =\underbrace{\bm{P}_{(i,j),:}(\bm{D}_{k}-\widehat{\bm{D}}_{k})\bm{P}_{:,(i,j)}}_{=:\,\beta_{1}}+\underbrace{\bm{P}_{(i,j),:}\widehat{\bm{D}}_{k}\bm{P}_{:,(i,j)}-\bm{P}_{(i,j),:}^{\star}\bm{D}_{k}^{\star}\bm{P}_{:,(i,j)}^{\star}}_{=:\,\beta_{2}}.
\end{align*}
In what follows, we shall control $\beta_{1}$ and $\beta_{2}$ individually.
\begin{itemize}
\item For $\beta_{1}$, one decomposes it as follows
\begin{align*}
\beta_{1} & =\underbrace{\bm{P}_{(i,j),:}^{\star}(\bm{D}_{k}-\widehat{\bm{D}}_{k})\bm{P}_{:,(i,j)}^{\star}}_{=:\,\gamma_{1}}+2\underbrace{(\bm{P}-\bm{P}^{\star})_{:,(i,j)}(\bm{D}_{k}-\widehat{\bm{D}}_{k})\bm{P}_{:,(i,j)}^{\star}}_{=:\,\gamma_{2}}\\
 & \quad+\underbrace{(\bm{P}-\bm{P}^{\star})_{:,(i,j)}(\bm{D}_{k}-\widehat{\bm{D}}_{k})(\bm{P}-\bm{P}^{\star})_{:,(i,j)}}_{=:\,\gamma_{3}}.
\end{align*}
The term $\gamma_{1}$ can be bounded by
\begin{align*}
\left|\gamma_{1}\right| & \leq\max_{(s,l,k)\in\Omega}\big|\widehat{E}_{s,l,k}^{2}-E_{s,l,k}^{2}\big|\sum_{1\leq s,l\leq d}p^{-1}\chi_{s,l,k}P_{(i,j),(s,l)}^{\star2}.
\end{align*}
Using (\ref{eq:P-true-norm}), we know from the Bernstein inequality
and the AM-GM inequality that with probability at least $1-O\left(d^{-13}\right)$,
\begin{align}
\sum_{1\leq s,l\leq d}p^{-1}\chi_{s,l,k}P_{(i,j),(s,l)}^{\star2} & \lesssim\big\|\bm{P}_{(i,j),:}^{\star}\big\|_{2}^{2}+p^{-1}\log d\,\big\|\bm{P}_{(i,j),:}^{\star}\big\|_{\infty}^{2}+\sqrt{p^{-1}\log d}\,\big\|\bm{P}_{(i,j),:}^{\star}\big\|_{\infty}\big\|\bm{P}_{(i,j),:}^{\star}\big\|_{2}\nonumber \\
 & \asymp\big\|\bm{P}_{(i,j),:}^{\star}\big\|_{2}^{2}+p^{-1}\log d\,\big\|\bm{P}_{(i,j),:}^{\star}\big\|_{\infty}^{2}\lesssim\frac{1}{\lambda_{\min}^{\star4/3}}\left\{ 1+\frac{\mu^{2}r\log d}{d^{2}p}\right\} \big\|\widetilde{\bm{U}}_{(i,j),:}^{\star}\big\|_{2}^{2}\nonumber \\
 & \asymp\frac{1}{\lambda_{\min}^{\star4/3}}\big\|\widetilde{\bm{U}}_{(i,j),:}^{\star}\big\|_{2}^{2},\label{eq:P-concentration}
\end{align}
where the last step arises from the condition $p\gtrsim\mu^{2}rd^{-2}\log d$.
This combined with (\ref{eq:noise-square-est-loss-UB}) leads to
\[
\left|\gamma_{1}\right|\lesssim\frac{\sigma_{\max}^{2}}{\lambda_{\min}^{\star4/3}}\sqrt{\frac{\mu^{3}r^{2}\log^{2}d}{d^{2}p}}\,\big\|\widetilde{\bm{U}}_{(i,j),:}^{\star}\big\|_{2}^{2}.
\]
As for $\gamma_{2}$, invoking Cauchy-Schwartz and applying (\ref{eq:noise-square-est-loss-UB})
and (\ref{eq:P-concentration}) give
\begin{align*}
\left|\gamma_{2}\right| & \leq\max_{(s,l,k)\in\Omega}\big|\widehat{E}_{s,l,k}^{2}-E_{s,l,k}^{2}\big|\left\Vert (\bm{P}-\bm{P}^{\star})_{(i,j),:}\right\Vert _{\infty}\sqrt{\sum_{1\leq s,l\leq d}p^{-1}\chi_{s,l,k}P_{(i,j),(s,l)}^{\star2}}\sqrt{\sum_{1\leq s,l\leq d}p^{-1}\chi_{s,l,k}}\\
 & \lesssim\sigma_{\max}^{2}\sqrt{\frac{\mu^{3}r^{2}\log^{2}d}{d^{2}p}}\cdot\frac{\sigma_{\max}}{\lambda_{\min}^{\star}}\sqrt{\frac{d\log d}{p}}\frac{\mu^{2}r}{d^{2}}\cdot\frac{1}{\lambda_{\min}^{\star2/3}}\big\|\widetilde{\bm{U}}_{(i,j),:}^{\star}\big\|_{2}\cdot d\\
 & \lesssim\frac{\sigma_{\max}^{2}}{\lambda_{\min}^{\star2/3}}\frac{\sigma_{\max}}{\lambda_{\min}^{\star}}\sqrt{\frac{d\log d}{p}}\frac{\mu^{2}r}{d}\big\|\widetilde{\bm{U}}_{(i,j),:}^{\star}\big\|_{2},
\end{align*}
as long as $p\gtrsim\mu^{3}r^{2}d^{-2}\log^{2}d$. Regarding $\gamma_{3}$,
we can upper bound 
\begin{align*}
\left|\gamma_{3}\right| & \leq\max_{(s,l,k)\in\Omega}\big|\widehat{E}_{s,l,k}^{2}-E_{s,l,k}^{2}\big|\left\Vert (\bm{P}-\bm{P}^{\star})_{(i,j),:}\right\Vert _{\infty}^{2}\sum_{1\leq s,l\leq d}p^{-1}\chi_{s,l,k}\\
 & \overset{(\mathrm{i})}{\lesssim}\sigma_{\max}^{2}\sqrt{\frac{\mu^{3}r^{2}\log^{2}d}{d^{2}p}}\left(\frac{\sigma_{\max}}{\lambda_{\min}^{\star}}\sqrt{\frac{d\log d}{p}}\frac{\mu^{2}r}{d^{2}}\right)^{2}d^{2}\\
 & \overset{(\mathrm{ii})}{\lesssim}\sigma_{\max}^{2}\left(\frac{\sigma_{\max}}{\lambda_{\min}^{\star}}\sqrt{\frac{d\log d}{p}}\frac{\mu^{2}r}{d}\right)^{2},
\end{align*}
where (i) uses (\ref{eq:noise-square-est-loss-UB}), (\ref{eq:obs-card-concentration})
and (\ref{eq:P-loss}); (ii) holds as long as $p\gtrsim\mu^{3}r^{2}d^{-2}\log^{2}d$.
Taking the above bounds for $\gamma_{1},\gamma_{2}$ and $\gamma_{3}$
together indicates that
\[
\left|\beta_{1}\right|\lesssim\frac{\sigma_{\max}^{2}}{\lambda_{\min}^{\star4/3}}\sqrt{\frac{\mu^{3}r^{2}\log^{2}d}{d^{2}p}}\,\big\|\widetilde{\bm{U}}_{(i,j),:}^{\star}\big\|_{2}^{2}+\frac{\sigma_{\max}^{2}}{\lambda_{\min}^{\star2/3}}\frac{\sigma_{\max}}{\lambda_{\min}^{\star}}\sqrt{\frac{d\log d}{p}}\frac{\mu^{2}r}{d}\big\|\widetilde{\bm{U}}_{(i,j),:}^{\star}\big\|_{2}+\sigma_{\max}^{2}\left(\frac{\sigma_{\max}}{\lambda_{\min}^{\star}}\sqrt{\frac{d\log d}{p}}\frac{\mu^{2}r}{d}\right)^{2}.
\]
\item Regarding $\beta_{2}$, we start by decomposing it as follows
\begin{align*}
\beta_{2} & =\bm{P}_{(i,j),:}\widehat{\bm{D}}_{k}\bm{P}_{:,(i,j)}-\bm{P}_{(i,j),:}^{\star}\widehat{\bm{D}}_{k}\bm{P}_{:,(i,j)}^{\star}+\bm{P}_{(i,j),:}^{\star}\widehat{\bm{D}}_{k}\bm{P}_{:,(i,j)}^{\star}-\bm{P}_{(i,j),:}^{\star}\bm{D}_{k}^{\star}\bm{P}_{:,(i,j)}^{\star}\\
 & =2\underbrace{\bm{P}_{(i,j),:}^{\star}\widehat{\bm{D}}_{k}(\bm{P}-\bm{P}^{\star})_{:,(i,j)}}_{=:\,\gamma_{4}}+\underbrace{(\bm{P}-\bm{P}^{\star})_{(i,j),:}(\widehat{\bm{D}}_{k}-\bm{D}_{k}^{\star})(\bm{P}-\bm{P}^{\star})_{:,(i,j)}}_{=:\,\gamma_{5}}+\underbrace{\bm{P}_{(i,j),:}^{\star}(\widehat{\bm{D}}_{k}-\bm{D}_{k}^{\star})\bm{P}_{:,(i,j)}^{\star}}_{=:\,\gamma_{6}}.
\end{align*}
To bound $\gamma_{4}$, we can combine (\ref{eq:P-true-norm}), (\ref{eq:P-loss})
and (\ref{eq:obs-card-concentration}) with the Cauchy-Schwartz inequality
and the Bernstein inequality, to obtain with probability at least
$1-O\left(d^{-13}\right)$,
\begin{align*}
\left|\gamma_{4}\right| & =\Big|\sum_{1\leq s,l\leq d}E_{s,l,k}^{2}p^{-1}\chi_{s,l,k}P_{(i,j),(s,l)}^{\star}\big(P_{(i,j),(s,l)}-P_{(i,j),(s,l)}^{\star}\big)\Big|\\
 & \leq\left\Vert (\bm{P}-\bm{P}^{\star})_{(i,j),:}\right\Vert _{\infty}\sqrt{\sum_{1\leq s,l\leq d}E_{s,l,k}^{4}p^{-1}\chi_{s,l,k}P_{(i,j),(s,l)}^{\star2}}\sqrt{\sum_{1\leq s,l\leq d}p^{-1}\chi_{s,l,k}}.\\
 & \leq\left\Vert (\bm{P}-\bm{P}^{\star})_{(i,j),:}\right\Vert _{\infty}\cdot\sigma_{\max}^{2}\left\{ \big\|\bm{P}_{(i,j),:}^{\star}\big\|_{2}+p^{-1/2}\log d\,\big\|\bm{P}_{(i,j),:}^{\star}\big\|_{\infty}\right\} \cdot d\\
 & \lesssim\frac{\sigma_{\max}^{2}}{\lambda_{\min}^{\star2/3}}\frac{\sigma_{\max}}{\lambda_{\min}^{\star}}\sqrt{\frac{d\log d}{p}}\frac{\mu^{2}r}{d}\big\|\widetilde{\bm{U}}_{(i,j),:}^{\star}\big\|_{2}
\end{align*}
as long as $p\gtrsim\mu^{2}rd^{-2}\log^{2}d$. As for $\gamma_{5}$,
combining (\ref{eq:noise-inf-UB}), (\ref{eq:obs-card-concentration})
and (\ref{eq:P-loss}) shows that with probability at least $1-O\left(d^{-13}\right)$,
\begin{align*}
\gamma_{5} & =\sum_{1\leq s,l\leq d}E_{s,l,k}^{2}p^{-1}\chi_{s,l,k}\big(P_{(i,j),(s,l)}-P_{(i,j),(s,l)}^{\star}\big)^{2}\\
 & \leq\left\Vert \bm{E}\right\Vert _{\infty}^{2}\left\Vert (\bm{P}-\bm{P}^{\star})_{(i,j),:}\right\Vert _{\infty}^{2}\sum_{1\leq s,l\leq d}p^{-1}\chi_{s,l,k}\lesssim\sigma_{\max}^{2}\left(\frac{\sigma_{\max}}{\lambda_{\min}^{\star}}\sqrt{\frac{d\log^{2}d}{p}}\frac{\mu^{2}r}{d}\right)^{2}.
\end{align*}
Finally, observe that $\gamma_{6}$ is a sum of independent random
variables. By (\ref{eq:P-true-norm}), invoking the Bernstein inequality
reveals that with probability at least $1-O\left(d^{-13}\right)$,
\begin{align*}
\left|\gamma_{6}\right| & \lesssim\sigma_{\max}^{2}\left\{ \frac{\log^{2}d}{p}\big\|\bm{P}_{(i,j),:}^{\star}\big\|_{\infty}^{2}+\sqrt{\frac{\log d}{p}}\big\|\bm{P}_{(i,j),:}^{\star}\big\|_{\infty}\big\|\bm{P}_{(i,j),:}^{\star}\big\|_{2}\right\} \\
 & \lesssim\frac{\sigma_{\max}^{2}}{\lambda_{\min}^{\star4/3}}\left\{ \frac{\mu^{2}r\log^{2}d}{d^{2}p}+\sqrt{\frac{\mu^{2}r\log d}{d^{2}p}}\right\} \big\|\widetilde{\bm{U}}_{(i,j),:}^{\star}\big\|_{2}^{2}\asymp\frac{\sigma_{\max}^{2}}{\lambda_{\min}^{\star4/3}}\sqrt{\frac{\mu^{2}r\log d}{d^{2}p}}\big\|\widetilde{\bm{U}}_{(i,j),:}^{\star}\big\|_{2}^{2},
\end{align*}
as long as $p\gtrsim\mu^{2}rd^{-2}\log^{3}d$. Therefore, we combine
bounds for $\gamma_{4},\gamma_{5}$ and $\gamma_{6}$ to conclude
that
\begin{align*}
\left|\beta_{2}\right| & \lesssim\frac{\sigma_{\max}^{2}}{\lambda_{\min}^{\star4/3}}\sqrt{\frac{\mu^{2}r\log d}{d^{2}p}}\big\|\widetilde{\bm{U}}_{(i,j),:}^{\star}\big\|_{2}^{2}+\frac{\sigma_{\max}^{2}}{\lambda_{\min}^{\star2/3}}\frac{\sigma_{\max}}{\lambda_{\min}^{\star}}\sqrt{\frac{d\log d}{p}}\frac{\mu^{2}r}{d}\big\|\widetilde{\bm{U}}_{(i,j),:}^{\star}\big\|_{2}+\sigma_{\max}^{2}\left(\frac{\sigma_{\max}}{\lambda_{\min}^{\star}}\sqrt{\frac{d\log^{2}d}{p}}\frac{\mu^{2}r}{d}\right)^{2}.
\end{align*}
\item Putting the above bounds for $\beta_{1}$ and $\beta_{2}$ together
reveals that
\begin{align*}
 & \big|\bm{P}_{(i,j),:}\bm{D}_{k}\bm{P}_{:,(i,j)}-\bm{P}_{(i,j),:}^{\star}\bm{D}_{k}^{\star}\bm{P}_{:,(i,j)}^{\star}\big|\\
 & \qquad\lesssim\frac{\sigma_{\max}^{2}}{\lambda_{\min}^{\star4/3}}\sqrt{\frac{\mu^{3}r^{2}\log d}{d^{2}p}}\,\big\|\widetilde{\bm{U}}_{(i,j),:}^{\star}\big\|_{2}^{2}+\frac{\sigma_{\max}^{2}}{\lambda_{\min}^{\star2/3}}\frac{\sigma_{\max}}{\lambda_{\min}^{\star}}\sqrt{\frac{d\log d}{p}}\frac{\mu^{2}r}{d}\big\|\widetilde{\bm{U}}_{(i,j),:}^{\star}\big\|_{2}+\sigma_{\max}^{2}\left(\frac{\sigma_{\max}}{\lambda_{\min}^{\star}}\sqrt{\frac{d\log^{2}d}{p}}\frac{\mu^{2}r}{d}\right)^{2}.
\end{align*}
\item Clearly, we can apply an analogous argument to bound $\bm{P}_{(i,k),:}\bm{D}_{j}\bm{P}_{:,(i,k)}-\bm{P}_{(i,k),:}^{\star}\bm{D}_{j}^{\star}\bm{P}_{:,(i,k)}^{\star}$
and $\bm{P}_{(j,k),:}\bm{D}_{i}\bm{P}_{:,(j,k)}-\bm{P}_{(j,k),:}^{\star}\bm{D}_{i}^{\star}\bm{P}_{:,(j,k)}^{\star}$.
Taken collectively with the lower bound of $v_{i,j,k}^{\star}$ (cf.~(\ref{eq:T-entry-var-LB}))
and the conditions $\sigma_{\max}/\sigma_{\min}\asymp1$ and $\kappa\asymp1$,
we obtain
\begin{align*}
\frac{\big|v_{i,j,k}-v_{i,j,k}^{\star}\big|}{v_{i,j,k}^{\star}} & \lesssim\sqrt{\frac{\mu^{3}r^{2}\log d}{d^{2}p}}+\Big(\big\|\widetilde{\bm{U}}_{(i,j),:}^{\star}\big\|_{2}+\big\|\widetilde{\bm{U}}_{(i,k),:}^{\star}\big\|_{2}+\big\|\widetilde{\bm{U}}_{(j,k),:}^{\star}\big\|_{2}\Big)^{-1}\frac{\sigma_{\max}}{\lambda_{\min}^{\star1/3}}\sqrt{\frac{d\log d}{p}}\frac{\mu^{2}r}{d}\\
 & \quad+\Big(\big\|\widetilde{\bm{U}}_{(i,j),:}^{\star}\big\|_{2}+\big\|\widetilde{\bm{U}}_{(i,k),:}^{\star}\big\|_{2}+\big\|\widetilde{\bm{U}}_{(j,k),:}^{\star}\big\|_{2}\Big)^{-2}\left(\frac{\sigma_{\max}}{\lambda_{\min}^{\star1/3}}\sqrt{\frac{d\log^{2}d}{p}}\frac{\mu^{2}r}{d}\right)^{2}.\\
 & \asymp\sqrt{\frac{\mu^{3}r^{2}\log d}{d^{2}p}}+\Big(\big\|\widetilde{\bm{U}}_{(i,j),:}^{\star}\big\|_{2}+\big\|\widetilde{\bm{U}}_{(i,k),:}^{\star}\big\|_{2}+\big\|\widetilde{\bm{U}}_{(j,k),:}^{\star}\big\|_{2}\Big)^{-1}\frac{\sigma_{\max}}{\lambda_{\min}^{\star1/3}}\sqrt{\frac{d\log d}{p}}\frac{\mu^{2}r}{d}=o\left(1\right),
\end{align*}
where the last line holds due to the assumptions (\ref{eq:U-tilde-2norm-LB-rank-r})
and $p\gg\mu^{3}r^{2}d^{-2}\log^{2}d$.
\end{itemize}

\subsubsection{Proof of Lemma \ref{lemma:P-property}}

Fix any $1\leq i,j\leq d$.
\begin{itemize}
\item We start with the norms of the rows of $\bm{P}^{\star}$ (cf.~\ref{def:P}).
By (\ref{eq:U-true-tilde-spectrum}), it is straightforward to deduce
that
\end{itemize}
\begin{align*}
\big\|\bm{P}_{(i,j),:}^{\star}\big\|_{2} & \leq\big\|\widetilde{\bm{U}}_{(i,j),:}^{\star}\big\|_{2}\big\|(\widetilde{\bm{U}}^{\star\top}\widetilde{\bm{U}}^{\star})^{-1}\big\|\big\|\widetilde{\bm{U}}^{\star}\big\|\lesssim\frac{1}{\lambda_{\min}^{\star2/3}}\big\|\widetilde{\bm{U}}_{(i,j),:}^{\star}\big\|_{2};\\
\big\|\bm{P}_{(i,j),:}^{\star}\big\|_{\infty} & \leq\big\|\widetilde{\bm{U}}_{(i,j),:}^{\star}\big\|_{2}\big\|(\widetilde{\bm{U}}^{\star\top}\widetilde{\bm{U}}^{\star})^{-1}\big\|\big\|\widetilde{\bm{U}}^{\star}\big\|_{2,\infty}\lesssim\frac{\mu\sqrt{r}}{d}\frac{1}{\lambda_{\min}^{\star2/3}}\big\|\widetilde{\bm{U}}_{(i,j),:}^{\star}\big\|_{2}.
\end{align*}

\begin{itemize}
\item Next, we move on to the $\ell_{2}$ norm of $(\bm{P}-\bm{P}^{\star})_{(i,j),:}$,
which can be decomposed as
\begin{align}
(\bm{P}-\bm{P}^{\star})_{(i,j),:} & =\widetilde{\bm{U}}_{(i,j),:}(\widetilde{\bm{U}}^{\top}\widetilde{\bm{U}})^{-1}\widetilde{\bm{U}}^{\top}-\widetilde{\bm{U}}_{(i,j),:}^{\star}(\widetilde{\bm{U}}^{\star\top}\widetilde{\bm{U}}^{\star})^{-1}\widetilde{\bm{U}}^{\star\top}\nonumber \\
 & =(\widetilde{\bm{U}}-\widetilde{\bm{U}}^{\star})_{(i,j),:}(\widetilde{\bm{U}}^{\top}\widetilde{\bm{U}})^{-1}\widetilde{\bm{U}}^{\top}+\widetilde{\bm{U}}_{(i,j),:}^{\star}\big((\widetilde{\bm{U}}^{\top}\widetilde{\bm{U}})^{-1}-(\widetilde{\bm{U}}^{\star\top}\widetilde{\bm{U}}^{\star})^{-1}\big)\widetilde{\bm{U}}^{\top}\nonumber \\
 & \quad+\widetilde{\bm{U}}_{(i,j),:}^{\star}(\widetilde{\bm{U}}^{\star\top}\widetilde{\bm{U}}^{\star})^{-1}(\widetilde{\bm{U}}-\widetilde{\bm{U}}^{\star})^{\top}.\label{eq:P-P-true-decomp}
\end{align}
By the triangle inequality, we shall control these three terms separately.
From (\ref{eq:U-tilde-loss-2inf}), one has
\begin{align*}
\big\|(\widetilde{\bm{U}}-\widetilde{\bm{U}}^{\star})_{(i,j),:}(\widetilde{\bm{U}}^{\top}\widetilde{\bm{U}})^{-1}\widetilde{\bm{U}}^{\top}\big\|_{2} & \leq\big\|(\widetilde{\bm{U}}-\widetilde{\bm{U}}^{\star})_{(i,j),:}\big\|_{2}\big\|(\widetilde{\bm{U}}^{\top}\widetilde{\bm{U}})^{-1}\big\|\big\|\widetilde{\bm{U}}\big\|\\
 & \lesssim\frac{\sigma_{\max}}{\lambda_{\min}^{\star}}\sqrt{\frac{\mu^{2}r\log d}{dp}}\,\lambda_{\max}^{\star2/3}\cdot\frac{1}{\lambda_{\min}^{\star4/3}}\cdot\lambda_{\max}^{\star2/3}.
\end{align*}
As for the remaining two terms, from (\ref{eq:U-true-tilde-norm}),
(\ref{eq:U-tilde-spectrum}) and (\ref{eq:U-tilde-Gram-inv-op-loss}),
one has
\begin{align*}
\big\|\widetilde{\bm{U}}_{(i,j),:}^{\star}\big\|_{2}\big\|(\widetilde{\bm{U}}^{\top}\widetilde{\bm{U}})^{-1}-(\widetilde{\bm{U}}^{\star\top}\widetilde{\bm{U}}^{\star})^{-1}\big\|\big\|\widetilde{\bm{U}}\big\| & \lesssim\frac{\sigma_{\max}}{\lambda_{\min}^{\star3}}\sqrt{\frac{d}{p}}\,\lambda_{\max}^{\star2/3}\cdot\lambda_{\max}^{\star2/3}\big\|\widetilde{\bm{U}}_{(i,j),:}^{\star}\big\|_{2}\\
 & \lesssim\frac{\sigma_{\max}}{\lambda_{\min}^{\star5/3}}\sqrt{\frac{d}{p}}\,\big\|\widetilde{\bm{U}}_{(i,j),:}^{\star}\big\|_{2}.
\end{align*}
Combining (\ref{eq:U-true-tilde-norm}), (\ref{eq:U-true-tilde-spectrum})
and (\ref{eq:U-tilde-op-loss}) yields
\begin{align*}
\big\|\widetilde{\bm{U}}_{(i,j),:}^{\star}\big\|\big\|(\widetilde{\bm{U}}^{\star\top}\widetilde{\bm{U}}^{\star})^{-1}\big\|\big\|\widetilde{\bm{U}}-\widetilde{\bm{U}}^{\star}\big\| & \lesssim\frac{1}{\lambda_{\min}^{\star4/3}}\cdot\frac{\sigma_{\max}}{\lambda_{\min}^{\star}}\sqrt{\frac{d}{p}}\,\lambda_{\max}^{\star2/3}\big\|\widetilde{\bm{U}}_{(i,j),:}^{\star}\big\|_{2}\\
 & \lesssim\frac{\sigma_{\max}}{\lambda_{\min}^{\star5/3}}\sqrt{\frac{d}{p}}\,\big\|\widetilde{\bm{U}}_{(i,j),:}^{\star}\big\|_{2}.
\end{align*}
The above bounds taken collectively allow us to obtain
\[
\big\|(\bm{P}-\bm{P}^{\star})_{(i,j),:}\big\|_{2}\lesssim\frac{\sigma_{\max}}{\lambda_{\min}^{\star5/3}}\sqrt{\frac{d}{p}}\left(\frac{\mu\sqrt{r\log d}}{d}\lambda_{\max}^{\star2/3}+\big\|\widetilde{\bm{U}}_{(i,j),:}^{\star}\big\|_{2}\right)\asymp\frac{\sigma_{\max}}{\lambda_{\min}^{\star}}\sqrt{\frac{d\log d}{p}}\frac{\mu\sqrt{r}}{d},
\]
where the last step arises from the incoherence condition that $\|\widetilde{\bm{U}}^{\star}\|_{2,\infty}\lesssim\mu\sqrt{r}\lambda_{\max}^{\star2/3}/d$
and $\kappa\asymp1$.
\item Finally, let us look at the $\ell_{\infty}$ norm of $(\bm{P}-\bm{P}^{\star})_{(i,j),:}$.
Armed with the decomposition in (\ref{eq:P-P-true-decomp}), we can
bound
\begin{align*}
 & \big\|(\bm{P}-\bm{P}^{\star})_{(i,j),:}\big\|_{\infty}\leq\big\|(\widetilde{\bm{U}}-\widetilde{\bm{U}}^{\star})_{(i,j),:}\big\|_{2}\big\|(\widetilde{\bm{U}}^{\top}\widetilde{\bm{U}})^{-1}\big\|\big\|\widetilde{\bm{U}}\big\|_{2,\infty}\\
 & \qquad\qquad+\big\|\widetilde{\bm{U}}_{(i,j),:}^{\star}\big\|_{2}\big\|(\widetilde{\bm{U}}^{\top}\widetilde{\bm{U}})^{-1}-(\widetilde{\bm{U}}^{\star\top}\widetilde{\bm{U}}^{\star})^{-1}\big\|\big\|\widetilde{\bm{U}}\big\|_{2,\infty}+\big\|\widetilde{\bm{U}}_{(i,j),:}^{\star}\big\|_{2}\big\|(\widetilde{\bm{U}}^{\star\top}\widetilde{\bm{U}}^{\star})^{-1}\big\|\big\|\widetilde{\bm{U}}-\widetilde{\bm{U}}^{\star}\big\|_{2,\infty}\\
 & \qquad\overset{(\mathrm{i})}{\lesssim}\frac{\sigma_{\max}}{\lambda_{\min}^{\star}}\sqrt{\frac{\mu^{2}r\log d}{dp}}\,\lambda_{\max}^{\star2/3}\cdot\frac{1}{\lambda_{\min}^{\star4/3}}\cdot\frac{\mu\sqrt{r}}{d}\lambda_{\max}^{\star2/3}+\frac{\sigma_{\max}}{\lambda_{\min}^{\star3}}\sqrt{\frac{d}{p}}\,\lambda_{\max}^{\star2/3}\cdot\frac{\mu\sqrt{r}}{d}\lambda_{\max}^{\star2/3}\big\|\widetilde{\bm{U}}_{(i,j),:}^{\star}\big\|_{2}\\
 & \qquad\quad+\frac{1}{\lambda_{\min}^{\star4/3}}\cdot\frac{\sigma_{\max}}{\lambda_{\min}^{\star}}\sqrt{\frac{\mu^{2}r\log d}{dp}}\,\lambda_{\max}^{\star2/3}\big\|\widetilde{\bm{U}}_{(i,j),:}^{\star}\big\|_{2}\\
 & \qquad\overset{(\mathrm{ii})}{\lesssim}\frac{\sigma_{\max}}{\lambda_{\min}^{\star5/3}}\sqrt{\frac{\mu^{2}r\log d}{dp}}\left(\frac{\mu\sqrt{r}}{d}\lambda_{\max}^{\star2/3}+\big\|\widetilde{\bm{U}}_{(i,j),:}^{\star}\big\|_{2}\right)\overset{(\mathrm{iii})}{\lesssim}\frac{\sigma_{\max}}{\lambda_{\min}^{\star}}\sqrt{\frac{\mu^{2}r\log d}{dp}}\frac{\mu\sqrt{r}}{d}.
\end{align*}
Here, (i) relies on (\ref{eq:U-tilde-property}) and (\ref{eq:U-tilde-Gram-inv-op-loss}),
(ii) is due to the condition $\kappa\asymp1$, whereas (iii) arises
from (\ref{eq:U-true-tilde-norm}) and $\kappa\asymp1$.
\end{itemize}

%% file: proof-estimation.tex
\section{Proof of $\ell_{2}$ estimation guarantees (Theorem~\ref{thm:optimality-L2-rank-r})}

\label{sec:proof-l2-loss}As before, we assume that $\bm{\Pi}=\bm{I}_{r}$
for simplicity of notation throughout this section.

\subsection{$\ell_{2}$ risk for tensor factor estimation}

\label{subsec:proof-l2-loss-Tensor-factors}

Fix an arbitrary $1\leq l\leq r$. Recalling the decomposition in
(\ref{eq:U-loss-decomp}), (\ref{def:W1-W4}) and (\ref{def:W}),
we can write $\bm{u}_{l}-\bm{u}_{l}^{\star}=\bm{Z}_{:,l}+\bm{W}_{:,l}.$
In what follows, we will first prove that $\bm{Z}_{:,l}$ converges
to a Gaussian random vector in distribution. Then we can use the standard
Gaussian concentration inequality to show that the $\ell_{2}$ norm
of the Gaussian random vector concentrates around its expectation.
Combined with the observation that the $\ell_{2}$ norm of $\bm{W}_{:,l}$
is negligible as shown in (\ref{eq:W-2inf-norm-UB}) (established
in Lemmas~\ref{lemma:U-loss-dist-W1}-\ref{lemma:U-loss-dist-W4}),
this implies the advertised bound on the $\ell_{2}$ norm of $\bm{u}_{l}-\bm{u}_{l}^{\star}$.

Now we begin the proof. For convenience of presentation, we adopt
the notation in (\ref{def:V-true}) that $\bm{V}^{\star}:=\widetilde{\bm{U}}^{\star}(\widetilde{\bm{U}}^{\star\top}\widetilde{\bm{U}}^{\star})^{-1}$.
Then we can express
\[
\bm{Z}_{:,l}=\sqrt{2}\sum_{1\leq i,k\leq d}p^{-1}E_{i,i,k}\chi_{i,i,k}V_{(i,i),l}^{\star}\bm{e}_{k}+\sum_{1\leq i,j,k\leq d}p^{-1}E_{i,j,k}\chi_{i,j,k}V_{(i,j),l}^{\star}\bm{e}_{k}
\]
as a sum of independent zero-mean random vectors in $\mathbb{R}^{d}$.
Let us first compute the covariance matrix $\bm{S}_{l}^{\star}:=\mathbb{E}\big[\bm{Z}_{:,l}(\bm{Z}_{:,l})^{\top}\big]$.
Straightforward computation yields that for each $1\leq i\leq d$,
\begin{align}
\left(\bm{S}_{l}^{\star}\right)_{i,i} & =2\sum_{1\leq k_{1},k_{2}\leq d}p^{-1}\sigma_{i,k_{1},k_{2}}^{2}V_{(k_{1},k_{2}),l}^{\star2};\label{eq:cov-u-diag-entry}
\end{align}
and for each $1\leq i\neq j\leq d$,
\begin{align}
\left(\bm{S}_{l}^{\star}\right)_{i,j} & =2\sqrt{2}\,p^{-1}\sigma_{i,i,j}^{2}V_{(i,i),l}^{\star}V_{(i,j),l}^{\star}+2\sqrt{2}\,p^{-1}\sigma_{i,j,j}^{2}V_{(j,j),l}^{\star}V_{(i,j),l}^{\star}+\sum_{k:k\neq i,j}4\,p^{-1}\sigma_{i,j,k}^{2}V_{(i,k),l}^{\star}V_{(j,k),l}^{\star}\nonumber \\
 & =4\sum_{1\leq k\leq d}p^{-1}\sigma_{i,j,k}^{2}V_{(i,k),l}^{\star}V_{(j,k),l}^{\star}-\big(4-2\sqrt{2}\big)p^{-1}\sigma_{i,i,j}^{2}V_{(i,i),l}^{\star}V_{(i,j),l}^{\star}-\big(4-2\sqrt{2}\big)p^{-1}\sigma_{i,j,j}^{2}V_{(j,j),l}^{\star}V_{(i,j),l}^{\star}.\label{eq:cov-u-off-diag-entry}
\end{align}
Lemma~\ref{lemma:u-l2-loss-cov-matrix} below collects several properties
of $\bm{S}_{l}^{\star}$ and the proof is deferred to the end of this
section.

\begin{lemma}\label{lemma:u-l2-loss-cov-matrix}Instate the assumptions
of Theorem~\ref{thm:optimality-L2}. One has
\begin{align*}
 & \lambda_{\max}(\bm{S}_{l}^{\star})\lesssim\frac{\sigma_{\max}^{2}}{p\left\Vert \bm{u}_{l}^{\star}\right\Vert _{2}^{4}},\qquad\lambda_{\min}(\bm{S}_{l}^{\star})\gtrsim\frac{\sigma_{\min}^{2}}{p\left\Vert \bm{u}_{l}^{\star}\right\Vert _{2}^{4}};\\
 & \mathsf{tr}(\bm{S}_{l}^{\star})=\frac{\left(2+o\left(1\right)\right)\sigma_{\max}^{2}d}{p\left\Vert \bm{u}_{l}^{\star}\right\Vert _{2}^{4}}\text{,\ensuremath{\qquad\left\Vert \bm{S}_{l}^{\star}\right\Vert _{\mathrm{F}}\lesssim\frac{\sigma_{\max}^{2}\sqrt{d}}{p\left\Vert \bm{u}_{l}^{\star}\right\Vert _{2}^{4}}.}}
\end{align*}
\end{lemma}

Recall that we want to show $\bm{Z}_{:,l}$ converges to a Gaussian
random vector $\bm{g}_{l}\sim\mathcal{N}(\bm{0},\bm{S}_{l}^{\star})$
in distribution. By the Cram\'er--Wold theorem, it suffices to prove
that for any $\bm{a}=(a_{1},\cdots,a_{d})^{\top}\in\mathbb{R}^{d}$,
$\bm{a}^{\top}\bm{Z}_{:,l}$ converges to $\bm{a}^{\top}\bm{g}_{l}$
in distribution. Towards this, we apply the Berry-Esseen theorem \cite[Theorem 1.1]{bentkus2005lyapunov}
(cf.~Appendix~\ref{subsec:Berry-Esseen-theorem}) again and upper
bound $\rho$ defined in (\ref{def:Berry-Esseen-rho}). Without loss
of generality, we assume $\|\bm{a}\|_{2}=1$. We can compute
\begin{align*}
\sum_{1\leq i,j,k\leq d}\mathbb{E}\Big[\big|a_{k}p^{-1}E_{i,j,k}\chi_{i,j,k}V_{(i,j),l}^{\star}\big|^{3}\Big] & \leq\frac{1}{p^{3}}\left\Vert \bm{a}\right\Vert _{\infty}\left\Vert \bm{V}_{:,l}^{\star}\right\Vert _{\infty}\sum_{1\leq i,j,k\leq d}a_{k}^{2}\,\mathbb{E}\big[\left|E_{i,j,k}\right|^{3}\chi_{i,j,k}\big]V_{(i,j),l}^{\star2}\\
 & \overset{(\mathrm{i})}{\lesssim}\frac{\sigma_{\max}^{3}}{p^{2}}\left\Vert \bm{a}\right\Vert _{\infty}\left\Vert \bm{a}\right\Vert _{2}^{2}\left\Vert \bm{V}_{:,l}^{\star}\right\Vert _{\infty}\left\Vert \bm{V}_{:,l}^{\star}\right\Vert _{2}^{2}\overset{(\mathrm{ii})}{\lesssim}\frac{\sigma_{\max}^{3}}{p^{2}}\cdot\frac{\mu\sqrt{r}}{d}\frac{1}{\lambda_{\min}^{\star2/3}}\cdot\frac{1}{\lambda_{\min}^{\star4/3}},
\end{align*}
where we use the property of sub-gaussian random variables in (i),
and (ii) follows from (\ref{eq:V-true-col-norm}) and $\left\Vert \bm{a}\right\Vert _{\infty}\leq\left\Vert \bm{a}\right\Vert _{2}=1$
. Moreover, from Lemma~\ref{lemma:u-l2-loss-cov-matrix}, it is easy
to see that
\[
\mathsf{Var}\big(\bm{a}^{\top}\bm{Z}_{:,l}\big)=\bm{a}^{\top}\bm{S}_{l}^{\star}\bm{a}\geq\lambda_{\min}(\bm{S}_{l}^{\star})\left\Vert \bm{a}\right\Vert _{2}^{2}\gtrsim\frac{\sigma_{\min}^{2}}{p\left\Vert \bm{u}_{l}^{\star}\right\Vert _{2}^{4}}\gtrsim\frac{\sigma_{\min}^{2}}{p\lambda_{\max}^{\star4/3}}.
\]
One can then bound $\rho$
\[
\rho=\big(\mathsf{Var}(\bm{a}^{\top}\bm{Z}_{:,l})\big)^{-3/2}\sum_{1\leq i,j,k\leq d}\mathbb{E}\Big[\big|a_{k}p^{-1}E_{i,j,k}\chi_{i,j,k}V_{(i,j),l}^{\star}\big|^{3}\Big]\lesssim\frac{p^{3/2}\lambda_{\max}^{\star2}}{\sigma_{\min}^{3}}\cdot\frac{\sigma_{\max}^{3}}{\lambda_{\min}^{\star2}}\frac{\mu\sqrt{r}}{dp^{2}}\lesssim\frac{\mu\sqrt{r}}{d\sqrt{p}}=o\left(1\right)
\]
where we use the condition that $\sigma_{\max}/\sigma_{\min},\kappa\asymp1$
and $p\gg\mu^{2}rd^{-3/2}$. Therefore, we justify the claimed distributional
convergence of $\bm{a}^{\top}\bm{Z}_{:,l}$, which further implies
the convergence of $\bm{Z}_{:,l}$ by the Cram\'er--Wold theorem.

Given that $\bm{Z}_{:,l}$ converges to $\bm{g}_{l}$ in distribution,
we now apply the Gaussian concentration inequality \cite[Proposition 1]{hsu2012tail}
to demonstrate the squared $\ell_{2}$ norm of $\bm{g}_{l}$ is tightly
concentrated around its mean with high probability. By Lemma~\ref{lemma:u-l2-loss-cov-matrix},
we can use the Gaussian concentration inequality \cite[Proposition 1]{hsu2012tail}
to find that with probability at least $1-O\left(d^{-11}\right)$,
\begin{align*}
\left\Vert \bm{g}_{l}\right\Vert _{2}^{2}-\mathsf{tr}(\bm{S}_{l}^{\star}) & \lesssim\left\Vert \bm{S}_{l}^{\star}\right\Vert _{\mathrm{F}}\sqrt{\log d}+\left\Vert \bm{S}_{l}^{\star}\right\Vert \log d\lesssim\frac{\sigma_{\max}^{2}(\sqrt{d\log d}+\log d)}{p\left\Vert \bm{u}_{l}^{\star}\right\Vert _{2}^{4}}=o\left(1\right)\frac{\sigma_{\max}^{2}d}{p\left\Vert \bm{u}_{l}^{\star}\right\Vert _{2}^{4}},
\end{align*}
and consequently,
\[
\left\Vert \bm{g}_{l}\right\Vert _{2}^{2}\leq\frac{2\left(1+o\left(1\right)\right)\sigma_{\max}^{2}d}{p\left\Vert \bm{u}_{l}^{\star}\right\Vert _{2}^{4}}.
\]
Moreover, we know from the continuous mapping theorem that $\left\Vert \bm{Z}_{:,l}\right\Vert _{2}^{2}$
converges to $\left\Vert \bm{g}_{l}\right\Vert _{2}^{2}$ in distribution
because $\|\cdot\|_{2}^{2}$ is a continuous function. Therefore,
we find that with probability at least $1-o\left(1\right)$,
\begin{equation}
\left\Vert \bm{Z}_{:,l}\right\Vert _{2}^{2}\leq\frac{2\left(1+o\left(1\right)\right)\sigma_{\max}^{2}d}{p\left\Vert \bm{u}_{l}^{\star}\right\Vert _{2}^{4}}.\label{eq:Z-col-2norm-UB}
\end{equation}

It remains to upper bound $\left\Vert \bm{W}_{:,l}\right\Vert _{2}$,
which is easily accomplished with the help of (\ref{eq:W-2inf-norm-UB}).
Indeed, it is straightforward to find that with probability at least
$1-O\left(d^{-11}\right)$,
\[
\left\Vert \bm{W}_{:,l}\right\Vert _{2}^{2}\leq\sum_{1\leq k\leq d}\left\Vert \bm{W}_{k,:}\right\Vert _{2,\infty}^{2}\leq d\cdot\frac{o\left(1\right)\sigma_{\max}^{2}}{\lambda_{\min}^{\star4/3}p}=\frac{o\left(1\right)\sigma_{\max}^{2}d}{p\left\Vert \bm{u}_{l}^{\star}\right\Vert _{2}^{4}},
\]
where we use the assumption that $\kappa\asymp1$ in the last step.
Taken collectively with (\ref{eq:Z-col-2norm-UB}) finishes the proof.

\subsubsection{Proof of Lemma \ref{lemma:u-l2-loss-cov-matrix}}

To begin with, let us consider the trace of $\bm{S}_{l}^{\star}$.
From (\ref{eq:cov-u-diag-entry}) and (\ref{eq:V-true-col-norm}),
it is straightforward to calculate
\begin{align*}
\mathsf{tr}(\bm{S}_{l}^{\star}) & =\sum_{1\leq i\leq d}\left(\bm{S}_{l}^{\star}\right)_{i,i}\leq2\sum_{1\leq i,j,k\leq d}p^{-1}\sigma_{i,j,k}^{2}V_{(j,k),l}^{\star2}\\
 & =\frac{2\,\sigma_{\max}^{2}d}{p}\left\Vert \bm{V}_{:,l}^{\star}\right\Vert _{2}^{2}=\frac{2\left(1+o\left(1\right)\right)\sigma_{\max}^{2}d}{p\left\Vert \bm{u}_{l}^{\star}\right\Vert _{2}^{4}}.
\end{align*}

As for the Frobenius norm of $\bm{S}_{l}^{\star}$, we note that it
is an immediate consequence of the claim for the spectrum of $\bm{S}_{l}^{\star}$.
Indeed, since $\bm{S}_{l}^{\star}$ is a positive semidefinite matrix,
we know that 
\[
\left\Vert \bm{S}_{l}^{\star}\right\Vert _{\mathrm{F}}^{2}=\mathsf{tr}(\bm{S}_{l}^{\star2})=\sum_{1\leq i\leq d}\lambda_{i}(\bm{S}_{l}^{\star2})=\sum_{1\leq i\leq d}\lambda_{i}^{2}(\bm{S}_{l}^{\star})\leq d\cdot\lambda_{\max}^{2}(\bm{S}_{l}^{\star})\lesssim\frac{\sigma_{\max}^{4}d}{p^{2}\left\Vert \bm{u}_{l}^{\star}\right\Vert _{2}^{8}}
\]
as claimed.

Hence, the remainder of the proof amounts to controlling the eigenvalues
of $\bm{S}_{l}^{\star}$. Let us decompose $\bm{S}_{l}^{\star}=:2\widehat{\bm{S}}_{l}^{\star}-\breve{\bm{S}}_{l}^{\star}\in\mathbb{R}^{d\times d}$,
where the entries of $\widehat{\bm{S}}_{l}^{\star}$ and $\breve{\bm{S}}_{l}^{\star}$
are given by
\[
\big(\widehat{\bm{S}}_{l}^{\star}\big)_{i,j}=\begin{cases}
\sum_{1\leq k_{1},k_{2}\leq d}p^{-1}\sigma_{i,k_{1},k_{2}}^{2}V_{(k_{1},k_{2}),l}^{\star2}, & \text{if}\quad i=j,\\
2\sum_{1\leq k\leq d}p^{-1}\sigma_{i,j,k}^{2}V_{(i,k),l}^{\star}V_{(j,k),l}^{\star}, & \text{if}\quad i\neq j,
\end{cases}
\]
and
\[
\big(\breve{\bm{S}}_{l}^{\star}\big)_{i,j}=\begin{cases}
0, & \text{if}\quad i=j,\\
\big(4-2\sqrt{2}\big)p^{-1}\sigma_{i,i,j}^{2}V_{(i,i),l}^{\star}V_{(i,j),l}^{\star}+\big(4-2\sqrt{2}\big)p^{-1}\sigma_{i,j,j}^{2}V_{(j,j),l}^{\star}V_{(i,j),l}^{\star}, & \text{if}\quad i\neq j.
\end{cases}
\]
Our proof strategy is to show that the spectrum of $\bm{S}_{l}^{\star}$
is mainly determined by $\widehat{\bm{S}}_{l}^{\star}$ (since $\|\breve{\bm{S}}_{l}^{\star}\|$
is a negligible term). One can then invoke Weyl's inequality to establish
the conclusion.

Now we start the analysis. Note that by the symmetric sampling pattern,
one equivalently express $\sigma_{i,j,k}^{2}=s_{i}^{2}s_{j}^{2}s_{k}^{2}$
for each $1\leq i,j,k\leq d$ with $\max_{1\leq i\leq d}s_{i}\leq\sigma_{\max}^{1/3}$.
We then can decompose
\[
\widehat{\bm{S}}_{l}^{\star}=2\bm{A}\bm{A}^{\top}+\mathcal{P}_{\mathsf{diag}}\big(\widehat{\bm{S}}_{l}^{\star}-2\bm{A}\bm{A}^{\top}\big),
\]
where $\mathcal{P}_{\mathsf{diag}}(\bm{Z})$ extracts out the diagonal
entries of a matrix $\bm{Z}$, and $\bm{A}\in\mathbb{R}^{d\times d}$
is a matrix with entries $A_{i,k}=\sqrt{1/p}\,s_{i}^{2}s_{k}V_{(i,k),l}^{\star}$.
Let us first control the spectral norm of $\mathcal{P}_{\mathsf{diag}}\big(\widehat{\bm{S}}_{l}^{\star}-2\bm{A}\bm{A}^{\top}\big)$.
From (\ref{eq:V-true-col-norm}), it is easy to see that for all $1\leq i\leq d$,
\begin{align}
\left|\big(\widehat{\bm{S}}_{l}^{\star}-2\bm{A}\bm{A}^{\top}\big)_{i,i}\right| & \leq\frac{\sigma_{\min}^{2}}{p}\left\Vert \bm{V}_{:,l}^{\star}\right\Vert _{2}^{2}+\frac{2\sigma_{\max}^{2}}{p}\sum_{1\leq k\leq d}V_{(i,k),l}^{\star2}\nonumber \\
 & \leq\frac{\left(1+o\left(1\right)\right)\sigma_{\min}^{2}}{p\left\Vert \bm{u}_{l}^{\star}\right\Vert _{2}^{4}}+\frac{2\sigma_{\max}^{2}}{p}\frac{\mu r}{d}\frac{1}{\lambda_{\min}^{\star4/3}}=\frac{\left(1+o\left(1\right)\right)\sigma_{\max}^{2}}{p\left\Vert \bm{u}_{l}^{\star}\right\Vert _{2}^{4}},\label{eq:cov-hat-diag-UB}\\
\big(\widehat{\bm{S}}_{l}^{\star}-2\bm{A}\bm{A}^{\top}\big)_{i,i} & \geq\frac{\sigma_{\min}^{2}}{p}\left\Vert \bm{V}_{:,l}^{\star}\right\Vert _{2}^{2}-\frac{2\sigma_{\max}^{2}}{p}\sum_{1\leq k\leq d}V_{(i,k),l}^{\star2}=\frac{\left(1-o\left(1\right)\right)\sigma_{\min}^{2}}{p\left\Vert \bm{u}_{l}^{\star}\right\Vert _{2}^{4}},\label{eq:cov-hat-diag-LB}
\end{align}
where we have used the condition that $\sigma_{\min}/\sigma_{\max}\asymp1$,
$\kappa\asymp1$, and $r=o\left(d/\mu\right)$. It then suffices to
focus on the spectrum of $\bm{A}\bm{A}^{\top}$, whose entries are
given by
\[
\left(\bm{A}\bm{A}^{\top}\right)_{i,j}=p^{-1}s_{i}^{2}s_{j}^{2}\sum_{1\leq k\leq d}s_{k}^{2}V_{(i,k),l}^{\star}V_{(j,k),l}^{\star},\qquad1\leq i,j\leq d.
\]
Recalling the definitions $\bm{V}^{\star}:=\widetilde{\bm{U}}^{\star}(\widetilde{\bm{U}}^{\star\top}\widetilde{\bm{U}}^{\star})^{-1}$
and $u_{l,i}^{\star}:=\left(\bm{u}_{l}^{\star}\right)_{i}$ for each
$1\leq l\leq r,1\leq i\leq d$, we can decompose
\begin{align}
V_{(i,k),l}^{\star} & =\widetilde{\bm{U}}_{(i,k),:}^{\star}(\widetilde{\bm{U}}^{\star\top}\widetilde{\bm{U}}^{\star})_{:,l}^{-1}=\widetilde{\bm{U}}_{(i,k),:}^{\star}\bm{\Lambda}_{:,l}^{\star-4/3}+\widetilde{\bm{U}}_{(i,k),:}^{\star}\big((\widetilde{\bm{U}}^{\star\top}\widetilde{\bm{U}}^{\star})^{-1}-\bm{\Lambda}^{\star-4/3}\big)_{:,l}\nonumber \\
 & =\left\Vert \bm{u}_{l}^{\star}\right\Vert _{2}^{-4}\widetilde{U}_{(i,k),l}^{\star}+\widetilde{\bm{U}}_{(i,k),:}^{\star}\big((\widetilde{\bm{U}}^{\star\top}\widetilde{\bm{U}}^{\star})^{-1}-\bm{\Lambda}^{\star-4/3}\big)_{:,l}\nonumber \\
 & =\left\Vert \bm{u}_{l}^{\star}\right\Vert _{2}^{-4}u_{l,i}^{\star}u_{l,k}^{\star}+\underbrace{\widetilde{\bm{U}}_{(i,k),:}^{\star}\big((\widetilde{\bm{U}}^{\star\top}\widetilde{\bm{U}}^{\star})^{-1}-\bm{\Lambda}^{\star-4/3}\big)_{:,l}}_{=:\,\delta_{i,k}}.\label{eq:V-ikl-decomp}
\end{align}
We will show shortly that $V_{(i,k),l}^{\star}$ is extremely close
to $\left\Vert \bm{u}_{l}^{\star}\right\Vert _{2}^{-4}u_{l,i}^{\star}u_{l,k}^{\star}$.
Then we obtain that for each $1\leq i,j\leq d$,
\begin{align*}
p\left(\bm{A}\bm{A}^{\top}\right)_{i,j} & =\left\Vert \bm{u}_{l}^{\star}\right\Vert _{2}^{-8}\Big(\sum_{1\leq k\leq d}s_{k}^{2}u_{l,k}^{\star2}\Big)\left(s_{i}^{2}u_{l,i}^{\star}\right)\left(s_{j}^{2}u_{l,j}^{\star}\right)+\left\Vert \bm{u}_{l}^{\star}\right\Vert _{2}^{-4}\left(s_{i}^{2}u_{l,i}^{\star}\right)\sum_{1\leq k\leq d}s_{k}^{2}u_{l,k}^{\star}\left(s_{j}^{2}\delta_{j,k}\right)\\
 & \quad+\left\Vert \bm{u}_{l}^{\star}\right\Vert _{2}^{-4}\left(s_{j}^{2}u_{l,j}^{\star}\right)\sum_{1\leq k\leq d}s_{k}^{2}u_{l,k}^{\star}\left(s_{i}^{2}\delta_{i,k}\right)+\underbrace{\sum_{1\leq k\leq d}s_{k}^{2}\left(s_{i}^{2}\delta_{i,k}\right)\left(s_{j}^{2}\delta_{j,k}\right)}_{=:\,\Upsilon_{i,j}},
\end{align*}
or equivalently,
\[
p\bm{A}\bm{A}^{\top}=\bm{a}\bm{a}^{\top}\sum_{1\leq k\leq d}s_{k}^{2}u_{l,k}^{\star2}+\bm{a}\bm{b}^{\top}+\bm{b}\bm{a}^{\top}+\bm{\Upsilon},
\]
where $\bm{a},\bm{b}\in\mathbb{R}^{d}$ with entries $a_{i}=s_{i}^{2}u_{l,i}^{\star}/\left\Vert \bm{u}_{l}^{\star}\right\Vert _{2}^{4}$,
$b_{i}=\sum_{1\leq k\leq d}s_{k}^{2}u_{l,k}^{\star}\left(s_{i}^{2}\delta_{i,k}\right)$
and $\bm{\Upsilon}=[\Upsilon_{i,j}]_{1\leq i,j\leq d}\in\mathbb{R}^{d\times d}$.
It is straightforward to see that $\bm{a}\bm{a}^{\top}\sum_{1\leq k\leq d}s_{k}^{2}u_{l,k}^{\star2}$
is a rank-$1$ matrix with the non-zero eigenvalue 
\begin{align}
 & \frac{\sigma_{\min}^{2}}{\left\Vert \bm{u}_{l}^{\star}\right\Vert _{2}^{4}}\leq\lambda_{\max}\Big(\bm{a}\bm{a}^{\top}\sum_{1\leq k\leq d}s_{k}^{2}u_{l,k}^{\star2}\Big)\leq\frac{\sigma_{\max}^{2}}{\left\Vert \bm{u}_{l}^{\star}\right\Vert _{2}^{4}}.\label{eq:cov-hat-rank1-eigval}
\end{align}
In addition, the remaining three terms are all small with respect
to the spectral norm. Indeed, recalling the decomposition in (\ref{eq:V-ikl-decomp}),
we can use (\ref{eq:U-true-tilde-Gram-Ind-diff}), $\kappa\asymp1$
and the condition $r=o(\sqrt{d/\mu})$ to show $\delta_{i,k}$ is
sufficiently small, i.e.
\begin{align}
\sum_{1\leq i,k\leq d}\delta_{i,k}^{2} & \leq\big\|\widetilde{\bm{U}}^{\star}\big\|_{\mathrm{F}}^{2}\big\|(\widetilde{\bm{U}}^{\star\top}\widetilde{\bm{U}}^{\star})^{-1}-\bm{\Lambda}^{\star-4/3}\big\|^{2}\lesssim r\lambda_{\max}^{\star4/3}\cdot\left(\frac{1}{\lambda_{\min}^{\star8/3}}\frac{\mu r}{d}\lambda_{\max}^{\star4/3}\right)^{2}=\frac{o\left(1\right)}{\lambda_{\min}^{\star4/3}}.\label{eq:proof-estimation-delta-square-sum}
\end{align}
It then follows from the Cauchy-Schwartz inequality that
\begin{align*}
\left\Vert \bm{b}\right\Vert _{2}^{2} & =\sum_{1\leq i\leq d}\Big|\sum_{1\leq k\leq d}s_{k}^{2}u_{l,k}^{\star}\left(s_{i}^{2}\delta_{i,k}\right)\Big|^{2}\leq\max_{1\leq i\leq d}s_{i}^{8}\sum_{1\leq k\leq d}u_{l,k}^{\star2}\sum_{1\leq i,k\leq d}\delta_{i,k}^{2}\lesssim\frac{\sigma_{\max}^{8/3}}{\lambda_{\min}^{\star4/3}}\frac{\mu^{2}r^{3}}{d^{2}}\left\Vert \bm{u}_{l}^{\star}\right\Vert _{2}^{2};\\
\left\Vert \bm{\Upsilon}\right\Vert _{\mathrm{F}}^{2} & =\sum_{1\leq i,j\leq d}\Big|\sum_{1\leq k\leq d}s_{k}^{2}\left(s_{i}^{2}\delta_{i,k}\right)\left(s_{j}^{2}\delta_{j,k}\right)\Big|^{2}\leq\max_{1\leq i\leq d}s_{i}^{12}\,\Big|\sum_{1\leq i,k\leq d}\delta_{i,k}^{2}\Big|^{2}\lesssim\left(\frac{\sigma_{\max}^{2}}{\lambda_{\min}^{\star4/3}}\frac{\mu^{2}r^{3}}{d^{2}}\right)^{2}.
\end{align*}
This combined with the condition $r=o\big(\sqrt{d/\mu}\big)$ and
$\kappa\asymp1$ reveals that
\begin{align*}
\left\Vert \bm{a}\bm{b}^{\top}+\bm{b}\bm{a}^{\top}+\bm{\Upsilon}\right\Vert  & \leq2\left\Vert \bm{a}\right\Vert _{2}\left\Vert \bm{b}\right\Vert _{2}+\left\Vert \bm{\Upsilon}\right\Vert _{\mathrm{F}}\lesssim\frac{\sigma_{\max}^{2/3}}{\left\Vert \bm{u}_{l}^{\star}\right\Vert _{2}^{3}}\frac{\sigma_{\max}^{4/3}}{\lambda_{\min}^{\star2/3}}\frac{\mu r^{3/2}}{d}\left\Vert \bm{u}_{l}^{\star}\right\Vert _{2}+\frac{\sigma_{\max}^{2}}{\lambda_{\min}^{\star4/3}}\frac{\mu^{2}r^{3}}{d^{2}}=\frac{o\left(1\right)\sigma_{\max}^{2}}{\left\Vert \bm{u}_{l}^{\star}\right\Vert _{2}^{4}}.
\end{align*}
Taken collectively with (\ref{eq:cov-hat-diag-UB}), (\ref{eq:cov-hat-diag-LB})
and (\ref{eq:cov-hat-rank1-eigval}), we conclude that
\[
\lambda_{\max}(\widehat{\bm{S}}_{l}^{\star})\lesssim\frac{\sigma_{\max}^{2}}{p\left\Vert \bm{u}_{l}^{\star}\right\Vert _{2}^{4}},\qquad\lambda_{\min}(\widehat{\bm{S}}_{l}^{\star})\gtrsim\frac{\sigma_{\min}^{2}}{p\left\Vert \bm{u}_{l}^{\star}\right\Vert _{2}^{4}}.
\]
Applying a similar argument, one can easily show that 
\begin{align*}
\|\breve{\bm{S}}_{l}^{\star}\| & \lesssim\frac{\sigma_{\max}^{2}}{p}\sum_{1\leq k\leq d}V_{(k,k),l}^{\star2}\lesssim\frac{\sigma_{\max}^{2}}{\lambda_{\min}^{\star4/3}p}\frac{\mu r}{d}=\frac{o\left(1\right)\sigma_{\min}^{2}}{p\left\Vert \bm{u}_{l}^{\star}\right\Vert _{2}^{4}}=o\left(1\right)\lambda_{\min}(\widehat{\bm{S}}_{l}^{\star}).
\end{align*}
Therefore, the advertised bound of the eigenvalues of $\bm{S}_{l}^{\star}=2\widehat{\bm{S}}_{l}^{\star}-\breve{\bm{S}}_{l}^{\star}$
immediately follows from Weyl's inequality.

\subsection{$\ell_{2}$ risk for tensor estimation}

\label{subsec:proof-l2-loss-Tensor}

To begin with, we recall the notation $\bm{\Delta}_{l}:=\bm{u}_{l}-\bm{u}_{l}^{\star},1\leq l\leq r$,
allowing us to expand
\begin{align*}
\bm{T}-\bm{T}^{\star} & =\sum_{1\leq l\leq r}\bm{\Delta}_{l}\otimes\bm{u}_{l}^{\star\otimes2}+\sum_{1\leq l\leq r}\bm{u}_{l}^{\star}\otimes\bm{\Delta}_{l}\otimes\bm{u}_{l}^{\star}+\sum_{1\leq l\leq r}\bm{u}_{l}^{\star\otimes2}\otimes\bm{\Delta}_{l}\\
 & \quad+\sum_{1\leq l\leq r}\bm{u}_{l}^{\star}\otimes\bm{\Delta}_{l}^{\otimes2}+\sum_{1\leq l\leq r}\bm{\Delta}_{l}\otimes\bm{u}_{l}^{\star}\otimes\bm{\Delta}_{l}+\sum_{1\leq l\leq r}\bm{\Delta}_{l}^{\otimes2}\otimes\bm{u}_{l}^{\star}+\sum_{1\leq l\leq r}\bm{\Delta}_{l}^{\otimes3}.
\end{align*}
By symmetry, straightforward calculation yields
\begin{align*}
\left\Vert \bm{T}-\bm{T}^{\star}\right\Vert _{\mathrm{F}}^{2} & =3\underbrace{\Big\|\sum_{1\leq l\leq r}\bm{\Delta}_{l}\otimes\bm{u}_{l}^{\star\otimes2}\Big\|_{\mathrm{F}}^{2}}_{=:\,\beta_{1}}+3\underbrace{\Big\|\sum_{1\leq l\leq r}\bm{u}_{l}^{\star}\otimes\bm{\Delta}_{l}^{\otimes2}\big\|_{\mathrm{F}}^{2}}_{=:\,\beta_{2}}+\underbrace{\Big\|\sum_{1\leq l\leq r}\bm{\Delta}_{l}^{\otimes3}\Big\|_{\mathrm{F}}^{2}}_{=:\,\beta_{4}}\\
 & \quad+6\underbrace{\Big\langle\sum_{1\leq l\leq r}\bm{\Delta}_{l}\otimes\bm{u}_{l}^{\star\otimes2},\sum_{1\leq l\leq r}\bm{u}_{l}^{\star\otimes2}\otimes\bm{\Delta}_{l}\Big\rangle}_{=:\,\beta_{4}}+\,\beta_{5},
\end{align*}
where
\begin{align*}
\beta_{5} & :=6\,\Big\langle\sum_{1\leq l\leq r}\bm{u}_{l}^{\star}\otimes\bm{\Delta}_{l}^{\otimes2},\sum_{1\leq l\leq r}\bm{\Delta}_{l}^{\otimes2}\otimes\bm{u}_{l}^{\star}\Big\rangle+6\,\Big\langle\sum_{1\leq l\leq r}\bm{u}_{l}^{\star}\otimes\bm{\Delta}_{l}^{\otimes2},\sum_{1\leq l\leq r}\bm{\Delta}_{l}^{\otimes3}\Big\rangle\\
 & \quad+6\,\Big\langle\sum_{1\leq l\leq r}\bm{\Delta}_{l}\otimes\bm{u}_{l}^{\star\otimes2},\sum_{1\leq l\leq r}\bm{u}_{l}^{\star}\otimes\bm{\Delta}_{l}^{\otimes2}\Big\rangle+6\,\Big\langle\sum_{1\leq l\leq r}\bm{\Delta}_{l}\otimes\bm{u}_{l}^{\star\otimes2},\sum_{1\leq l\leq r}\bm{\Delta}_{l}^{\otimes3}\Big\rangle\\
 & \quad+12\,\Big\langle\sum_{1\leq l\leq r}\bm{\Delta}_{l}\otimes\bm{u}_{l}^{\star\otimes2},\sum_{1\leq l\leq r}\bm{\Delta}_{l}^{\otimes2}\otimes\bm{u}_{l}^{\star}\Big\rangle.
\end{align*}
In what follows, we shall control the $\beta_{i}$'s separately. In
particular, we want to show that the $\ell_{2}$ loss of interest
is mainly controlled by $\beta_{1}$, with the remaining four terms
being negligible with high probability.
\begin{enumerate}
\item We start with $\beta_{1}$. Recalling (\ref{eq:U-loss-decomp}), (\ref{def:W1-W4})
and (\ref{def:W}) that $\bm{\Delta}:=\bm{U}-\bm{U}^{\star}=\bm{Z}+\bm{W}$
as well as the notation $\widetilde{\bm{U}}^{\star}:=\big[\bm{u}_{l}^{\star\otimes2}\big]_{1\leq l\leq r}\in\mathbb{R}^{d^{2}\times r}$,
we can easily see that
\begin{align*}
\beta_{1} & =\big\|\bm{\Delta}\widetilde{\bm{U}}^{\star\top}\big\|_{\mathrm{F}}^{2}=\big\|\bm{Z}\widetilde{\bm{U}}^{\star\top}\big\|_{\mathrm{F}}^{2}+\big\|\bm{W}\widetilde{\bm{U}}^{\star\top}\big\|_{\mathrm{F}}^{2}+2\,\big\langle\bm{Z}\widetilde{\bm{U}}^{\star\top},\bm{W}\widetilde{\bm{U}}^{\star\top}\big\rangle.
\end{align*}
One can apply an analogous argument as in Appendix~\ref{subsec:proof-l2-loss-Tensor-factors}
to show that the distribution of $\bm{Z}\widetilde{\bm{U}}^{\star\top}$
converges to a multivariate normal distribution, whose Euclidean norm
concentrates around its expectation with high probability. We omit
the detailed proof for conciseness. One can verify that with probability
exceeding $1-o\left(1\right)$,
\[
\|\bm{Z}\widetilde{\bm{U}}^{\star\top}\|_{\mathrm{F}}^{2}=\frac{\left(2+o\left(1\right)\right)\sigma_{\max}^{2}d}{p}\big\|\widetilde{\bm{U}}^{\star}(\widetilde{\bm{U}}^{\star\top}\widetilde{\bm{U}}^{\star})^{-1}\widetilde{\bm{U}}^{\star\top}\big\|_{\mathrm{F}}^{2}=\frac{\left(2+o\left(1\right)\right)\sigma_{\max}^{2}rd}{p}.
\]
In addition, we know from (\ref{eq:U-true-tilde-spectrum}) and (\ref{eq:W-2inf-norm-UB})
that
\[
\big\|\bm{W}\widetilde{\bm{U}}^{\star\top}\big\|_{\mathrm{F}}^{2}\leq\big\|\widetilde{\bm{U}}^{\star}\big\|^{2}\big\|\bm{W}\big\|_{\mathrm{F}}^{2}\lesssim\lambda_{\max}^{\star4/3}\cdot d\left\Vert \bm{W}\right\Vert _{2,\infty}^{2}=o\left(1\right)\frac{\sigma_{\max}^{2}d}{p},
\]
which further implies that
\[
\big|\big\langle\bm{Z}\widetilde{\bm{U}}^{\star\top},\bm{W}\widetilde{\bm{U}}^{\star\top}\big\rangle\big|\leq\big\|\bm{Z}\widetilde{\bm{U}}^{\star\top}\big\|_{\mathrm{F}}\big\|\bm{W}\widetilde{\bm{U}}^{\star\top}\big\|_{\mathrm{F}}=o\left(1\right)\frac{\sigma_{\max}^{2}rd}{p}.
\]
As a result, we find that
\begin{equation}
\Big\|\sum_{1\leq l\leq r}\bm{\Delta}_{l}\otimes\bm{u}_{l}^{\star\otimes2}\Big\|_{\mathrm{F}}^{2}=\frac{\left(2+o\left(1\right)\right)\sigma_{\max}^{2}rd}{p}.\label{eq:T-l2-loss-beta1}
\end{equation}
\item Next, let us look at $\beta_{2}$. We denote by $\widetilde{\bm{\Delta}}:=\big[\bm{\Delta}_{l}^{\otimes2}\big]_{1\leq l\leq r}\in\mathbb{R}^{d^{2}\times r}$,
whose Frobenius norm can be bounded by
\begin{equation}
\big\|\widetilde{\bm{\Delta}}\big\|_{\mathrm{F}}^{2}=\sum_{1\leq l\leq r}\big\|\bm{\Delta}_{l}^{\otimes2}\big\|_{2}^{2}=\sum_{1\leq l\leq r}\big\|\bm{\Delta}_{l}\big\|_{2}^{4}\leq\max_{1\leq l\leq r}\big\|\bm{\Delta}_{l}\big\|_{2}^{2}\left\Vert \bm{\Delta}\right\Vert _{\mathrm{F}}^{2}\leq\left\Vert \bm{U}-\bm{U}^{\star}\right\Vert _{2,\infty}^{2}\left\Vert \bm{U}-\bm{U}^{\star}\right\Vert _{\mathrm{F}}^{2}.\label{eq:delta-tilde-fro-UB}
\end{equation}
Consequently, we use (\ref{eq:U-loss-fro}) and (\ref{eq:U-loss-2inf})
to obtain
\begin{align}
\beta_{2} & =\big\|\bm{U}^{\star}\widetilde{\bm{\Delta}}^{\top}\big\|_{\mathrm{F}}^{2}\leq\left\Vert \bm{U}^{\star}\right\Vert ^{2}\big\|\widetilde{\bm{\Delta}}\big\|_{\mathrm{F}}^{2}\lesssim\left\Vert \bm{U}^{\star}\right\Vert ^{2}\left\Vert \bm{U}-\bm{U}^{\star}\right\Vert _{2,\infty}^{2}\left\Vert \bm{U}-\bm{U}^{\star}\right\Vert _{\mathrm{F}}^{2}\nonumber \\
 & \lesssim\lambda_{\max}^{\star2/3}\cdot\frac{\sigma_{\max}^{2}}{\lambda_{\min}^{\star2}}\frac{\mu r\log d}{p}\lambda_{\max}^{\star2/3}\cdot\frac{\sigma_{\max}^{2}}{\lambda_{\min}^{\star2}}\frac{rd\log d}{p}\lambda_{\max}^{\star2/3}=o\left(1\right)\frac{\sigma_{\max}^{2}rd}{p},\label{eq:T-l2-loss-beta2}
\end{align}
where the last step holds as long $\sigma_{\max}/\lambda_{\min}^{\star}\ll\sqrt{p/(\mu r\log^{2}d)}$
and $\kappa\asymp1$.
\item In a similar way, we can use (\ref{eq:U-T-loss-UB}) and (\ref{eq:delta-tilde-fro-UB})
to upper bound $\beta_{3}$ as follows
\begin{align}
\beta_{3} & =\big\|\bm{\Delta}\widetilde{\bm{\Delta}}^{\top}\big\|_{\mathrm{F}}^{2}\leq\|\bm{\Delta}\|^{2}\|\widetilde{\bm{\Delta}}\|_{\mathrm{F}}^{2}\lesssim\left\Vert \bm{U}-\bm{U}^{\star}\right\Vert ^{2}\left\Vert \bm{U}-\bm{U}^{\star}\right\Vert _{2,\infty}^{2}\left\Vert \bm{U}-\bm{U}^{\star}\right\Vert _{\mathrm{F}}^{2}\nonumber \\
 & \lesssim\frac{\sigma_{\max}^{2}}{\lambda_{\min}^{\star2}}\frac{rd\log d}{p}\lambda_{\max}^{\star2/3}\cdot\frac{\sigma_{\max}^{2}}{\lambda_{\min}^{\star2}}\frac{\mu r\log d}{p}\lambda_{\max}^{\star2/3}\cdot\frac{\sigma_{\max}^{2}}{\lambda_{\min}^{\star2}}\frac{rd\log d}{p}\lambda_{\max}^{\star2/3}=o\left(1\right)\frac{\sigma_{\max}^{2}rd}{p},\label{eq:T-l2-loss-beta3}
\end{align}
where the last step follows from the conditions that $\sigma_{\max}/\lambda_{\min}^{\star}\ll\sqrt{p/(\mu rd\log^{2}d)}$
and $\kappa\asymp1$.
\item As for $\beta_{4}$, one can apply the triangle inequality and Cauchy-Schwartz
to upper bound
\begin{align*}
\left|\beta_{4}\right| & =\Big|\sum_{1\leq l\leq r}\left\langle \bm{\Delta}_{l},\bm{u}_{l}^{\star}\right\rangle ^{2}\left\Vert \bm{u}_{l}^{\star}\right\Vert _{2}^{2}+\sum_{1\leq l\neq s\leq r}\left\langle \bm{\Delta}_{l},\bm{u}_{s}^{\star}\right\rangle \left\langle \bm{u}_{l}^{\star},\bm{\Delta}_{s}\right\rangle \left\langle \bm{u}_{l}^{\star},\bm{u}_{s}^{\star}\right\rangle \Big|\\
 & \lesssim\sum_{1\leq l\leq r}\left\langle \bm{\Delta}_{l},\bm{u}_{l}^{\star}\right\rangle ^{2}\left\Vert \bm{u}_{l}^{\star}\right\Vert _{2}^{2}+\max_{1\leq l\neq s\leq r}\left|\left\langle \bm{u}_{l}^{\star},\bm{u}_{s}^{\star}\right\rangle \right|\Big(\sum_{1\leq l\leq r}\left\Vert \bm{\Delta}_{l}\right\Vert _{2}\left\Vert \bm{u}_{l}^{\star}\right\Vert _{2}\Big)^{2}\\
 & \leq\max_{1\leq l\leq r}\left\langle \bm{\Delta}_{l},\bm{u}_{l}^{\star}\right\rangle ^{2}\left\Vert \bm{U}^{\star}\right\Vert _{\mathrm{F}}^{2}+\max_{1\leq l\neq s\leq r}\left|\left\langle \bm{u}_{l}^{\star},\bm{u}_{s}^{\star}\right\rangle \right|\left\Vert \bm{U}-\bm{U}^{\star}\right\Vert _{\mathrm{F}}^{2}\left\Vert \bm{U}^{\star}\right\Vert _{\mathrm{F}}^{2}.
\end{align*}
We then use the incoherence assumption~(\ref{assumption:u-inner-prod}),
Lemma~\ref{lemma:u-loss-u-true-inner-product-UB} in Appendix~\ref{subsec:U-loss-dist-W1},
(\ref{eq:U-true-norm}) and (\ref{eq:U-loss-fro}) to find that
\begin{equation}
\left|\beta_{4}\right|=o\left(1\right)\frac{\sigma^{2}}{\lambda_{\min}^{\star2/3}}\frac{d}{p}\cdot r\lambda_{\max}^{\star2/3}+\sqrt{\frac{\mu}{d}}\,\lambda_{\max}^{\star2/3}\cdot\frac{\sigma_{\max}^{2}}{\lambda_{\min}^{\star2}}\frac{rd\log d}{p}\lambda_{\max}^{\star2/3}\cdot r\lambda_{\max}^{\star2/3}=o\left(1\right)\frac{\sigma_{\max}^{2}rd}{p},\label{eq:T-l2-loss-beta4}
\end{equation}
where we use the assumption that $r=o\,\big(\sqrt{d/(r\log^{2}d)}\big)$
and $\kappa\asymp1$.
\item It remains to bound $\beta_{5}$. Given the Cauchy-Schwartz inequality
$|\langle\bm{A},\bm{B}\rangle|\leq\left\Vert \bm{A}\right\Vert _{\mathrm{F}}\left\Vert \bm{B}\right\Vert _{\mathrm{F}}$,
it immediately follows from (\ref{eq:T-l2-loss-beta1}), (\ref{eq:T-l2-loss-beta2})
and (\ref{eq:T-l2-loss-beta3}) that
\begin{align*}
\left|\beta_{5}\right| & \lesssim\Big\|\sum_{1\leq l\leq r}\bm{u}_{l}^{\star}\otimes\bm{\Delta}_{l}^{\otimes2}\Big\|_{\mathrm{F}}^{2}+\Big(\Big\|\sum_{1\leq l\leq r}\bm{u}_{l}^{\star\otimes2}\otimes\bm{\Delta}_{l}\Big\|_{\mathrm{F}}+\Big\|\sum_{1\leq l\leq r}\bm{u}_{l}^{\star}\otimes\bm{\Delta}_{l}^{\otimes2}\Big\|_{\mathrm{F}}\Big)\Big\|\sum_{1\leq l\leq r}\bm{\Delta}_{l}^{\otimes3}\Big\|_{\mathrm{F}}\\
 & \quad+\Big\|\sum_{1\leq l\leq r}\bm{u}_{l}^{\star}\otimes\bm{\Delta}_{l}^{\otimes2}\Big\|_{\mathrm{F}}\Big\|\sum_{1\leq l\leq r}\bm{u}_{l}^{\star\otimes2}\otimes\bm{\Delta}_{l}\Big\|_{\mathrm{F}}=o\left(1\right)\frac{\sigma_{\max}^{2}rd}{p}.
\end{align*}
This taken collectively with (\ref{eq:T-l2-loss-beta4}) finishes
the proof.
\end{enumerate}

%% file: proof_lower_bounds.tex
\section{Proof of lower bounds (Theorem \ref{thm:lower-bound-entrywise-rank-r}
and Theorem \ref{thm:l2-error-lower-bound-rankr})}

\label{sec:Proof-of-lower-bounds}

In this section, we establish the lower bounds claimed in Theorems
\ref{thm:lower-bound-entrywise-rank-r} and \ref{thm:l2-error-lower-bound-rankr}
(which subsume Theorems \ref{thm:lower-bound-entrywise} and \ref{thm:l2-error-lower-bound}
as special cases, respectively). Recall the assumption that $\{E_{i,j,k}\}$
are independent Gaussians. For the sake of notational simplicity,
we shall assume throughout this proof that $\sigma_{i,j,k}^{2}\equiv\sigma_{\min}^{2}$
for all $1\leq i,j,k\leq d$.

\begin{comment}
Clearly, if we reduce the variance of all noise components to $\sigma_{\min}^{2}$,
then the fundamental lower limits on the variance$\,$/$\,$covariance
can only be further decreased. As a result, 
\end{comment}

Given that the noise components $\{E_{i,j,k}\mid(i,j,k)\in\Omega,1\leq i\leq j\leq k\leq d\}$
are assumed to be independent Gaussian, the probability density function
(conditional on $\Omega)$ can be computed as
\begin{align*}
f(\bm{T}^{\mathsf{obs}}) & =c\prod_{1\leq i\leq j\leq k\leq d,\,(i,j,k)\in\Omega}\exp\left(-\frac{\big(T_{i,j,k}^{\mathsf{obs}}-\sum_{l=1}^{r}u_{l,i}^{\star}u_{l,j}^{\star}u_{l,k}^{\star}\big)^{2}}{2\sigma_{\min}^{2}}\right)
\end{align*}
for some normalization constant $c>0$. Here, we abuse the notation
$f(\cdot)$ to represent the probability density function whenever
it is clear from the context. Denote by $\mathsf{vec}(\bm{U}^{\star})$
the vectorization of $\bm{U}^{\star}=[\bm{u}_{1}^{\star},\cdots,\bm{u}_{r}^{\star}]$,
namely,
\[
\mathsf{vec}(\bm{U}^{\star}):=\left[\begin{array}{c}
\bm{u}_{1}^{\star}\\
\vdots\\
\bm{u}_{r}^{\star}
\end{array}\right]\in\mathbb{R}^{dr}.
\]
By virtue of the Cram\'er-Rao lower bound, any unbiased estimator
$\widehat{\bm{U}}$ for $\bm{U}^{\star}$ necessarily obeys
\[
\mathsf{Cov}\big[\mathsf{vec}(\widehat{\bm{U}})\big]\succeq(\bm{\mathcal{I}}_{\Omega})^{-1},
\]
where $\bm{\mathcal{I}}_{\Omega}\in\mathbb{R}^{dr\times dr}$ denotes
the corresponding Fisher information matrix (conditional on $\Omega$)
as follows
\begin{align}
\bm{\mathcal{I}}_{\Omega}:=\bm{\mathcal{I}}_{\Omega}(\bm{U}^{\star}) & =\mathbb{E}\Big[\nabla_{\mathsf{vec}(\bm{U}^{\star})}\log f(\bm{T}^{\mathsf{obs}})\big(\nabla_{\mathsf{vec}(\bm{U}^{\star})}\log f(\bm{T}^{\mathsf{obs}})\big)^{\top}\Big]\nonumber \\
 & =\bigg[\mathbb{E}\Big[\nabla_{\bm{u}_{l_{1}}^{\star}}\log f(\bm{T}^{\mathsf{obs}})\big(\nabla_{\bm{u}_{l_{2}}^{\star}}\log f(\bm{T}^{\mathsf{obs}})\big)^{\top}\Big]\bigg]_{1\leq l_{1},l_{2}\leq r}.\label{eq:defn-Fisher-information}
\end{align}

It then suffices to compute the Fisher information matrix. Towards
this end, we start by observing that
\begin{equation}
\frac{\partial T_{i,j,k}^{\star}}{\partial u_{l,s}^{\star}}=\sum_{\tau=1}^{r}\frac{\partial u_{\tau,i}^{\star}u_{\tau,j}^{\star}u_{\tau,k}^{\star}}{\partial u_{l,s}^{\star}}=u_{l,j}^{\star}u_{l,k}^{\star}\ind_{\{i=s\}}+u_{l,i}^{\star}u_{l,k}^{\star}\ind_{\{j=s\}}+u_{l,i}^{\star}u_{l,j}^{\star}\ind_{\{k=s\}},\label{eq:T-ijk-u-derivative}
\end{equation}
and
\begin{align}
\sigma_{\min}^{2}\frac{\partial\log f(\bm{T}^{\mathsf{obs}})}{\partial u_{l,s}^{\star}} & =\sum_{(i,j,k)\in\Omega,\,i\leq j\leq k}\Big(T_{i,j,k}^{\mathsf{obs}}-\sum_{\tau=1}^{r}u_{\tau,i}^{\star}u_{\tau,j}^{\star}u_{\tau,k}^{\star}\Big)\Big(u_{l,j}^{\star}u_{l,k}^{\star}\ind_{\{i=s\}}+u_{l,i}^{\star}u_{l,k}^{\star}\ind_{\{j=s\}}+u_{l,i}^{\star}u_{l,j}^{\star}\ind_{\{k=s\}}\Big)\nonumber \\
 & =\sum_{(i,j,k)\in\Omega,\,i\leq j\leq k}E_{i,j,k}\Big(u_{l,j}^{\star}u_{l,k}^{\star}\ind_{\{i=s\}}+u_{l,i}^{\star}u_{l,k}^{\star}\ind_{\{j=s\}}+u_{l,i}^{\star}u_{l,j}^{\star}\ind_{\{k=s\}}\Big)\label{eq:log-f-calculation}
\end{align}
for any $1\leq l\leq r$ and $1\leq s\leq d$. In addition, let us
define a collection of vectors $\{\bm{h}_{i,j,k}\}_{1\leq i\leq j\leq k\leq d}$
in $\mathbb{R}^{dr}$ with entries
\begin{equation}
h_{i,j,k}(l,s):=h_{i,j,k}\big((l-1)\times r+s\big):=u_{l,j}^{\star}u_{l,k}^{\star}\ind_{\{i=s\}}+u_{l,i}^{\star}u_{l,k}^{\star}\ind_{\{j=s\}}+u_{l,i}^{\star}u_{l,j}^{\star}\ind_{\{k=s\}}\label{eq:h-entry}
\end{equation}
for any $1\leq l\leq r$ and $1\leq s\leq d$. One can then express
\begin{align}
\bm{\mathcal{I}}_{\Omega} & =\frac{1}{\sigma_{\min}^{4}}\sum_{(i,j,k)\in\Omega,\,i\leq j\leq k}\mathbb{E}\left[E_{i,j,k}^{2}\right]\bm{h}_{i,j,k}\bm{h}_{i,j,k}^{\top}=\frac{1}{\sigma_{\min}^{2}}\sum_{1\leq i\leq j\leq k\leq d}\chi_{i,j,k}\bm{h}_{i,j,k}\bm{h}_{i,j,k}^{\top},\label{eq:Fisher-Omega}
\end{align}
where we recall the notation $\chi_{i,j,k}:=\ind_{\{(i,j,k)\in\Omega\}}$.
Let us further define 
\begin{equation}
\bm{\mathcal{I}}:=\mathbb{E}_{\Omega}\big[\bm{\mathcal{I}}_{\Omega}\big]=\frac{p}{\sigma_{\min}^{2}}\sum_{1\leq i\leq j\leq k\leq d}\bm{h}_{i,j,k}\bm{h}_{i,j,k}^{\top}\label{eq:Fisher}
\end{equation}
where the expectation is taken over randomness of $\{\chi_{i,j,k}\}_{1\leq i,j,k\leq d}$.
In what follows, we shall compute the spectrum of $\bm{\mathcal{I}}$,
and show that $\mathcal{\bm{\mathcal{I}}}_{\Omega}$ is, with high
probability, sufficiently close to $\bm{\mathcal{I}}$ in the spectral
norm. In addition, denote $(l,s):=(l-1)\times r+d$ for all $1\leq l\leq r,1\leq s\leq d.$

\paragraph{Spectrum of $\bm{\mathcal{I}}$.}

First, it is straightforward to calculate that for any $1\leq s\leq d$
and any $1\leq l_{1},l_{2}\leq d$,
\begin{align*}
\frac{\sigma_{\min}^{2}}{p}\mathcal{I}_{(l_{1},s),(l_{2},s)} & =\sum_{(i,j):i\leq j}u_{l_{1},i}^{\star}u_{l_{2},i}^{\star}u_{l_{1},j}^{\star}u_{l_{2},j}^{\star}+3\,u_{l_{1},s}^{\star}u_{l_{2},s}^{\star}\sum_{1\leq i\leq d}u_{l_{1},i}^{\star}u_{l_{2},i}^{\star}+5\,u_{l_{1},s}^{\star2}u_{l_{2},s}^{\star2}\\
 & =\frac{1}{2}\langle\bm{u}_{l_{1}}^{\star},\bm{u}_{l_{2}}^{\star}\rangle^{2}+\frac{1}{2}\sum_{1\leq i\leq d}u_{l_{1},i}^{\star2}u_{l_{2},i}^{\star2}+3\,u_{l_{1},s}^{\star}u_{l_{2},s}^{\star}\langle\bm{u}_{l_{1}}^{\star},\bm{u}_{l_{2}}^{\star}\rangle+5\,u_{l_{1},s}^{\star2}u_{l_{2},s}^{\star2},
\end{align*}
and for any $1\leq s_{1}\neq s_{2}\leq d$, $1\leq l_{1},l_{2}\leq d$,
\[
\frac{\sigma_{\min}^{2}}{p}\mathcal{I}_{(l_{1},s_{1}),(l_{2},s_{2})}=u_{l_{1},s_{2}}^{\star}u_{l_{2},s_{1}}^{\star}\langle\bm{u}_{l_{1}}^{\star},\bm{u}_{l_{2}}^{\star}\rangle.
\]
As a consequence, one can express
\begin{align*}
\frac{\sigma_{\min}^{2}}{p}\bm{\mathcal{I}} & =\left[\begin{array}{ccc}
\frac{1}{2}\|\bm{u}_{1}^{\star}\|_{2}^{4}\bm{I}_{d}+\|\bm{u}_{1}^{\star}\|_{2}^{2}\bm{u}_{1}^{\star}\bm{u}_{1}^{\star\top}\\
 & \ddots\\
 &  & \frac{1}{2}\|\bm{u}_{r}^{\star}\|_{2}^{4}\bm{I}_{d}+\|\bm{u}_{r}^{\star}\|_{2}^{2}\bm{u}_{r}^{\star}\bm{u}_{r}^{\star\top}
\end{array}\right]\\
 & \quad+\underbrace{\left[\begin{array}{ccc}
\bm{0} & \cdots & \langle\bm{u}_{r}^{\star},\bm{u}_{1}^{\star}\rangle\bm{u}_{r}^{\star}\bm{u}_{1}^{\star\top}\\
\vdots & \ddots & \vdots\\
\langle\bm{u}_{1}^{\star},\bm{u}_{r}^{\star}\rangle\bm{u}_{1}^{\star}\bm{u}_{r}^{\star\top} & \cdots & \bm{0}
\end{array}\right]}_{=:\,\bm{\Phi}}+\,\bm{\Psi}
\end{align*}
where $\bm{\Phi}$ is a matrix in $\mathbb{R}^{dr\times dr}$ with
all-zero diagonal blocks and $(l_{1,}l_{2})$-th block equal to $\langle\bm{u}_{l_{2}}^{\star},\bm{u}_{l_{1}}^{\star}\rangle\bm{u}_{l_{2}}^{\star}\bm{u}_{l_{1}}^{\star\top}$,
and $\bm{\Psi}$ is a matrix in $\mathbb{R}^{dr\times dr}$ with entries
\begin{align*}
\Psi_{(l,s),(l,s)} & =\frac{1}{2}\sum_{1\leq i\leq d}u_{l,i}^{\star4}+2\,u_{l,s}^{\star2}\|\bm{u}_{l}^{\star}\|_{2}^{2}+5u_{l,s}^{\star4};\quad\text{if}\,\,l_{1}=l_{2},s_{1}=s_{2},\\
\Psi_{(l_{1},s),(l_{2},s)} & =\frac{1}{2}\langle\bm{u}_{l_{1}}^{\star},\bm{u}_{l_{2}}^{\star}\rangle^{2}+\frac{1}{2}\sum_{1\leq i\leq d}u_{l_{1},i}^{\star2}u_{l_{2},i}^{\star2}+2\,u_{l_{1},s}^{\star}u_{l_{2},s}^{\star}\langle\bm{u}_{l_{1}}^{\star},\bm{u}_{l_{2}}^{\star}\rangle+5\,u_{l_{1},s}^{\star2}u_{l_{2},s}^{\star2},\quad\text{if\,\,}l_{1}\neq l_{2},s_{1}=s_{2},\\
\Psi_{(l_{1},s_{1}),(l_{2},s_{2})} & =0,\quad\text{otherwise.}
\end{align*}

In the sequel, we shall control the spectral norm of $\bm{\Phi}$
and $\bm{\Psi}$ separately.
\begin{itemize}
\item For $\bm{\Phi}$, we can bound
\begin{align*}
\|\bm{\Phi}\| & \leq\|\bm{\Phi}\|_{\mathrm{F}}\leq\sqrt{\sum_{1\leq l_{1}\neq l_{2}\leq r}\langle\bm{u}_{l_{1}}^{\star},\bm{u}_{l_{2}}^{\star}\rangle^{2}\big\|\bm{u}_{l_{2}}^{\star}\bm{u}_{l_{1}}^{\star\top}\big\|_{\mathrm{F}}^{2}}\leq\sqrt{\max_{1\leq l_{1}\neq l_{2}\leq r}\langle\bm{u}_{l_{1}}^{\star},\bm{u}_{l_{2}}^{\star}\rangle^{2}\sum_{1\leq l_{1},l_{2}\leq r}\|\bm{u}_{l_{1}}^{\star}\|_{2}^{2}\|\bm{u}_{l_{2}}^{\star}\|_{2}^{2}}\\
 & \leq r\sqrt{\frac{\mu}{d}}\max_{1\leq l\leq r}\|\bm{u}_{l}^{\star}\|_{2}^{4}=o\Big(\min_{1\leq l\leq r}\|\bm{u}_{l}^{\star}\|_{2}^{4}\Big),
\end{align*}
where we have used the incoherence condition (\ref{assumption:u-inner-prod})
and the assumptions $r=o\big(\sqrt{d/\mu}\big)$ and $\kappa\asymp1$.
\item As for $\bm{\Psi}$, we note that each block of $\bm{\Psi}$ is a
diagonal matrix. By the incoherence conditions, its entries can be
bounded by
\[
\big|\Psi_{(l,s),(l,s)}\big|\lesssim\|\bm{u}_{l}^{\star}\|_{\infty}^{2}\|\bm{u}_{l}^{\star}\|_{2}^{2}\leq\frac{\mu}{d}\max_{1\leq l\leq r}\|\bm{u}_{l}^{\star}\|_{2}^{4},
\]
and
\begin{align*}
\big|\Psi_{(l_{1},s),(l_{2},s)}\big| & \lesssim\langle\bm{u}_{l_{1}}^{\star},\bm{u}_{l_{2}}^{\star}\rangle^{2}+\|\bm{u}_{l_{1}}^{\star}\|_{\infty}^{2}\|\bm{u}_{l_{2}}^{\star}\|_{2}^{2}+2\,\|\bm{u}_{l_{1}}^{\star}\|_{\infty}\|\bm{u}_{l_{2}}^{\star}\|_{\infty}\|\bm{u}_{l_{1}}^{\star}\|_{2}\|\bm{u}_{l_{2}}^{\star}\|_{2}\lesssim\frac{\mu}{d}\max_{1\leq l\leq r}\|\bm{u}_{l}^{\star}\|_{2}^{4},
\end{align*}
thus indicating that
\[
\|\bm{\Psi}\|_{\infty}\lesssim\frac{\mu}{d}\max_{1\leq l\leq r}\|\bm{u}_{l}^{\star}\|_{2}^{4}.
\]
Given the special structure of $\bm{\Psi}$, one can easily permutation
its columns and rows to arrive at another matrix $\widetilde{\bm{\Psi}}=[\widetilde{\bm{\Psi}}_{i,j}]_{1\leq i,j\leq r}$
such that (1) $\widetilde{\bm{\Psi}}$ is a block diagonal matrix;
(2) $\widetilde{\bm{\Psi}}$ contains $d\times d$ blocks each of
size $r\times r$; (3) each diagonal block $\widetilde{\bm{\Psi}}_{i,i}$
of $\widetilde{\bm{\Psi}}$ has spectral norm at most 
\[
\|\widetilde{\bm{\Psi}}_{i,i}\|\leq\|\widetilde{\bm{\Psi}}_{i,i}\|_{\mathrm{F}}\leq r\|\bm{\Psi}\|_{\infty}\lesssim\frac{\mu r}{d}\max_{1\leq l\leq r}\|\bm{u}_{l}^{\star}\|_{2}^{4}.
\]
Consequently, one can derive
\[
\|\bm{\Psi}\|=\|\widetilde{\bm{\Psi}}\|\leq\max_{1\leq i\leq r}\|\bm{\Psi}_{i,i}\|\lesssim\frac{\mu r}{d}\max_{1\leq l\leq r}\|\bm{u}_{l}^{\star}\|_{2}^{4}=o\left(\min_{1\leq l\leq r}\|\bm{u}_{l}^{\star}\|_{2}^{4}\right),
\]
provided that $r=o(d/\mu)$ and $\kappa\asymp1$.
\item Putting the above two estimates together, we conclude that
\begin{align}
\bm{\mathcal{I}} & \preceq\frac{p}{\sigma_{\min}^{2}}\left[\begin{array}{ccc}
\frac{1}{2}\|\bm{u}_{1}^{\star}\|_{2}^{4}\bm{I}_{d}+\|\bm{u}_{1}^{\star}\|_{2}^{2}\bm{u}_{1}^{\star}\bm{u}_{1}^{\star\top}\\
 & \ddots\\
 &  & \frac{1}{2}\|\bm{u}_{r}^{\star}\|_{2}^{4}\bm{I}_{d}+\|\bm{u}_{r}^{\star}\|_{2}^{2}\bm{u}_{r}^{\star}\bm{u}_{r}^{\star\top}
\end{array}\right]+\frac{p}{\sigma_{\min}^{2}}\left(\|\bm{\Phi}\|+\|\bm{\Psi}\|\right)\bm{I}_{dr}\nonumber \\
 & =\big(1+o(1)\big)\frac{p}{\sigma_{\min}^{2}}\left[\begin{array}{ccc}
\frac{1}{2}\|\bm{u}_{1}^{\star}\|_{2}^{4}\bm{I}_{d}+\|\bm{u}_{1}^{\star}\|_{2}^{2}\bm{u}_{1}^{\star}\bm{u}_{1}^{\star\top}\\
 & \ddots\\
 &  & \frac{1}{2}\|\bm{u}_{r}^{\star}\|_{2}^{4}\bm{I}_{d}+\|\bm{u}_{r}^{\star}\|_{2}^{2}\bm{u}_{r}^{\star}\bm{u}_{r}^{\star\top}
\end{array}\right].\label{eq:fisher-eig-UB}
\end{align}
\end{itemize}

\paragraph{Controlling $\|\bm{\mathcal{I}}_{\Omega}-\bm{\mathcal{I}}\|$.}

By construction, $\bm{\mathcal{I}}_{\Omega}-\bm{\mathcal{I}}=\frac{1}{\sigma_{\min}^{2}}\sum_{i\leq j\leq k}(\chi_{i,j,k}-p)\bm{h}_{i,j,k}\bm{h}_{i,j,k}^{\top}$
is a sum of independent zero-mean random matrix in $\mathbb{R}^{dr\times dr}$.
By the incoherence conditions, it is straightforward to bound
\begin{align*}
B & :=\max_{1\leq i\leq j\leq k\leq d}\big\|(\chi_{i,j,k}-p)\bm{h}_{i,j,k}\bm{h}_{i,j,k}^{\top}\big\|\leq\max_{1\leq i\leq j\leq k\leq d}\|\bm{h}_{i,j,k}\|_{2}^{2}\\
 & \,\overset{(\mathrm{i})}{\lesssim}r\max_{1\leq l\leq r}\|\bm{u}_{l}^{\star}\|_{\infty}^{4}\leq\frac{\mu^{2}r}{d^{2}}\max_{1\leq l\leq r}\|\bm{u}_{l}^{\star}\|_{2}^{4},
\end{align*}
where (i) arises from the definition of $\bm{h}_{i,j,k}$ in (\ref{eq:h-entry}).
In addition, we also have
\begin{align*}
V & :=\Big\|\sum_{i\leq j\leq k}\mathbb{E}\big[(\chi_{i,j,k}-p)^{2}\big]\|\bm{h}_{i,j,k}\|_{2}^{2}\bm{h}_{i,j,k}\bm{h}_{i,j,k}^{\top}\Big\|\\
 & \,\leq p\max_{1\leq i\leq j\leq k\leq d}\|\bm{h}_{i,j,k}\|_{2}^{2}\Big\|\sum_{1\leq i\leq j\leq k\leq d}\bm{h}_{i,j,k}\bm{h}_{i,j,k}^{\top}\Big\|\\
 & \,\lesssim pr\max_{1\leq l\leq r}\|\bm{u}_{l}^{\star}\|_{\infty}^{4}\max_{1\leq l\leq r}\|\bm{u}_{l}^{\star}\|_{2}^{4}\\
 & \,\lesssim\frac{\mu^{2}rp}{d^{2}}\max_{1\leq l\leq r}\|\bm{u}_{l}^{\star}\|_{2}^{8}.
\end{align*}
Here, the second inequality arises from (\ref{eq:h-entry}) and (\ref{eq:fisher-eig-UB}),
whereas the third comes from our bound above for $\bm{\mathcal{I}}$.
Invoking the matrix Bernstein inequality, we conclude that with probability
at least $1-O(d^{-10})$,
\begin{align}
\|\bm{\mathcal{I}}_{\Omega}-\bm{\mathcal{I}}\| & \lesssim\frac{B\log d+\sqrt{V\log d}}{\sigma_{\min}^{2}}\lesssim\frac{p}{\sigma_{\min}^{2}}\max_{1\leq l\leq r}\|\bm{u}_{l}^{\star}\|_{2}^{4}\left\{ \frac{\mu^{2}r\log d}{d^{2}p}+\frac{\mu\sqrt{r\log d}}{d\sqrt{p}}\right\} \nonumber \\
 & =o(1)\frac{p}{\sigma_{\min}^{2}}\min_{1\leq l\leq r}\|\bm{u}_{l}^{\star}\|_{2}^{4},\label{eq:fisher-dev-UB}
\end{align}
where the last step holds as long as $p\gg\mu^{2}rd^{-2}\log^{2}d$
and $\kappa\asymp1$.

\paragraph{Combining the spectrum of $\bm{\mathcal{I}}_{\Omega}$ and the bound
on $\|\bm{\mathcal{I}}_{\Omega}-\bm{\mathcal{I}}\|$.}

Combining (\ref{eq:fisher-eig-UB}) and (\ref{eq:fisher-dev-UB})
with Weyl's inequality reveals that with probability exceeding $1-O(d^{-10})$,
\[
\bm{\mathcal{I}}_{\Omega}\preceq\big(1+o(1)\big)\frac{p}{\sigma_{\min}^{2}}\left[\begin{array}{ccc}
\frac{1}{2}\|\bm{u}_{1}^{\star}\|_{2}^{4}\bm{I}_{d}+\|\bm{u}_{1}^{\star}\|_{2}^{2}\bm{u}_{1}^{\star}\bm{u}_{1}^{\star\top}\\
 & \ddots\\
 &  & \frac{1}{2}\|\bm{u}_{r}^{\star}\|_{2}^{4}\bm{I}_{d}+\|\bm{u}_{r}^{\star}\|_{2}^{2}\bm{u}_{r}^{\star}\bm{u}_{r}^{\star\top}
\end{array}\right].
\]
By the Woodbury matrix identity, it is straightforward to check
\[
\left(\frac{1}{2}\|\bm{u}_{i}^{\star}\|_{2}^{4}\bm{I}_{d}+\|\bm{u}_{i}^{\star}\|_{2}^{2}\bm{u}_{i}^{\star}\bm{u}_{i}^{\star\top}\right)^{-1}=\frac{2}{\|\bm{u}_{i}^{\star}\|_{2}^{4}}\left(\bm{I}_{d}-\frac{2}{3}\frac{\bm{u}_{i}^{\star}\bm{u}_{i}^{\star\top}}{\|\bm{u}_{i}^{\star}\|_{2}^{2}}\right).
\]
Hence, for any unbiased estimator $\widehat{\bm{U}}$ of $\bm{U}^{\star}$
we have
\begin{equation}
\mathsf{Cov}\big[\mathsf{vec}(\widehat{\bm{U}})\big]\succeq(\bm{\mathcal{I}}_{\Omega})^{-1}\succeq\big(1-o(1)\big)\frac{\sigma_{\min}^{2}}{p}\left[\begin{array}{ccc}
\frac{2}{\|\bm{u}_{1}^{\star}\|_{2}^{4}}\left(\bm{I}_{d}-\frac{2}{3}\frac{\bm{u}_{1}^{\star}\bm{u}_{1}^{\star\top}}{\|\bm{u}_{1}^{\star}\|_{2}^{2}}\right)\\
 & \ddots\\
 &  & \frac{2}{\|\bm{u}_{r}^{\star}\|_{2}^{4}}\left(\bm{I}_{d}-\frac{2}{3}\frac{\bm{u}_{r}^{\star}\bm{u}_{r}^{\star\top}}{\|\bm{u}_{r}^{\star}\|_{2}^{2}}\right)
\end{array}\right].\label{eq:cov-vec-U-LB}
\end{equation}

A few consequences from the above Cram\'er-Rao lower bound are in
order. 
\begin{itemize}
\item For each unbiased estimator $\widehat{\bm{u}}_{l}$ of $\bm{u}_{l}^{\star}$,
one necessarily has
\begin{align*}
\mathbb{E}\big[(\widehat{u}_{l,k}-u_{l,k}^{\star})^{2}\big] & \geq\big((\bm{\mathcal{I}}_{\Omega})^{-1})_{(l,k),(l,k)}\geq\big(1-o(1)\big)\frac{2\sigma_{\min}^{2}}{p\|\bm{u}_{l}^{\star}\|_{2}^{4}}\left(1-\frac{2}{3}\frac{u_{l,k}^{\star2}}{\|\bm{u}_{l}^{\star}\|_{2}^{2}}\right)\\
 & \geq\big(1-o(1)\big)\frac{2\sigma_{\min}^{2}}{p\|\bm{u}_{l}^{\star}\|_{2}^{4}}\left(1-\frac{2}{3}\frac{\|\bm{u}_{l}^{\star}\|_{\infty}^{2}}{\|\bm{u}_{l}^{\star}\|_{2}^{2}}\right)\\
 & \overset{(\mathrm{i})}{\geq}\big(1-o(1)\big)\frac{2\sigma_{\min}^{2}}{p\|\bm{u}_{l}^{\star}\|_{2}^{4}}\left(1-\frac{2}{3}\frac{\mu}{d}\right)\\
 & \overset{(\mathrm{ii})}{\geq}\big(1-o(1)\big)\frac{2\sigma_{\min}^{2}}{p\|\bm{u}_{l}^{\star}\|_{2}^{4}}=\big(1-o(1)\big)(\bm{\Sigma}_{k}^{\star})_{l,l},
\end{align*}
where (i) arises from the incoherence assumption (\ref{assumption:incoherence})
and (ii) holds as long as $\mu=o(d)$. This further implies that
\[
\mathbb{E}\big[\|\widehat{\bm{u}}_{l}-\bm{u}_{l}^{\star}\|_{2}^{2}\big]=\sum_{k=1}^{d}\mathbb{E}\big[(\widehat{u}_{l,k}-u_{l,k}^{\star})^{2}\big]\geq\big(1-o(1)\big)\frac{2\sigma_{\min}^{2}d}{p\|\bm{u}_{l}^{\star}\|_{2}^{4}}.
\]
\item Any unbiased estimator $\widehat{T}_{i,j,k}$ of $T_{i,j,k}^{\star}$
necessarily obeys \cite{shao2003mathematical}
\begin{align*}
\mathbb{E}\Big[\big(\widehat{T}_{i,j,k}-T_{i,j,k}^{\star}\big)^{2}\Big] & \geq\left[\frac{\partial T_{i,j,k}^{\star}}{\partial\mathsf{vec}(\bm{U}^{\star})}\right]^{\top}(\bm{\mathcal{I}}_{\Omega})^{-1}\frac{\partial T_{i,j,k}^{\star}}{\partial\mathsf{vec}(\bm{U}^{\star})}\\
 & \overset{(\mathrm{i})}{\geq}\left(1-o(1)\right)\sum_{1\leq s\leq d}\sum_{1\leq l\leq r}\frac{2\sigma_{\min}^{2}}{p\|\bm{u}_{l}^{\star}\|_{2}^{4}}\left(\frac{\partial T_{i,j,k}^{\star}}{\partial u_{l,s}^{\star}}\right)^{2}\\
 & \overset{(\mathrm{ii})}{=}\big(1-o(1)\big)v_{i,j,k}^{\star},
\end{align*}
where (i) uses (\ref{eq:cov-vec-U-LB}), and (ii) follows from (\ref{eq:T-ijk-u-derivative}),
(\ref{def:T-entry-var}) and direct algebraic manipulations.
\item Any unbiased estimator $\widehat{\bm{T}}$ of $\bm{T}^{\star}$ necessarily
satisfies
\begin{align*}
\mathbb{E}\big[\|\widehat{\bm{T}}-\bm{T}^{\star}\|_{\mathrm{F}}^{2}\big] & =\sum_{i,j,k}\mathbb{E}\Big[\big(\widehat{T}_{i,j,k}-T_{i,j,k}^{\star}\big)^{2}\Big]\geq\big(1-o(1)\big)\sum_{i,j,k}v_{i,j,k}^{\star}\\
 & \overset{(\mathrm{i})}{\geq}\big(1-o(1)\big)3\sum_{1\leq i\leq d}\sum_{1\leq j,k\leq d}\sum_{1\leq l\leq r}\frac{2\sigma_{\min}^{2}}{p\|\bm{u}_{l}^{\star}\|_{2}^{4}}u_{l,j}^{\star2}u_{l,k}^{\star2}\\
 & =\big(1-o(1)\big)\frac{6\sigma_{\min}^{2}dr}{p},
\end{align*}
where (i) arises from the definition of $v_{i,j,k}^{\star}$ in (\ref{def:T-entry-var}).
\end{itemize}

%% file: proof-preliminaries.tex
\section{Proof of auxiliary lemmas: preliminary facts}

\label{sec:Proof-Preliminaries}

\subsection{Proof of Lemma \ref{lemma:U-property}}

\label{subsec:U-property}
\begin{enumerate}
\item To begin with, by the incoherence assumption (\ref{assumption:u-inf-norm}),
it is easy to derive
\begin{align*}
\left\Vert \bm{U}^{\star}\right\Vert _{\mathrm{F}} & \leq\sqrt{r}\,\max_{1\leq l\leq r}\left\Vert \bm{u}_{l}^{\star}\right\Vert _{2}\leq\sqrt{r}\,\lambda_{\max}^{\star1/3},\\
\left\Vert \bm{U}^{\star}\right\Vert _{2,\infty} & \leq\sqrt{r}\,\max_{1\leq l\leq r}\left\Vert \bm{u}_{l}^{\star}\right\Vert _{\infty}\leq\sqrt{\frac{\mu r}{d}}\,\lambda_{\max}^{\star1/3}.
\end{align*}
\item Regarding the properties about the spectrum of $\bm{U}^{\star}$,
we refer the reader to the proof of \cite[Lemma D.1]{cai2019nonconvex}.
\item From Lemma~\ref{lemma:U-loss-property}, it is straightforward to
show that: there exists a permutation $\pi(\cdot):[d]\mapsto[d]$
such that
\begin{align*}
\max_{1\leq i\leq r}\left\Vert \bm{u}_{\pi(i)}-\bm{u}_{i}^{\star}\right\Vert _{2} & \leq\left\Vert \bm{U}\bm{\Pi}-\bm{U}^{\star}\right\Vert _{\mathrm{F}}\lesssim\frac{\sigma}{\lambda_{\min}^{\star}}\sqrt{\frac{rd\log d}{p}}\,\lambda_{\max}^{\star1/3}=o\left(\lambda_{\min}^{\star1/3}\right),\\
\max_{1\leq i\leq r}\left\Vert \bm{u}_{\pi(i)}-\bm{u}_{i}^{\star}\right\Vert _{\infty} & \leq\left\Vert \bm{U}\bm{\Pi}-\bm{U}^{\star}\right\Vert _{2,\infty}\lesssim\frac{\sigma}{\lambda_{\min}^{\star}}\sqrt{\frac{\mu r\log d}{p}}\,\lambda_{\max}^{\star1/3}=o\Big(\frac{\lambda_{\min}^{\star1/3}}{\sqrt{d}}\Big),
\end{align*}
where we have used the conditions that $\sigma_{\max}/\lambda_{\min}^{\star}\ll\sqrt{p/(\mu rd^{3/2}\log d)}$
and $\kappa\asymp1$. Recognizing that $\lambda_{\min}^{\star1/3}\leq\left\Vert \bm{u}_{i}^{\star}\right\Vert _{2}\leq\lambda_{\max}^{\star1/3}$
and that $\sqrt{1/d}\,\lambda_{\min}^{\star1/3}\leq\left\Vert \bm{u}_{i}^{\star}\right\Vert _{\infty}\leq\sqrt{\mu/d}\,\lambda_{\max}^{\star1/3}$
for all $1\leq i\leq r$, one immediately obtains (\ref{eq:u-loss-u-relation})
and (\ref{eq:u-norm}) by invoking the triangle inequality.
\item For any $1\leq i\neq j\leq r$, applying the triangle inequality and
the Cauchy-Schwartz inequality yields
\begin{align*}
\left|\left\langle \bm{u}_{i},\bm{u}_{j}\right\rangle \right| & \leq\left|\left\langle \bm{u}_{i}^{\star},\bm{u}_{j}^{\star}\right\rangle \right|+\left|\left\langle \bm{u}_{i}-\bm{u}_{i}^{\star},\bm{u}_{j}\right\rangle \right|+\left|\left\langle \bm{u}_{i}^{\star},\bm{u}_{j}-\bm{u}_{j}^{\star}\right\rangle \right|\\
 & \leq\left|\left\langle \bm{u}_{i}^{\star},\bm{u}_{j}^{\star}\right\rangle \right|+\left\Vert \bm{u}_{i}-\bm{u}_{i}^{\star}\right\Vert _{2}\|\bm{u}_{j}\|_{2}+\|\bm{u}_{j}-\bm{u}_{j}^{\star}\|_{2}\left\Vert \bm{u}_{i}^{\star}\right\Vert _{2}\\
 & \lesssim\sqrt{\frac{\mu}{d}}\,\lambda_{\max}^{\star2/3}+\frac{\sigma_{\max}}{\lambda_{\min}^{\star}}\sqrt{\frac{rd\log d}{p}}\,\lambda_{\max}^{\star2/3}.
\end{align*}
\item Next, we move on to the spectrum of $\bm{U}$. In view of (\ref{eq:U-loss-fro})
and the conditions that $\sigma_{\max}/\lambda_{\min}^{\star}\ll\sqrt{p/(rd^{3/2}\log d)}$
and $\kappa\asymp1$, one can deduce that
\[
\left\Vert \bm{U}\bm{\Pi}-\bm{U}^{\star}\right\Vert \leq\left\Vert \bm{U}\bm{\Pi}-\bm{U}^{\star}\right\Vert _{\mathrm{F}}\lesssim\frac{\sigma_{\max}}{\lambda_{\min}^{\star}}\sqrt{\frac{rd\log d}{p}}\,\lambda_{\max}^{\star1/3}=o\left(\lambda_{\min}^{\star1/3}\right).
\]
Therefore, (\ref{eq:U-spectrum}) is an immediate consequence of Weyl's
inequality and (\ref{eq:U-true-spectrum}).
\item Finally, we know from Lemma~\ref{lemma:U-loo-property} that the
estimation error bounds for $\left\Vert \bm{U}\bm{\Pi}-\bm{U}^{\star}\right\Vert _{\mathrm{F}}$
and $\left\Vert \bm{U}\bm{\Pi}-\bm{U}^{\star}\right\Vert _{2,\infty}$
continue to hold if we replace $\bm{U}$ with $\bm{U}^{(m)}$. Hence,
the above results are also valid for $\bm{U}^{(m)}$ and $\bm{u}_{l}^{(m)}$
($1\leq l\leq r$).
\end{enumerate}

\subsection{Proof of Lemma \ref{lemma:U-tilde-property}}

\label{subsec:U-tilde-property}
\begin{enumerate}
\item To begin with, it is straightforward to compute
\[
\big\|\widetilde{\bm{U}}^{\star}\big\|_{\mathrm{F}}^{2}=\sum_{1\leq s\leq r}\left\Vert \bm{u}_{s}^{\star}\otimes\bm{u}_{s}^{\star}\right\Vert _{2}^{2}=\sum_{1\leq s\le r}\left\Vert \bm{u}_{s}^{\star}\right\Vert _{2}^{4}\leq r\lambda_{\max}^{\star4/3}.
\]
\item For any $1\leq i,j\leq d$, the incoherence assumption (\ref{assumption:u-inf-norm})
yields 
\[
\big\|\widetilde{\bm{U}}_{(i,j),:}^{\star}\big\|_{2}^{2}=\sum_{1\leq s\leq r}u_{s,i}^{\star2}u_{s,j}^{\star2}\leq\frac{\mu^{2}r}{d^{2}}\lambda_{\max}^{\star4/3}.
\]
This leads to the claimed bound regarding $\|\widetilde{\bm{U}}^{\star}\|_{2,\infty}$
.
\item Regarding the spectrum of $\widetilde{\bm{U}}^{\star}$ and $\widetilde{\bm{U}}$,
we refer the reader to the proof of \cite[Lemma 4.1 and Lemma D.1]{cai2019nonconvex}.
\item Next, we turn to $\|\widetilde{\bm{U}}\bm{\Pi}-\widetilde{\bm{U}}^{\star}\|_{\mathrm{F}}$.
Without loss of generality, assume that $\bm{\Pi}=\bm{I}_{r}$. Using
the fact that $\langle\bm{a}^{\otimes2},\bm{b}^{\otimes2}\rangle=\left\langle \bm{a},\bm{b}\right\rangle ^{2}$
for any vectors $\bm{a},\bm{b}\in\mathbb{R}^{d}$, we can straightforwardly
calculate that
\begin{align*}
\left\Vert \bm{u}_{s}^{\otimes2}-\bm{u}_{s}^{\star\otimes2}\right\Vert _{2}^{2} & =\left\Vert \bm{u}_{s}^{\otimes2}\right\Vert _{2}^{2}+\left\Vert \bm{u}_{s}^{\star\otimes2}\right\Vert _{2}^{2}-2\left\langle \bm{u}_{s}^{\otimes2},\bm{u}_{s}^{\star\otimes2}\right\rangle =\left\Vert \bm{u}_{s}\right\Vert _{2}^{4}+\left\Vert \bm{u}_{s}^{\star}\right\Vert _{2}^{4}-2\left\langle \bm{u}_{s},\bm{u}_{s}^{\star}\right\rangle ^{2}\\
 & =\left\Vert \bm{u}_{s}\right\Vert _{2}^{4}+\left\Vert \bm{u}_{s}^{\star}\right\Vert _{2}^{4}-\frac{1}{2}\big(\left\Vert \bm{u}_{s}\right\Vert _{2}^{2}+\left\Vert \bm{u}_{s}^{\star}\right\Vert _{2}^{2}-\left\Vert \bm{u}_{s}-\bm{u}_{s}^{\star}\right\Vert _{2}^{2}\big)^{2}\\
 & =\frac{1}{2}\big(\left\Vert \bm{u}_{s}\right\Vert _{2}^{2}-\left\Vert \bm{u}_{s}^{\star}\right\Vert _{2}^{2}\big)^{2}+\big(\left\Vert \bm{u}_{s}\right\Vert _{2}^{2}+\left\Vert \bm{u}_{s}^{\star}\right\Vert _{2}^{2}\big)\left\Vert \bm{u}_{s}-\bm{u}_{s}^{\star}\right\Vert _{2}^{2}-\frac{1}{2}\left\Vert \bm{u}_{s}-\bm{u}_{s}^{\star}\right\Vert _{2}^{4}.
\end{align*}
From the triangle inequality and the Cauchy-Schwartz inequality, we
know that
\[
\big(\left\Vert \bm{u}_{s}\right\Vert _{2}^{2}-\left\Vert \bm{u}_{s}^{\star}\right\Vert _{2}^{2}\big)^{2}=\big(\left\Vert \bm{u}_{s}\right\Vert _{2}+\left\Vert \bm{u}_{s}^{\star}\right\Vert _{2}\big)^{2}\big(\left\Vert \bm{u}_{s}\right\Vert _{2}-\left\Vert \bm{u}_{s}^{\star}\right\Vert _{2}\big)^{2}\leq2\big(\left\Vert \bm{u}_{s}\right\Vert _{2}^{2}+\left\Vert \bm{u}_{s}^{\star}\right\Vert _{2}^{2}\big)\left\Vert \bm{u}_{s}-\bm{u}_{s}^{\star}\right\Vert _{2}^{2}.
\]
The above two results taken together with (\ref{eq:U-property}) reveal
that
\[
\left\Vert \bm{u}_{s}^{\otimes2}-\bm{u}_{s}^{\star\otimes2}\right\Vert _{2}^{2}\leq2\,\big(\left\Vert \bm{u}_{s}\right\Vert _{2}^{2}+\left\Vert \bm{u}_{s}^{\star}\right\Vert _{2}^{2}\big)\left\Vert \bm{u}_{s}-\bm{u}_{s}^{\star}\right\Vert _{2}^{2}\lesssim\lambda_{\max}^{\star2/3}\left\Vert \bm{u}_{s}-\bm{u}_{s}^{\star}\right\Vert _{2}^{2},
\]
and consequently,
\[
\big\|\widetilde{\bm{U}}-\widetilde{\bm{U}}^{\star}\big\|_{\mathrm{F}}^{2}=\sum_{1\leq s\leq r}\left\Vert \bm{u}_{s}^{\otimes2}-\bm{u}_{s}^{\star\otimes2}\right\Vert _{2}^{2}\lesssim\lambda_{\max}^{\star2/3}\sum_{1\leq s\leq r}\left\Vert \bm{u}_{s}-\bm{u}_{s}^{\star}\right\Vert _{2}^{2}=\lambda_{\max}^{\star2/3}\left\Vert \bm{U}-\bm{U}^{\star}\right\Vert _{\mathrm{F}}^{2}.
\]
Then the advertised bound on $\|\widetilde{\bm{U}}-\widetilde{\bm{U}}^{\star}\|_{\mathrm{F}}$
follows immediately from (\ref{eq:U-T-loss-UB}).
\item We proceed to the term $\|\widetilde{\bm{U}}\bm{\Pi}-\widetilde{\bm{U}}^{\star}\|_{2,\infty}$.
Again , let us assume $\bm{\Pi}=\bm{I}_{r}$ and recall the notation
$u_{s,i}:=\left(\bm{u}_{s}\right)_{i}$ and $u_{s,i}^{\star}:=\left(\bm{u}_{s}^{\star}\right)_{i}$
for any $1\leq s\leq r,1\leq i\leq d$. Then we can upper bound
\begin{align*}
\sum_{1\leq s\leq r}\left(\bm{u}_{s}^{\otimes2}-\bm{u}_{s}^{\star\otimes2}\right)_{(i,j)}^{2} & =\sum_{1\leq s\leq r}\left(u_{s,i}u_{s,j}-u_{s,i}^{\star}u_{s,j}^{\star}\right)^{2}\lesssim\sum_{1\leq s\leq r}\left(u_{s,i}-u_{s,i}^{\star}\right)^{2}u_{s,j}^{2}+\sum_{1\leq s\leq r}u_{s,i}^{\star2}\left(u_{s,j}-u_{s,j}^{\star}\right)^{2}\\
 & \lesssim\max_{1\leq s\leq r}\left\Vert \bm{u}_{s}\right\Vert _{\infty}^{2}\left\Vert \bm{U}-\bm{U}^{\star}\right\Vert _{2,\infty}^{2}
\end{align*}
for any $1\leq i,j\leq d$. This taken collectively with (\ref{eq:u-norm})
and (\ref{eq:U-loss-2inf}) yields the claim.
\item Finally, we note that all bounds for $\bm{u}_{l}$ are also true for
$\bm{u}_{l}^{(m)}$ . Hence the above-mentioned results continue to
hold for $\bm{U}^{(m)}$ and $\bm{u}_{l}^{(m)}$ ($1\leq l\leq r$).
\end{enumerate}

%% file: aux-lemma.tex
\section{Other auxiliary lemmas}

\label{sec:Auxiliary-lemmas}

\begin{lemma}\label{lemma:Omega-I-T-op-UB}Let $\bm{T}\in\mathbb{R}^{d\times d\times d}$
be an order-3 tensor with decomposition $\bm{T}=\sum_{i=1}^{r}\bm{u}_{i}\otimes\bm{v}_{i}\otimes\bm{w}_{i}$.
Here, $\left\{ \bm{u}_{i},\bm{v}_{i},\bm{w}_{i}\right\} _{i=1}^{r}$
is a collection of vectors in $\mathbb{R}^{d}$. Then for any index
subset $\Omega\subset\left[d\right]^{3}$ and any $t\in\mathbb{R}$,
one has
\[
\left\Vert \mathcal{P}_{\Omega}\left(\bm{T}\right)-t\bm{T}\right\Vert \leq\left\Vert \mathcal{P}_{\Omega}\left(\bm{1}^{\otimes3}\right)-t\bm{1}^{\otimes3}\right\Vert \sum_{i=1}^{r}\left\Vert \bm{u}_{i}\right\Vert _{\infty}\left\Vert \bm{v}_{i}\right\Vert _{\infty}\left\Vert \bm{w}_{i}\right\Vert _{\infty},
\]
where $\bm{1}\in\mathbb{R}^{d}$ denotes the all-one vector. Here,
$\Omega$ can be arbitrary.\end{lemma}

\begin{proof}Fix arbitrary vectors $\bm{x},\bm{y},\bm{z}\in\mathbb{R}^{d}$
with $\left\Vert \bm{x}\right\Vert _{2}=\left\Vert \bm{y}\right\Vert _{2}=\left\Vert \bm{z}\right\Vert _{2}=1$.
We have
\begin{align*}
\left|\left\langle \mathcal{P}_{\Omega}\left(\bm{T}\right)-t\bm{T},\,\bm{x}\otimes\bm{y}\otimes\bm{z}\right\rangle \right| & =\left|\left\langle \mathcal{P}_{\Omega}\left(\bm{1}^{\otimes3}\right)-t\bm{1}^{\otimes3},\,\bm{T}\odot\left(\bm{x}\otimes\bm{y}\otimes\bm{z}\right)\right\rangle \right|\\
 & \leq\big\|\mathcal{P}_{\Omega}\left(\bm{1}^{\otimes3}\right)-t\bm{1}^{\otimes3}\big\|\big\|\bm{T}\odot\left(\bm{x}\otimes\bm{y}\otimes\bm{z}\right)\big\|_{*},
\end{align*}
where we denote by $\left\Vert \cdot\right\Vert _{*}$ the tensor
nuclear norm \cite{yuan2016tensor}. By the linearity of the Hadamard
and tensor product, we can express
\begin{align*}
\bm{T}\odot\left(\bm{x}\otimes\bm{y}\otimes\bm{z}\right) & =\Big(\sum_{i=1}^{r}\bm{u}_{i}\otimes\bm{v}_{i}\otimes\bm{w}_{i}\Big)\odot\left(\bm{x}\otimes\bm{y}\otimes\bm{z}\right)\\
 & =\sum_{i=1}^{r}\left(\bm{u}_{i}\otimes\bm{v}_{i}\otimes\bm{w}_{i}\right)\odot\left(\bm{x}\otimes\bm{y}\otimes\bm{z}\right)\\
 & =\sum_{i=1}^{r}\left(\bm{u}_{i}\odot\bm{x}\right)\otimes\left(\bm{v}_{i}\odot\bm{y}\right)\otimes\left(\bm{w}_{i}\odot\bm{z}\right).
\end{align*}
From the triangle inequality, we can upper bound
\begin{align*}
\big\|\bm{T}\odot\left(\bm{x}\otimes\bm{y}\otimes\bm{z}\right)\big\|_{*} & \leq\sum_{i=1}^{r}\left\Vert \left(\bm{u}_{i}\odot\bm{x}\right)\otimes\left(\bm{v}_{i}\odot\bm{y}\right)\otimes\left(\bm{w}_{i}\odot\bm{z}\right)\right\Vert _{*}\\
 & \leq\sum_{i=1}^{r}\left\Vert \bm{u}_{i}\odot\bm{x}\right\Vert _{2}\left\Vert \bm{v}_{i}\odot\bm{y}\right\Vert _{2}\left\Vert \bm{w}_{i}\odot\bm{z}\right\Vert _{2}\\
 & \leq\sum_{i=1}^{r}\left\Vert \bm{u}_{i}\right\Vert _{\infty}\left\Vert \bm{v}_{i}\right\Vert _{\infty}\left\Vert \bm{w}_{i}\right\Vert _{\infty}.
\end{align*}
Here, the second inequality holds due to the fact that $\left\Vert \bm{a}\otimes\bm{b}\otimes\bm{c}\right\Vert =\left\Vert \bm{a}\otimes\bm{b}\otimes\bm{c}\right\Vert _{*}=\left\Vert \bm{a}\right\Vert _{2}\left\Vert \bm{b}\right\Vert _{2}\left\Vert \bm{c}\right\Vert _{2}$
for any vectors $\bm{a},\bm{b},\bm{c}\in\mathbb{R}^{d}$, whereas
the last inequality follows by observing the following inequality
\[
\left\Vert \bm{u}_{i}\odot\bm{x}\right\Vert _{2}^{2}=\sum_{j=1}^{d}\left(\bm{u}_{i}\right)_{j}^{2}x_{j}^{2}\leq\left\Vert \bm{u}_{i}\right\Vert _{\infty}^{2}\sum_{j=1}^{d}x_{j}^{2}=\left\Vert \bm{u}_{i}\right\Vert _{\infty}^{2}
\]
and similarly $\left\Vert \bm{v}_{i}\odot\bm{y}\right\Vert _{2}\leq\|\bm{v}_{i}\|_{\infty}$
and $\left\Vert \bm{w}_{i}\odot\bm{z}\right\Vert _{2}\leq\|\bm{w}_{i}\|_{\infty}$.
Consequently, one arrives at
\[
\left|\left\langle \mathcal{P}_{\Omega}\left(\bm{T}\right)-t\bm{T},\,\bm{x}\otimes\bm{y}\otimes z\right\rangle \right|\leq\left\Vert \mathcal{P}_{\Omega}\left(\bm{1}^{\otimes3}\right)-t\bm{1}^{\otimes3}\right\Vert \sum_{i=1}^{r}\left\Vert \bm{u}_{i}\right\Vert _{\infty}\left\Vert \bm{v}_{i}\right\Vert _{\infty}\left\Vert \bm{w}_{i}\right\Vert _{\infty}.
\]
Given that this holds for arbitrary $\bm{x,y,\bm{z}}\in\mathbb{R}^{d}$
with $\left\Vert \bm{x}\right\Vert _{2}=\left\Vert \bm{y}\right\Vert _{2}=\left\Vert \bm{z}\right\Vert _{2}=1,$
we finish the proof by the definition of the spectral norm.\end{proof}